\documentclass{article}

\usepackage[final]{neurips_2022}

\usepackage[utf8]{inputenc} 
\usepackage[T1]{fontenc}    
\usepackage{hyperref}       
\usepackage{url}            
\usepackage{booktabs}       
\usepackage{amsfonts}       
\usepackage{nicefrac}       
\usepackage{microtype}      
\usepackage{xcolor}         
\usepackage{graphicx}
\usepackage{subfigure}
\usepackage{booktabs} 
\usepackage{algorithm}
\usepackage{algorithmic}
\usepackage{wrapfig}
\usepackage{color}
\usepackage{framed}
\definecolor{shadecolor}{rgb}{0.92,0.92,0.92}

\usepackage{amsmath}
\usepackage{amssymb}
\usepackage{mathtools}
\usepackage{amsthm}
\usepackage{mathrsfs} 
\usepackage{ulem}
\usepackage{makecell}
\usepackage[capitalize,noabbrev]{cleveref}

\theoremstyle{plain}
\newtheorem{theorem}{Theorem}[section]
\newtheorem{proposition}[theorem]{Proposition}
\newtheorem{lemma}[theorem]{Lemma}
\newtheorem{corollary}[theorem]{Corollary}
\theoremstyle{definition}

\newtheorem{assumption}[theorem]{Assumption}
\theoremstyle{remark}
\newtheorem{remark}[theorem]{Remark}

\newcommand{\RED}[1]{{\color{red} #1}}
\newcommand{\BLUE}[1]{{\color{blue} #1}}

\newcommand{\xki}{x_{k,i}}
\newcommand{\xkz}{x_{k,0}}
\newcommand{\mlki}{m_{l,k,i}}
\newcommand{\vlki}{v_{l,k,i}}
\newcommand{\mlkz}{m_{l,k,0}}
\newcommand{\vlkz}{v_{l,k,0}}
\newcommand{\AAA}{\sqrt{\frac{2\rho_3^2}{\beta_2^n}}}
\newcommand{\Ex}{\mathbb{E}}
\newcommand{\Ext}{\mathbb{E}_k}

\newcommand{\btwo}{\beta_2}
\newcommand{\partialtauki}{\partial_{l} f_{\tau_{k, i}}}
\newcommand{\betabeta}{(\beta_1,\beta_2)}
\newcommand{\functionclass}{\mathcal{F}_{L, D_0, D_1}^{n,f^*}  (\mathbb{R}^d)}

\title{Adam Can Converge Without Any Modification On Update Rules}

\author{%
  Yushun Zhang$^{13}$, Congliang Chen$^1$, Naichen Shi$^2$, Ruoyu Sun$^{13}$\thanks{Correspondence author} , 
    Zhi-Quan Luo$^{13}$
   \\
   $^1$The Chinese University of Hong Kong, Shenzhen, China $^2$University of Michigan, US\\
    $^3$Shenzhen Research Institute of Big Data
  \\
  \texttt{\{yushunzhang,congliangchen\}@link.cuhk.edu.cn}, naichens@umich.edu \\
  \texttt{\{sunruoyu,luozq\}@cuhk.edu.cn}
}

\begin{document}

\maketitle

\begin{abstract}

  Ever since \citet{reddi2019convergence} pointed out the divergence issue of Adam, many new variants  have been designed to obtain convergence. However, vanilla Adam remains exceptionally popular and it works well in practice. Why is there a gap between theory and practice? We point out there is a mismatch between the settings of theory and practice: \citet{reddi2019convergence} pick the problem after  picking the hyperparameters of Adam, i.e., $(\beta_1,\beta_2)$; while practical applications often fix the problem first and then tune $(\beta_1,\beta_2)$. Due to this observation, we conjecture that the empirical convergence can be theoretically justified, only if we change the order of picking the problem and hyperparameter.  In this work, we confirm this conjecture.  We prove that, when the 2nd-order momentum parameter $\beta_2$ is large and 1st-order momentum parameter $\beta_1 < \sqrt{\beta_2}<1$, Adam converges to the neighborhood of critical points. The size of the neighborhood is propositional to the variance of stochastic gradients. Under an extra condition (strong growth condition), Adam converges to critical points. It is worth mentioning that our results cover a wide range of hyperparameters: as $\beta_2$  increases, our convergence result can cover any $\beta_1 \in [0,1)$ including $\beta_1=0.9$, which is the default setting in deep learning libraries. To our knowledge, this is the first result showing that Adam can converge {\it without any modification} on its update rules. Further, our analysis does not require assumptions of bounded gradients or bounded 2nd-order momentum. When $\beta_2$ is small, we further point out  a large region of  $(\beta_1,\beta_2)$ combinations where  Adam can diverge to infinity. Our divergence result considers the same setting (fixing the optimization problem ahead) as our convergence result, indicating that there is a phase transition from divergence to convergence when increasing $\beta_2$. These positive and negative  results provide suggestions on  how to tune Adam hyperparameters: for instance,  when Adam does not work well, we suggest tuning up $\beta_2$ and trying $\beta_1< \sqrt{\beta_2}$.

\end{abstract}

\vspace{-6mm}
\section{Introduction}
\label{section:intro}
\vspace{-1mm}

Modern machine learning tasks often aim to solve the following finite-sum problem.
\begin{equation} \label{finite_sum}
  \min _{x \in \mathbb{R}^{d}} f(x)=\sum_{i=0}^{n-1} f_{i}(x),
\end{equation}
where $n$ is the number of samples or mini-batches and $x$ denotes the trainable parameters. 
In deep learning, Adam \citep{kingma2014adam} is one of the most popular algorithms for solving \eqref{finite_sum}.
It has been applied  to various machine learning domains such as natural language processing (NLP) \citep{vaswani2017attention,brown2020language,devlin2018bert}, generative adversarial networks (GANs) \citep{radford2015unsupervised,isola2017image,zhu2017unpaired}
and computer vision (CV) \citep{dosovitskiy2021an}.
Despite its prevalence,
\citet{reddi2019convergence} point out that 
Adam can diverge with a wide range of hyperparameters. A main result in  \citep{reddi2019convergence} states that \footnote{
We formally re-state their results in Appendix \ref{appendix:reddi}.}:
\begin{center}
\vspace{-1mm}
   {\it  For any $\beta_1, \beta_2$ s.t. $0 \leq \beta_1 < \sqrt{\beta_2} <1$, there exists a problem such that Adam diverges. } 
    \vspace{-1mm}
\end{center}
Here, $\beta_1$ and $\beta_2$ are the hyperparameter to control Adam's 1st-order and 2nd-order momentum.
More description of Adam can be seen in Algorithm \ref{algorithm} (presented later in Section \ref{section:preliminaries}). 
Ever since \citep{reddi2019convergence} pointed out the divergence issue, many new variants have been designed. For instance, 
AMSGrad \citep{reddi2019convergence} enforced the adaptor $v_t$ (defined later in Algorithm \ref{algorithm}) to be non-decreasing; AdaBound \citep{luo2018adaptive} imposed constraint $v_t \in [C_l, C_u]$ to ensure the boundedness on effective stepsize.  We introduce more variants in Appendix \ref{appendix:more_related}.


On the other hand, counter-intuitively, vanilla Adam remains exceptionally popular (see evidence at \citep{adamcitation}).  Without any modification on its update rules, Adam works well in practice. 
 Even more mysteriously, we find that the commonly reported hyperparameters actually satisfy the divergence condition stated  earlier. 
 For instance, 
\citet{kingma2014adam} claimed that  $\betabeta=(0.9,0.999)$ is a ``good choice for the tested machine learning problems" and it is indeed the default setting in deep learning libraries.
In super-large models GPT-3 and Megatron \citep{brown2020language, smith2022using}, $\betabeta$ is chosen to  be $(0.9,0.95)$.
GAN researchers (e.g. \citet{radford2015unsupervised,isola2017image}) use $\betabeta=(0.5,0.999)$. 
All these hyperparameters live  in the divergence region  {\small $\beta_1 < \sqrt{\beta_2}$}.
Surprisingly, instead of observing the divergence issue, these hyperparameters achieve good performances and they actually show the sign of convergence. 

 
 Why does Adam work well despite its theoretical divergence issue? Is there any mismatch between deep learning problems and the divergent example? 
We take a closer look into the divergence example and find out the mismatch {\it does} exist. 
In particular, we notice an important (but often ignored) characteristic of the divergence  example:
 \citep{reddi2019convergence} picks   $\betabeta$ {\it before} picking the sample size $n$.
 Put  in  another way, to construct the divergence example, they change $n$ for different $\betabeta$.
   For instance,  for $\betabeta=(0,0.99)$, they use one $n$ to construct the  divergent example; for $\betabeta = (0, 0.9999)$, they use  another $n$ to construct another divergent example. 
On the other hand, in practical applications of Adam listed above, 
practitioners tune the hyperparameters $\betabeta$ \textit{after} the sample size $n$ is fixed. 
So there is a gap between the setting of theory and practice: the order of picking $n$ and $\betabeta$ is different.

Considering the good performance of Adam under fixed $n$, we conjecture that Adam can converge in this setting. 
Unfortunately, the behavior of vanilla Adam is far less studied than its variants (perhaps due to the criticism of divergence).
 To verify this conjecture, 
 we run experiments for different choices of  $(\beta_1,\beta_2)$ on a few tasks.  First, we run Adam for a convex function \eqref{counterexample1} with fixed $n$ (see the definition in Section \ref{sec_div}). Second, we run Adam for the classification problem on data MNIST and CIFAR-10 with fixed batchsize. 
We observe some interesting phenomena in Figure \ref{fig:intro_paper} (a), (b) and (c). 

First,  
when $\beta_2$ is large, the optimization error is small for almost all values of $\beta_1$.  Second, when $\beta_1,\ \beta_2$ are both small, there is a 
red region with relatively large error. 
On MNIST, CIFAR-10, the error in the red region is  increased by 1.4 times than that in the blue region. The situation is a lot worse on function \eqref{counterexample1} (defined later in Section \ref{sec_div}): the error in the red region is 70 times higher. 

\begin{figure*}[htbp]
  \vspace{-4mm}
    \centering
    \subfigure[Function \eqref{counterexample1}]{
    \begin{minipage}[t]{0.25\linewidth}
    \centering
    \includegraphics[width=\linewidth]{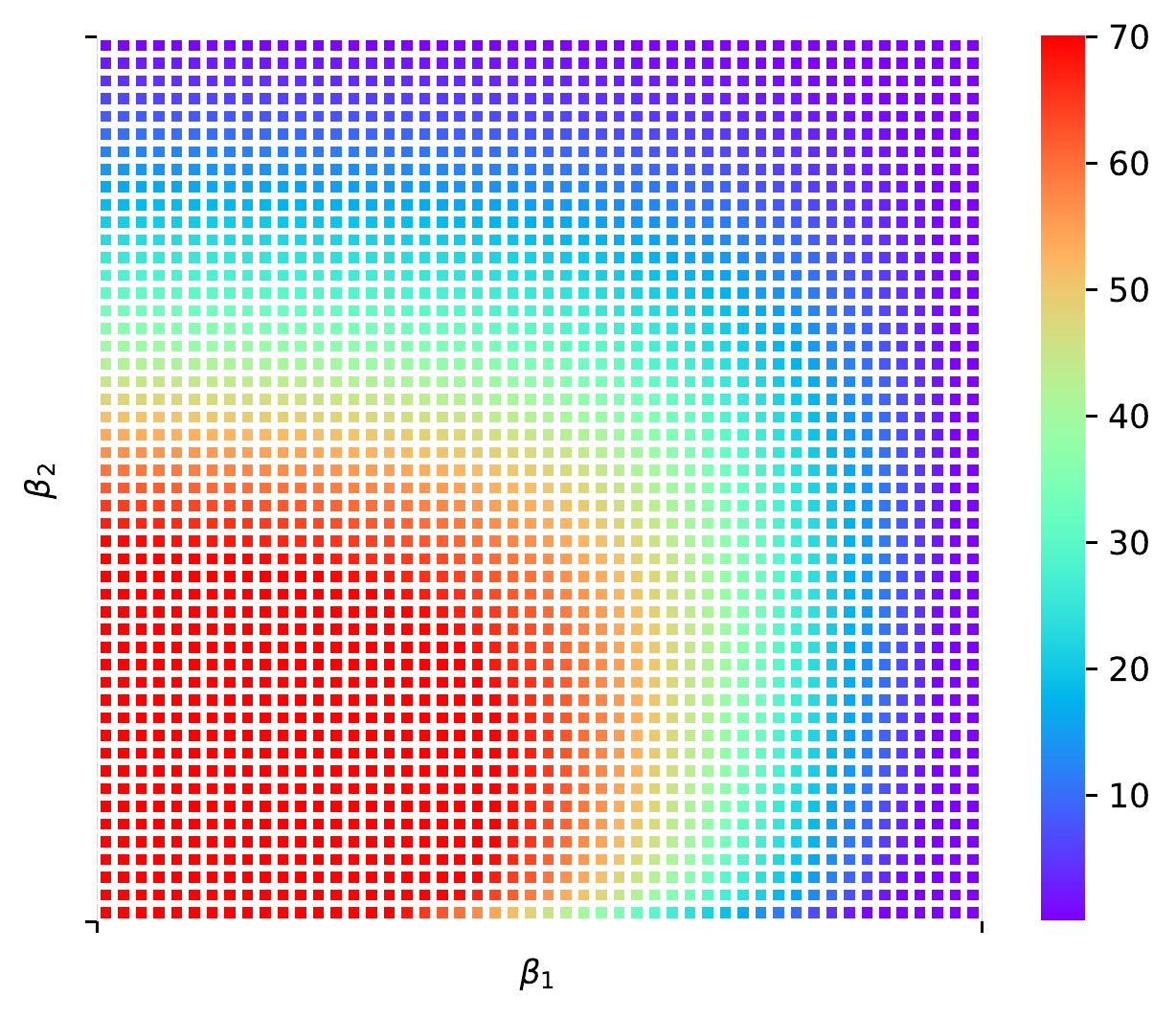}
    \end{minipage}%
    }%
    \subfigure[MNIST]{
      \begin{minipage}[t]{0.25\linewidth}
      \centering
    \includegraphics[width=\linewidth]{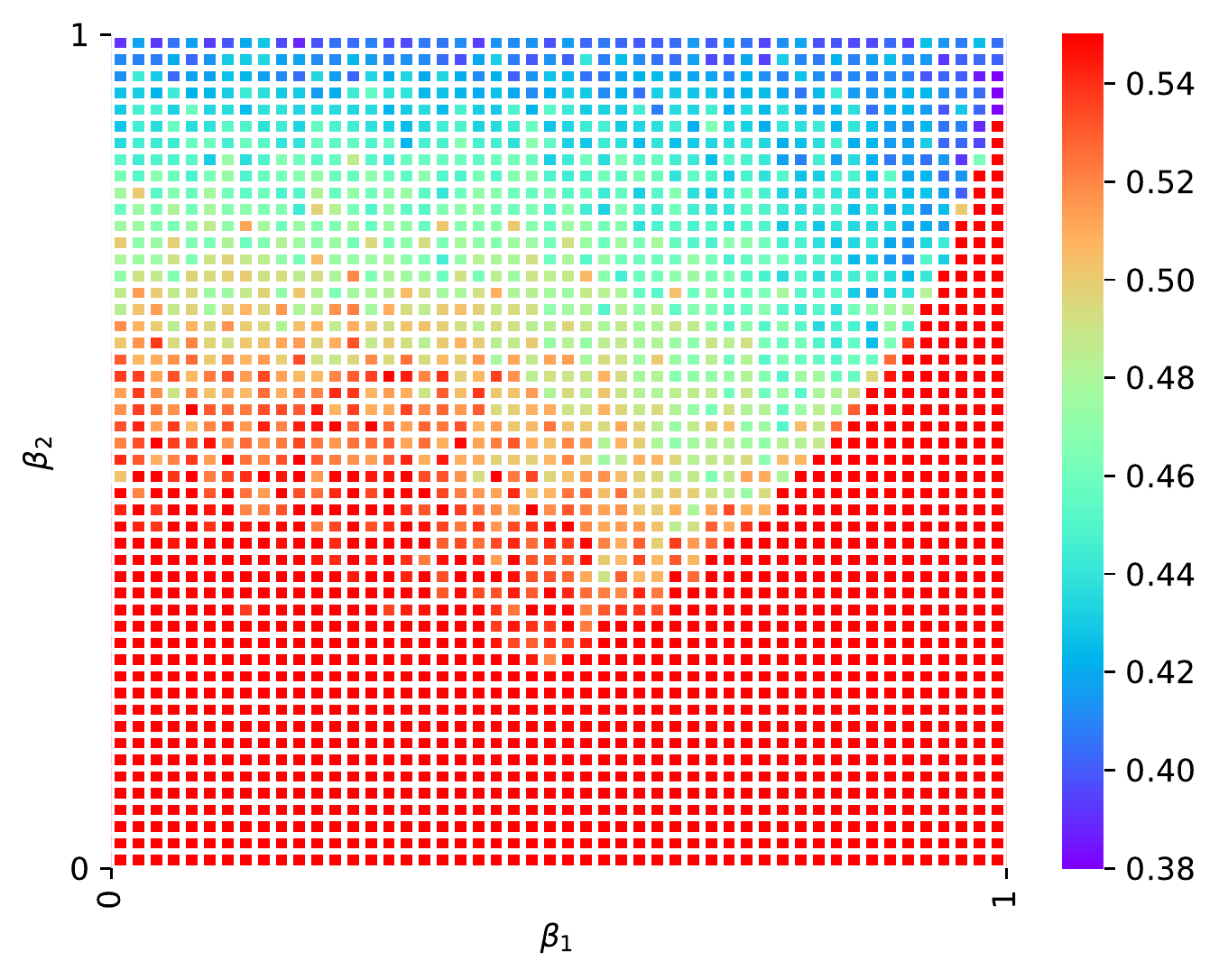}
      \end{minipage}%
      }%
    \subfigure[CIFAR-10]{
      \begin{minipage}[t]{0.25\linewidth}
      \centering
    \includegraphics[width=\linewidth]{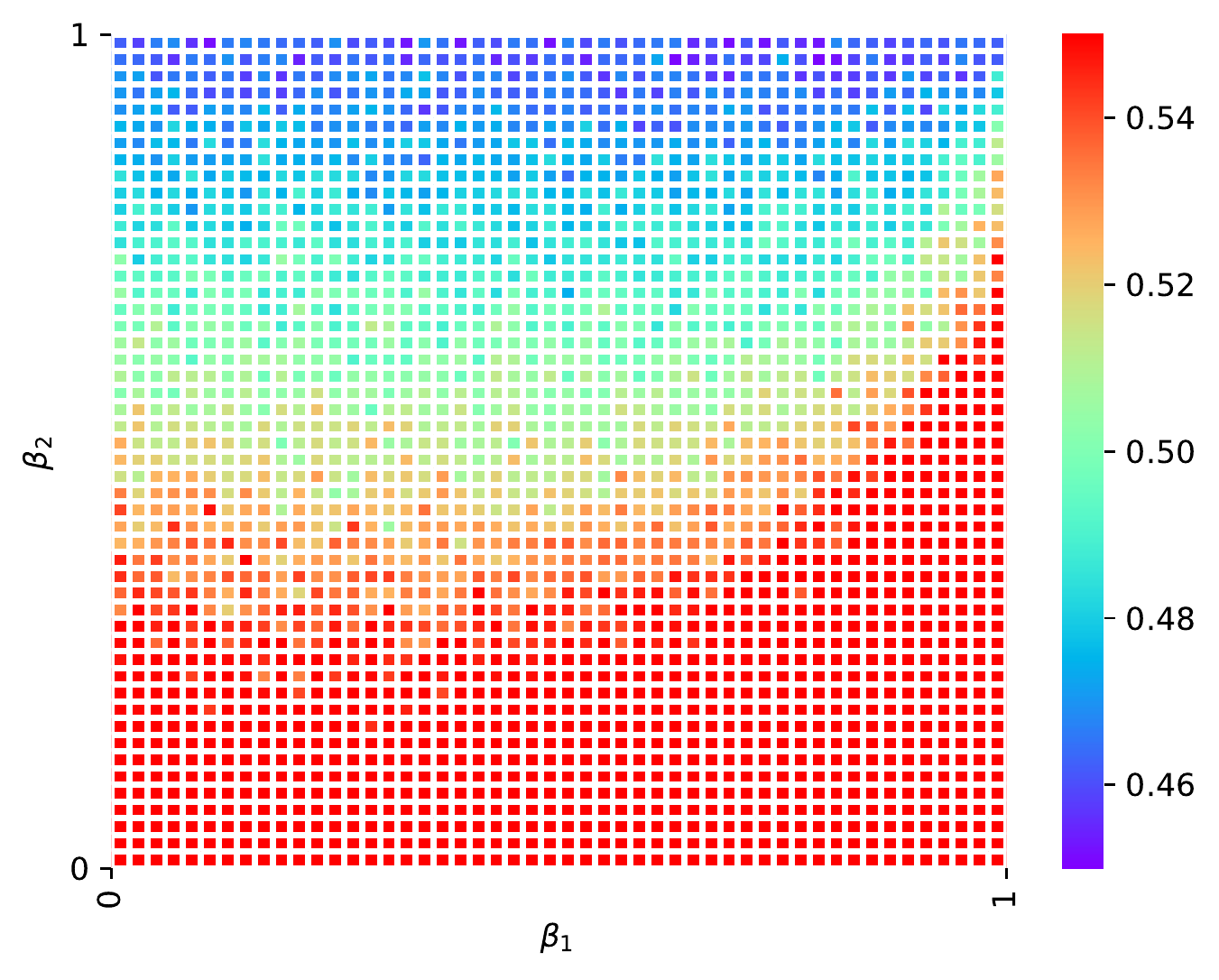}
      \end{minipage}%
      }%
    \subfigure[Our contribution]{
      \begin{minipage}[t]{0.25\linewidth}
      \centering
   \includegraphics[width=\linewidth]{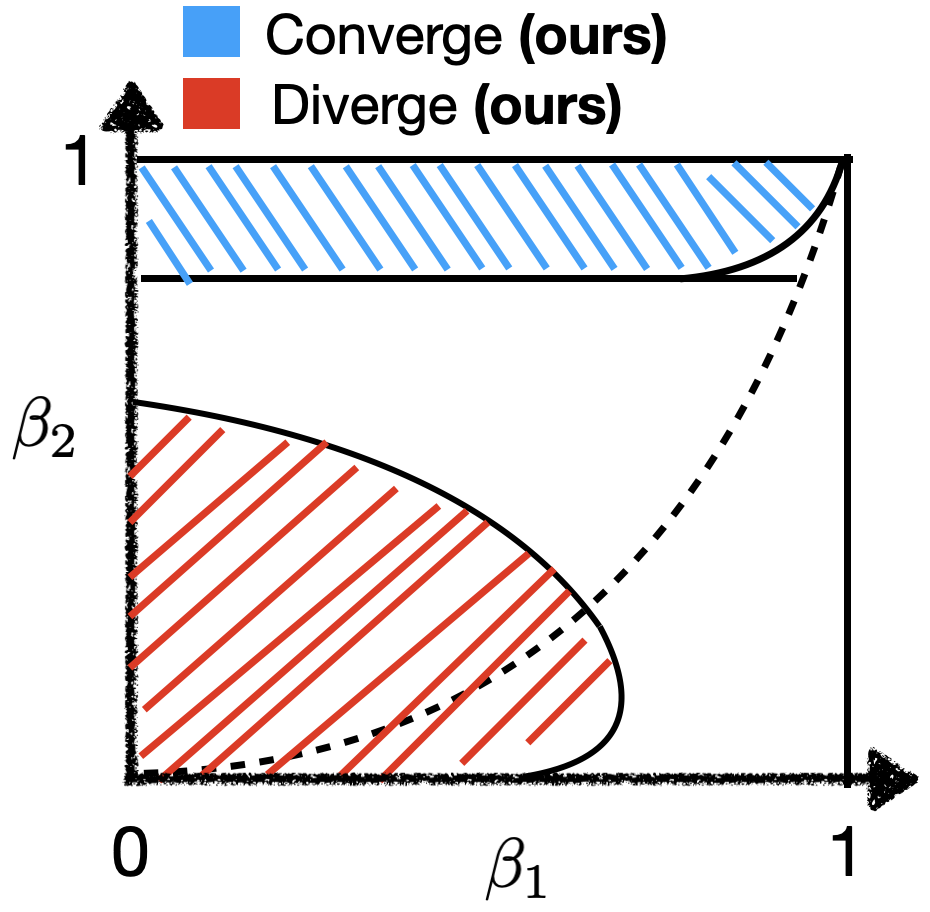}
      \end{minipage}%
      }%
    \centering
    \vspace{-4mm}
    \caption{ {\small {\bf (a), (b), (c)}: The performance of Adam on different tasks. For each task, we show the results with $\beta_1$ and  $\beta_2$ in grids $\{(k_1/50,k_2/50)| k_1 = 0,\cdots,49,k_2=0, \cdots, 49\}$. {\bf (a)}: the optimality gap $x-x^*$ on the convex function \eqref{counterexample1}. {\bf (b) (c)}: The training loss on MNIST and CIFAR-10.
    {\bf (d)}: An illustration of our contribution in $\betabeta$ phase diagram. The shape of the region follows the solution to our analytic conditions. The size of the region depends on $n$. The dotted curve satisfies $\beta_1 =\sqrt{\beta_2}$. In all figures, $n$ is fixed before picking $\betabeta$. }}%
    \label{fig:intro_paper}
\vspace{-1.5mm}
\end{figure*}

While Adam's performances seem unstable in the red region, we find that {\it Adam always performs well in the top blue region in Figure \ref{fig:intro_paper}.} This seems to suggest that Adam can converge without any algorithmic modification, as long as $\beta_1$ and $\beta_2$ are chosen properly.
We ask  the following question: 

\vspace{-1mm}
\begin{center}
\vspace{-2mm}
        {\it Can  Adam provably converge without any modification on its update rules?} 
\vspace{-1mm}
\end{center}
\vspace{-1mm}






In this work, we theoretically explore this question.
Our contributions are visualized in Figure \ref{fig:intro_paper} (d).  We prove the following results when $n$ is fixed (or more rigorously, when the function class is fixed):

\vspace{-2mm}
\begin{itemize} 
    \item  We prove that when $\beta_2$ is large enough
    and $\beta_1< \sqrt{\beta_2}$,  Adam converges to  the neighborhood of critical points. The size of the neighborhood is propositional to the variance of stochastic gradients.  With an extra condition (so-called strong growth condition), we prove that Adam can converge to critical points.
    As $\beta_2$  increases, these results can cover any momentum parameter $\beta_1 \in [0,1)$ including the default setting $\beta_1=0.9$. In particular, our analysis does {\it not} require bounded gradient assumption.
    
    \item 
  We study the  divergence issue of small-$\beta_2$ Adam.
    We prove that: for any fixed $n$ (or more rigorously, for any fixed function class), there exists a function such that, Adam diverges to infinity when $\betabeta$ is picked in the red region
    in Figure \ref{fig:intro_paper} (d). The size of the red region increases with $n$. The shape of the region follows the solution to our analytic conditions. 
    

\item 
We emphasize a few characteristics of our results. 
\textbf{(1) phase transition}. The divergence  result considers the same setting as our convergence result, indicating that there
is a phase transition from divergence to  convergence when changing $\beta_2$.
\textbf{(2) problem-dependent bounds}. 
Our convergence and divergence regions  of $\betabeta$ are problem-dependent,
which is drastically different from \citep{reddi2019convergence}  which 
 established the problem-independent worst-case choice of $\betabeta$. 
\textbf{(3) non-asymptotic characterization}. the ``divergence region'' of $\betabeta$ expands as $n$ increases and converges to the whole region $ [0,1)^2$ as $n$ goes to infinity, 
 which recovers (actually stronger than) the problem-independent divergence result  of \citep{reddi2019convergence} that requires $\beta_1 < \sqrt{\beta_2}$. 
 In this sense, we can view the divergence result of \citep{reddi2019convergence} 
as  an asymptotic characterization  of the divergence region (as $n \rightarrow \infty$)
 and our divergence  result as a non-asymptotic characterization (for any fixed $n$). We provide more discussion in Section \ref{sec:reconcile}.
 
 \item  Our positive and negative results  can provide suggestions for tuning $\beta_1$ and $\beta_2$: for instance,when Adam does not work well, we suggest tuning up $\beta_2$  and trying $\beta_1 <\sqrt{\beta_2}$. We provide more tuning suggestions in Appendix \ref{sec:implication}.

\end{itemize}


We believe our results can boost new understandings for Adam.
While  \citet{reddi2019convergence} reveal that ``Adam can diverge", our results show the other side of the coin: when $n$ is fixed (or when function class is fixed), Adam can still converge without any  modification on its update rules.  
Our results suggest that Adam is still a theoretically justified algorithm and practitioners can use it confidently. 



We further emphasize that 
our convergence results can cover any $\beta_1 \in [0,1)$, which allows the algorithm to bring arbitrarily heavy momentum signals.
It turns out that large-momentum Adam is not easy to analyze. 
Even with stronger assumptions like bounded gradient ($\| \nabla f(x) \| < \text{C}, \forall x $), its convergence is not well understood (see related works in Section \ref{section:related}).
To our best knowledge, this is the first result that proves vanilla Adam with any $\beta_1$ can converge {\it without} any assumption of  bounded gradient or bounded 2nd-order momentum. 
The proof contains a new method to handle unbounded momentum in the stochastic non-linear dynamics system. 
We will highlight our technical novelties in Section \ref{sec:proofidea}.

 \vspace{-2mm}
\section{Preliminaries}
 \vspace{-2mm}

\subsection{Review of Adam}
\label{section:preliminaries}
 \vspace{-2mm}

We consider finite-sum problem \eqref{finite_sum}.
We use $x$  to denote the optimization variable. 
We denote $\nabla f_{j}$ as the gradient of $f_{j}$ and let $\circ$ be the component-wise product. The division and square-root operator are component-wise as well. 
We present randomly shuffled Adam in Algorithm \ref{algorithm}.

\vspace{-1mm}
\begin{algorithm}
    \caption{Adam }
    \label{algorithm}
    \begin{algorithmic}
     \STATE Initialize $x_{1,0}  = x_0$, $m_{1,-1}= 
      \nabla f(x_0)$ 
     and $v_{1,-1}= \max_i \nabla f_i(x_0)\circ \nabla f_i(x_0)$.
      \FOR {$k=1\to\infty$}
      \STATE Sample $\{\tau_{k,0},\tau_{k,1},\cdots,\tau_{k,n-1}\}$ as a random permutation of $\{0,1,2,\cdots,n-1\}$
      \FOR  {$i=0 \to n-1$}
      \STATE 
      $m_{k,i}=\beta_1 m_{k,i-1}+\left(1-\beta_1\right)\nabla f_{\tau_{k,i}}(x_{k,i})$ 
      \STATE $v_{k,i}= \beta_2 v_{k,i-1}+\left(1-\beta_2\right)\nabla f_{\tau_{k,i}}(x_{k,i})\circ \nabla f_{\tau_{k,i}}(x_{k,i})$ 
      \STATE $x_{k,i+1}=x_{k,i}-\frac{\eta_{k}}{\sqrt{v_{k,i}}+\epsilon}\circ m_{k,i}$	
       \ENDFOR
      \STATE $x_{k+1,0}=x_{k,n}$; $v_{k+1,-1}=v_{k,n-1}$; $m_{k+1,-1}=m_{k,n-1}$
       \ENDFOR
    \end{algorithmic}
  \end{algorithm}

In Algorithm \ref{algorithm}, $m$ denotes the 1st-order momentum and $v$ denotes the 2nd-order momentum.  they are weighted averaged by hyperparameter  $\beta_{1}, \beta_{2}$, respectively.
Larger $\beta_1$, $\beta_{2}$ will adopt more history information. 
We denote $x_{k, i}, m_{k, i}, v_{k, i} \in \mathbb{R}^{d}$ as the value of $x, m, v$ at the $k$-th outer loop (epoch) and $i$-th inner loop (batch), respectively. 
We choose $\eta_{k}=\frac{\eta_{1}}{\sqrt{nk}}$ as the stepsize. In practice, $\epsilon$ is adopted for numerical stability and it is often chosen to be $10^{-8}$. In our theory, we allow $\epsilon$ to be an arbitrary non-negative constant including $0$.

In the original version of Adam in \citep{kingma2014adam}, it has an additional ``bias correction'' step. This ``bias correction'' step can be implemented by changing the stepsize $\eta_k$ into $\hat{\eta_k}=\frac{\sqrt{1-\beta_2^k}}{1-\beta_1^k} \eta_k$ and using zero initialization. In Algorithm \ref{algorithm}, the ``bias correction'' step is replaced by a special initialization, which  corrects the bias as well. Note that  $\hat{\eta_k} \in [\sqrt{1-\beta_2} \eta_k,\frac{1}{1-\beta_1}\eta_k]$ is well-bounded near $\eta_k$, so $\eta_k$ and $\hat{\eta_k}$ brings the same convergence rate. In addition, 
as the effect of initialization becomes negligible when the training progresses, Adam with zero \& our initialization will have the same asymptotic behavior. In the main body of our proof, we follow the form of Algorithm \ref{algorithm}, which makes results cleaner. For completeness, we add the proof on the convergence of Adam with ``bias correction'' steps in Appendix \ref{appendix:bias_correction}. 

In our analysis, we make the  assumptions below.

\begin{assumption} \label{assum1}
    we consider $x \in \mathbb{R}^d$ and $f_i(x) $ satisfies gradient Lipschitz continuous with constant $L$. We assume $f(x)$ is lower bounded by a finite constant $f^*$. 
 \end{assumption}
\begin{assumption}\label{assum2}
  $f_i(x) $ and $f(x)$ satisfy: $      \sum_{i=0}^{n-1}\left\|\nabla f_{i}(x)\right\|_{2}^{2} \leq D_{1}\|\nabla f(x)\|_{2}^{2}+D_{0}, \forall x\in \mathbb{R}^d.$
  \end{assumption}
  \vspace{-1mm}
  
      Assumption \ref{assum2} is quite general. When $D_1=1/n$, it becomes the ``constant variance'' with constant $D_0/n$.
        ``constant variance" condition is commonly used in both SGD and Adam analysis (e.g. \citep{ghadimi2016mini, Manzil2018adaptive, huang2021super}).  Assumption \ref{assum2}  allows more flexible choices of $D_1 \neq n$ and thus it is weaker than ``constant variance''. 
      
      When $D_{0}>0$, the problem instance is sometimes called ``non-realizable" \citep{shi2020rmsprop}. In this case, adaptive gradient methods are not guaranteed to reach the exact critical points. Instead, they only converge to a bounded region (near critical points) \citep{Manzil2018adaptive,shi2020rmsprop}. 
      This phenomenon indeed  occurs for Adam in experiments, even with diminishing stepsize (see Figure \ref{fig:cifar_nlp_mnist} (a)). The behavior of SGD is similar: constant stepsize SGD converges to a bounded region with its size propositional to the noise level $D_0$ \citep{yan2018unified,yu2019linear,liu2020improved}.

    When $D_{0}=0$,  Assumption \ref{assum2} 
      is often called ``strong growth condition" (SGC) \citep{vaswani2019fast}. 
      When $\|\nabla f(x)\|=0$, under SGC we have $\left\|\nabla f_{j}(x)\right\|=0$ for all $j$.
     SGC is increasingly popular recently e.g.\citep{ schmidt2013fast,vaswani2019fast}. 
     This condition is known to be reasonable in the overparameterized regime where neural networks  can interpolate all data points
    \citep{vaswani2019fast}. 
     We will show that Adam can converge to critical points if SGC holds.
      




When  $n, f^*, L, D_0, D_1$ are fixed a priori, we use  $\functionclass$ to denote the function class containing $f(x)$ satisfying Assumption \ref{assum1} and \ref{assum2} with constant $n, f^*$, etc..  Since  $n$ is fixed when the function class $\functionclass$, we 
 introduce this notation to clearly present the divergence result in Proposition \ref{thm_diverge}. Without this pre-defined function class, the claim of divergence might be confusing.


\vspace{-2mm}
\subsection{Related Works}
\label{section:related}
\vspace{-2mm}

Ever since \citet{reddi2019convergence} pointed out the divergence issue, there are many attempts on designing new variants of Adam. 
Since we focus on understanding Adam {\it without modification} on its update rules, we introduce more variants later in Appendix \ref{appendix:more_related}. 

Compared with proposing new variants, 
the convergence of vanilla Adam is far less studied than its variants (perhaps due to the criticism of divergence). 
There are only a few works analyzing vanilla Adam and they require extra assumptions. 
\citet{zhou2018adashift} analyze the counter-example  in \citep{reddi2019convergence} and find certain hyperparameter can work. However, their analysis is restricted to the counter-example.
\citet{Manzil2018adaptive} study the relation between   mini-batch sizes and (non-)convergence of Adam. 
However,  this work require $\beta_1=0$ and Adam is reduced to  RMSProp \citep{hinton2012neural}.  \citet{de2018convergence} analyze RMSProp and non-zero-$\beta_1$ Adam, but they assume the sign of all stochastic gradients to keep the same. It seems unclear how to check this condition a priori.
Additionally, they require $\beta_1$ to be inversely related to the upper bound of gradient, which forces $\beta_1$ to be small (as a side note, this result only applies to full-batch Adam).
\citet{defossez2020convergence} analyze Adam with $\beta_1<\beta_2$ and  provide some insights on the momentum mechanisms. 
However, their bound is inversely proportional to $\epsilon$ (the hyperparameter for numerical stability) and the bound goes to infinity when $\epsilon$ goes to 0. 
This is different from practical application since small $\epsilon $ such as $10^{-8}$ often works well. Further, using large $\epsilon$ is against the nature of adaptive gradient methods because $\sqrt{v}$ no longer dominates in the choice of stepsize. In this case, Adam is essentially transformed back to SGD. 
 Two recent works \citep{huang2021super} and \citep{guo2021novel}  propose novel and simple frameworks to analyze Adam-family with large $\beta_1$. 
Yet, they require the effective stepsize of Adam to be bounded in certain interval, i.e.,  $\frac{1}{\sqrt{v_t}+ \epsilon} \in [C_l, C_u]$ \footnote{For completeness, we explain why they require this condition in Appendix \ref{appendix:more_related}.}. This boundedness condition changes Adam into AdaBound \citep{luo2018adaptive} and thus they cannot explain the observations on original Adam in Section \ref{section:intro}.
To summarize, all these works require at least one  strong assumption (e.g. large $\epsilon$).
Additionally, they all (including those for new variants) require bounded gradient assumptions.

A recent work \citep{shi2020rmsprop} takes the first attempt to analyze RMSProp 
without  bounded gradient assumption.
They show that RMSProp can converge 
to the neighborhood of critical points. 
\footnote{We notice that they also provide a convergence result for Adam with $\beta_1$ close enough to 0. However, a simple calculation by \citet{yushun2022doesadamconverge} shows that they require $\beta_1<10^{-7}$. Thus their result does not provide much extra information other than RMSProp.} We believe it is important to study Adam rather than RMSProp: Numerically, Adam often outperforms RMSProp on complicated tasks (e.g. on Atari games, the mean reward is improved from 88\% to 110\% \citep{agarwal2020optimistic}).  Theoretically, 
literature on RMSProp cannot reveal the interaction between $\beta_1$ and $\beta_2$; or how these hyperparameters jointly affect (or jeopardize) the convergence of Adam.
However, it is non-trivial to jointly analyze the effect of  $\beta_1$ and $\beta_2$. 
We point out there are at least three challenges. First, it seems unclear how to control the massive momentum $m_t$ of Adam. Second, $m_t$ is multiplied by  $1/\sqrt{v_t}$, causing non-linear perturbation. Third, $m_t$ and  $1/\sqrt{v_t}$ are statistically dependent and cannot be decoupled. 
 We propose new methods to resolve these issues. We highlight our  technical novelties in Section \ref{sec:proofidea}. 

\vspace{-2mm}
\subsection{The Importance and Difficulties of Removing Bounded Gradient Assumptions}
\label{section:beta1}
\vspace{-2mm}

Here, we emphasize the importance to remove bounded gradient assumption.
First, unlike the assumptions in Section \ref{section:preliminaries}, bounded gradient is {\it not} common in SGD analysis. So it is of theoretical interests to remove this condition for Adam. 
Second, bounded gradient condition rules out the chances of  gradient divergence a priori. However, there are numerical evidences  showing that Adam's gradient can diverge (see Section \ref{section:exp} and Appendix \ref{appendix:more_exp}).
Removing the boundedness assumption helps us point out the divergence and convergence phase transition in the $\betabeta$ diagram. 




However, it is
often difficult to analyze convergence without  bounded gradient assumption.  First,  it is non-trivial to control stochastic momentum. Even for SGD, this task is challenging. 
For instance, An early paper \citet{bertsekas2000gradient} analyzed SGD-type methods without any boundedness condition. But it is not until recently that \citet{yu2019linear, liu2020improved, jin2022convergence} prove SGDM (SGD with momentum) converges without bounded gradient assumption.
 Such attempts of removing boundedness assumptions are often appreciated for general optimization problems  where ``bounded-assumption-free" is considered as a major contribution.  

Secondly, for Adam, the role of momentum $m_{t}$ is even more intricate since it is multiplied by  $1/\sqrt{v_t}$. Combined with $v_t$, the impact of previous signals not only affect the update direction, but also change the stepsize for each component. Further, both momentum $m_t$ and stepsize $1/\sqrt{v_t}$ are random variables and they are highly correlated. Such statistical dependency causes  trouble for analysis. 
In summary, the role of momentum in Adam could be much different from that in  SGDM or GDM.
 Even with boundedness conditions, the convergence of large-$\beta_1$ Adam is still not well understood (see related works in Section \ref{section:related}).  
In this work, we propose new techniques to handle Adam's momentum under any large $\beta_1$, regardless of the gradient magnitude. These techniques are not revealed in any existing works. We introduce our technical contribution in Section \ref{sec:proofidea}.

\vspace{-2mm}
\section{Main Results}
\label{sec_main}
\vspace{-2mm}
\subsection{Convergence Results}
\label{sec_conv}
\vspace{-2mm}
Here, we give the convergence results under large $\beta_2$.
\begin{theorem}
\label{thm1}
For any $f(x) \in \mathcal{F}_{L, D_0, D_1}^{n,f^*}  (\mathbb{R}^d)$, we assume the hyperparameters in  Algorithm \ref{algorithm} satisfy: $\beta_1 < \sqrt{\beta_2} <1$;  $\beta_2$ is greater or equal to a threshold $ \gamma_1(n)$; and $\eta_{k}=\frac{\eta_{1}}{\sqrt{nk}}$. Let $k_m \in \mathbb{N}$ satisfies $k_m\ge 4$ and  {\small $\beta_{1}^{(k_m-1) n} \leq \frac{\beta_{1}^{n}}{\sqrt{k_m-1}}$}, \footnote{When $\beta_1=0.9$, $k_m = 15$ for any $n \geq1$. } we  have the following results for any $T>k_m$: 

\vspace{-3mm}
{\small
\begin{align*}
     \min_{k \in [k_m, T]} \Ex\left\{\min  \left[ \sqrt{\frac{ 2D_1 d }{D_{0} } } \|\nabla f(x_{k,0})\|_2^2 , \|\nabla f(x_{k,0})\|_2\right]  \right\}
   =\mathcal{O}\left(\frac{\log T }{\sqrt{T}} \right)  + \mathcal{O}(\sqrt{D_0}) .\nonumber  
\end{align*}
}%





\end{theorem}

\vspace{-2mm}
  \paragraph{Remark 1: the choice of $\beta_2$.} 
    Our theory suggests that large $\beta_2$ should be used to ensure convergence. This message matches our experimental findings  in Figure \ref{fig:intro_paper}. We would like to point out that the requirement of ``large $\beta_2$" is neccessary, because small $\beta_2$ will indeed lead to divergence (shown later in Section \ref{sec_div}).
    We here comment a bit on the the threshold $\gamma_1(n)$. $\gamma_1(n)$ satisfies $\beta_2 \geq 1- \mathcal{O}\left(\frac{1-\beta_1^n}{n^{2}\rho}\right)$ (see inequality \eqref{condition} and Remark \ref{remark_rho}), where $\rho$ is a constant that depends on the training trajectory. In worst cases, $\rho$ is upper bounded by $n^{2.5}$, but we find the practical $\rho$ to be much smaller.
    In Appendix \ref{appendix:more_exp}, we estimate $\rho$ on MNIST and CIFAR-10. In practical training process,  we empirically observe that $\rho \approx \mathcal{O}(n)$, thus the required $\gamma_1(n) \approx 1- \mathcal{O}\left(n^{-3}\right)$.  
    Note that our threshold of $\beta_2$ is a sufficient condition for convergence, so there may be a gap between the practical choices and the theoretical bound of $\beta_2$. Closing the gap will be an interesting future direction. 
  
  We find that $\gamma_1(n)$ increases with $n$. This property suggests that larger $\beta_2$ should be used when $n$ is large. This phenomenon is also verified by our experiments in Appendix \ref{appendix:more_exp}.
We also remark that $\gamma_1(n)$ slowly increases with $\beta_1$. This property is visualized in Figure \ref{fig:intro_paper} (d) where the lower boundary of blue region slightly lifts up when $\beta_1$ increases.
  


\vspace{-2mm}
 \paragraph{ Remark 2: the choice of $\beta_1$.} Theorem \ref{thm1} requires  $\beta_1 < \sqrt{\beta_2}$. 
  Since $\beta_2$ is suggested to be large, our convergence result can cover flexible choice of $\beta_1 \in [0,1)$. For instance,  $\beta_2=0.999$ brings the threshold of $\beta_1 < 0.9995$, which covers  basically all practical choices of $\beta_1$ reported in the literature (see Section \ref{section:intro}), including the default setting $\beta_1=0.9$. This result is much stronger than those in the RMSProp literature (e.g. \citep{shi2020rmsprop,Manzil2018adaptive}). To our knowledge, we are the first to prove convergence of Adam under any $\beta_1 \in [0,1)$ without bounded gradient assumption.
  

\paragraph{Remark 3: convergence to a bounded region.}
  When {\small $D_0>0$}, Adam converges to a bounded region near critical points.
  As discussed in Section \ref{section:preliminaries}, converging to bounded region is common for stochastic methods including constant-stepsize SGD  \citep{yan2018unified,yu2019linear,liu2020improved} and  diminishing-stepsize RMSProp \citep{Manzil2018adaptive,shi2020rmsprop}.  This phenomenon is also observed in practice: even for convex quadratic function with  {\small $D_0>0$},  Adam with diminishing stepsize {\it cannot} reach exactly zero gradient (see Figure \ref{fig:cifar_nlp_mnist} (a) in Section  \ref{section:exp}).  This is because: even though $\eta_k$ is decreasing, the effective stepsize $\eta_k/\sqrt{v_{k,i}}$ might not decay.   
  The good news is that, the constant  {\small $\mathcal{O}(\sqrt{D_0})$} vanishes to 0 as $\beta_2$ goes to 1 (both in theory and experiments). The relation between $\beta_2$ and constant {\small $\mathcal{O}(\sqrt{D_0})$} are introduced in Remark \ref{remark_f4} in Appendix  \ref{appendix_lemma_a1}. The size shrinks to 0 because the movement of  $\sqrt{v_{k,i}}$ shrinks as  $\beta_2$ increases.

  As a corollary of Theorem \ref{thm1}, we have the following result under SGC (i.e., $D_0=0$).
  
  
\begin{corollary}\label{thm1_sgc}
  Under the setting in Theorem \ref{thm1}. When $D_0=0$ for Assumption \ref{assum2}, we have
  \vspace{-0.5mm}
{\small
\begin{align*}
     \min_{k \in [k_m, T]} \Ex\|\nabla f(x_{k,0})\|_2
   =\mathcal{O}\left(\frac{\log T }{\sqrt{T}} \right).\nonumber  
\end{align*}
}%
 \end{corollary}
  Under  SGC (i.e. $D_0=0$), Corollary \ref{thm1_sgc} states that Adam can converge to critical points.  This is indeed the case in practice. For instance, function \eqref{counterexample1} satisfies SGC and we observe 0 gradient norm after  Adam converges (see Section \ref{section:exp}  and Appendix \ref{appendix:more_exp}). The convergence rate in Corollary \ref{thm1_sgc}  is comparable to that of SGD under the same condition in \citep{vaswani2019fast}.

\vspace{-2mm}
\subsection{Divergence Results}
\label{sec_div}
\vspace{-2mm}
Theorem \ref{thm1} shows that when $\beta_2$ is large, any $\beta_1 < \sqrt{\beta_2}$ ensures convergence. Now we consider the case where $\beta_2$ is small. We will show that in this case, a wide range of $\beta_1$ is facing the risk of diverging to infinity. The divergence of small-$\beta_2$ Adam  suggests that ``large $\beta_2$" is necessary in the convergence result Theorem \ref{thm1}. 
We construct a counter-example in $\functionclass$.  Consider $f(x)=\sum_{i=0}^{n-1} f_i(x)$ for  $x \in \mathbb{R}$
, we define  $ f_i(x)$ as:
{\small
\begin{eqnarray}
  f_{i}(x)&=&\left\{\begin{array}{ll} nx, &  x \geq -1\\ \frac{n}{2}(x+2)^2 -\frac{3n}{2}, & x < -1 \end{array}\right.  \text{for $i=0$,} \nonumber \\
  f_{i}(x)&=&\left\{\begin{array}{ll} -x, &  x \geq -1\\ -\frac{1}{2}(x+2)^2 +\frac{3}{2}, & x < -1 \end{array}\right.  \text{for $i>0$.} \label{counterexample1}
\end{eqnarray}
}%
Summing up all the $f_i(x)$, we can see that 
{\small
$$
f(x)= \left\{\begin{array}{ll} x, &  x \geq -1\\ \frac{1}{2}(x+2)^2 -\frac{3}{2}, & x < -1 \end{array}\right.
$$
}%
is a lower bounded convex smooth function with optimal solution $x^*=-2$.   
Function \eqref{counterexample1} allows both iterates and gradients to diverge to infinity.
As shown in Figure \ref{fig:intro_paper} (a), when running Adam on \eqref{counterexample1}, there exists a red large-error region. This shows the sign of divergence. We further theoretically verify the conjecture in Proposition \ref{thm_diverge}.

\begin{proposition}\label{thm_diverge} 
For any  function class $\functionclass$, there exists a $f(x) \in \functionclass$, s.t. when $\betabeta$ satisfies analytic condition \eqref{diverge_c1}, \eqref{diverge_c2}, \eqref{diverge_c3} in Appendix \ref{appendix:thm_diverge}, Adam's iterates and function values diverge to infinity.
By solving these  conditions in \texttt{NumPy}, 
we plot the orange region in Figure  \ref{fig:counter_boundary_paper}.
The size of the region depends on $n$ and it expands to the whole region when $n$ goes to infinity.


\end{proposition}
  \vspace{-1mm}
\begin{wrapfigure}{r}{0.3\textwidth}
\begin{center}
\vspace{-6mm}
\centerline{\includegraphics[width=0.3\textwidth]{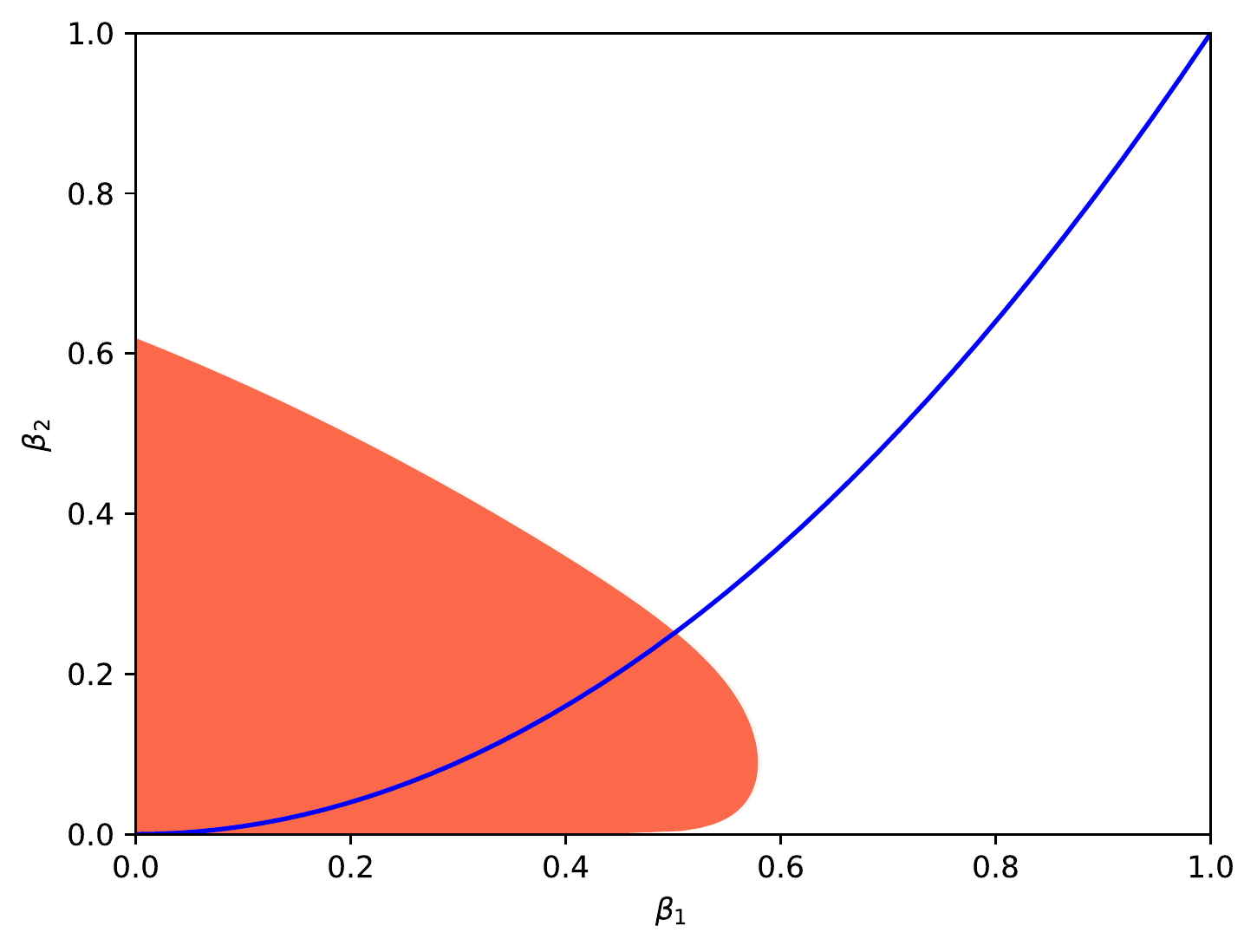}}
\vspace{-3mm}
\caption{{\small 
On function \eqref{counterexample1} with  $n=20$, Adam diverges in the colored region. The region is plotted by solving  condition \eqref{diverge_c1}, \eqref{diverge_c2}, \eqref{diverge_c3} in \texttt{NumPy}.
The blue curve satisfies  $\beta_1 = \sqrt{\beta_2}$.
}
}
\label{fig:counter_boundary_paper}
\vspace{-3mm}
\end{center}
\vspace{-6mm}
\end{wrapfigure}
\vspace{-1mm}

  The proof can be seen in Appendix \ref{appendix:thm_diverge}. We  find 
 the ``divergence region" 
 always stays below the ``convergence threshold" $\gamma_1(n)$ in Theorem \ref{thm1}, so the two results are self-consistent (see the remark in Appendix \ref{appendix:thm_diverge}). Proposition \ref{thm_diverge} states the divergence of iterates and function values. 
 Consistently, our experiments also show the divergence of gradient (see Section \ref{section:exp} and Appendix \ref{appendix:more_exp}). These results characterize Adam's divergence behavior both numerically and theoretically.


We emphasize that the orange region is {\it not} discussed in \citep{reddi2019convergence} because we consider $n$ fixed while they allow $n$ changing. When $n$ is allowed to increase,  our orange region will expand to the whole region and thus we can derive a similar (actually stronger)  result as \citep{reddi2019convergence}.
We provide more explanation in Section \ref{sec:reconcile}.
Combining Theorem \ref{thm1} and Proposition \ref{thm_diverge}, we establish a clearer image on the relation between $\betabeta$ and qualitative behavior of Adam.

\vspace{-2mm}
\section{Reconciling Our Results with \citep{reddi2019convergence} }
\label{sec:reconcile}
\vspace{-2mm}

\begin{wrapfigure}{r}{0.3\textwidth}
\vskip -0.2in
\begin{center}
\vspace{-4mm}
\centerline{\includegraphics[width=0.3\textwidth]{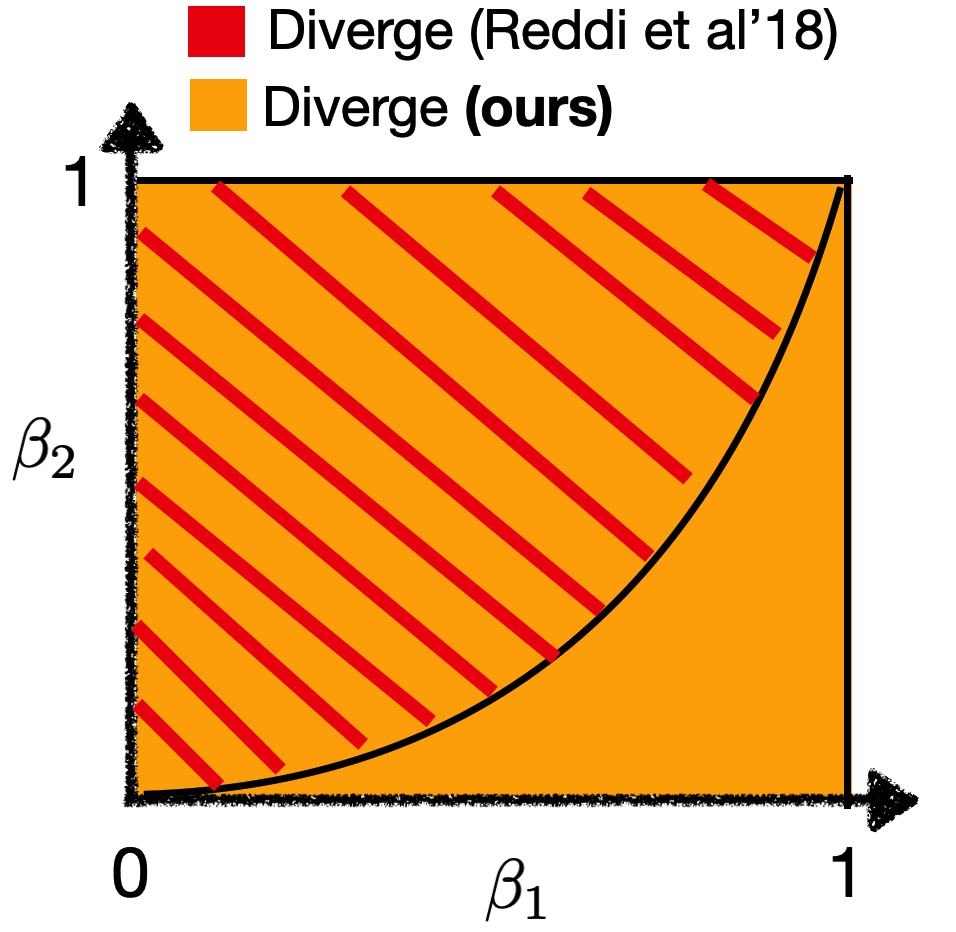}}
\vspace{-2mm}
\caption{{\small Adam's behavior when $\betabeta$ in {\bf Case I}. 
}
}
\label{fig:reconcile_case1}
\end{center}
\vskip -0.4in
\end{wrapfigure}

We  discuss more on the relation between \citep{reddi2019convergence} and our results. The divergence result shown in Section \ref{section:intro} does not contradict with our convergence results in   Theorem \ref{thm1}. Further, it is different from our divergence result in Proposition \ref{thm_diverge}. The key difference lies in whether  $\betabeta$ is picked {\it before or after} picking  the function class $\functionclass$. We discuss the following two cases.


\paragraph{Case I: When $\betabeta$ is picked before picking $\functionclass$.} 
As discussed in Section \ref{section:intro}, the divergence result requires different $n$ for different $\betabeta$. In this sense, the considered function class is constantly changing. It does {\it not}  contradict with our Theorem \ref{thm1}  
which considers a fixed function class with fixed $n$. For {\bf Case I}, we illustrate Adam's behavior in Figure \ref{fig:reconcile_case1}. The red region  is proved by \citep{reddi2019convergence}. For completeness, we remove the condition ``$\beta_1 <\sqrt{\beta_2}$" and further prove that Adam will diverge to infinity for {\it any} $\betabeta \in [0,1)^2$. The result is shown in  the following Corollary \ref{corollary_diverge}.

\begin{corollary} \label{corollary_diverge}
For any $\betabeta \in [0,1)^2$, there exists a function satisfying Assumption \ref{assum1} and \ref{assum2} that the Adam's iterates and function values diverge to infinity.
\end{corollary}
\vspace{-1mm}


Proof of Corollary \ref{corollary_diverge} can be seen in the final paragraph in Appendix \ref{appendix:thm_diverge}. In the proof, we also require different $n$ to cause divergence for different $\betabeta$. So the function class is constantly changing.
As a result, in {\bf Case I}, we cannot prove any convergence result.





\vspace{-4mm}
{\small 
\begin{table*}[h]
\caption{Possible algorithmic behaviors of Adam  in {\bf Case II}.
}
\label{sample-table}
\begin{center}
\begin{small}
\begin{tabular}{lll}
\toprule
Setting  & Hyperparameters  & Adam's behavior\\
\midrule
\makecell[l]{$\forall f \in \functionclass$ \\with $D_0=0$ }& $\beta_2$ is large and  $\beta_1 < \sqrt{\beta_2}$ &\makecell[l]{converge to \\critical points {\bf(Ours)} }\\
\makecell[l]{$\forall f \in \functionclass$ \\ with $D_0 \neq 0$}& $\beta_2$ is large and  $\beta_1 < \sqrt{\beta_2}$ &\makecell[l]{converge to a bounded \\ region with size    $\mathcal{O}(D_0)$ {\bf(Ours)} }\\
$\exists f \in \functionclass$ & The orange region in Figure \ref{fig:counter_boundary_paper} &diverge to infinity {\bf(Ours)} \\
\bottomrule
\end{tabular}
\end{small}
\end{center}
\vspace{-2mm}
\label{table_adam_behavior}
\end{table*} }

\paragraph{Case II: When $\betabeta$ is picked after picking  $\functionclass$.} When the function class is picked in advance,
sample size $n$ will also be fixed. This case is closer to most practical applications.  In this case, we find that Adam's  behavior changes significantly in the different region of Figure \ref{fig:reconcile_case1}.
First, $\forall f(x) \in \functionclass$ will converge when $\beta_1<\sqrt{\beta_2}$ and  $\beta_2$ is large.  Second,  $\exists f(x) \in \functionclass$ will diverge to infinity when $(\beta_1, \beta_2)$ are in the orange region in Figure \ref{fig:counter_boundary_paper}. Since {\bf Case II} is closer to practical scenarios, these results can provide better guidance for hyperparameter tuning for Adam users. We provide some suggestions for practitioners in  Appendix \ref{sec:implication}.

For {\bf Case II}, we summarize the possible behaviors of Adam in Table \ref{table_adam_behavior}. 
We  also illustrate our convergence and divergence results in Figure \ref{fig:intro_paper} (d). Note that there are some blanket areas where Adam's behavior remains unknown, this part will be left as interesting future work.

\vspace{-2mm}
\section{Proof Ideas for the Convergence Result}
\label{sec:proofidea}
\vspace{-2mm}

We now (informally) introduce our proof ideas for the convergence result in Theorem \ref{thm1}. Simply put, we want to control the update direction $m_{k,i}/\sqrt{v_{k,i}}$ inside the dual cone of  gradient direction. Namely:
{\small 
\begin{equation} \label{eq_sketch_goal_paper_0}
  \Ex \langle \nabla f(x_{k,0}), \sum_{i=0}^{n-1} \frac{m_{k,i}}{\sqrt{v_{k,i}}} \rangle > 0.
\end{equation}
}
However, directly proving \eqref{eq_sketch_goal_paper_0} could be difficult because both $m_{k,i}$ and $v_{k,i}$ distort the trajectory. To start with, we try to control the movement of $v_{k,i}$ by increasing $\beta_2$ (similar idea as  \citep{shi2020rmsprop,zou2019sufficient,chen2021towards}). 
Recall $ v_{k,i} =  (1-\beta_2) \sum_{j=1}^{i}\beta_2^{i-j}\nabla f_{\tau_{k,j}} (x_{k,j}) \circ \nabla f_{\tau_{k,j}} (x_{k,j}) +\beta_2^{i} v_{k,0}$, we have $v_{k,i} \approx v_{k,0} $ when $\beta_2$ is large.
In this case, we have: 
{\small
$$ \Ex \left\langle \nabla f(x_{k,0}), \sum_{i=0}^{n-1} \frac{m_{k,i}}{\sqrt{v_{k,i}}} \right\rangle \approx \Ex \left\langle \frac{\nabla f(x_{k,0})}{\sqrt{v_{k,0}}}, \sum_{i=0}^{n-1} m_{k,i} \right\rangle \approx \Ex \left\langle \frac{\nabla f(x_{k,0})}{\sqrt{v_{k,0}}}, \nabla f(x_{k,0}) \right\rangle >0,$$ }
where the first ``$\approx$" is due to the large $\beta_2$ and the second ``$\approx$" is our goal. Now we need to show:
{\small
\begin{equation} \label{eq_sketch_goal_paper}
 \Ex \left\langle \frac{\nabla f(x_{k,0})}{\sqrt{v_{k,0}}}, \left(\sum_{i=0}^{n-1} m_{k,i} \right) -\nabla f(x_{k,0}) \right\rangle  \overset{(*)}{=}  \Ex \left(\sum_{l=1}^d \sum_{i=0}^{n-1} \frac{\partial_l f(x_{k,0})}{\sqrt{v_{l,k,0}}}\left( m_{l,k,i} -\partial_l f_{\tau_{k,i}}(x_{k,0})\right)   \right) 
  \approx 0,
\end{equation}
}
where $\partial_l f(x_{k,0})$ is the $l$-th component of $\nabla f(x_{k,0})$, similarly for $m_{l,k,0}$ and $v_{l,k,0}$.  $(*)$ is due to the finite-sum structure.
However, it is not easy to prove \eqref{eq_sketch_goal_paper}. We point out some technical issues below.

\vspace{-1mm}
   \paragraph{Issue I: massive momentum.}
     Directly proving \eqref{eq_sketch_goal_paper} is still not easy. We need to first consider a simplified problem: for every $l \in [d]$, assume we treat  $ \partial_l f(x_{k,0})/\sqrt{v_{l,k,0}}$ as a constant, how to bound {\small$ \Ex \sum_{i=0}^{n-1}\left(m_{l,k,i} -\partial_l f_{\tau_{k,i}}(x_{k,0}) \right) $}?
     
     It turns out that this simplified problem is still non-trivial.
    When $\beta_1$ is large, $m_{l,k,i}$ contains heavy historical signals which significantly distort the trajectory from gradient direction.  
Existing literature \citep{Manzil2018adaptive,de2018convergence,shi2020rmsprop} take a naive approach: 
	they set  $\beta_1 \approx 0$ so that  $m_{l,k,i} \approx  \partial_l f_{\tau_{k,i}} (x_{k,i})$. Then we get $\eqref{eq_sketch_goal_paper}  \approx 0$. 
	However, this method cannot be applied here since we are interested in practical cases where $\beta_1$ is large in $[0,1)$. 
	
	\vspace{-1mm}
	\paragraph{Issue II: stochastic non-linear dynamics.}   Even if we solve {\bf Issue I}, it is still unclear how to prove \eqref{eq_sketch_goal_paper}. This is because: for every $l \in [d]$, $ \partial_l f(x_{k,0})/\sqrt{v_{l,k,0}}$ is a r.v. instead of a constant. With this term involved, we are facing with a  stochastic non-linear dynamics, which could be difficult to analyze. Further, $ \partial_l f(x_{k,0})/\sqrt{v_{l,k,0}}$ is statistically dependent with  $\left(m_{l,k,i} -\partial_l f_{\tau_{k,i}}(x_{k,0}) \right)$, so we are not allowed to handle the expectation $\Ex(\partial_l f(x_{k,0})/\sqrt{v_{l,k,0}})$    separately .

Unfortunately, even with additional assumptions like bounded gradient, there is no general approach to tackle the above issues.  In this work, we propose solutions regardless of gradient magnitude.

\vspace{-1mm}
\paragraph{Solution to Issue I.} We prove the following Lemma to resolve {\bf Issue I}.

\begin{lemma} 
\label{lemma_paper_toy}
(Informal)
Consider Algorithm \ref{algorithm}. For every $l \in [d]$ and any $\beta_1 \in [0,1)$, we have the following result  under Assumption \ref{assum1}.
{\small
\begin{equation*}
   \delta(\beta_1) := \Ex \sum_{i=0}^{n-1}\left(m_{l,k,i} -\partial_l f_{\tau_{k,i}}(x_{k,0}) \right)  = \mathcal{O}\left(\frac{1}{\sqrt{k}}\right),
\end{equation*}}
where $\partial_l f(x_{k,0})$ is the $l$-th component of $\nabla f(x_{k,0})$; $m_{l,k,i} = (1-\beta_1)\partial_l f_{\tau_{k,i}}(x_{k,i})+ \beta_1 m_{l,k,i-1}$. 
\end{lemma}

We present the proof idea in Appendix \ref{appendix:sketch}. Simply put,
we construct a simple  toy example called ``color-ball" model (of the 1st kind). This toy model shows a special property of $\delta(\beta_1)$. We find out:  for Algorithm \ref{algorithm}, error terms from successive epochs can be  canceled, which keeps the momentum roughly in the descent direction. This important property is not revealed in any existing work. 

\vspace{-1mm}
\paragraph{Remark 4:}
When assuming bounded gradient {\small $\|\nabla f(x)\| \leq G$},
a naive upper bound would be  $\delta(\beta_1) = \mathcal{O}(G) $. However, such constant upper bound does not imply $\delta (\beta_1)$ is close to 0. It will not help prove the convergence. This might be partially the reason why large-$\beta_1$ Adam is hard to analyze even under bounded gradient (see related works in Section \ref{section:related}). 
We emphasize Lemma \ref{lemma_paper_toy} holds true regardless of gradient norm, so it could be deployed in both bounded or unbounded gradient analysis.

\paragraph{Solution to Issue II.}

We try to show \eqref{eq_sketch_goal_paper}  by adopting Lemma \ref{lemma_paper_toy}. However, the direct application cannot work since $\frac{\partial_l f(\xkz)}{\sqrt{\vlkz}}$ is random. 
Despite its randomness, 
we find out that when $\beta_2$ is large, the changes of {\small $\frac{\partial_l f(\xkz)}{\sqrt{\vlkz}} $}  shrinks along iteration.
As such, although $\frac{\partial_l f(\xkz)}{\sqrt{\vlkz}}$ brings extra perturbation, 
the quantity in \eqref{eq_sketch_goal_paper} share the similar asymptotic behavior as $\delta(\beta_1)$. We prove the following Lemma \ref{lemma_paper_idea_2}.

\begin{lemma} \label{lemma_paper_idea_2}
(Informal)  Under Assumption \ref{assum1} and \ref{assum2}, consider Algorithm \ref{algorithm} with large $\beta_2$ and $\beta_1< \sqrt{\beta_2}$. For those $l$ with gradient component larger than certain threshold, we have:
{\small
\begin{equation}\label{eq_k-k-1_paper}
\left|\frac{\partial_l f(x_{k,0})}{\sqrt{v_{k,0}}}  - \frac{\partial_l f(x_{k-1,0})}{\sqrt{v_{k-1,0}}}\right| = \mathcal{O}\left(\frac{1}{\sqrt{k}}\right);
\end{equation}
}
{\small
\begin{equation} \label{eq_b1_paper}
\Ex \left( \frac{\partial_l f(\xkz)}{\sqrt{\vlkz}} \sum_{i=0}^{n-1}(m_{l,k,i} - \partial_l f_{\tau_{k,i}}(x_{k,0})) \right) = \mathcal{O}\left(\frac{1}{\sqrt{k}}\right).
\end{equation}
}

\end{lemma}

In Appendix \ref{appendix:sketch}, we introduce how to derive \eqref{eq_b1_paper} from \eqref{eq_k-k-1_paper}. To do so, we introduce a new type of ``color-ball" model (we call it color-ball of the 2nd kind) which adopts the random perturbation of $\frac{\partial_l f(\xkz)}{\sqrt{\vlkz}}$. Understanding color-ball model of the 2nd kind is crucial for proving  Lemma \ref{lemma_paper_idea_2}.

We conclude the proof of \eqref{eq_sketch_goal_paper} by some additional analysis on ``those $l$ with small gradient component". This case is a bit easier since it reduces to bounded gradient case. 
For readers who wants to learn more about the idea of tackling {\bf Issue I} and {\bf II}, please refer to Appendix \ref{appendix:sketch}
where we formally introduce the 1st and 2nd kind of color-ball models.  
Since the whole proof is quite long, we  provide a proof sketch in Appendix \ref{appendix:roadmap}. The whole proof is presented in Appendix \ref{appendix:thm1}.

\section{Experiments}
\label{section:exp}
\vspace{-2mm}

To support our theory, we provide more simulations and real-data experiments. All the experimental settings and hyperparameters are presented in Appendix \ref{appendix:exp_setting}. We aim to show:


{\bf (I).} When $\beta_2$ is large, a large range of $\beta_1$ gives good performance, including all $\beta_1 <\sqrt{\beta_2}$. 

{\bf (II).} When $\beta_2$ is small, a large range of $\beta_1$ performs relatively badly.

\vspace{-2mm}
\begin{figure*}[htbp]
  \vspace{-2mm}
    \centering
    \subfigure[{\tiny Function \eqref{eq_non_realizable} with $D_0>0$}]{
      \begin{minipage}[t]{0.25\linewidth}
      \centering
    \includegraphics[width=\linewidth]{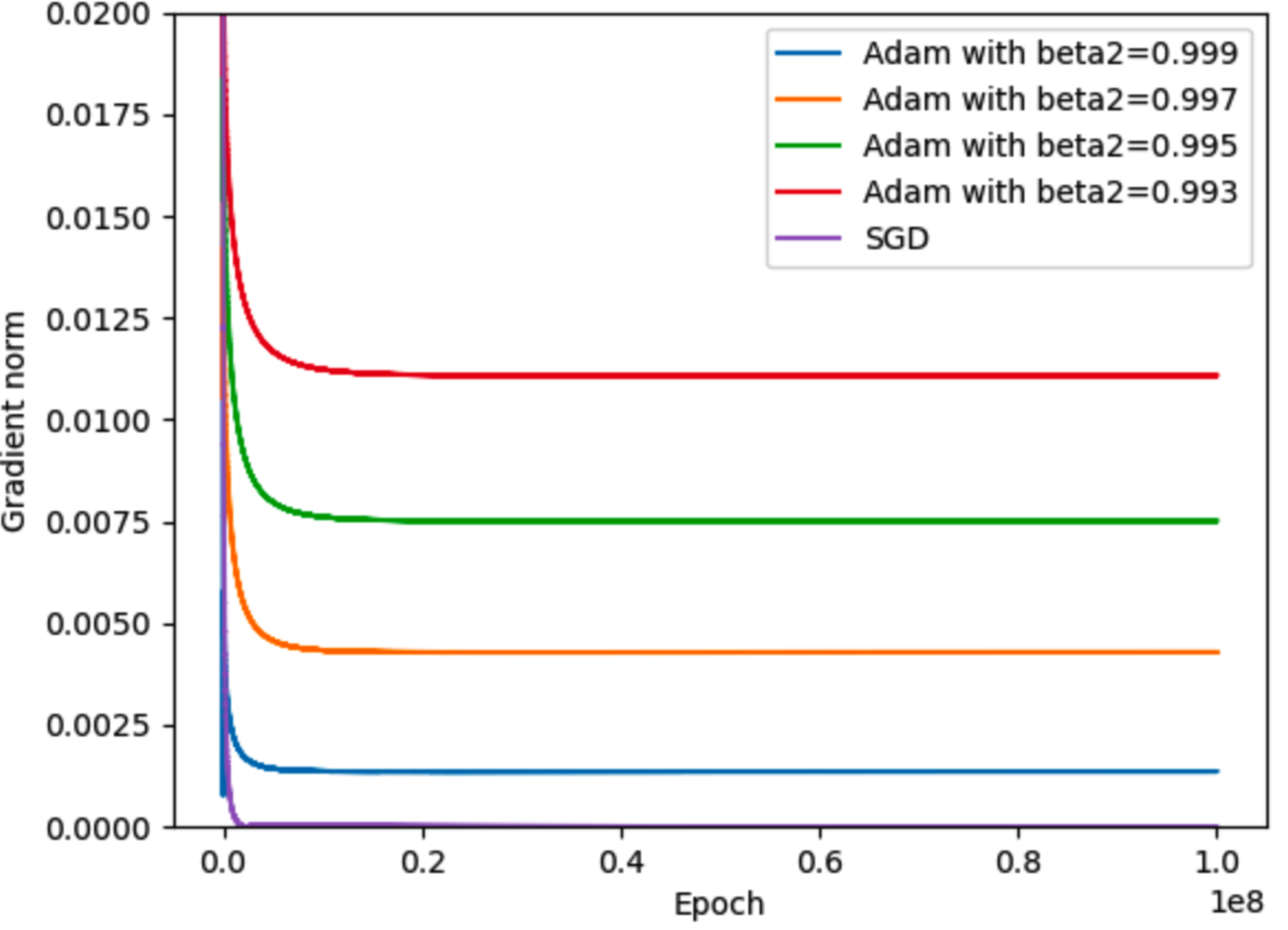}
      \end{minipage}%
      }%
    \subfigure[{\tiny Function \eqref{counterexample1} with $D_0 = 0$}]{
    \begin{minipage}[t]{0.25\linewidth}
    \centering
\includegraphics[width=\linewidth]{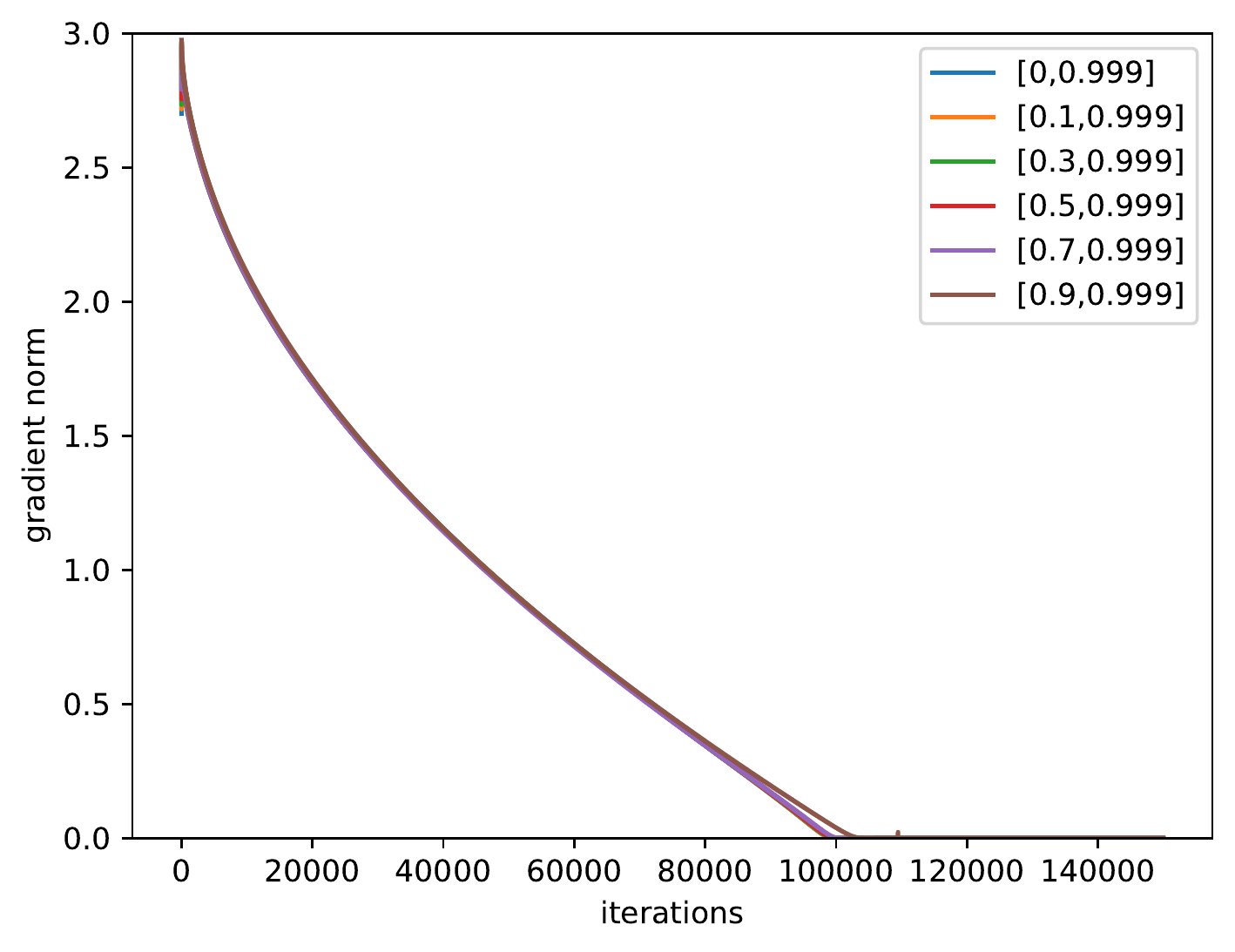}
    \end{minipage}%
    }%
    \subfigure[{\tiny The trajectories of  small-$\beta_2$ Adam }]{
    \begin{minipage}[t]{0.25\linewidth}
    \centering
    \includegraphics[width=\linewidth]{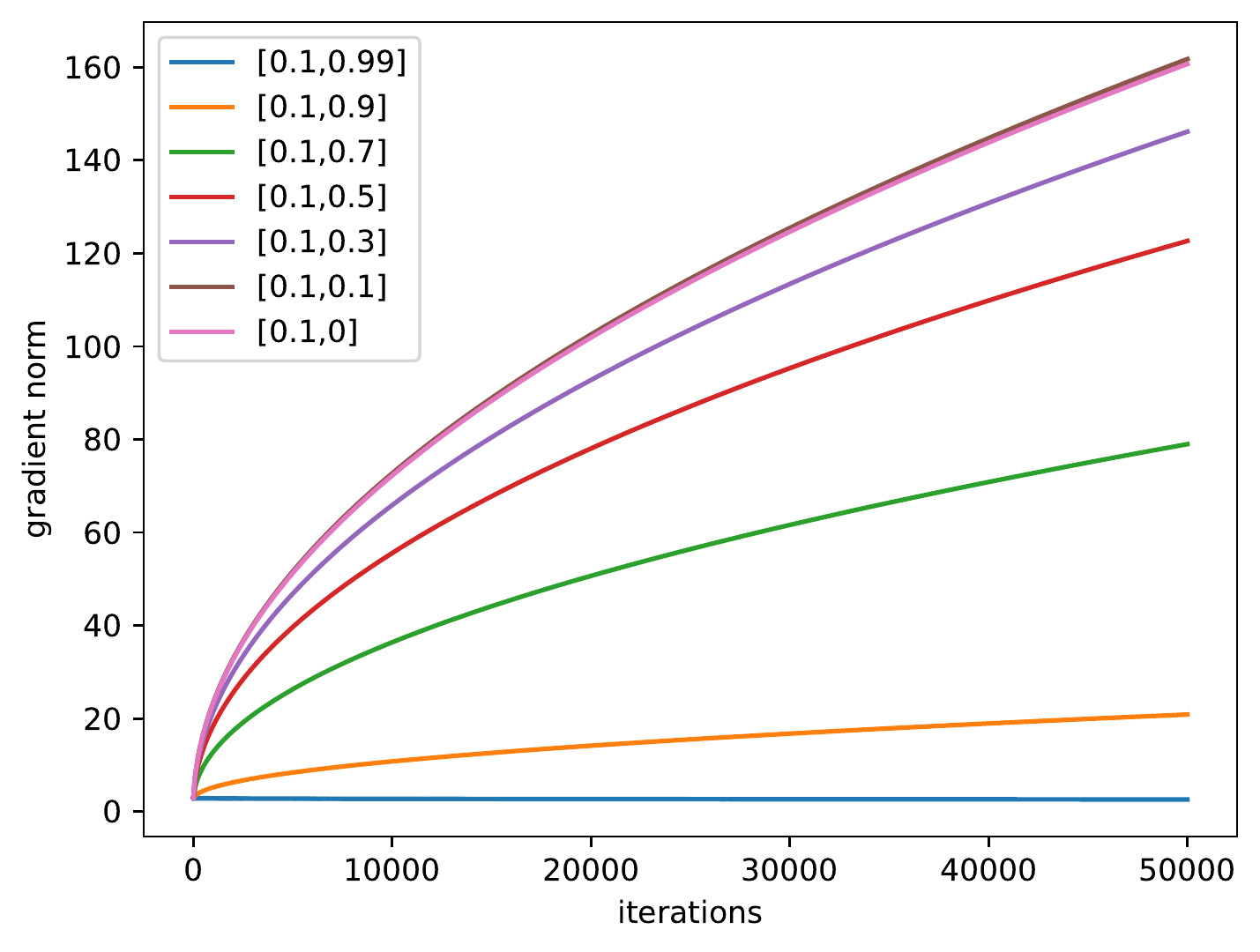}
    \end{minipage}%
    }%
    \subfigure[{\tiny WikiText-103}]{
      \begin{minipage}[t]{0.25\linewidth}
      \centering
    \includegraphics[width=\linewidth]{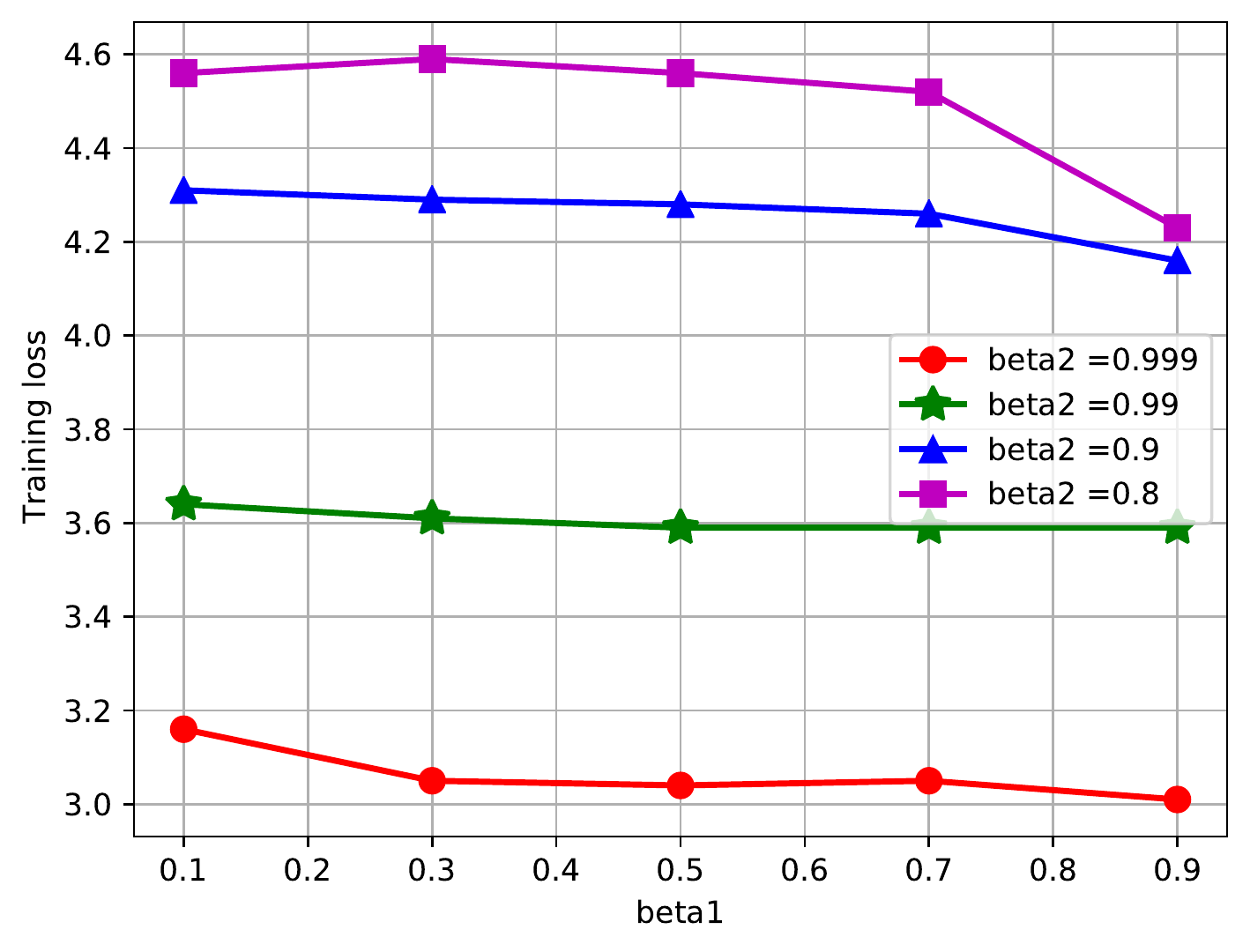}
      \end{minipage}%
      }%
    \centering
    \vspace{-3mm}
    \caption{{\small The performance of Adam on different tasks. 
    {\bf (a) (b):} 
  large-$\beta_2$ Adam converges to bounded region when $D_0>0$ and converges to critical points when $D_0 = 0$. We use diminishing stepsize $\eta_k = 0.1 /\sqrt{k}$. 
    {\bf (c):} When $\beta_2$  is small, gradient norm of Adam iterates can be unbounded. We use function \ref{counterexample1} with initialization $x =-5$ and $n=20$.  The legends in (b) and (c) stand for $[\beta_1,\beta_2]$.
    {\bf (d):}
    The training loss under different $\betabeta$ on NLP tasks.  We use Adam to train Transformer XL on WikiText-103 dataset. }} 
    \vspace{-2mm}
    \label{fig:cifar_nlp_mnist}
\end{figure*}

\paragraph{Convergence to bounded region when $D_0>0$.} In Figure \ref{fig:cifar_nlp_mnist}, we run large-$\beta_2$ Adam on function \eqref{eq_non_realizable} (defined later in Appendix \ref{appendix:more_exp}). This function satisfies with $D_0 >0$. We find that even with diminishing stepsize $\eta_k = 1/ \sqrt{k}$,  Adam may not converge to an exact critical point. Instead, it converges to a bounded region.  This is because: even though $\eta_k$ is decreasing, the effective stepsize $\eta_k/\sqrt{v_{k,i}}$ might not decay.   Further, the size of the region shrinks when $\beta_2$ increases. This is because the movement of  $\sqrt{v_{k,i}}$ shrinks as  $\beta_2$ increases. These phenomena  match  {\bf Remark 3} and claim {\bf (I)}.

\paragraph{Convergence to critical points when $D_0 = 0$} Since function  \eqref{counterexample1} satisfies $D_0=0$, we run more experiments on \eqref{counterexample1} with initialization $x=-5$ and $n=5,10,15,20$.  We show the result of $n=20$ in Figure \ref{fig:cifar_nlp_mnist} (a), (b); the rest are shown in Appendix \ref{appendix:more_exp}. We find that: when $\beta_2$ is large, Adam converges to critical points for $\beta_1< \sqrt{\beta_2}$. These phenomena  match  claim {\bf  (I)}.

\paragraph{Gradient norm of iterates can be unbounded when $\beta_2$  is small.} On function \eqref{counterexample1}, We further run Adam with small $\beta_2$ at initialization $x=-5$. 
In this case, gradient norms of iterates increase dramatically.  This emphasizes the importance of discarding bounded gradient assumptions. These phenomena  match claim {\bf (II)}.


\vspace{-2mm}
\paragraph{MNIST and CIFAR-10.} 

As shown in Figure \ref{fig:intro_paper} (b)\& (c) in Section \ref{section:intro},  the training results match both  claim {\bf (I)} and {\bf (II).} 
In addition, there is a convex-shaped boundary on the transition from low loss to higher loss, this boundary roughly matches the condition in Theorem \ref{thm1}. 


\vspace{-2mm}
\paragraph{NLP.}  We use Adam to train Transformer XL \citep{dai2019transformer} on the WikiText-103 dataset \citep{merity2016pointer}.  
This architecture and dataset is widely used in NLP tasks (e.g. \citep{howard2018universal, see2017get}).
As shown in Figure \ref{fig:cifar_nlp_mnist} (d), the training results match both claim {\bf (I)} and {\bf (II).}

\vspace{-2mm}
\section{Conclusions}
\label{sec_conclusion}
\vspace{-2mm}
In this work, we explore the (non-)convergence of Adam. When $\beta_2$ is large, we  
prove that Adam can converge 
 with any $\beta_1 < \sqrt{\beta_2}$.
When $\beta_2$ is small, we further show that Adam might diverge to infinity for  a wide range of $\beta_1$.  One interesting question is to
verify  the advantage of Adam over SGD.
In this work, we focus on the fundamental issue of convergence. Proving faster convergence of Adam would be our future work.

\begin{ack}
 Yushun Zhang would  like to thank Bohan Wang, Reviewer xyuf, Reviewer UR9H, Reviwer tmNN and Reviewer V9yg for the careful proof reading and helpful comments. Yushun Zhang would like to thank Bohan Wang for the valuable discussions around Lemma G.3. Yushun Zhang would also like to  thank Reviewer UR9H for the valuable discussions around Lemma G.13. This work is supported by the Internal Project Fund from Shenzhen Research Institute of Big Data under Grant J00220220001. 
 This work is supported by NSFC-A10120170016, NSFC-617310018 and the Guandong
 Provincial Key Laboratory of Big Data Computing.   

\end{ack}

\bibliographystyle{icml2022}
\bibliography{references}








\section*{Checklist}


\begin{enumerate}

\item For all authors...
\begin{enumerate}
  \item Do the main claims made in the abstract and introduction accurately reflect the paper's contributions and scope?
    \answerYes{}
  \item Did you describe the limitations of your work?
    \answerYes{}
  \item Did you discuss any potential negative societal impacts of your work?
    \answerYes{}
  \item Have you read the ethics review guidelines and ensured that your paper conforms to them?
    \answerYes{}
\end{enumerate}

\item If you are including theoretical results...
\begin{enumerate}
  \item Did you state the full set of assumptions of all theoretical results?
    \answerYes{}
        \item Did you include complete proofs of all theoretical results?
    \answerYes{}
\end{enumerate}

\item If you ran experiments...
\begin{enumerate}
  \item Did you include the code, data, and instructions needed to reproduce the main experimental results (either in the supplemental material or as a URL)?
    \answerNo{}
  \item Did you specify all the training details (e.g., data splits, hyperparameters, how they were chosen)?
    \answerYes{}{}
        \item Did you report error bars (e.g., with respect to the random seed after running experiments multiple times)?
    \answerNo{}
        \item Did you include the total amount of compute and the type of resources used (e.g., type of GPUs, internal cluster, or cloud provider)?
    \answerNo{}
\end{enumerate}

\item If you are using existing assets (e.g., code, data, models) or curating/releasing new assets...
\begin{enumerate}
  \item If your work uses existing assets, did you cite the creators?
    \answerYes{}
  \item Did you mention the license of the assets?
    \answerNo{}
  \item Did you include any new assets either in the supplemental material or as a URL?
    \answerNo{}
  \item Did you discuss whether and how consent was obtained from people whose data you're using/curating?
    \answerNo{}
  \item Did you discuss whether the data you are using/curating contains personally identifiable information or offensive content?
    \answerNo{}
\end{enumerate}

\item If you used crowdsourcing or conducted research with human subjects...
\begin{enumerate}
  \item Did you include the full text of instructions given to participants and screenshots, if applicable?
    \answerNA{}
  \item Did you describe any potential participant risks, with links to Institutional Review Board (IRB) approvals, if applicable?
    \answerNA{}
  \item Did you include the estimated hourly wage paid to participants and the total amount spent on participant compensation?
    \answerNA{}
\end{enumerate}

\end{enumerate}


\newpage
\appendix


\section*{Negative Social Impact}
This script may provide better guidance for neural nets training.  It would have certain negative social impact if neural nets are deployed for illegal usage.

\section*{Appendix Organization}

The Appendix is organized as follows.

\begin{itemize}
    \item Appendix \ref{appendix:sketch} introduces two ``color-ball" models  and their applications in proving Lemma \ref{lemma_paper_toy} and Lemma \ref{lemma_paper_idea_2}. This part is important for proving Theorem \ref{thm1}.
    
      \item Appendix \ref{appendix:more_exp} introduces more experiments to support our theory.
    \item Appendix \ref{sec:implication} provide some suggestions for hyperparameter tuning of Adam.
    \item Appendix \ref{appendix:more_related} provide more discussions on some recent related works.  Appendix \ref{appendix:reddi} re-state the non-convergence results in \citep{reddi2019convergence}. 
    
    \item Appendix \ref{appendix:thm_diverge} provides detailed proof for Proposition \ref{thm_diverge}. 
    \item Appendix \ref{appendix_notations} provides some more notations and technical lemmas that serve for the proof of Theorem \ref{thm1}.
    \item Appendix \ref{appendix:thm1} provides detailed proof for Theorem \ref{thm1}. Especially, Appendix \ref{appendix:roadmap} provide a proof roadmap.
\end{itemize}

\section{Introduction to the ``Color-Ball'' Models: the Key Ingredients to Prove Theorem \ref{thm1}}
 \label{appendix:sketch}

 We now introduce two  ``color-ball" models  and their applications in tackling {\bf Issue I} and {\bf Issue II} mentioned in Section \ref{sec:proofidea}.  These two color-ball models are important for proving Theorem \ref{thm1}.

\paragraph{Solution to Issue I.} As discussed in Section \ref{sec:proofidea}, we wish to show that $\delta(\beta_1)$ vanishes with $k$.  Formally, we wish to get the following equation \eqref{sketch_m_goal} for every $l \in [d]$.

\begin{equation}
  \label{sketch_m_goal}
  \delta(\beta_1) = \left|\Ex \left[\sum_{i=0}^{n-1}\left(m_{l, k, i}-\partial_l f_{i}\left(x_{k, 0}\right)\right)\right] \right|= \mathcal{O}(\beta_1^{nk})  , \quad \forall \beta_1 \in [0,1)
\end{equation}

When $k$ is large, $\mathcal{O}(\beta_1^{nk})$ vanishes faster than   $\mathcal{O}(\frac{1}{\sqrt{k}})$ and thus Lemma \ref{lemma_paper_toy} can be proved.
In the following context, we will carefully quantify the mismatch between $m_{l,k,i}$  and  the $\partial_l f(x_{k,0})$. 
We find out that the error terms from successive epochs can be cancelled, which keeps the momentum roughly in the descent direction.  
To help readers understand our idea, we introduce the  ``color-ball " model (of the 1st kind) as follows. 

\paragraph{The color-ball model of the 1st kind.} Consider a box containing two balls labeled with constant $c_0, c_1 \in \mathbb{R}$, respectively. In each round (epoch), we randomly sample balls from the box without replacement, then we put them back. We denote the 1st sampled label in the $k$-th epoch as $a_{k}$ and  the 2nd sampled one as $b_{k}$;  $a_{k}, b_k \in \{c_0, c_1\}$. We define two random variables  $m_0$  and $m_1$ as follows (assume $\beta \in [0,1)$):

$$m_1 = \underbrace{ b_k +\beta
a_{k} }_{m_{1,k}} + \underbrace{\beta^{2} b_{k-1}   + \beta^{3} a_{k-1}  }_{m_{1,k-1}} +  \cdots  +\underbrace{\beta^{2(k-1)} b_{1}   + \beta^{2(k-1)+1} a_{1}  }_{m_{1,1}};  $$ 

$$m_0 = \underbrace{\quad\quad\quad
   a_{k} }_{m_{0,k}} + \underbrace{\beta^{1} b_{k-1}   + \beta^{2} a_{k-1}  }_{m_{0,k-1}} +  \cdots  +\underbrace{\beta^{2(k-1)-1} b_{1}   + \beta^{2(k-1)} a_{1}  }_{m_{1,1}};  $$

   where $m_{0,k}$ denotes the summand of $m_{0}$ in $k$-th epoch, similarly for $m_{1,k}$. Note that in each epoch,  $m_0$ and $m_1$ share the same sample order.
  Further, we introduce the following deterministic constants.

    $$f_1 = 
    c_1(\underbrace{1+\beta}_{f_{1,k}}+ \underbrace{\beta^2 +\beta^3}_{f_{1,k-1} }   \cdots +\underbrace{\beta^{2(k-1)}+ \beta^{2(k-1)+1} }_{f_{1,1}});  $$ 

    $$f_0 = 
    c_0(\underbrace{1+\beta}_{f_{0,k}}+ \underbrace{\beta^2 +\beta^3}_{f_{0,k-1} }   \cdots +\underbrace{\beta^{2(k-1)}+ \beta^{2(k-1)+1} }_{f_{0,1}});  $$ 

    where $f_{0,k}$ denotes the summand of $f_{0}$ in $k$-th epoch, similarly for $f_{1,k}$.  Now we prove the following Lemma \ref{lemma_toy_2}.

    \begin{lemma}\label{lemma_toy_2}
      In the color-ball model of the 1st kind, we have 
    
      \begin{equation*}
       \left\vert \Ex \left[ \sum_{i=0}^1 m_i -\sum_{i=0}^1 f_i \right] \right\vert = \beta^{2(k-1)+1} \left(\frac{c_0}{2} + \frac{c_1}{2}\right),
      \end{equation*}
    
      where the expectation is taken on all the possible draws.  For the color-ball example with  $n\geq 2$ balls, we have
    
      \begin{equation*}
        \left\vert\Ex \left[ \sum_{i=0}^{n-1} m_i -\sum_{i=0}^{n-1} f_i \right] \right \vert= \beta^{n(k-1)} \sum_{i=0}^{n-1} c_i (\frac{1}{n} \beta^1 \cdots+\frac{n-1}{n} \beta^{n-1}).
      \end{equation*}
    \end{lemma}

    {\bf How is Lemma \ref{lemma_toy_2} related to \eqref{sketch_m_goal}? } In this color-ball toy example, $m_i$ mimics the possible realization of momentum up to the $i$-th inner loop in $k$-th epoch. $f_{i}$ mimics the stochastic gradient $\nabla f_{i}\left(x_{k, 0}\right)$. This is because we can expand  $\nabla f_{i}\left(x_{k, 0}\right)$ into an infinite-sum sequence $\nabla f_{i}\left(x_{k, 0}\right)=(1-\beta_1) \nabla f_{i}\left(x_{k, 0}\right) \sum_{j=0}^{\infty} \beta_1^j$, which shares a similar structure as  $f_{i}$. In this sense, Lemma \ref{lemma_toy_2} may provide ideas to prove \eqref{sketch_m_goal}. Nevertheless, there are still gap between these two, we will explain the gap later.

\begin{proof}
  We use $\Ext \left[\cdot\right]$ to denote the conditional expectation given all the history up to the beginning of $k$-th epoch.  
  Since $\Ex \left[ \cdot\right]= \Ex \left[\Ext \left[\cdot\right]\right]$, we first calculate $\Ext \left[ \sum_{i=0}^1 m_i \right]$.  Since all the history before $k$-th epoch is fixed, we relegate this part to later discussion and first focus on the expectation of $ \sum_{i=0}^1 m_{i,k} $. 
  As shown in Figure \ref{fig:toy_2_k} (upper part), there are $2$ possible realization.

\begin{figure}[htbp]
  \centering 
  \includegraphics[width=3.5in]{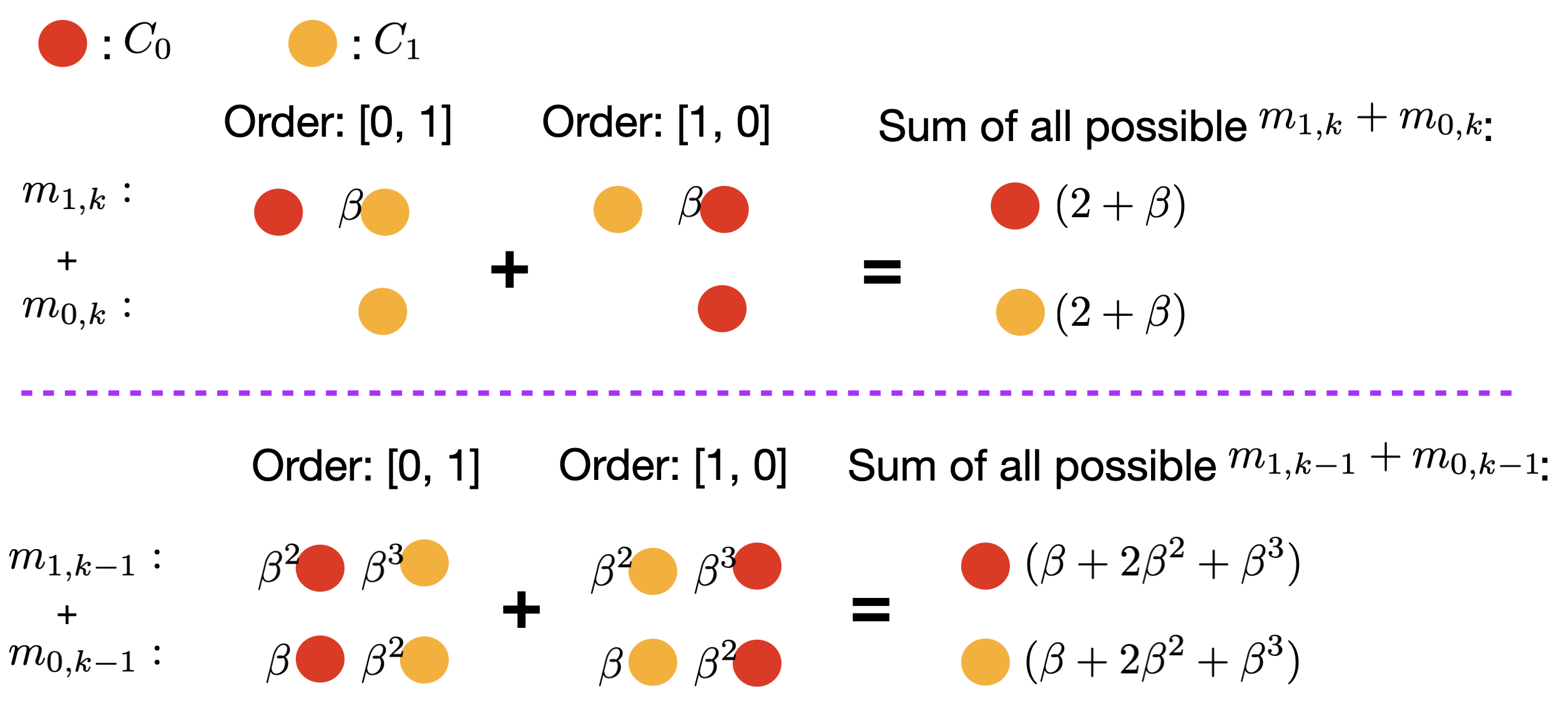}
 
  \caption{ All possible realization of $\sum_{i=0}^1 m_{i,k} $ and $\sum_{i=0}^1 m_{i,k-1} $.}
  \label{fig:toy_2_k}
\end{figure}

With the help of Figure \ref{fig:toy_2_k} (upper part), we have the following result.

 \begin{eqnarray}
  \Ext \left[\sum_{i=0}^1 m_{i,k}  -\sum_{i=0}^1 f_{i,k} \right] &=&
  \Ext \left[\sum_{i=0}^1 m_{i,k}  \right]- \left(1+ \beta\right)(c_0+c_1) \nonumber \\
  &=&  -\frac{\beta}{2}(c_0+c_1)  
  \label{sketch_toy_k}
 \end{eqnarray}

 Now we move one step further to calculate 
$ \Ex_{k-1} \Ext \left[ \sum_{i=0}^1 m_i -\sum_{i=0}^1 f_i \right]$. Using the similar strategy as \eqref{sketch_toy_k}, we have the following result. The calculation is  illustrated in Figure \ref{fig:toy_2_k} (lower part) and Figure \ref{fig:toy_2_cancel}.

\begin{figure}[htbp]
  \centering 
  \includegraphics[width=3.5in]{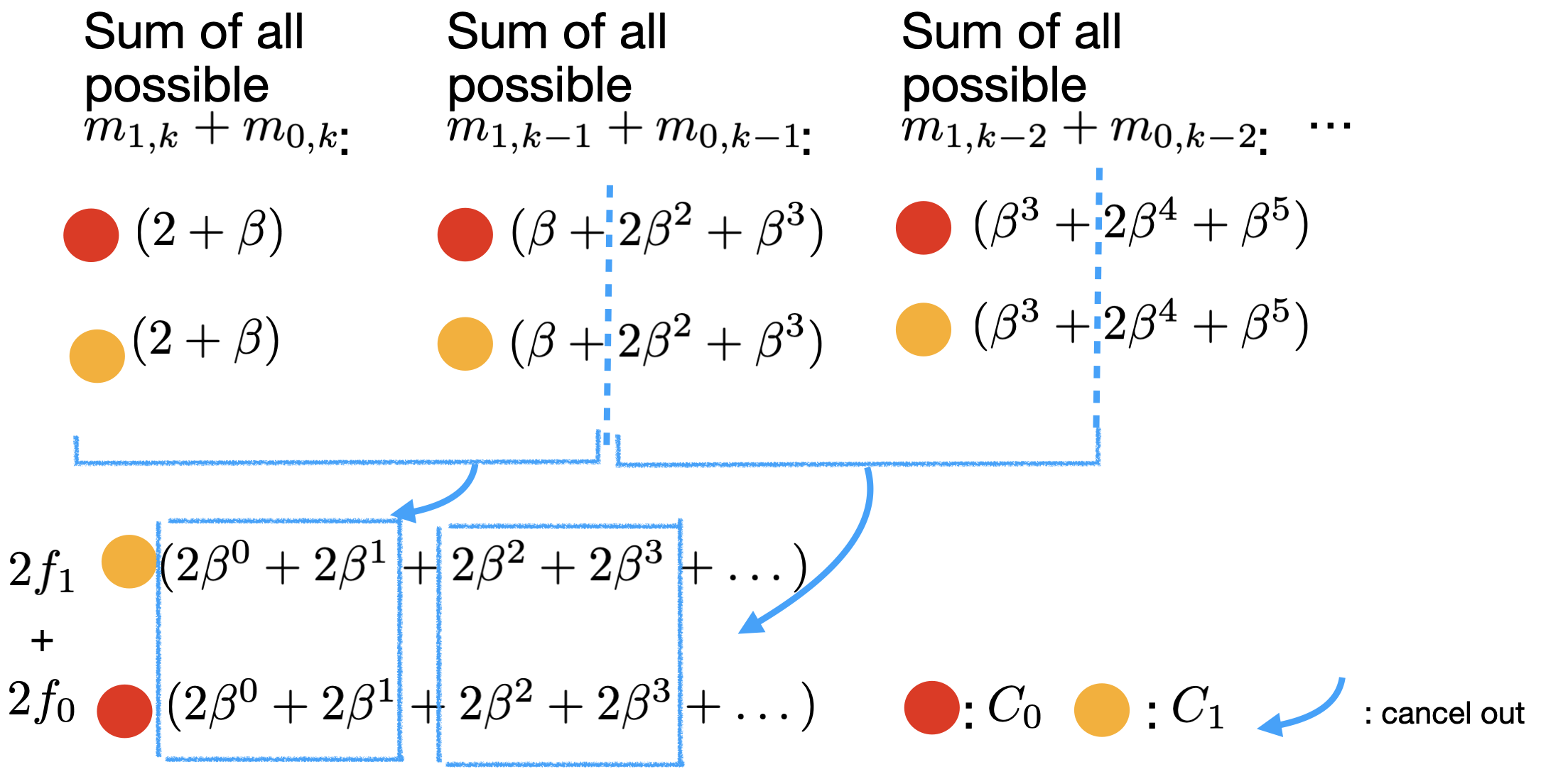}
  
  \caption{For every $k$, $\Ext \left[\sum_{i=0}^1 m_{i,k}  -\sum_{i=0}^1 f_{i,k} \right]$ will create residues, while these residues will be canceled out in the $(k-1)$-th epoch.}
  \label{fig:toy_2_cancel}
\end{figure}

\begin{eqnarray*}
  &&\Ex_{k-1} \left\{\Ext  \left[ \BLUE{\sum_{i=0}^1 m_{i,k} -\sum_{i=0}^1 f_{i,k} }
  + \RED{\sum_{i=0}^1 m_{i,k-1} }  -\sum_{i=0}^1 f_{i,k-1}\right] \right\} \\
  &\overset{ \eqref{sketch_toy_k}}{=}&  
    \BLUE{-\frac{\beta}{2}(c_0+c_1)}   + \RED{\Ex_{k-1} \left\{\sum_{i=0}^1 m_{i,k-1}  \right\}} -\sum_{i=0}^1 f_{i,k-1} \\
    &\overset{\text{Figure \ref{fig:toy_2_k} and \ref{fig:toy_2_cancel}}}{=}&\BLUE{ 
     -\frac{\beta}{2}(c_0+c_1)  }  + \RED{\frac{1}{2} (\beta+2\beta^2+\beta^3)(c_0+c_1) }  -(\beta^2+\beta^3)(c_0+c_1) \\
    &=& -\frac{\beta^3}{2}(c_0+c_1)
 \end{eqnarray*}

 We observe that {\it only the highest order term remains in the calculation.} 
Repeat this process until $k=1$, we will get the results in Lemma \ref{lemma_toy_2}.
The above analysis also holds for general $n \geq 2$. 

\end{proof}

\paragraph{The gap between Lemma \ref{lemma_toy_2} and  Equation \eqref{sketch_m_goal}.}  Lemma \ref{lemma_toy_2} shows the key idea of 
proving equation \eqref{sketch_m_goal}. However, due to its idealized setting, the color-ball toy example is still far from our real goal \eqref{sketch_m_goal}. We list some of the gaps here.

\begin{itemize}
    \item In each possible trajectory: $x_{k,i}$ is changing with $k$ and $i$, while the balls are fixed in the color-ball example.
    \item When taking expectation: $x$ is changing in different trajectory, while the balls is fixed in the color-ball example.
    \item $x$ is a vector in $\mathbb{R}^d$ while the label of the balls are constant in $ \mathbb{R}$. 
\end{itemize}
It requires extra technical lemmas to handle these differences. We provide more discussions in Appendix \ref{appendix:roadmap}.

 \paragraph{Solution to Issue II.} 
  We now discuss how to resolve {\bf Issue II} mentioned in Section \ref{sec:proofidea}. The solution contains two parts. First, we need to prove \eqref{eq_k-k-1_paper}. Second,  
  we  derive \eqref{eq_b1_paper} from \eqref{eq_k-k-1_paper}.
  Due to the limited space, we relegate the first part to Appendix \ref{appendix:roadmap} (related to Lemma \ref{lemma_k-k-1}). Now,  we discuss the second part: Assume we have \eqref{eq_k-k-1_paper}, how do we  use it to prove \eqref{eq_b1_paper}? To answer this question, we introduce the color-ball model of the 2nd kind.

\paragraph{The color-ball model of the 2nd kind.} Consider the same setting as the color-ball model in {\bf Step 1}. We define a sequence of real random variable  $\{r_j\}_{j=1}^{k}$ with 
the following relation.
    \begin{equation*}
    |r_j - r_{j-1}| = \frac{1}{\sqrt{j}}, \quad j = 1, \cdots k.
  \end{equation*}
  Further, we assume $r_j$ is fixed when fixing the history up to $j$-th round.   The sequence $\{r_j\}_{j=1}^{k}$ mimics the sequence $\{\frac{\partial_l f(x_{k,0})}{\sqrt{v_{l,k,0}}}\}_{k=1}^{\infty}$ in \eqref{eq_k-k-1_paper}. 
Now, we consider the following quantities. 

	{\small	
	$ \quad\quad\quad\quad \BLUE{r_k } m_1 =\BLUE{r_k } \left( \underbrace{ b_k +\beta
	a_{k} }_{m_{1,k}} + \underbrace{\beta^{2} b_{k-1}   + \beta^{3} a_{k-1}  }_{m_{1,k-1}} +  \cdots  +\underbrace{\beta^{2(k-1)} b_{1}   + \beta^{2(k-1)+1} a_{1}  }_{m_{1,1}} \right);  $ 
	
	$\quad\quad\quad\quad \BLUE{r_k }m_0 = \BLUE{r_k } \left(\underbrace{\quad\quad\quad
	   a_{k} }_{m_{0,k}} + \underbrace{\beta^{1} b_{k-1}   + \beta^{2} a_{k-1}  }_{m_{0,k-1}} +  \cdots  +\underbrace{\beta^{2(k-1)-1} b_{1}   + \beta^{2(k-1)} a_{1}  }_{m_{1,1}} \right);  $
		
	}

	{\small	 
	$\quad\quad\quad\quad  \BLUE{r_k } f_1 = \BLUE{r_k } \left(
	c_1(\underbrace{1+\beta}_{f_{1,k}}+ \underbrace{\beta^2 +\beta^3}_{f_{1,k-1} }   \cdots +\underbrace{\beta^{2(k-1)}+ \beta^{2(k-1)+1} }_{f_{1,1}})  \right);  $
	
	$ \quad\quad\quad\quad \BLUE{r_k }f_0 = \BLUE{r_k }\left(
	c_0(\underbrace{1+\beta}_{f_{0,k}}+ \underbrace{\beta^2 +\beta^3}_{f_{0,k-1} }   \cdots +\underbrace{\beta^{2(k-1)}+ \beta^{2(k-1)+1} }_{f_{0,1}}) \right);  $}.
	
	We now prove the following Lemma \ref{lemma_toy_new}.

\begin{lemma}\label{lemma_toy_new}
  Consider the color-ball model of the 2nd kind,  we have 
  \begin{equation*}
    \Ex \left[ \sum_{i=0}^1 r_k m_i -\sum_{i=0}^1 r_k f_i \right]=  \beta^{2(k-1)+1} \left(-\frac{c_0}{2} - \frac{c_1}{2}\right) + \mathcal{O}(\frac{1}{\sqrt{k}}).
  \end{equation*}

 For general $n=1,2,3,\cdots$, we have

  \begin{equation*}
    \Ex \left[ \sum_{i=0}^{n-1}r_k  m_i -\sum_{i=0}^{n-1} r_k f_i \right]=  \sum_{i=0}^{n-1} c_i \beta^{(k-1)n}(-\frac{1}{n} \beta^1 \cdots-\frac{n-1}{n} \beta^{n-1})+ \mathcal{O}(\frac{1}{\sqrt{k}}).
  \end{equation*}
\end{lemma}

\begin{proof}
We only describe the proof idea here. The proof contains the following 4 steps.
 
 {\bf Step 2.1. } We firstly take $\Ext [\cdot]$ and thus $ r_k$  can be viewed as a constant. We use the color-ball procedure as in Figure \ref{fig:step2_1} to calculate $\Ext \left[r_k \sum_{i=0}^{1} m_{i, k} - r_k \sum_{i=0}^{1} f_{i} \right]$. 
 
 	 \begin{figure}[htbp]
		\centering
		\includegraphics[width=3.5in]{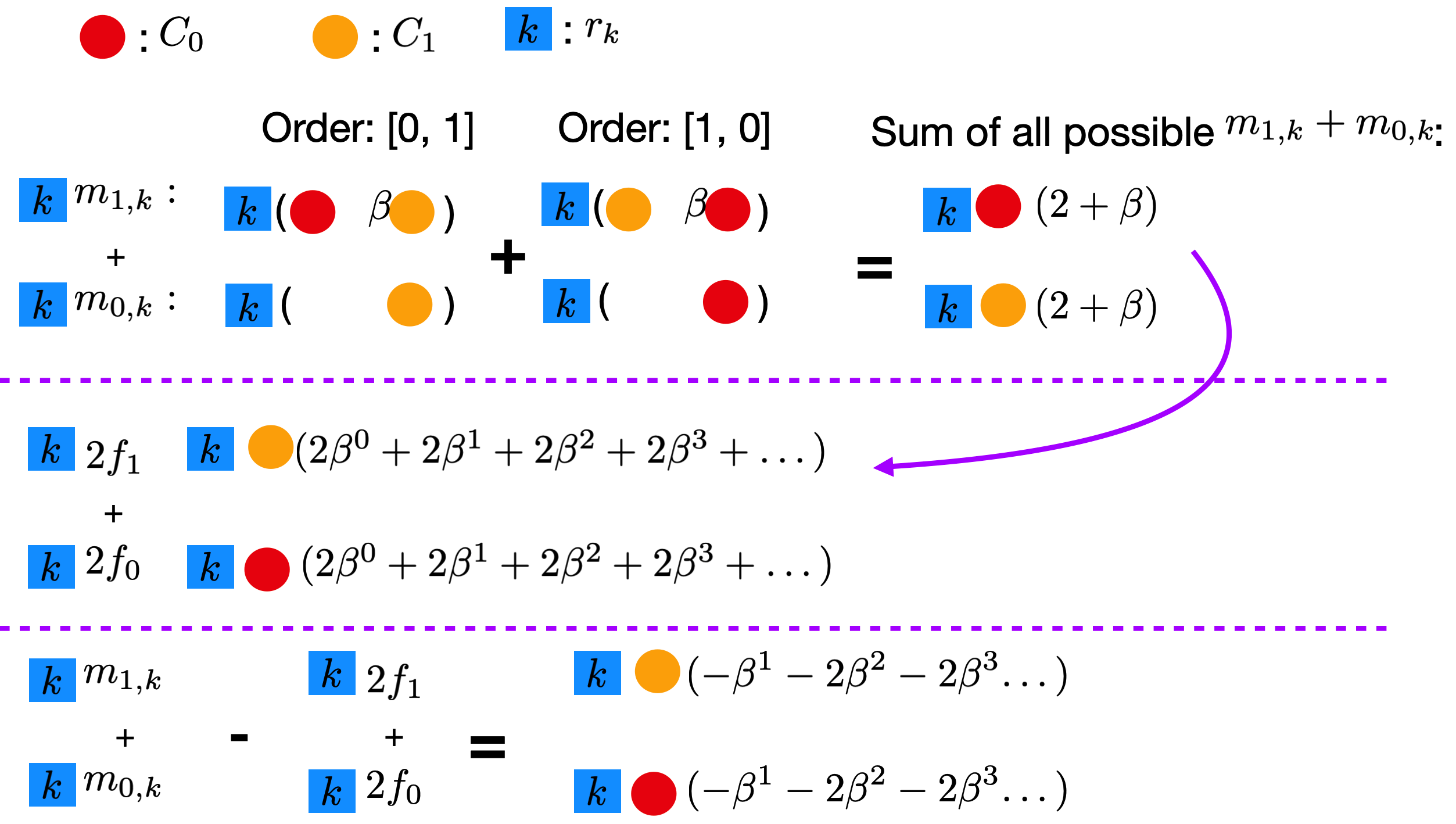}
		\caption{\small {\bf Step 2.1} in the new color-ball model.}
		\label{fig:step2_1}
	\end{figure}
	
	 {\bf Step 2.2. }  We change $\Ext\left[r_k \sum_{i=0}^{1} m_{i, k-1}  \right]$ into $\Ext\left[r_{k-1} \sum_{i=0}^{1} m_{i, k-1}  \right] + \text{Error}$; where $\text{Error} = \mathcal{O}(1/ \sqrt{k}) $. 

	{\bf Step 2.3. }  We  calculate $\Ex_{k-1}\Ext\left[r_{k-1} \sum_{i=0}^{1} m_{i, k-1}  \right]  = \Ex_{k-1}\left[r_{k-1} \sum_{i=0}^{1} m_{i, k-1}  \right] $. This part is illustrated in the top row in Figure  
	\ref{fig:step2_3}

	{\bf Step 2.4. } For the leftovers  in {\bf Step 2.1},  we change all  $r_k$ into $r_{k-1}$. Then we do the cancellation to calculate $\Ex_{k-1} \Ext \left[   r_{k-1} \sum_{i=0}^{1} m_{i, k-1}     + r_{k-1} \sum_{i=0}^{1} m_{i, k} - r_{k-1} \sum_{i=0}^{1} f_{i}   \right]$. This step is shown in the second and third row in Figure \ref{fig:step2_3}.

	 \begin{figure}[htbp]
		\centering
		\includegraphics[width=3.5in]{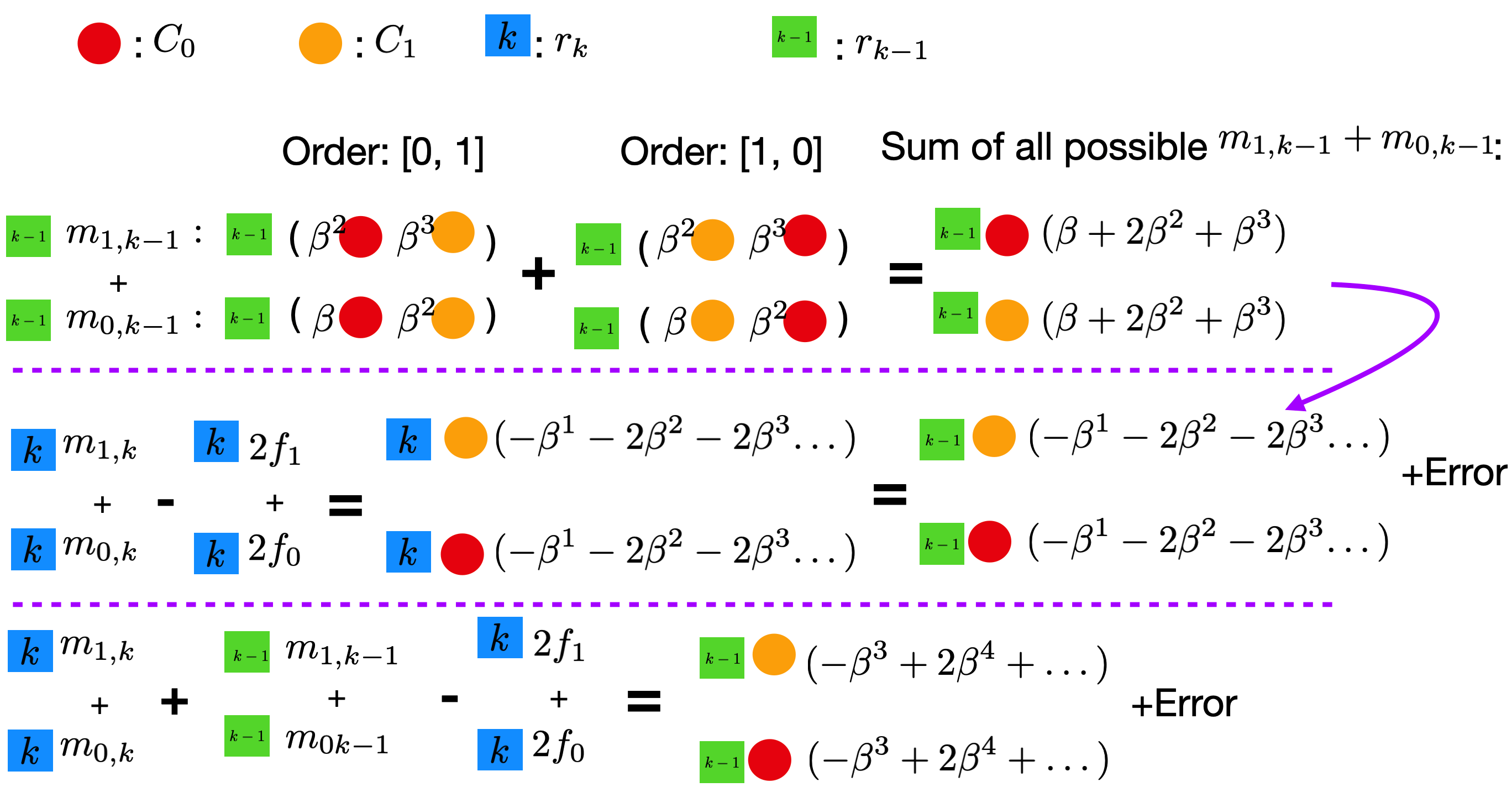}
		\caption{\small {\bf Step 2.3 and 2.4} in the new color-ball model.}
		\label{fig:step2_3}
	\end{figure}

	Repeat this process to the 1st epoch. We can   prove  the Lemma \ref{lemma_toy_new}.

\end{proof}

We emphasize that here are still some gap between Lemma \ref{lemma_toy_new} and our goal in \eqref{eq_b1_paper} in Lemma \ref{lemma_paper_idea_2}. First, we have the similar gap as discussed at the end of the {\bf Solution to Issue I}. Second, the condition in Lemma \ref{lemma_paper_idea_2} has requirement on the gradient norm, while this requirement is temporarily ignored in the  color-ball method of the 2nd kind.  It requires some technical lemmas to handle these gaps. 

  For more details, please refer to the complete  proof in Appendix \ref{appendix:thm1}. 



\section{More experiments}
\label{appendix:more_exp}

\paragraph{Estimation on $\rho$ in Theorem \ref{thm1}.} To ensure convergence, Theorem \ref{thm1} requires $\beta_2 \geq \gamma_1(n) = 1-\mathcal{O}((1-\beta_1^n)/n^2 \rho)$. Now we estimate the constant $\rho$.
According to our definition in Appendix \ref{sec_notations} and Remark \ref{remark_rho} in Appendix \ref{appendix:roadmap}, $\rho= \rho_1\rho_2\rho_3$, where $\rho_1, \rho_2, \rho_3$ are defined as follows.

  $$\rho_{1} \geq \frac{\sum_{i=1}^{n}\left|\partial_l f_i (x_{k,0})\right|}{\sqrt{\sum_{i=1}^{n}\left| \partial_l f_i (x_{k,0})\right|^{2}}};$$
  
  $$\rho_{2} \geq \frac{\left|\max_i \partial_l f_i (x_{k,0})\right|^{2}}{\frac{1}{n} \sum_{i=1}^{n}\left|\partial_l f_i (x_{k,0})\right|^{2}};$$
  
  $$\rho_{3} \geq \frac{\left|\sum_{i=1}^{n} \partial_l f_i (x_{k,0})\right|}{\sqrt{\frac{1}{n} \sum_{i=1}^{n}\left|\partial_l f_i (x_{k,0})\right|^{2}}}.$$

  These constants are firstly introduced by \citep{shi2020rmsprop}.  In worst case, we have $0 \leq \rho_{3} \leq \sqrt{n} \rho_{1} \leq n$.
 However, $\rho$ is highly dependent on the problem instance $f(x)$ and training process. We now estimate how $\rho$ changes with Adam's trajectory on MNIST and CIFAR-10. We use  $\beta_1= 0.9$, $\beta_2 = 0.99$. On both datasets, we set batchsize to be 64, which brings $n= 937$ on MNIST and  $n= 781$ on CIFAR-10. We collect $\rho_1, \rho_2, \rho_3$ along the training process and estimate their distribution density. The results are shown in Figure \ref{fig:rho_mnist} and \ref{fig:rho_cifar}.

\begin{figure*}[htbp]
      \vspace{-2mm}
        \centering
        \subfigure[Histogram of $\rho_1$]{
        \begin{minipage}[t]{0.3\linewidth}
        \centering
        \includegraphics[width=\linewidth]{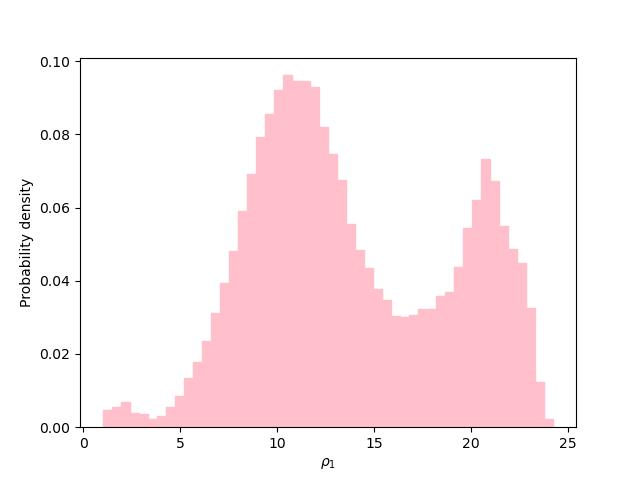}
        \end{minipage}%
        }%
        \subfigure[Histogram of $\rho_2$]{
          \begin{minipage}[t]{0.3\linewidth}
          \centering
        \includegraphics[width=\linewidth]{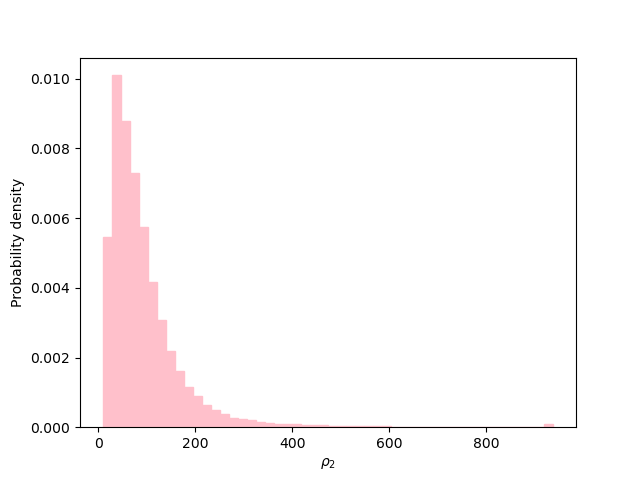}
          \end{minipage}%
          }%
        \subfigure[Histogram of $\rho_3$]{
          \begin{minipage}[t]{0.3\linewidth}
          \centering
        \includegraphics[width=\linewidth]{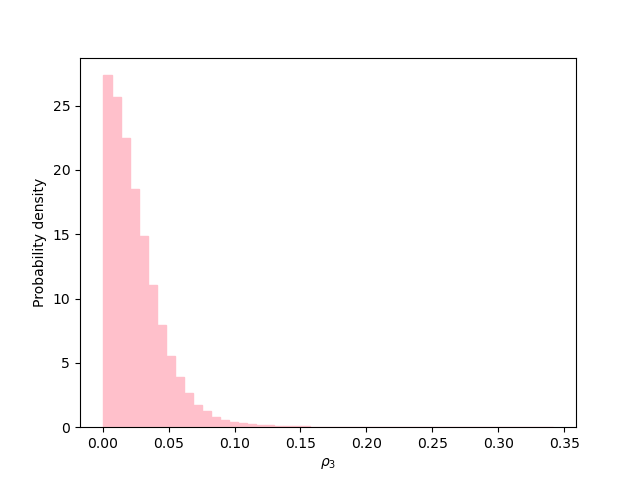}
          \end{minipage}%
          }%
        \centering
        \vspace{-3mm}
        \caption{Histograms of $\rho_1$, $\rho_2$, $\rho_3$ along the training process on MNIST. }
        \vspace{-2mm}
        \label{fig:rho_mnist}
\end{figure*}

\begin{figure*}[htbp]
      \vspace{-2mm}
        \centering
        \subfigure[Histogram of $\rho_1$]{
        \begin{minipage}[t]{0.3\linewidth}
        \centering
        \includegraphics[width=\linewidth]{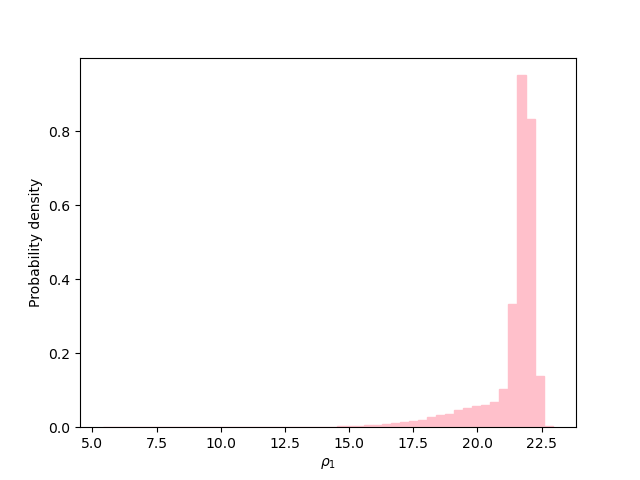}
        \end{minipage}%
        }%
        \subfigure[Histogram of $\rho_2$]{
          \begin{minipage}[t]{0.3\linewidth}
          \centering
        \includegraphics[width=\linewidth]{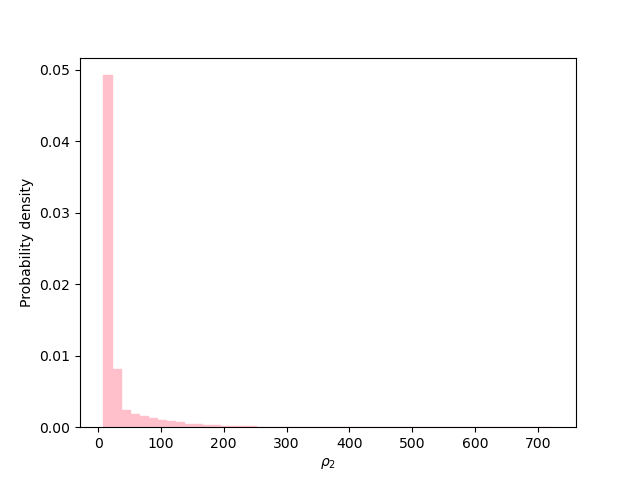}
          \end{minipage}%
          }%
        \subfigure[Histogram of $\rho_3$]{
          \begin{minipage}[t]{0.3\linewidth}
          \centering
        \includegraphics[width=\linewidth]{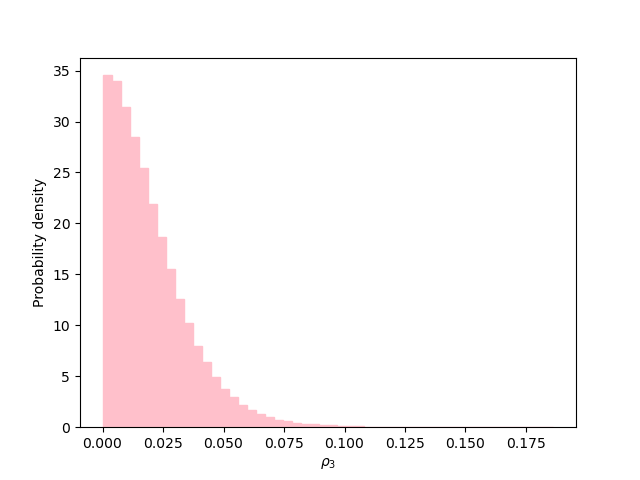}
          \end{minipage}%
          }%
        \centering
        \vspace{-3mm}
        \caption{Histograms of $\rho_1$, $\rho_2$, $\rho_3$ along the training process on CIFAR-10. }
        \vspace{-2mm}
        \label{fig:rho_cifar}
\end{figure*}

 On both  CIFAR-10 and MNIST, we observe that the maximal $\rho_1 < 25  \approx \mathcal{O}(\sqrt{n})$, 
 $\rho_2 < 400  \approx \mathcal{O}(n)$,
 $\rho_3 < 0.1  \approx \mathcal{O}(1/\sqrt{n})$. Therefore, $\rho = \rho_1\rho_2\rho_3 \approx \mathcal{O}(n)$. 
 
\paragraph{Batchsize and $\beta_2$.} 
As shown  in Figure \ref{fig:batchsize}: on MNIST, smaller batchsize requires larger $\beta_2$ to reach small loss. Since batchsize equals to (number of total sample)/(number of batches). In the context of finite-sum setting with $n$ summand, $n$ usually stands for the number of batches (e.g., In the extreme case  when batchsize $=1$, $n$ equals to the  number of total samples).  Therefore, smaller batchsize brings larger $n$.
As such,  Figure \ref{fig:batchsize}  matches the message by Theorem \ref{thm1}: the threshold of $\beta_2$ increases with $n$.

\begin{figure}[ht]
\begin{center}
\centerline{\includegraphics[width=2 in]{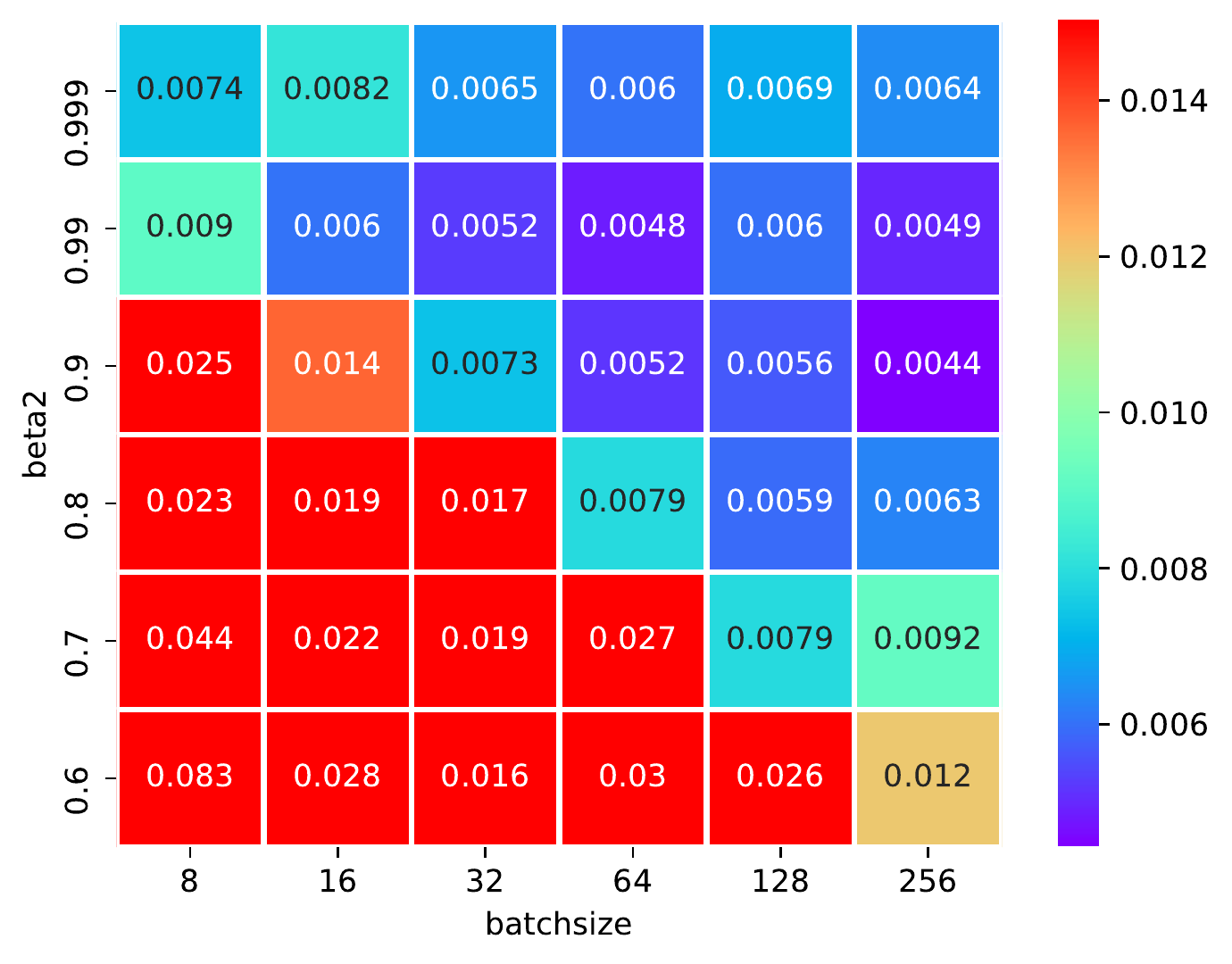}}
\caption{The training
loss on MNIST under different batchsize and $\beta_2$. Here, $\beta_1$ is fixed to be 0.9
}
\label{fig:batchsize}
\end{center}
\vskip -0.2in
\end{figure}

\paragraph{More experiments on function \eqref{counterexample1}.}  Figure \ref{fig:intro_paper} (a) shows the optimality gap after 50k iterations when $n=10$ and initialization  $x=1$. Here, we provide more relevant experiments. First, when initialized at $x=1$, we run experiments with $n =5, 10 ,15, 20$. The results are shown in Figure \ref{fig:toy}. We observe that the blue region shrinks as $n$ increases. This matches our conclusion in Theorem \ref{thm1}: when $n$ increases, the convergent threshold of $\beta_2$ increases, which means we need larger $\beta_2$ to ensure convergence. 
  It also matches the conclusion in Theorem \ref{thm_diverge}: when $n$ increases, the divergence region will expand (more evidence can be seen in Appendix \ref{appendix:thm_diverge}).

\begin{figure*}[htbp]
  \vspace{-2mm}
    \centering
    \subfigure[$n=5$]{
    \begin{minipage}[t]{0.25\linewidth}
    \centering
    \includegraphics[width=\linewidth]{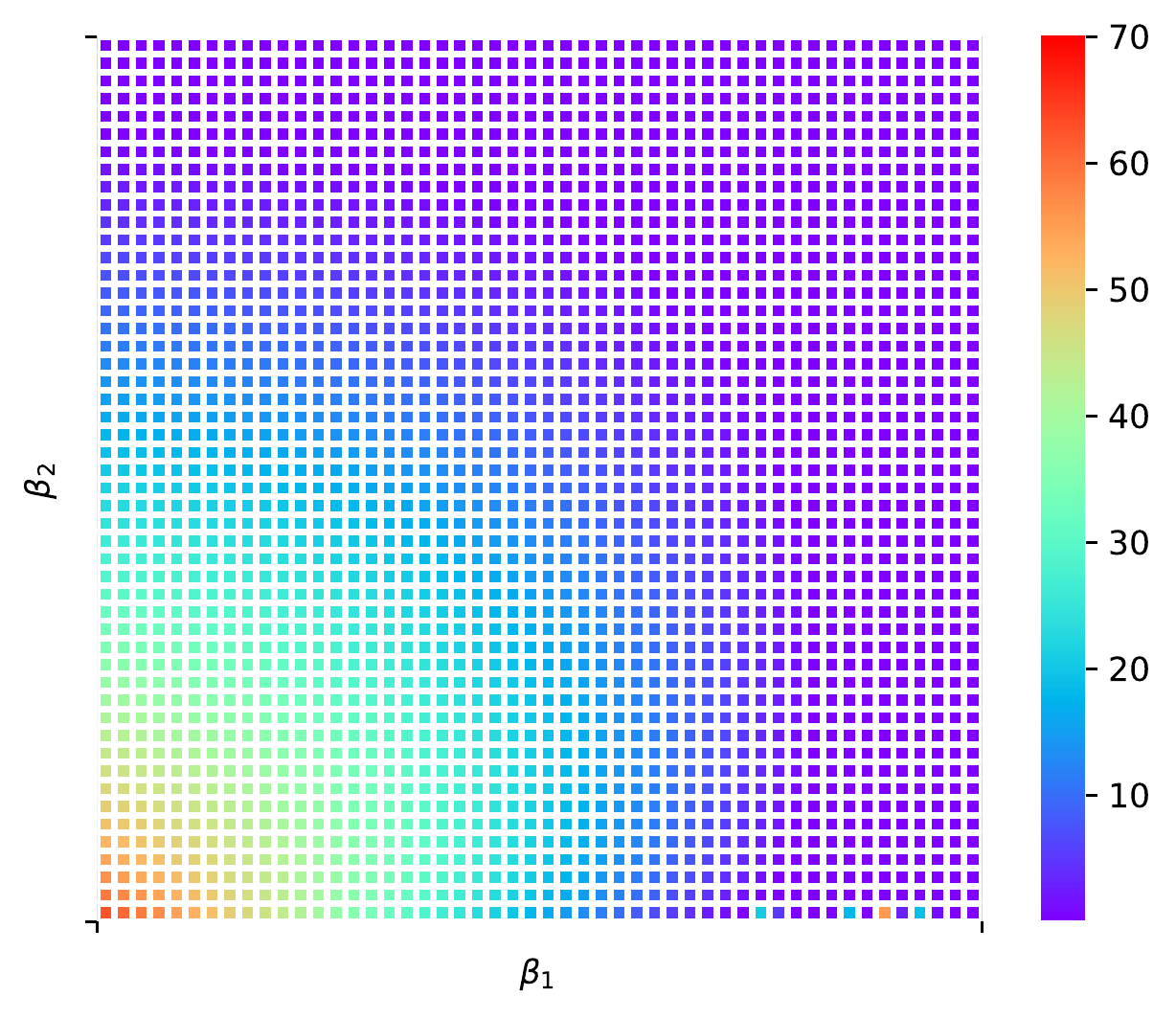}
    \end{minipage}%
    }%
    \subfigure[$n=10$]{
      \begin{minipage}[t]{0.25\linewidth}
      \centering
    \includegraphics[width=\linewidth]{./figures/0127_nosearch_toyexample_50000_10lr_0.1.pdf}
      \end{minipage}%
      }%
    \subfigure[$n=15$]{
      \begin{minipage}[t]{0.25\linewidth}
      \centering
    \includegraphics[width=\linewidth]{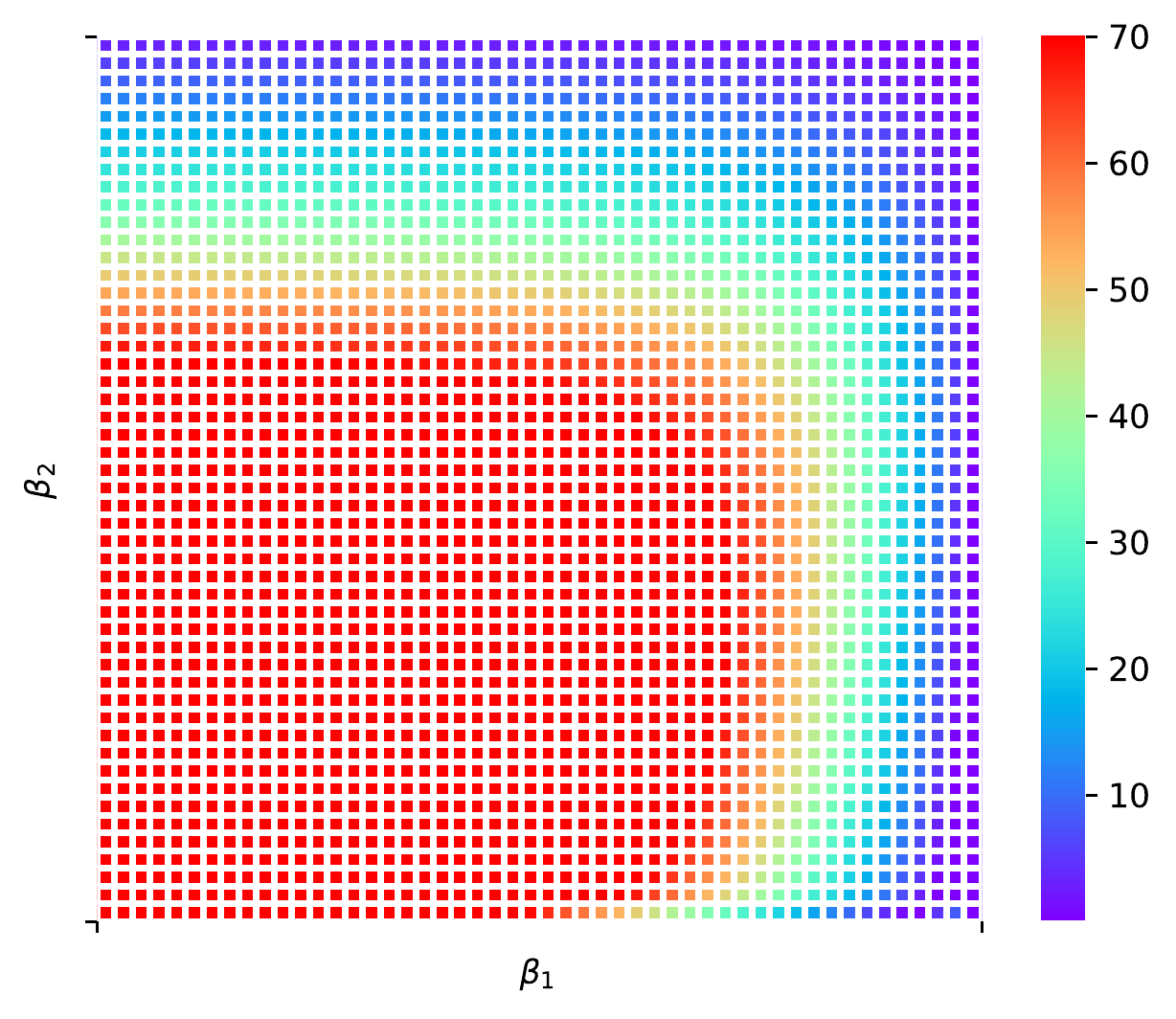}
      \end{minipage}%
      }%
    \subfigure[$n=20$]{
      \begin{minipage}[t]{0.25\linewidth}
      \centering
  \includegraphics[width=\linewidth]{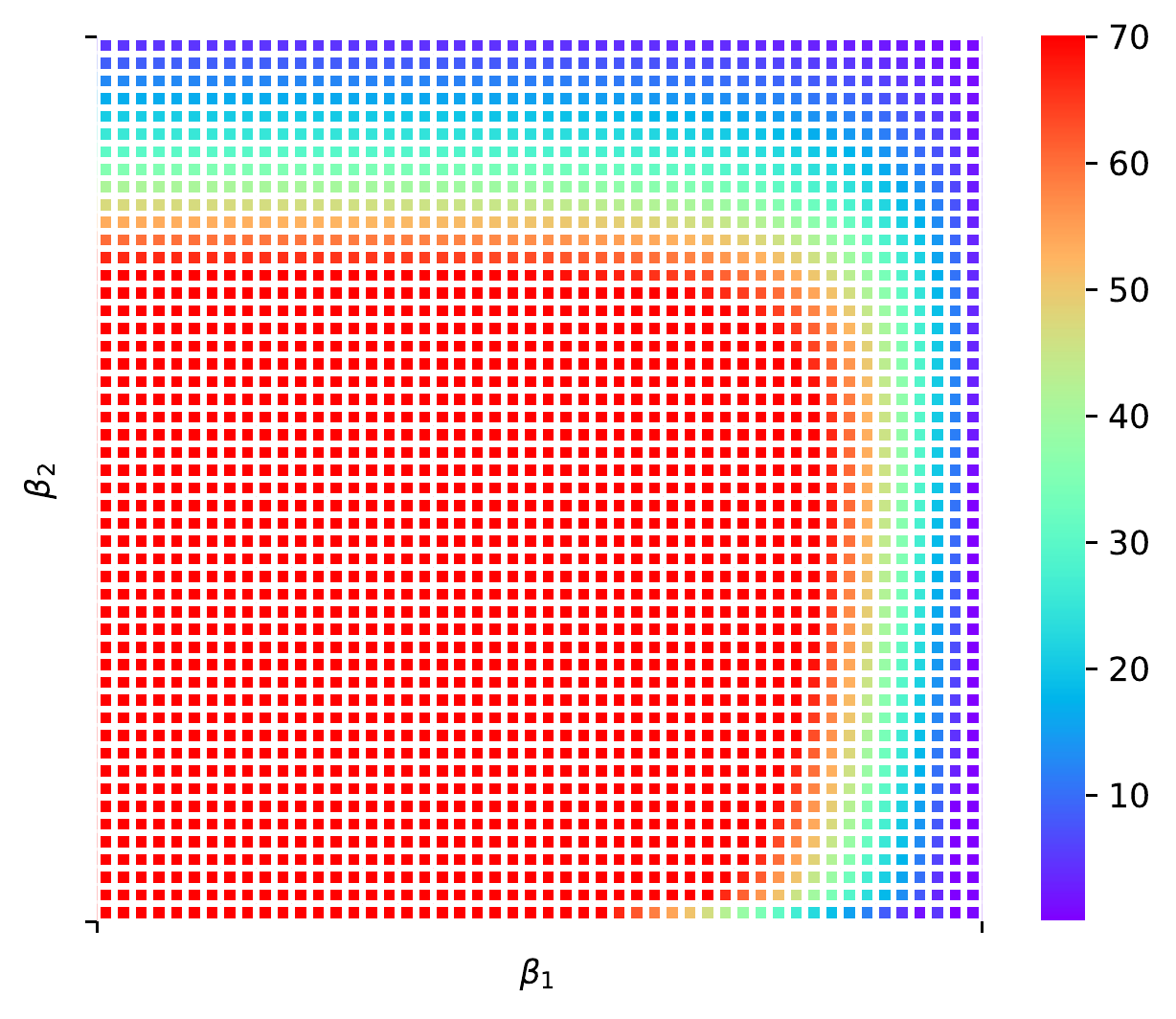}
      \end{minipage}%
      }%
    \centering
    \vspace{-2mm}
    \caption{ The optimality gap $x-x^*$ after running 50k iterations of Adam on function \eqref{counterexample1}. We use initialization $x = 1$.
    }%
    \label{fig:toy}
\end{figure*}

When initializing at $x=-5$, we further demonstrate that the gradient norm  of $f(x)$ can dramatically increase. All the setting is the same as that in Figure \ref{fig:toy} except for the change of initialization.   Starting at $x=-5$,  the algorithm can touch the ``quadratic" side of function \eqref{counterexample1} where the gradient is unbounded.   The results are shown in Figure \ref{fig:toy_gradient}.  We observe a  similar pattern as that in Figure \ref{fig:toy}. As a result, the gradient norm of $f(x)$ is large in the left bottom corner.

We further plot the change of gradient norm along the iterations. We pick $\beta_1=0.1$ and $\beta_2 = 0, 0.1,0.3,0.5,0.7,0.9,0.99$ to see the phase transition when increasing $\beta_2$.  The result is shown in Figure \ref{fig:toy_gradient_trajectory}. When $\beta_2$ is small, the gradient norm of  $f(x)$  increases rapidly along the iteration. Most of them are  even much larger than the upper bound of color bar in Figure \ref{fig:toy_gradient}. As a result, there is a  phase transition from diverge to converge when increasing $\beta_2$ from 0 to 1.

\begin{figure*}[htbp]
      \vspace{-2mm}
        \centering
        \subfigure[$n=5$]{
        \begin{minipage}[t]{0.25\linewidth}
        \centering
        \includegraphics[width=\linewidth]{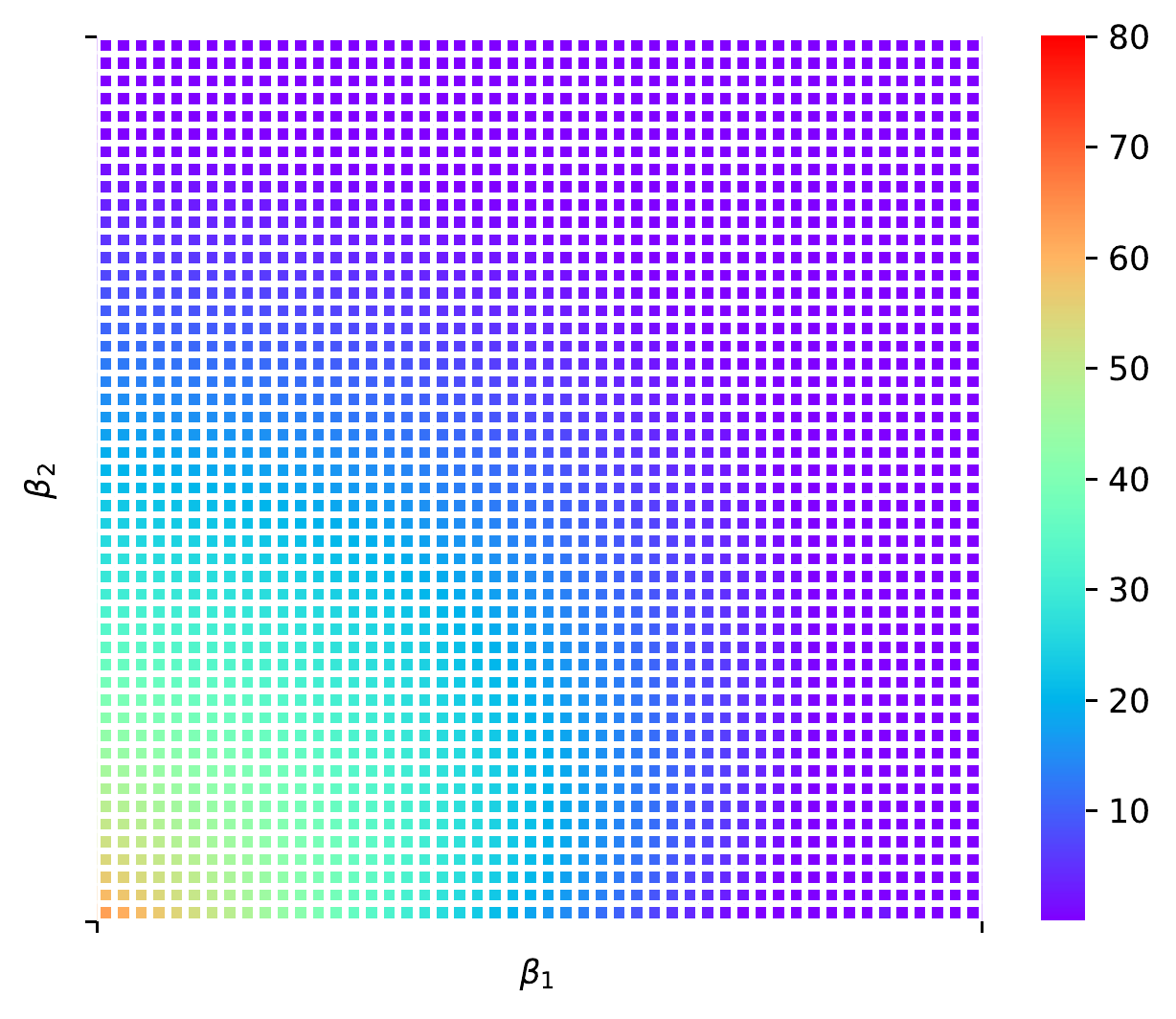}
        \end{minipage}%
        }%
        \subfigure[$n=10$]{
          \begin{minipage}[t]{0.25\linewidth}
          \centering
        \includegraphics[width=\linewidth]{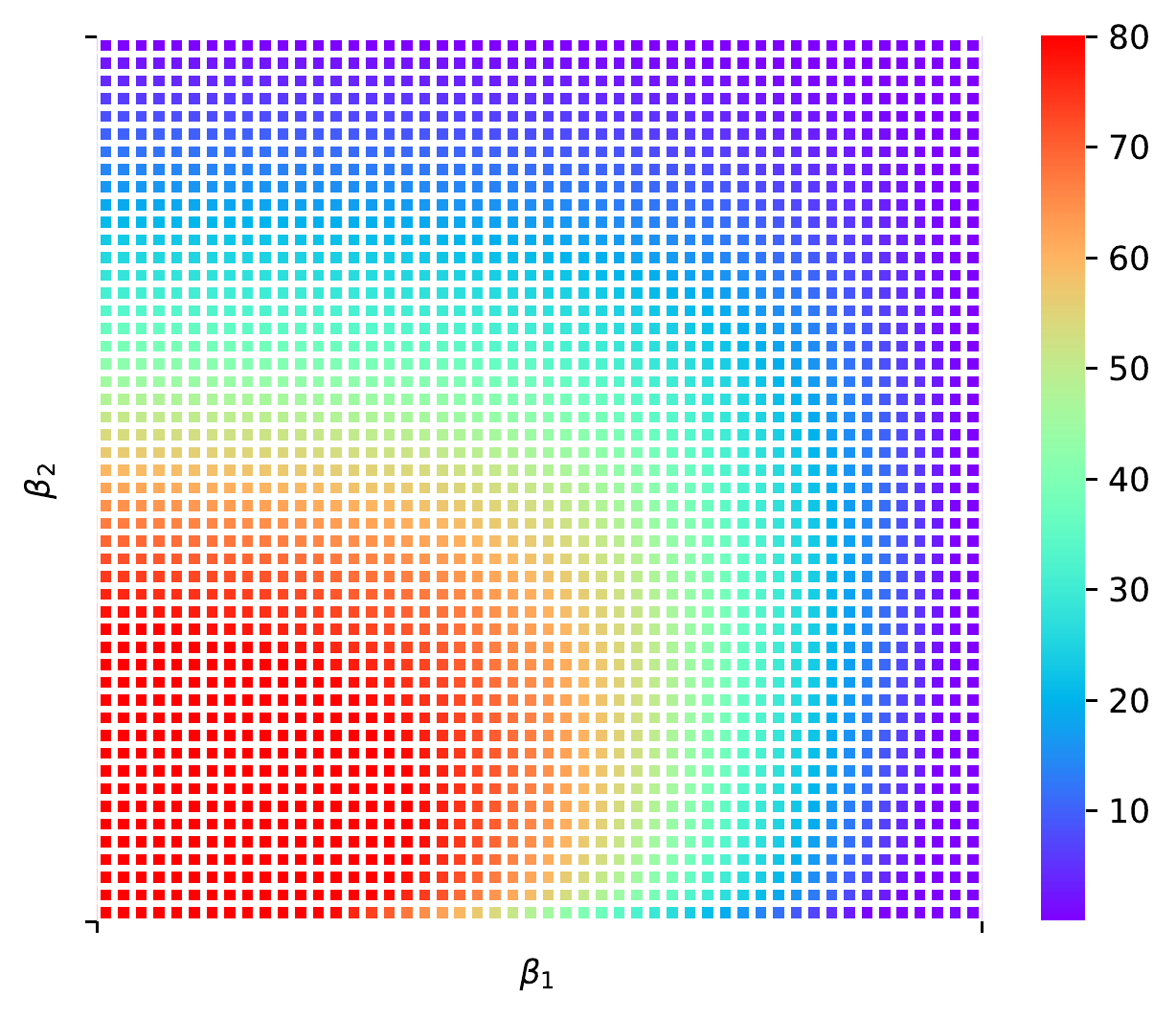}
          \end{minipage}%
          }%
        \subfigure[$n=15$]{
          \begin{minipage}[t]{0.25\linewidth}
          \centering
        \includegraphics[width=\linewidth]{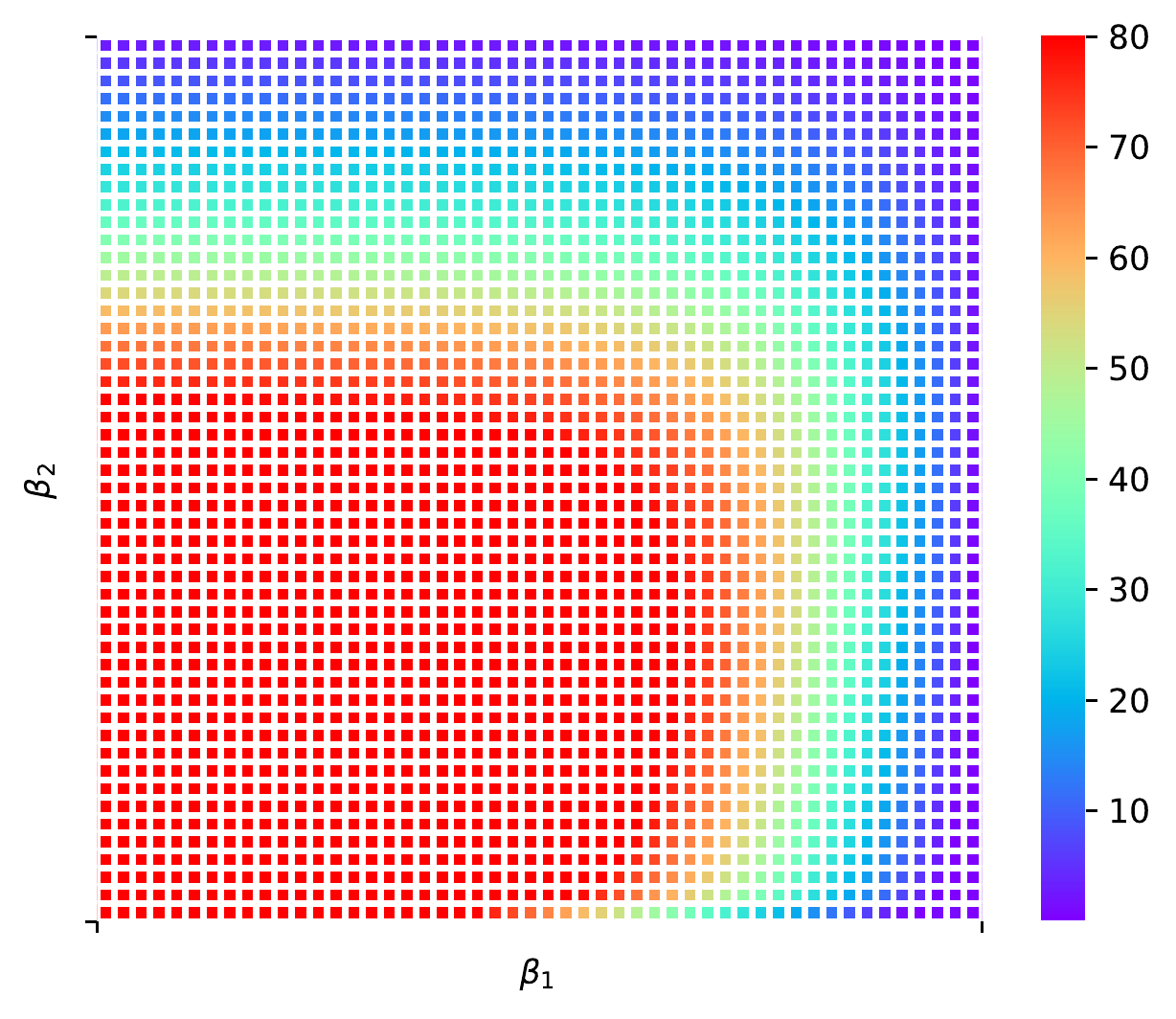}
          \end{minipage}%
          }%
        \subfigure[$n=20$]{
          \begin{minipage}[t]{0.25\linewidth}
          \centering
        \includegraphics[width=\linewidth]{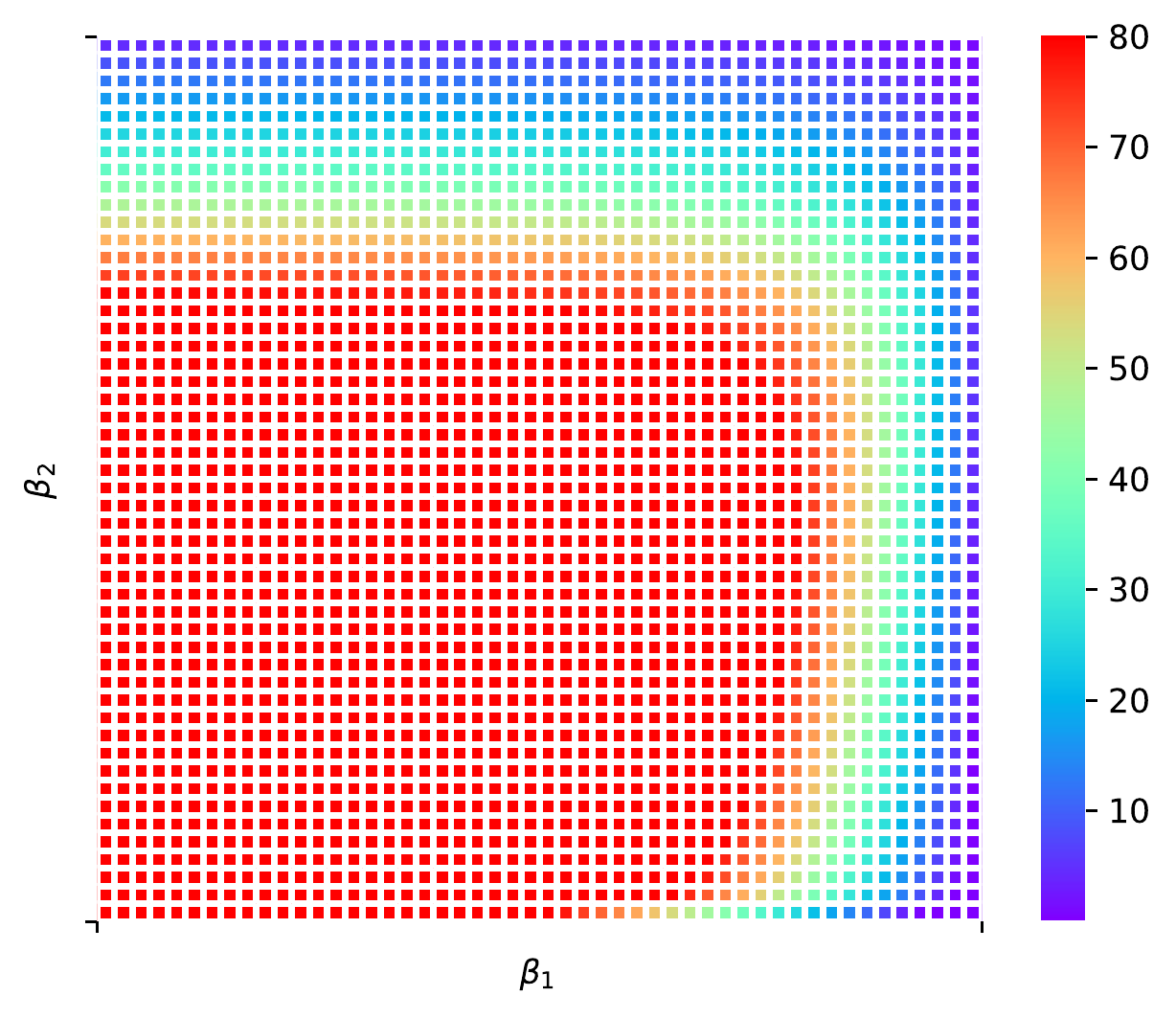}
          \end{minipage}%
          }%
        \centering
        \vspace{-3mm}
        \caption{ The gradient norm of $f(x)$ after running 50k iterations of Adam on function \eqref{counterexample1}. We use initialization $x = -5$.}
        \vspace{-2mm}
        \label{fig:toy_gradient}
\end{figure*}

\begin{figure*}[htbp]
      \vspace{-2mm}
        \centering
        \subfigure[$n=5$]{
        \begin{minipage}[t]{0.25\linewidth}
        \centering
         \includegraphics[width=\linewidth]{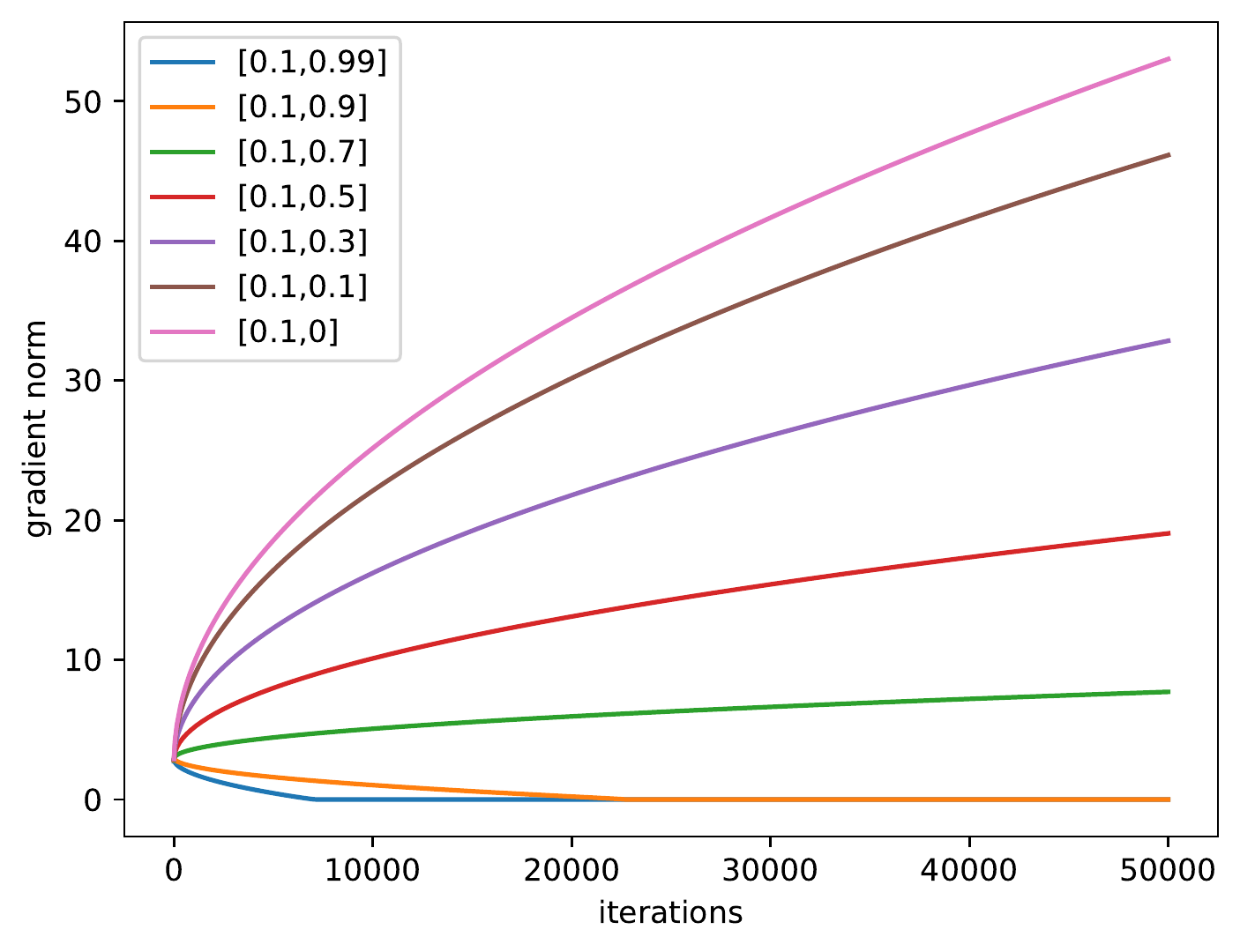}
        \end{minipage}%
        }%
        \subfigure[$n=10$]{
          \begin{minipage}[t]{0.25\linewidth}
          \centering
          \includegraphics[width=\linewidth]{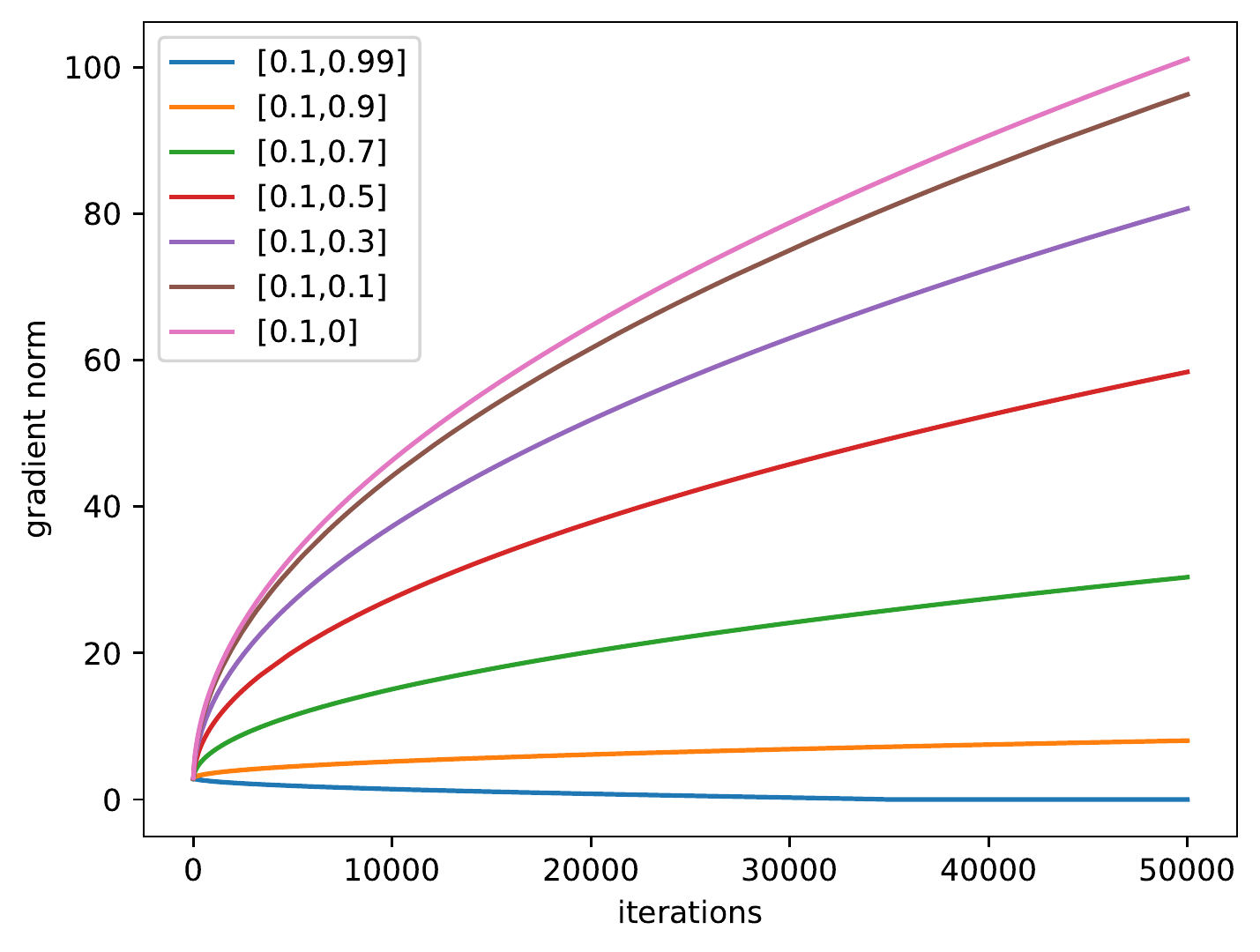}
          \end{minipage}%
          }%
        \subfigure[$n=15$]{
          \begin{minipage}[t]{0.25\linewidth}
          \centering
          \includegraphics[width=\linewidth]{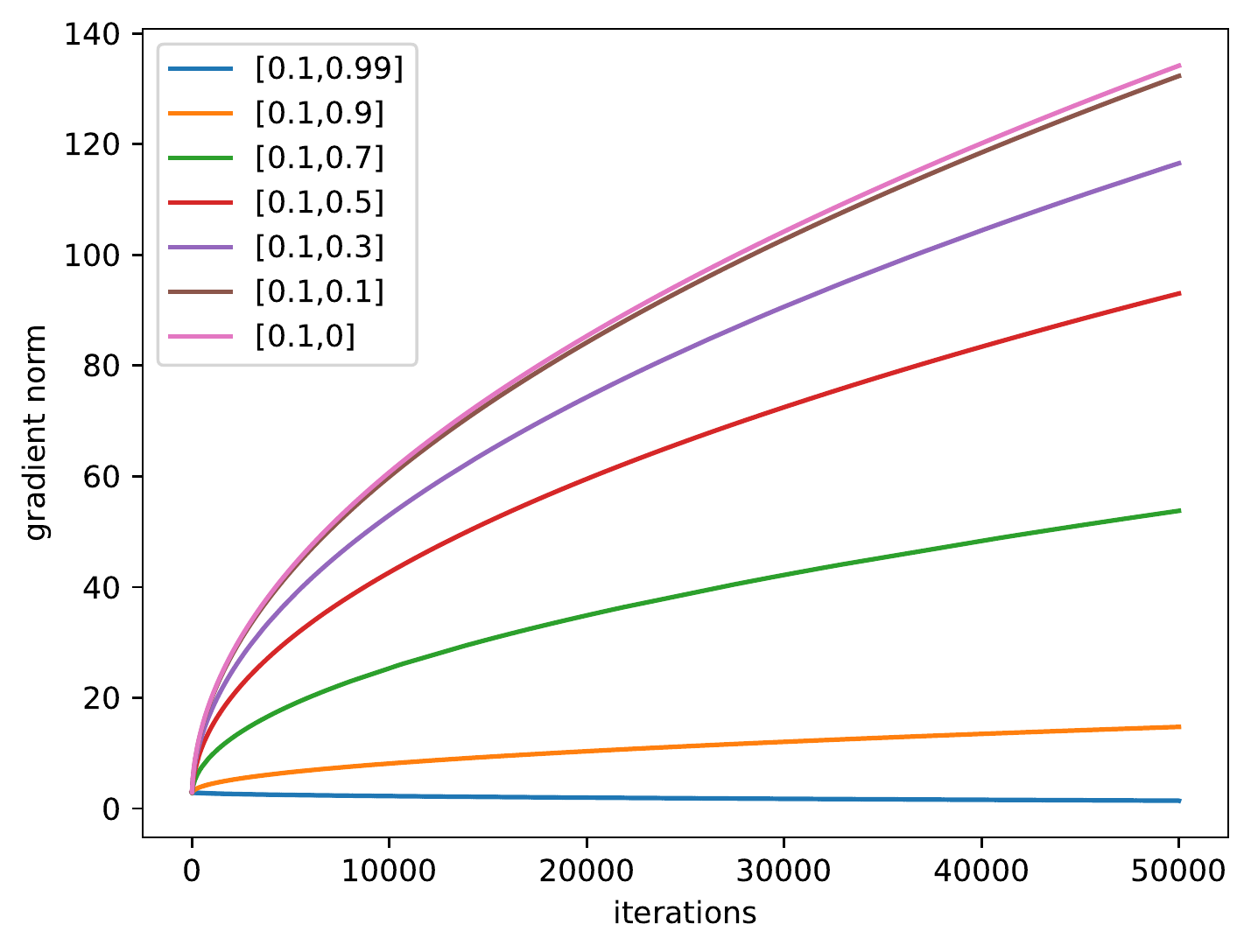}
          \end{minipage}%
          }%
        \subfigure[$n=20$]{
          \begin{minipage}[t]{0.25\linewidth}
          \centering
          \includegraphics[width=\linewidth]{./figures/0127_gradient_colomn_toyexample_50000_20.pdf}
          \end{minipage}%
          }%
        \centering
        \vspace{-3mm}
        \caption{ The change of gradient norm along the iterations of Adam on function \eqref{counterexample1}. We use initialization $x = -5$. We use $[\beta_1,\beta_2]$ to label the curves trained with corresponding hyperparameters.}
        \vspace{-2mm}
     \label{fig:toy_gradient_trajectory}
\end{figure*}

\begin{figure*}[ht]
      \vspace{-2mm}
        \centering
        \subfigure[$\beta_2=0.999$]{
        \begin{minipage}[t]{0.25\linewidth}
        \centering
          \includegraphics[width=\linewidth]{./figures/0128_gradient_SGC_toyexample_150000_20beta_20.999.pdf}
        \end{minipage}%
        }%
        \subfigure[$\beta_2=0.9995$]{
          \begin{minipage}[t]{0.25\linewidth}
          \centering
          \includegraphics[width=\linewidth]{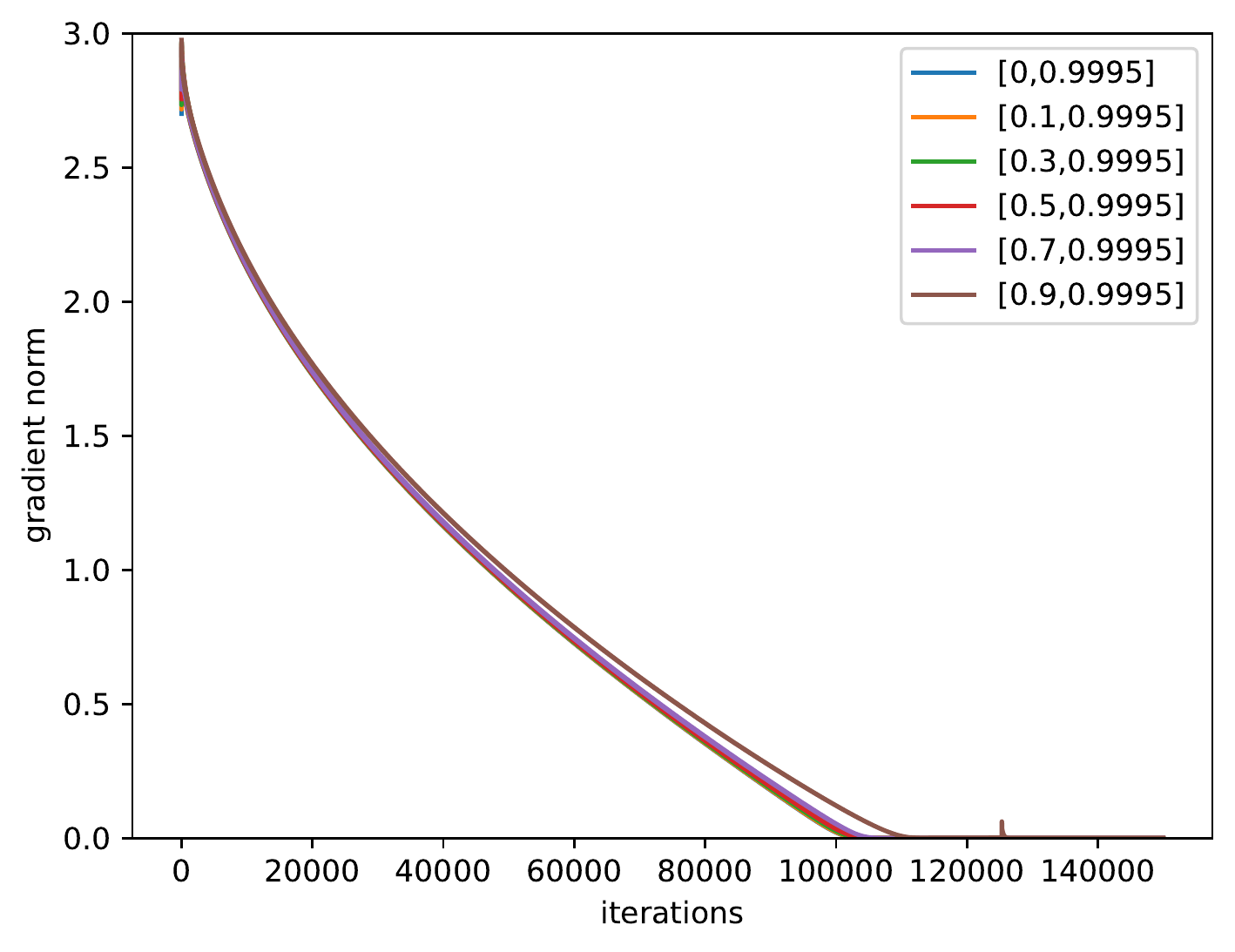}
          \end{minipage}%
          }%
        \subfigure[$\beta_2=0.9999$]{
          \begin{minipage}[t]{0.25\linewidth}
          \centering
          \includegraphics[width=\linewidth]{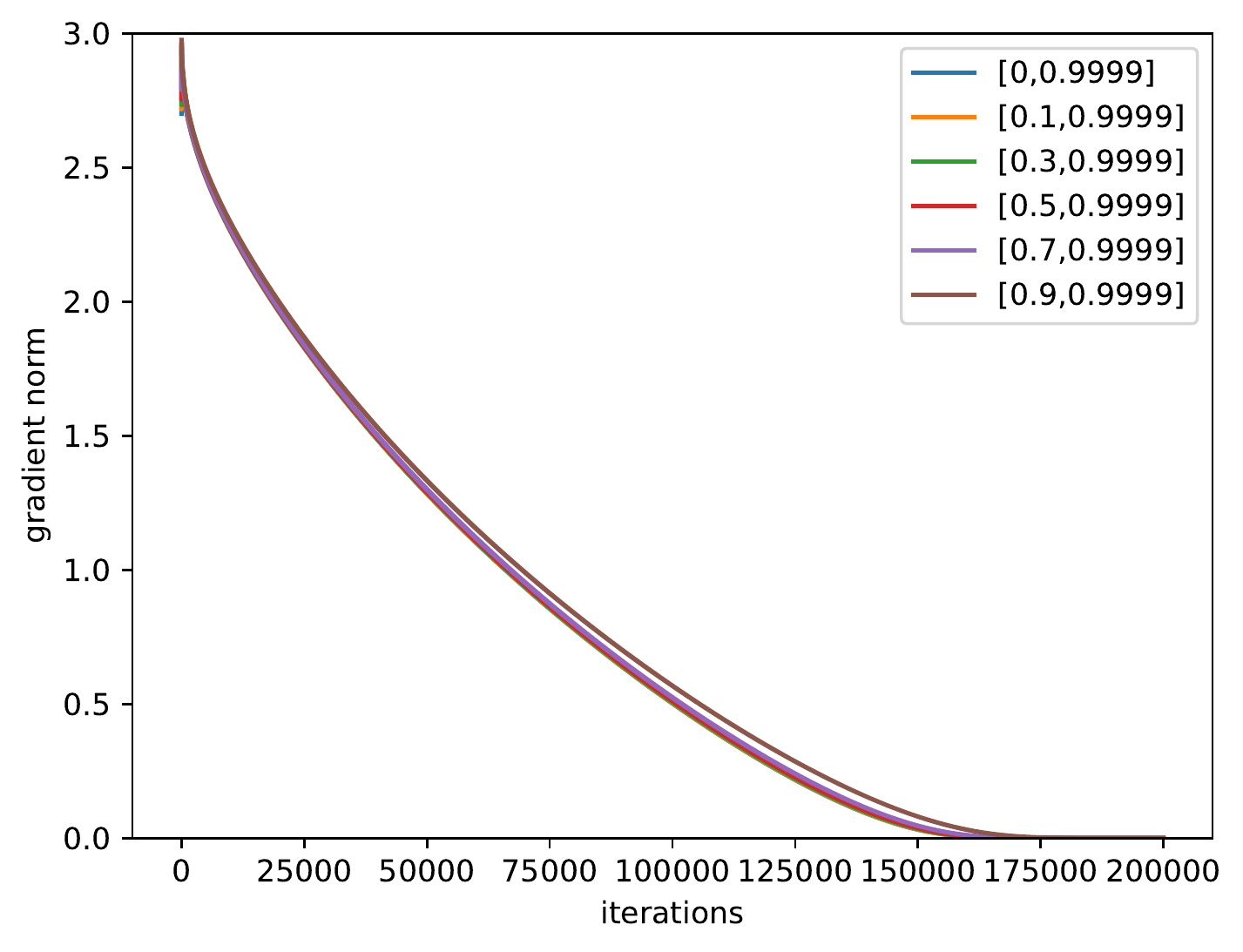}
          \end{minipage}%
          }%
        \subfigure[$\beta_2=0.99995$]{
          \begin{minipage}[t]{0.25\linewidth}
          \centering
          \includegraphics[width=\linewidth]{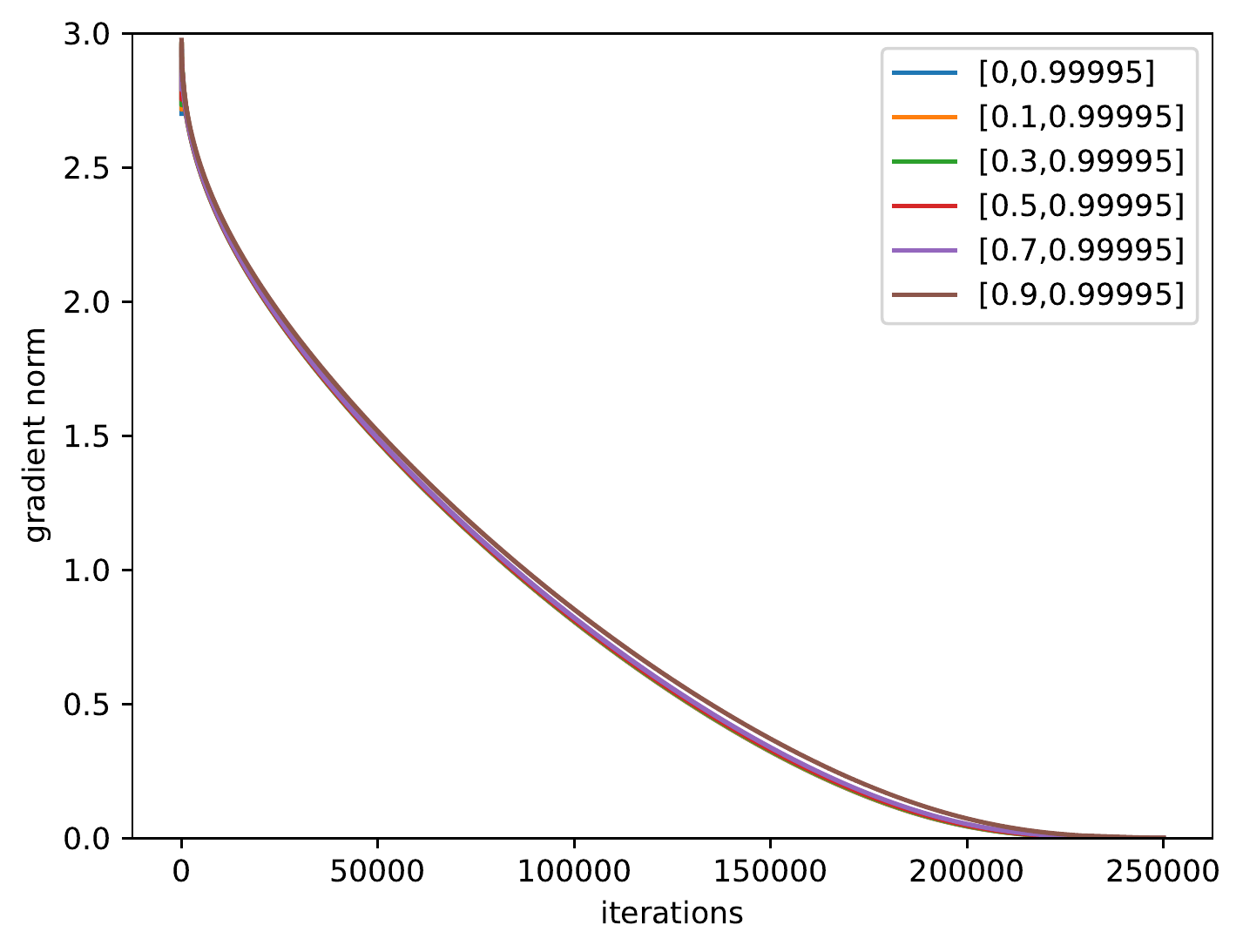}
          \end{minipage}%
          }%
        \centering
        \vspace{-3mm}
        \caption{ Under SGC (When $D_0 = 0$), we observe 0 gradient norm after  Adam converges. We use function \eqref{counterexample1} with $n=20$ and initialization $x=-5$.  We use $[\beta_1,\beta_2]$ to label the curves trained with corresponding hyperparameters.}
        \vspace{-2mm}
     \label{fig:toy_gradient_SGC}
\end{figure*}

\paragraph{Adam converges to a exact critical point when  $D_0=0$.}
Since function \eqref{counterexample1} satisfies $D_0=0$, we further provide empirical evidence that Adam converges to a exact critical point when $D_0=0$. We use  function \eqref{counterexample1} $n=20$ and initialization $x=-5$. We choose some large enough $\beta_2$ to ensure the convergence. 
As shown in Figure \ref{fig:toy_gradient_SGC},
We observe 0 gradient norm after  Adam converges. All the hyperparameter setting is the same as before.

\paragraph{Adam converges to a bounded region when $D_0>0$.} Now we show that Adam may not converge to an exact critical point when $D_0>0$. Instead, it converges to a bounded region near the critical point. For this part, we re-state the example from \citep{shi2020rmsprop}. Consider the following function.

\begin{equation}\label{eq_non_realizable}
    f_{j}(x)=\left\{\begin{array}{l}(x-a)^{2} \text { if  } j=0 \\ -0.1\left(x-\frac{10}{9} a\right)^{2} \text { if  } 1 \leq j \leq 9\end{array}\right.
\end{equation}

Summing up $f_j(x)$ we get

$$f(x)=\sum_{j=0}^{9} f_{j}(x)=\frac{1}{10} x^{2}-\frac{1}{9} a^{2}.$$

This $f(x)$ is lower bounded by $\frac{-a^2}{9}$ with optimal solution $x^*=0$. At the optimal point $x^*=0$, $\nabla f_j(x^*) \neq 0$ so we have $D_0 >0$. When running Algorithm \ref{algorithm} on this function, we observe that Adam with diminishing stepsize does not converge to 0 gradient norm. Instead, it converge to a bounded region.  Further, the size of the region shrinks when $\beta_2$ increases. These phenomena  matches our discussion in {\bf Remark 3.}  The result is shown in Figure \ref{fig:cifar_nlp_mnist} (a). In the experiment, we choose $\beta_1 =0.9$, $a =3$, $x_0 = -2$ and diminishing stepsize $\eta_k = \frac{0.1}{\sqrt{k}}$.



\subsection{Experimental Settings}
\label{appendix:exp_setting}

Here, we introduce our experimental settings.
\begin{itemize}
    \item {\bf Experiments on function \eqref{counterexample1}.} We use Algorithm \ref{algorithm} with cyclic order $f_0$, $f_1$, $f_2$ and so on. We report the optimality gap $x-x^*$ after  50k iteration, or equivalently  $50000/n$ epochs. We use $\epsilon =10^{-8}$ for numerical stability. We use diminishing stepsize $\eta_k=0.1/\sqrt{k}$, where $k$ is the index of epoch.  Unless otherwise stated, this setting applies to all the other experiments on function \eqref{counterexample1}.
    
    In Figure \ref{fig:intro_paper} (a), we use $n=10$ and  initialization $x =1$. We will report more results with different $n$ and different initialization in Appendix \ref{appendix:more_exp}.
    
    \item {\bf MNIST \citep{deng2012mnist}.} We use one-hidden-layer neural network with width =16. We set batchsize =1, weight decay =0, stepsize =0.0001 and train for 20 epochs.  We use $\epsilon =10^{-8}$ for numerical stability. 
    \item {\bf CIFAR-10 \citep{krizhevsky2009learning}.} We use ResNet-18 \citep{he2016deep} as the architecture. We choose batchsize =16, weight decay =5e-4 , initial stepsize=1e-3.   We use a stage-wise constant learning rate scheduling with a multiplicative factor of 0.1 on epoch 30, 60 and 90.   We use $\epsilon =10^{-8}$ for numerical stability.  
    
    For  MNIST and CIFAR-10, larger batchsize will bring similar pattern as that in Figure \ref{fig:intro_paper}, but
the phase transition will occur at some smaller $\beta_2$.

    \item {\bf NLP.} The WikiText-103 dataset is a collection of over 100 million tokens extracted from the set of verified ‘Good’ and ‘Featured’ articles on Wikipedia.  The base model of Transformer XL contains 16 self-attention layers. In each self-attention layer, there are 10 heads. The encoding dimension of each head is set to be 41.  We set batchsize = 60, number of iteration = 200k, and initial stepsize = 0.00025. We use cosine learning rate scheduler, which is a popular choice for training Transformer. We use $\epsilon =10^{-8}$ for numerical stability. 
\end{itemize}

\section{ Some Potential Implications for Practical Users}
\label{sec:implication}

Theorem \ref{thm1}  and Proposition \ref{thm_diverge} establish a clearer image on the relation between $(\beta_1,\beta_2)$ and qualitative behavior of Adam, which may have certain implication for practical users.  Many practitioners are still doing grid or random search over $\betabeta$, which could be costly.   The following advice may cut down a large portion of search space.
Suppose we start with some random point $(\beta_1^*,\beta_2^*)$, we provide the following suggestions for tuning $\beta_1$ and $\beta_2$. 

{\bf Case 1:  If $(\beta_1^*,\beta_2^*)$ lies  above the blue curve in Figure \ref{fig:counter_boundary_paper}.} We point out two sub-cases. 
First,  if Adam with $(\beta_1^*,\beta_2^*)$ diverges, then any points with ``$\beta_1 <\sqrt{\beta_2}$" and ``$\beta_2<\beta_2^*$" shall not be tried due to the risk of divergence.   
Second, if Adam with $(\beta_1^*,\beta_2^*)$ converges, then any points above this point will converge. You may either fix $\beta_1 = \beta_1^*$ and increase $\beta_2$, or fix $\beta_2 = \beta_2^*$ and try smaller $\beta_1$. Both ways have convergence guarantee.


{\bf Case 2: If $(\beta_1^*,\beta_2^*)$ lies  below the blue curve in Figure \ref{fig:counter_boundary_paper}. } We still discuss two sub-cases. First, if  you observe divergence at $(\beta_1^*,\beta_2^*)$, then do not further try any point on the left.  We do not suggest exploring the points in below either, since the majority of them will still face the risk of divergence. Instead, we suggest fix $\beta_1=\beta_1^*$ and increase $\beta_2$. Since our Theorem \ref{thm1} applies for any $\beta_1 <\sqrt{\beta_2}$,  algorithm will converge when $\beta_2$ is large enough. 

Second, if  you observe convergence at $(\beta_1^*,\beta_2^*)$ (which is also possible according to Figure \ref{fig:intro_paper} (a)), then we suggest: (i) either fix $\beta_1 = \beta_1^*$ and increase $\beta_2$, or (ii) fix $\beta_2 = \beta_2^*$ and try smaller $\beta_1$. Both ways push $\betabeta$ into the region of Theorem \ref{thm1} with convergence guarantee.




\section{More Discussion on Related Works.}

\subsection{More Related Works on the Convergence Analysis of Adam Family}
\label{appendix:more_related}

\paragraph{New variants of Adam.}
Ever since \citet{reddi2019convergence} pointed out the non-convergence issue of Adam, one active line of work has tried to design new variants of Adam that can be proved to converge. For instance,  \citet{zou2019sufficient,gadat2020asymptotic,chen2018convergence,chen2021towards} 
replace the constant hyperparameters by iterate-dependent ones e.g. $\beta_{1t}$ or $\beta_{2t}$. AMSGrad \citep{reddi2019convergence} and AdaFom \citep{chen2018convergence} modify $\{v_t\}$ to be an non-decreasing sequence.  Similarly,
AdaBound \citep{luo2018adaptive}  impose lower and upper bounds on $\{v_t\}$ to prevent the effective stepsize from  vanishing or exploding.   \citet{zhou2018adashift} also adopt a new estimate of  $v_t$ to correct the  bias.
There are also attempts to combine Adam with Nesterov momentum \citep{dozat2016incorporating} as well as warm-up techniques \citep{Liu2020On}.  Padam \citep{chen2018closing} also introduce a partial adaptive parameter to improve the generalization performance. 
There are also some works providing theoretical analysis on the variants of Adam.  For instance, \citet{zhou2018convergence} study the convergence of  AdaGrad and AMSGrad under bounded gradient condition. \citet{gadat2020asymptotic} study the asymptotic behavior of a subclass of adaptive gradient methods from landscape point of view. Their analysis applies to  Adam-variants with $\beta_1 =0$ and $\beta_2$ increasing along the iterates (it could also be understood as RMSProp with increasing $\beta_2$). When this script is under review, a new work \citep{iiduka2022theoretical} appear on line. \citet{iiduka2022theoretical} analyze the convergence of AMSGrad by relaxing the Lipschitz-gradient condition. However, their analysis requires extra conditions on both bounded gradient and bounded domain. 

\paragraph{Some more discussions on \citep{guo2021novel} and \citep{huang2021super}.}
Here, we discuss more on two recent works \citet{guo2021novel} and \citet{huang2021super}. As mentioned in Section \ref{section:related}, they both require some extra conditions. First,  both \citet{guo2021novel} and \citet{huang2021super}
requires bounded gradient assumption.
This can be seen in Assumption 2 in \citep{guo2021novel}. In  \citep{huang2021super}, they require bounded iterates ( their Theorem 1) or change Adam into  AdaBound \citep{luo2018adaptive} by clipping (their Remark 2, Corollary 1). Both settings inherent boundedness on gradient. 

Besides bounded gradient, both \citep{huang2021super} and \citep{guo2021novel} requires $1/(\sqrt{v_t}+\epsilon)\leq C_u$. 
This condition is stated in Assumption 2 in \citep{guo2021novel} and Assumption 3 in \citep{huang2021super} (they presented it as $H_{t} \succeq \rho I_{d} \succ 0$, where matrix $H_{t}= \text{diag}(\sqrt{v_t}+ \epsilon)$).  Combining these two conditions,  the effective stepsize of Adam will be bounded in certain interval  $\frac{1}{\sqrt{v_t}+ \epsilon} \in [C_l, C_u]$. Such boundedness condition changes Adam into AdaBound \citep{luo2018adaptive} and thus cannot explain the observations of Adam in Section \ref{section:intro}.

\subsection{A Brief Introduction to \citep{reddi2019convergence}}
\label{appendix:reddi}

Here, we re-state two counter examples by \citep{reddi2019convergence}. For the consistence of notation, we will re-state their results under our notation in the full script. 
They consider the convex problem (\citep{reddi2019convergence}):  $\min \sum_{i=0}^{n-1} f_i(x)$ where $x \in [-1,1]$, $n \geq 3$:
	\begin{equation}\label{counterexample_reddi}
		f_{i}(x)=\left\{\begin{array}{ll}n x, & \text { for } i =0 \\ -x, & \text { otherwise, }\end{array}\right.
	\end{equation}

Note that \eqref{counterexample_reddi} satisfy both Assumption \ref{assum1} and \ref{assum2} (with $D_1= n^2 +n -1$ and $D_0 =0$), so our assumptions do not rule out this counter-example a priori.
This is a constrained problem with feasible set $x \in [-1,1]$, the optimal solution is $x^*=-1$. Since they consider constrained problems, their claimed ``divergence" actually means the iterates will stay in a huge region with the size of whole feasible set. Here, we call it ``non-convergence" to distinguish from our result of ``diverge to infinity" in Proposition \ref{thm_diverge}.


They consider sampling $f_i$ in the  cyclic order: $f_0$, $f_1$, $f_2$.    In  \citep{reddi2019convergence}, 
Function \eqref{counterexample_reddi} is presented  as an ``online optimization problem with non-zero average regret". We choose to use the form of \eqref{counterexample_reddi} since it is more consistent with our notation in Algorithm \ref{algorithm}. We re-state their results as follows.


\begin{theorem}[Theorem 2 in \citep{reddi2019convergence}]
 For any fixed $\betabeta$ s.t. $\beta_1 < \sqrt{\beta_2}$, there exists function \eqref{counterexample_reddi} with large enough $n$, s.t.  Adam will  converge to highly sub-optimal solution $x=1$.
\end{theorem}

We briefly re-state the non-convergent condition for this Theorem.  As stated in Equation (7), Appendix B in \citep{reddi2019convergence},  for every fixed $\betabeta$, they need a ``constant $n$ that depends on $\beta_1$ and $\beta_2$". 
As such, they require  different $n$ to cause non-convergence on different $\betabeta$. So the considered function class is constantly changing.



For completeness, we further re-state Theorem 1 in \citep{reddi2019convergence}.

\begin{theorem}[Theorem 1 in \citep{reddi2019convergence}]
 For function \eqref{counterexample_reddi}, when $\beta_1=0$ and $\beta_2 = 1/(n^2+1)$, Adam will  converge to highly sub-optimal solution $x=1$.
\end{theorem}

This theorem considers choosing $\betabeta$ after $n$. However, this result only show non-convergence on {\it a single point} $\betabeta = (0, 1/(n^2+1))$. This point lies somewhere on the left boundary of Figure \ref{fig:reconcile_case1}. It seems unclear how Adam's behavior would change as we change the $\betabeta $ to anywhere else.




\section{Proof of Proposition \ref{thm_diverge} }
\label{appendix:thm_diverge}

We restate our counter-example here. 
Consider $f(x)=\sum_{i=0}^{n-1} f_i(x)$ for  $x \in \mathbb{R}$
, we define  $ f_i(x)$ as follows:
\begin{eqnarray}
  f_{i}(x)&=&\left\{\begin{array}{ll} nx, &  x \geq -1\\ \frac{n}{2}(x+2)^2 -\frac{3n}{2}, & x < -1 \end{array}\right.  \text{for $i=0$,} \nonumber \\
  f_{i}(x)&=&\left\{\begin{array}{ll} -x, &  x \geq -1\\ -\frac{1}{2}(x+2)^2 +\frac{3}{2}, & x < -1 \end{array}\right.  \text{for $i>0$.} \label{counterexample}
\end{eqnarray}

Summing up all the $f_i(x)$, we can see that 

$$
f(x)= \left\{\begin{array}{ll} x, &  x \geq -1\\ \frac{1}{2}(x+2)^2 -\frac{3}{2}, & x < -1 \end{array}\right.
$$

is a convex smooth function with optimal solution $x^*=-2$ and optimal value $f(x^*)=-3/2$.

 However, we are going to show that, for any fixed $n>2$, there exists an orange region shown in Figure \ref{fig:counter_boundary_paper}, s.t., Adam with any $\betabeta$ combination in the yellow region diverge to $x=\infty$ rather than the optimal solution $x=-2$, causing the divergence. 
Now we introduce the formal statement of Proposition \ref{thm_diverge}.

\begin{proposition}\label{thm_diverge2} {\bf(Formal statement of Proposition \ref{thm_diverge}.)}
Consider the convex function \eqref{counterexample} for a fixed $n$. Starting at the initialization $x=1$ and initial stepsize $\eta_1$,  the iterates of Adam diverge to infinity if the following holds:

\begin{equation}\label{diverge_c1}
    {\bf (C1):}   \left(n-1-\min \left\{n-1, \log_{\beta_2} (\frac{1}{10n^2})   \right\}\right) \frac{1-\beta_1^{\min \left\{n-1, \log_{\beta_2} (\frac{1}{10n^2})   \right\}}}{\sqrt{1+ \max\left\{0.1, \beta_2^{n-1} n^2   \right\}}}
\geq  \frac{1-\beta_1}{\sqrt{1-\beta_2}} + \frac{\beta_1  }{\sqrt{1-\beta_2}} n; 
\end{equation}		

\begin{equation}\label{diverge_c2}
      {\bf (C2):}  (1-\beta_1^{n-1})>(1-\beta_1)\beta_1^{n-1}n.
\end{equation}

\begin{equation}\label{diverge_c3}
    {\bf (C3): } \eta_1 \leq 2\sqrt{(1-\beta_2)\beta_2^n}.
\end{equation}

\end{proposition}

The analysis is motivated by that in  \citep[Theorem 1]{reddi2019convergence}. However, \citep[Theorem 1]{reddi2019convergence} considers a simplified case with  $\beta_1 = 0$. Here, we consider non-zero $\beta_1$, especially for those $\beta_1 > \sqrt{\beta_2}$.   To show the divergence, we aim to prove the following claim: (we denote $x_{k, i}$ as the value of $x$ at the $k$-th outer loop and $i$-th inner loop ) 

\begin{center}
  Claim: for any fixed $n>2$, there exists an orange region shown in Figure \ref{fig:counter_boundary_paper} s.t., Adam with any $\beta_1$-$\beta_2$ combination in the orange region gives $x_{k+1,0}>1$ as long as $x_{k,0}=1$. 
\end{center}

Since the gradient stays constant when $x>1$, so $x$ will go to infinity if the claim holds. 
To prove this claim, we only need to analyze the trajectory of Adam within one particular outer loop, e.g., the $k$-th outer loop. We will show that  $x_{k+1,0} >1$ if this outer loop is initialized with $x_{k,0}=1$. 
Similarly as \citep{reddi2019convergence}, we assume $f_i(x)$ are sampled in the order of $f_0(x), f_1(x), \cdots, f_{n-1}(x)$ within the  $k$-th outer loop. 



Now let us prove the claim. For function \eqref{counterexample}, the update rule of Adam is shown as follows.

{\small
\begin{equation}\label{eq_counter_update1}
x_{k,1}= \left(x_{k,0}+\delta_{k,0}\right), \quad \delta_{k,0}=- \frac{\eta_1}{\sqrt{k}} \left(\frac{n(1-\beta_1)+\beta_1 m_{k-1,n-1}}{\sqrt{(1-\beta_2)n^2+\beta_2 v_{k-1,n-1}}}\right)
\end{equation}

\begin{equation}\label{eq_counter_update2}
x_{k,i+1} = \left(x_{k,i}+\delta_{k,i}\right) , \quad i=1, \cdots, n-1;
\end{equation}

where $ \delta_{k,i} = - \frac{\eta_1}{\sqrt{k}} \left( \frac{(1-\beta_1) \sum_{j=0}^{i-1} (-1) \beta_1^{j} + (1-\beta_1)\beta_1^i n+\beta_1^{i+1}m_{k-1,n-1} }{\sqrt{(1-\beta_2)+ \beta_2 v_{k,i-1}}}\right).$

}

We decompose the  total movement $ \sum_{i=0}^{n-1}\delta_{k,i} $ into three terms as follows.

{\small
\begin{eqnarray*}
 \sum_{i=0}^{n-1}\delta_{k,i}  &=& \frac{\eta_1}{\sqrt{k}} \underbrace{\left( -\frac{\beta_1  m_{k-1,n-1}}{\sqrt{(1-\beta_2)n^2+\beta_2 v_{k-1,n-1}}} -  \frac{\beta_1^2  m_{k-1,n-1}}{\sqrt{(1-\beta_2)+\beta_2 v_{k,0}}}   -\cdots -  \frac{\beta_1^n  m_{k-1,n-1}}{\sqrt{(1-\beta_2)+\beta_2 v_{k,n-2}}}   \right) }_{(I)}  \\
 && +  \frac{\eta_1}{\sqrt{k}}  \underbrace{\left(\frac{1-\beta_1 }{ \sqrt{(1-\beta_2) +\beta_2 v_{k,0}}} + \frac{(1-\beta_1)+\beta_1(1-\beta_1) }{ \sqrt{(1-\beta_2) +\beta_2 v_{k,1}}} + \cdots+ \frac{(1-\beta_1)\sum_{j=0}^{n-2} \beta_1^j
  }{ \sqrt{(1-\beta_2) +\beta_2 v_{k,n-2}}}   \right) }_{(II)} \\
  &&+ \frac{\eta_1}{\sqrt{k}} \underbrace{\left( -  \frac{n(1-\beta_1)}{\sqrt{(1-\beta_2)n^2 +\beta_2 v_{k-1,n-1}}}  - \frac{n(1-\beta_1)\beta_1}{\sqrt{(1-\beta_2) +\beta_2 v_{k,0}}} - \cdots- \frac{n(1-\beta_1)\beta_1^{n-1}}{\sqrt{(1-\beta_2)+\beta_2 v_{k,n-2}}}     \right) }_{(III)}.
\end{eqnarray*}

}

We will show that for some $\beta_1$ and $\beta_2$: $(I)$, $(II)>0$ and $(III)<0$. However, $(I)$ and $(II)$  outweigh $(III)$, causing the divergence. 

First, we show that $m_{k-1, n-1}<0 $ when $\beta_1$ is small. 

\begin{eqnarray*}
  -m_{k-1,n-1} &=& (1-\beta_1) \sum_{j=0}^{n-2} \beta_1^j -  (1-\beta_1)\beta_1^{n-1}n -\beta_1^n m_{k-2,n-1} \\
    &=& (1-\beta_1^{n-1}) -  (1-\beta_1)\beta_1^{n-1}n -\beta_1^n m_{k-2,n-1}\\
    &=& \left[(1-\beta_1^{n-1})-(1-\beta_1)\beta_1^{n-1}n \right] \sum_{j=0}^k \left(\beta_1^{n}\right)^j.
 \end{eqnarray*}

when $\beta_1$ is small,we have  $(1-\beta_1^{n-1})>(1-\beta_1)\beta_1^{n-1}n$, which implies  $-m_{k-1,n-1}>0$. For these choices of $\beta_1$, we  have $(I)>0$. Now we derive a lower bound for $(II)$.
 
 \begin{eqnarray*} 
   (II)&\geq & \frac{1-\beta_1 }{ \sqrt{1+\beta_2 n^2}} + \frac{(1-\beta_1)+\beta_1(1-\beta_1) }{ \sqrt{1+\beta_2^2 n^2}} + \cdots+ \frac{(1-\beta_1)\sum_{j=0}^{n-2} \beta_1^j
   }{ \sqrt{1+\beta_2^{n-1} n^2}}  \\
   &=&  \frac{1-\beta_1 }{ \sqrt{1+\beta_2 n^2}} + \frac{1-\beta_1^2 }{ \sqrt{1+\beta_2^2 n^2}} + \cdots+ \frac{1-\beta_1^{n-1}
   }{ \sqrt{1+\beta_2^{n-1} n^2}}.
 \end{eqnarray*}

 The inequality is due to the fact that $v_{k,0}\leq n^2$. Since $\beta_2^{j} n^2$ is small when $\beta_2$ is small and  $j$ is close to $n$, there exists some small $\beta_2$ such that  $\beta_2^{j} n^2 \leq 0.1$ for at least one $j<n$. For these small enough $\beta_2$, we keep the summand with $j \geq \log_{\beta_2} (0.1/n^2)$ and drop the rest.   We have the following lower bound for $(II)$.

 \begin{eqnarray*} 
  (II)&\geq &
  \left(n-1-  \log_{\beta_2} (\frac{1}{10n^2})   \right) \frac{1-\beta_1^{ \log_{\beta_2} (\frac{1}{10n^2})   }}{\sqrt{1+  0.1  }}
\end{eqnarray*}

However, this lower bound only holds for the small $\beta_2$. With simple modification, we derive a universal lower bound of $(II)$ for any $\beta_2 \in (0,1)$.

 \begin{eqnarray*} 
  (II)&\geq &
  \left(n-1-\min \left\{n-1, \log_{\beta_2} (\frac{1}{10n^2})   \right\}\right) \frac{1-\beta_1^{\min \left\{n-1, \log_{\beta_2} (\frac{1}{10n^2})   \right\}}}{\sqrt{1+ \max\left\{0.1, \beta_2^{n-1} n^2   \right\}}}
\end{eqnarray*}

Now we derive a lower bound for $(III)$. 

\begin{eqnarray*}
  (III) & = &   -  \frac{n(1-\beta_1)}{\sqrt{(1-\beta_2)n^2 +\beta_2 v_{k-1,n-1}}}  - \frac{n(1-\beta_1)\beta_1}{\sqrt{(1-\beta_2) +\beta_2 v_{k,0}}} - \cdots- \frac{n(1-\beta_1)\beta_1^{n-1}}{\sqrt{(1-\beta_2)+\beta_2 v_{k,n-2}}} \\
  &\geq & - \frac{1-\beta_1}{\sqrt{1-\beta_2}}\left( 1+n(\sum_{j=1}^{n-1}\beta_1^j)    \right) \\
  &=& - \frac{1-\beta_1}{\sqrt{1-\beta_2}} -  \frac{\beta_1 (1-\beta_1^{n-1}) }{\sqrt{1-\beta_2}} n \\
  &\geq & - \frac{1-\beta_1}{\sqrt{1-\beta_2}} -  \frac{\beta_1  }{\sqrt{1-\beta_2}} n.
\end{eqnarray*}

The remaining step is to show for small enough step size $\eta_1$, the iterates will stay in the linear region, thus the above gradient holds for all iterates in the trajectory.

As we initial $x$ as $x_0 = 1$, if for all $m<n$ and $k>0$, $\sum_{i=0}^m \delta_{k,i}\geq-2$, we can conclude all iterates in the trajectory stay in the linear region.

Because it holds that $m_{k-1,n-1}\leq 0$ and $v_{k-1,n-1}\geq 0$, we have the following result after dropping most negative terms in the definition of $\delta_{k,i}$:
\[
\delta_{k,i} \geq -\frac{\eta_1}{\sqrt{k}} \frac{n(1-\beta_1)\beta_1^i}{\sqrt{n^2(1-\beta_2)\beta_2^i}} \geq -\eta_1\frac{(1-\beta_1)\beta_1^i}{\sqrt{(1-\beta_2)\beta_2^n}}.
\]

Therefore, to make $\sum_{i=0}^m \delta_{k,i}\geq \sum_{i=0}^m -\eta_1\frac{(1-\beta_1)\beta_1^i}{\sqrt{(1-\beta_2)\beta_2^n}}\geq-2$, we have
\[
\eta_1 \leq 2\sqrt{(1-\beta_2)\beta_2^n}.
\]
To show the divergence, we want to show that there exists some $\beta_1$ and $\beta_2$ s.t. the both of the following  conditions hold:


$$ {\bf (C1):}   \left(n-1-\min \left\{n-1, \log_{\beta_2} (\frac{1}{10n^2})   \right\}\right) \frac{1-\beta_1^{\min \left\{n-1, \log_{\beta_2} (\frac{1}{10n^2})   \right\}}}{\sqrt{1+ \max\left\{0.1, \beta_2^{n-1} n^2   \right\}}}
\geq  \frac{1-\beta_1}{\sqrt{1-\beta_2}} + \frac{\beta_1  }{\sqrt{1-\beta_2}} n;  $$

$$
{\bf (C2):} \text{$\beta_1$ is small s.t. } (1-\beta_1^{n-1})>(1-\beta_1)\beta_1^{n-1}n.
$$

$$
{\bf (C3): } \eta_1 \leq 2\sqrt{(1-\beta_2)\beta_2^n}.
$$
The proof of Theorem \ref{thm_diverge} is completed.
With the help of Python, we visualize the region where  $ {\bf (C1)} $ and $ {\bf (C2)} $ hold. The results are shown in Figure \ref{fig:counter_boundary}. We use orange color to indicate the region where $ {\bf (C1)} $ holds.  White color is used for the counter part. As for $ {\bf (C2)} $, we use the gray vertical line to indicate 
the line where $(1-\beta_1^{n-1})=(1-\beta_1)\beta_1^{n-1}n$. Note that there are two solutions to this equation: one solution is $\beta_1=1$ and the other solution lies in  $ 0<\beta_1<1$, this is why there are two vertical lines in the figure. $ {\bf (C2)} $ holds on the left hand side of the left gray vertical line.

\begin{figure}[h]
  \vspace{-2mm}
    \subfigure[$n=5$]{
      \begin{minipage}[t]{0.5\linewidth}
      \centering
      \includegraphics[width=2in]{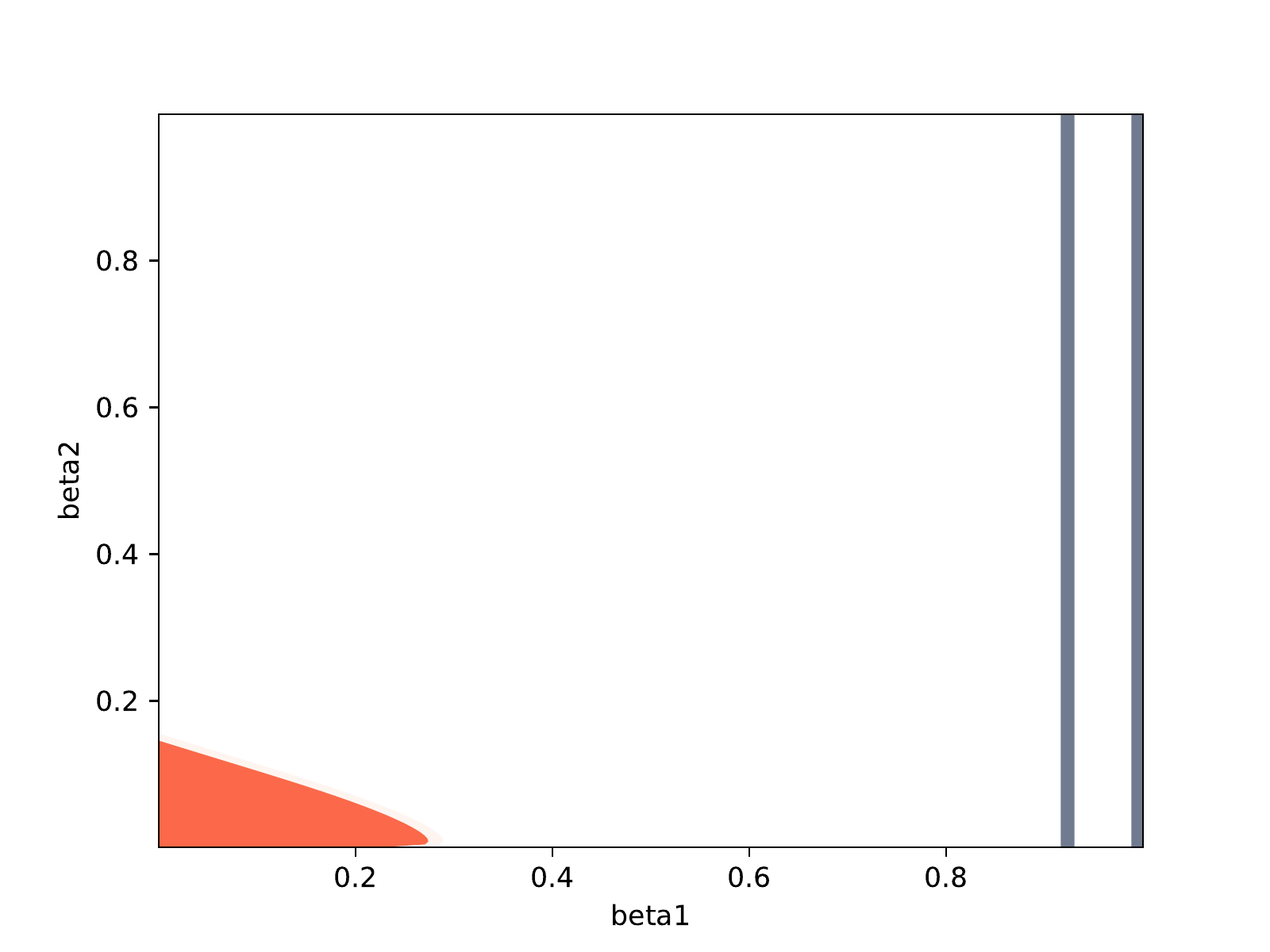}
      \end{minipage}%
      }%
      \subfigure[$n=10$]{
        \begin{minipage}[t]{0.5\linewidth}
        \centering
        \includegraphics[width=2in]{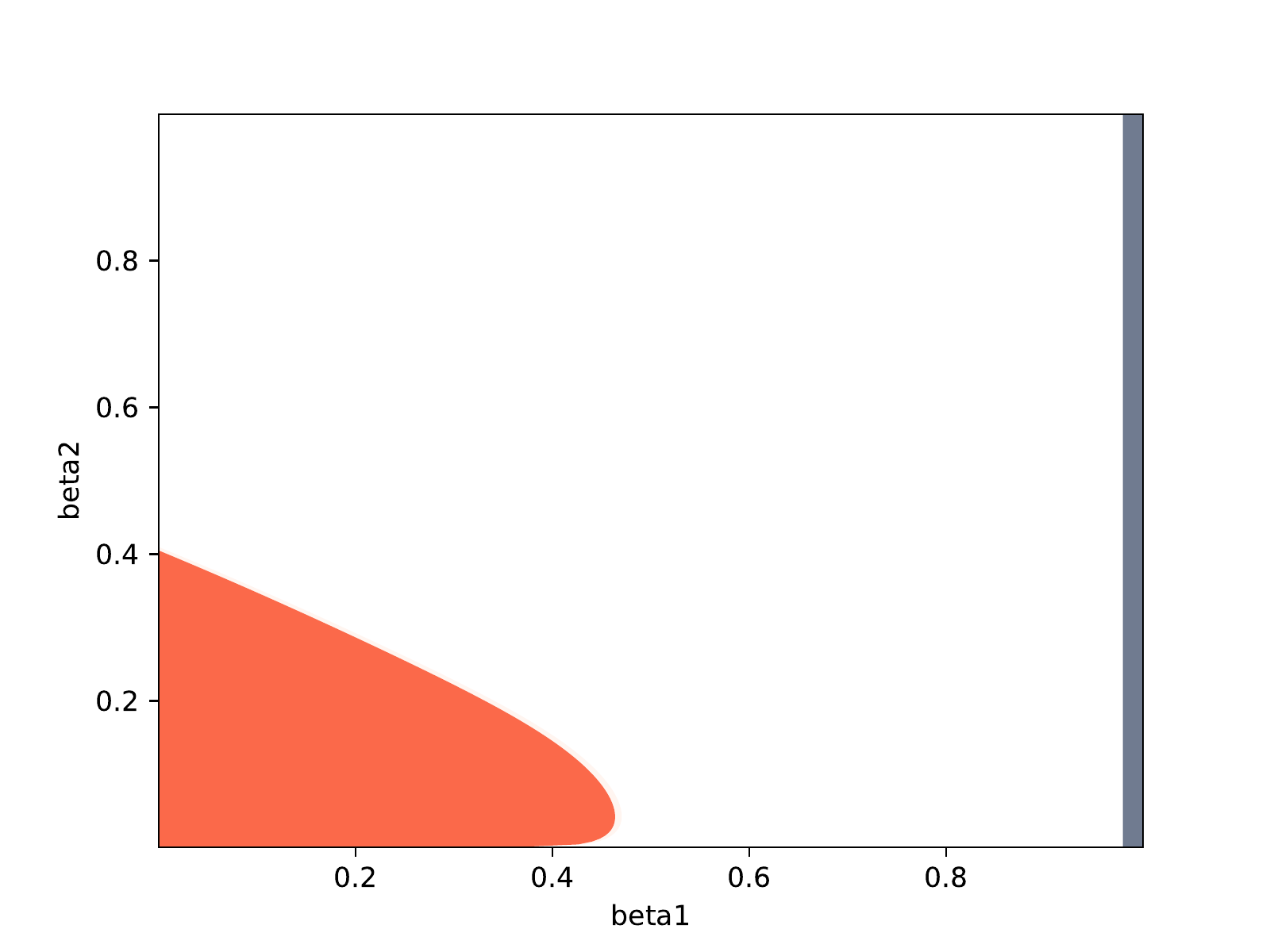}
        \end{minipage}%
        }\\%
        \subfigure[$n=50$]{
        \begin{minipage}[t]{0.5\linewidth}
          \vspace{-3mm}
        \centering
        \includegraphics[width=2in]{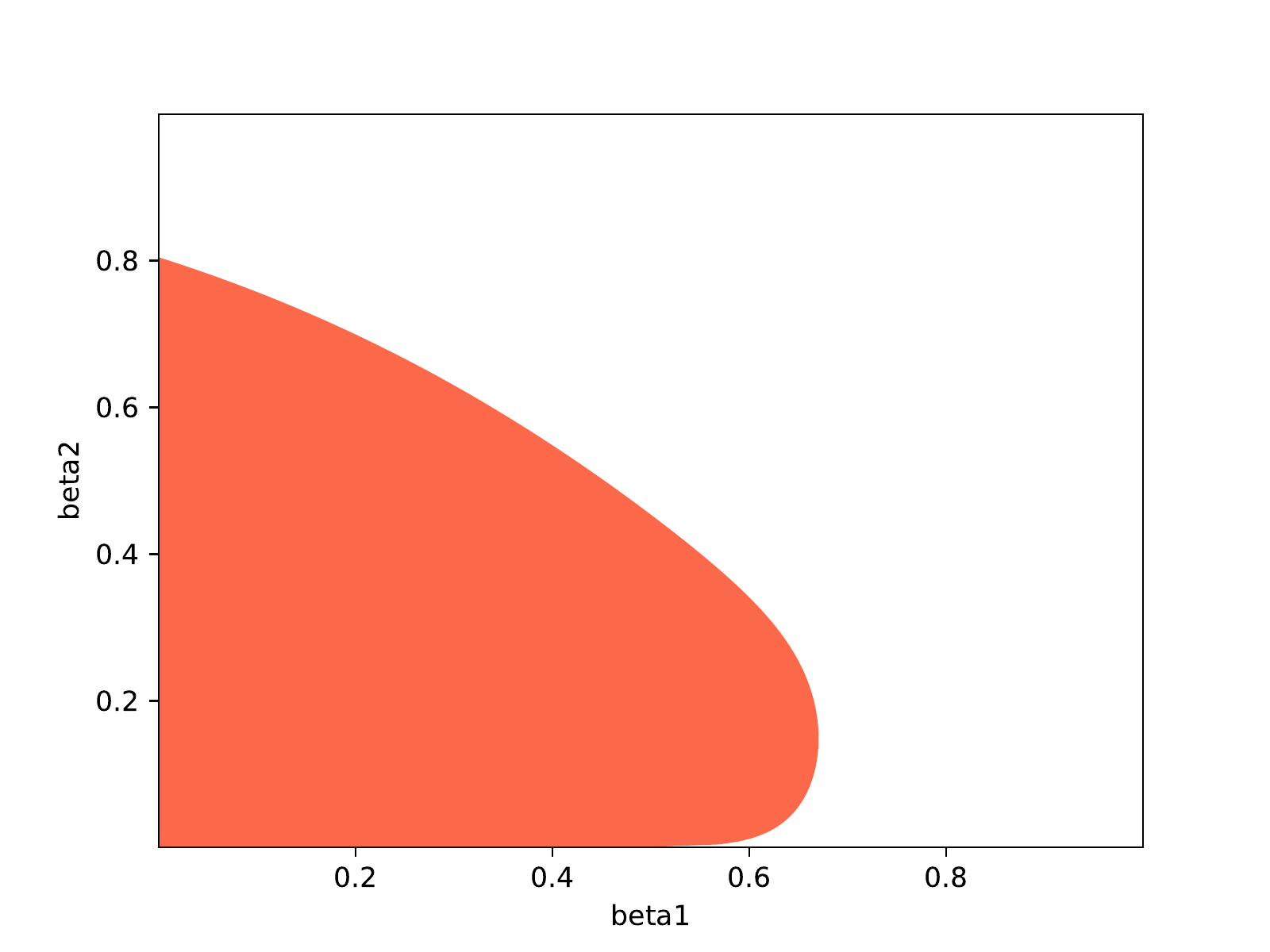}
        \end{minipage}%
        }%
        \subfigure[$n=100$]{
          \begin{minipage}[t]{0.5\linewidth}
          \vspace{-3mm}
          \centering
          \includegraphics[width=2in]{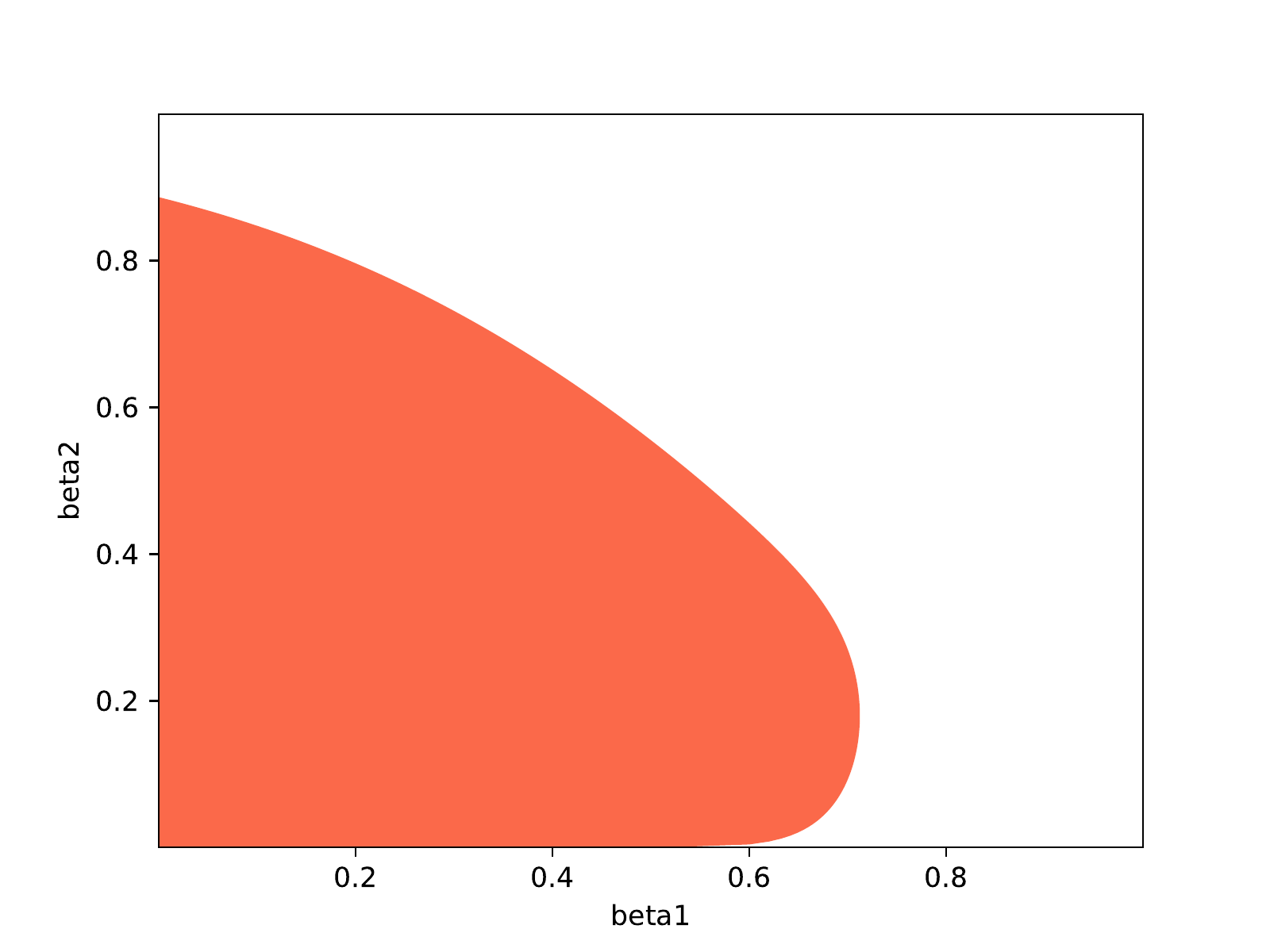}
          \end{minipage}%
          }%
    \caption{ This figure illustrates the region where both  $ {\bf (C1)} $ and $ {\bf (C2)} $ in Proposition \ref{thm_diverge2} hold. The orange color indicates the region where $ {\bf (C1)} $ holds. White color is used for the counter part. The gray vertical lines are used to indicate 
    the boundary of $ {\bf (C2)} $. Note that there are two solutions to the equation in $ {\bf (C2)} $: one solution is $\beta_1=1$ and the other solution lies in  $ 0<\beta_1<1$, this is why there are two vertical lines in the figure. $ {\bf (C2)} $ holds on the left hand side of the left gray vertical line.  
    This figure is visualized in Python. }
    \label{fig:counter_boundary}
    \vspace{-1mm}
  \end{figure}

The intersection of two regions will be the region where Adam diverges, which is actually the orange region in Figure \ref{fig:counter_boundary}. As $n$ increases, $ {\bf (C2)} $ holds for a wide range of $\beta_1$, so the grey vertical lines move towards $\beta_1=1$ and finally get overlapped. In addition, the size of  divergence region increases with $n$

\paragraph{Relation with $\gamma_1(n)$ in Theorem \ref{thm1}} According to Theorem \ref{thm1}, $\gamma_1(n) $ is at least in the order of $1-\mathcal{O}(n^{-2})$. Combining with Figure \ref{fig:counter_boundary}. It is not hard to see that $\gamma_1(n)$ is always larger than the upper boundary of the orange region, so there is no contradiction.

\paragraph{Proof of Corollary \ref{corollary_diverge}.} For any $\betabeta \in [0,1)^2$, condition {\bf (C1), (C2)} and {\bf (C3)} can be satisfied by some sufficiently large $n$. Therefore, Adam will diverge and the proof is concluded.

\section{Some More Notations and Useful Lemmas for Convergence Analysis}
\label{appendix_notations}
\subsection{More notations}
\label{sec_notations}
\begin{itemize}
  \item We denote $x_{k, i}, m_{k, i}, v_{k, i} \in \mathbb{R}^{d}$ as the value of $x, m, v$ at the $k$-th outer loop and $i$-th inner loop. Further, we denote  $x_{l, k, i}, m_{l, k, i}, v_{l, k, i} \in \mathbb{R}$ as the $l$-th component of $x_{k, i}, m_{k, i}, v_{k, i}$.
  \item We denote $\eta_{k}$ as the stepsize at the $k$-th epoch (outer loop). We will focus mainly on diminishing step size, especially $\eta_{k}=\frac{\eta_{1}}{\sqrt{nk}}$, where $n$ is the number of batches (inner loop). 
  \item We denote $ \partial_l f(x)=\frac{\partial}{\partial x_{l}} f(x)$, $ \partial_l f_{j}(x)=\frac{\partial}{\partial x_{l}} f_{j}(x)$. Further, we will use $\tau_{k,i}$ to index  the $i$-th randomly chosen batch in the $k$-th epoch. In this sense,  we denote  $ \partial_l f_{\tau_{k,i}}(x)$ as $\frac{\partial}{\partial x_{l}} f_{\tau_{k,i}}(x)$.
  \item Given an epoch $k$, we denote  $\alpha$ as the index of the coordinate with the greatest gradient:
  $$
  \alpha=\arg \max _{l=1,2, \cdots, d}\left|\partial_{l} f\left(x_{k, 0}\right)\right|.
  $$
  \item We define $\rho_1, \rho_2, \rho_3$ as  constants satisfying the following condition for any $l =1,\cdots, d$:

  $$\rho_{1} \geq \frac{\sum_{i=1}^{n}\left|\partial_l f_i (x_{k,0})\right|}{\sqrt{\sum_{i=1}^{n}\left| \partial_l f_i (x_{k,0})\right|^{2}}};$$
  
  $$\rho_{2} \geq \frac{\left|\max_i \partial_l f_i (x_{k,0})\right|^{2}}{\frac{1}{n} \sum_{i=1}^{n}\left|\partial_l f_i (x_{k,0})\right|^{2}};$$
  
  $$\rho_{3} \geq \frac{\left|\sum_{i=1}^{n} \partial_l f_i (x_{k,0})\right|}{\sqrt{\frac{1}{n} \sum_{i=1}^{n}\left|\partial_l f_i (x_{k,0})\right|^{2}}}.$$
  $\rho_{1}, \rho_{2}, \rho_{3}$ are problem-dependent constants. In worst case, we have $0 \leq \rho_{3} \leq \sqrt{n} \rho_{1} \leq n$.  These constants are firstly introduced by  \citep{shi2020rmsprop}. 
  
  \item We define the constant $\triangle_{x}:= \frac{\eta_{0}}{\sqrt{x}} \frac{L \sqrt{d}}{\sqrt{1-\beta_{2}}} \frac{1-\beta_{1} }{1-\frac{\beta_{1}}{\sqrt{\beta_{2}}}}.$ This constant will be used repeatedly, especially in Lemma \ref{lemma_delta}.
  \item We define $Q_k:= \triangle_1 \frac{n\sqrt{n}}{\sqrt{k}} \frac{32\sqrt{2}}{(1-\btwo)^n \btwo^n}.$ This constant will be used repeatedly.
  \item  We define 
  \begin{equation}\label{eq_delta1}
   \delta_{1}=\frac{(1-\beta_2)4n \rho_2  }{\beta_2^n}+\left(\frac{1}{\sqrt{\beta_{2}^{n}}}-1\right).
  \end{equation}
 
This constant will be used repeatedly. Note that $\delta_{1}\rightarrow 0$ when $\beta_2 \rightarrow 1$.
  \item $\mathbb{E}(\cdot)$ means taking expectation over the whole trajectory.
  \item We will use $\Ext \left[ \cdot \right]$ as a shorthand for  $\Ex [ \cdot \bigm |  x_{k,0}, x_{k-1,n-1}, \cdots, x_{1,0}]$, i.e. the conditional expectation given all the history up to $x_{k,0}$. Similarly, $\mathbb{E}_{k-1}$ stands for  the conditional expectation given all the history up to $x_{k-1,0}$. 
\end{itemize}

We further introduce some `index' notation of the stochastic gradients, these notations will be useful in the proof. For all the stochastic gradient, we have 

\begin{itemize}
  \item {\bf True index:} every stochastic gradient has its own true index: $f_1(\cdot), f_2(\cdot), \cdots, f_{n-1}(\cdot)$. All these $\{f_i(\cdot)\}_{i=0}^{n-1}$ are fixed once the optimization problem is formulated.
  
  Note that for a fixed $x$, we have the following relationship between $f_i$ and $f$:
  
  \begin{equation}\label{prop_i}
    \sum_{i=0}^{n-1}\partial_l f_i(x)= \partial_l f(x).
  \end{equation}

  \item {\bf Random-shuffle index:} At the $k$-th epoch, all the stochastic gradients are sampled in the order of $f_{\tau_{k,0}}(\cdot),f_{\tau_{k,1}}(\cdot), \cdots, f_{\tau_{k,n-1}}(\cdot)$. 
  
  Since Algorithm \ref{algorithm} is sampling without replacement, there is an implicit bijective mapping between $\{f_i(\cdot)\}_{i=0}^{n-1}$ and $\{f_{\tau_{k,i}}(\cdot)\}_{i=0}^{n-1}$. Further, for a fixed $x$, we have the 
following useful property: 

\begin{equation}\label{prop_t}
  \sum_{i=0}^{n-1} \partialtauki(x)= \partial_l f(x).
\end{equation}

  
  
  

  \item {\bf Index for $m$ and $v$:} As shown in Algorithm \ref{algorithm}, we denote :  $$ \mlki= (1-\beta_1)\{\partialtauki(x_{k,i})+ \cdots + \beta_1^{i}\partial_l f_{\tau_{k,0}}(x_{k,0}) \}+\beta_1^{i+1}\mlkz,$$
 $$ \vlki= (1-\beta_2)\{\partialtauki(x_{k,i})^2+ \cdots + \beta_1^{i}\partial_l f_{\tau_{k,0}}(x_{k,0})^2 \}+\beta_2^{i+1}\vlkz.$$ 
 
\end{itemize}

\subsection{Some useful lemmas}

We begin with proving some useful lemmas. 



\begin{lemma} \label{lemma_beta}
  For any $\beta \neq 1$, we have
  $$\left(1-\beta\right) \sum_{j=1}^{\infty} \beta^{j-1}=1,$$
$$\left(1-\beta\right) \sum_{j=1}^{\infty} j \beta^{j-1}=\frac{1}{1-\beta},$$
$$\left(1-\beta\right) \sum_{j=1}^{\infty} j^{2} \beta^{j-1}=\frac{1+\beta}{\left(1-\beta\right)^{2}}.$$
\end{lemma}

\begin{proof}
  Proof only involves basic calculation, we omit the proof here.
\end{proof}

\begin{lemma}\label{lemma_delta}
  Under Assumption \ref{assum1}, if $\beta_{1}<\sqrt{\beta_{2}}$,  we have
  \begin{equation}
    \left| \partial_{l} f_{\tau_{k, i}}(x_{k,i+1}) -\partial_{l} f_{\tau_{k, i}}(x_{k,i})\right| \leq \frac{\eta_{0} }{\sqrt{nk}}\frac{L\sqrt{d}}{\sqrt{1-\beta_{2}}} \frac{1-\beta_{1} }{1-\frac{\beta_{1}}{\sqrt{\beta_{2}}}} = \triangle_{nk}.
  \end{equation}
  
\end{lemma}

\begin{proof}
  We start with bounding  $\left| x_{l,k,i+1} - x_{l,k,i}\right| $. By the update formula of Adam, we have

{\small
  \begin{eqnarray}
    \left| x_{l,k,i+1} - x_{l,k,i}\right|  &=& \eta_k \frac{|\mlki|}{\sqrt{\vlki}}  
     \nonumber \\
    &\leq &\eta_k (1-\beta_1)\left\{ \beta_1^0\frac{|\partialtauki(\xki)|}{\sqrt{\vlki}} +  \beta_1\frac{|\partial_l f_{\tau_{k,i-1}}(x_{k,i-1})|}{\sqrt{\vlki}} + \ldots \right\}\nonumber  \\
    & \leq & \eta_k(1-\beta_1)\left\{ \beta_1^0\frac{|\partialtauki(\xki)|}{\sqrt{(1-\beta_2)\beta_2^0}|\partialtauki(\xki)|} +  \beta_1\frac{|\partial_l f_{\tau_{k,i-1}}(x_{k,i-1})|}{\sqrt{(1-\beta_2)\beta_2}|\partial_l f_{\tau_{k,i-1}}(x_{k,i-1})|} + \ldots \right\}  \nonumber \\
    &\leq& \eta_k\frac{(1-\beta_1)}{\sqrt{1-\beta_2}} \left(  (\frac{\beta_1}{\sqrt{\beta_2}})^0 + (\frac{\beta_1}{\sqrt{\beta_2}})^1+(\frac{\beta_1}{\sqrt{\beta_2}})^2 +\ldots\right)\nonumber  \\
    &=& \eta_k\frac{(1-\beta_1)}{\sqrt{1-\beta_2}} \frac{1}{1-\frac{\beta_1}{\sqrt{\beta_2}}}.  \label{eq:m_over_v}
  \end{eqnarray}
}

  Now, by Assumption \ref{assum1}, we have

  \begin{eqnarray}
    \left| \partial_{l} f_{\tau_{k, i}}(x_{k,i+1}) -\partial_{l} f_{\tau_{k, i}}(x_{k,i})\right|
    & \leq & \|\nabla f_{\tau_{k, i}}(x_{k,i+1}) 
    - \nabla f_{\tau_{k, i}}
    (x_{k,i}) \|_2 \nonumber\\
    & \leq & L \|x_{k,i+1} -x_{k,i}\|_2 \nonumber \\
    & \leq & L \sqrt{d} \max_{l}\left| x_{l,k,i+1} - x_{l,k,i}\right| \nonumber \\
    & \leq & \eta_{k} \frac{L \sqrt{d}}{\sqrt{1-\beta_{2}}} \frac{1-\beta_{1} }{1-\frac{\beta_{1}}{\sqrt{\beta_{2}}}} \nonumber \\
    &=& \frac{\eta_{0} }{\sqrt{nk}}\frac{L\sqrt{d}}{\sqrt{1-\beta_{2}}} \frac{1-\beta_{1} }{1-\frac{\beta_{1}}{\sqrt{\beta_{2}}}}  
  \end{eqnarray}

Proof is completed.

\end{proof}

\begin{lemma} \label{lemma_fi_f}
Under Assumption \ref{assum1} and \ref{assum2}, for any integer $i\in [0, n-1]$, we have the following two results,

\begin{equation}
  |\partial_l f_{\tau_{k,i}}(x_{k,i})|  \leq i \triangle_{nk} + \sqrt{D_{1}} \rho_{1} d\left(\left|\partial_{\alpha} f\left(x_{k, 0}\right)\right|+\sqrt{\frac{D_{0}}{D_{1} d}}\right),  \label{eq_fi_f_1}
\end{equation}

\begin{equation} 
  |\partial_l f_{\tau_{k,i}}(x_{k,0}) | \leq \sum_{i=0}^{n-1}  |\partial_l f_{\tau_{k,i}}(x_{k,0}) |  \leq  \sum_{l=1}^d\sum_{i=0}^{n-1}  |\partial_l f_{\tau_{k,i}}(x_{k,0}) |  \leq    \sqrt{D_{1}} \rho_{1} d\left(\left|\partial_{\alpha} f\left(x_{k, 0}\right)\right|+\sqrt{\frac{D_{0}}{D_{1} d}}\right),  \label{eq_fi_f_2}
\end{equation}

where $\left|\partial_{\alpha} f\left(x_{k, 0}\right)\right| = \max_{l\in [1, d]} \left|\partial_{l} f\left(x_{k, 0}\right)\right|$.

\end{lemma}

\begin{proof}

  \begin{eqnarray}
    |\partial_l f_{\tau_{k,j}}(x_{k,j}) |  &\overset{\text{Lemma }\ref{lemma_delta}}{\leq}& j \triangle_{nk} +  |\partial_l f_{\tau_{k,j}}(x_{k,0}) |  \nonumber \\
    &\leq & j \triangle_{nk} + \sum_{l=1}^d \sum_{j=0}^{n-1} |\partial_l f_{\tau_{k,j}}(x_{k,0}) |  \nonumber \\  
    &\overset{(*)}{\leq } &j \triangle_{nk} +  \rho_{1} \sum_{l=1}^d \sqrt{\sum_{i=0}^{n-1}\left|\partial_l f_i (x_{k,0}) \right|^{2}}\nonumber  \\
    & \overset{\text{Cauchy-Swartz inequality}}{\leq} & j \triangle_{nk} + \rho_{1} \sqrt{d} \sqrt{ \sum_{l=1}^d \sum_{i=0}^{n-1}\left|\partial_l f_i (x_{k,0}) \right|^{2}} \nonumber  \\
    &\overset{\text{Assumption \ref
    {assum2}}}{\leq }& j \triangle_{nk} + \rho_{1} \sqrt{d}  \sqrt{D_1 \|\nabla f(x_{k,0})\|_2^2 + D_0} \nonumber \\
    &  \overset{ (**)}{\leq}   & j \triangle_{nk} + \sqrt{D_{1}} \rho_{1} d \sqrt{\left|\partial_{\alpha} f\left(x_{k, 0}\right)\right|^{2}+\frac{D_{0}}{D_{1} d}} \nonumber \\
    & \overset{}{\leq} & j \triangle_{nk} + \sqrt{D_{1}} \rho_{1} d\left(\left|\partial_{\alpha} f\left(x_{k, 0}\right)\right|+\sqrt{\frac{D_{0}}{D_{1} d}}\right)   \nonumber
  \end{eqnarray}
  
  where $(*)$ is because of the definition (see Appendix \ref{sec_notations}):  $\rho_1$ is a constant satisfying $\rho_{1} \geq \frac{\sum_{i=1}^{n}\left|\partial_l f_i (x_{k,0})\right|}{\sqrt{\sum_{i=1}^{n}\left| \partial_l f_i (x_{k,0})\right|^{2}}}$;
  $(**)$ is due to
  $\|\nabla f (x_{k,0})\|_{2}^{2} \leq d\left|\partial_{\alpha} f\left(x_{k, 0}\right)\right|^{2}$. 
  The proof of \eqref{eq_fi_f_2} follows the same  procedure as above.
\end{proof}

\clearpage
\section{Proof of Theorem \ref
{thm1} }
\label{appendix:thm1}

\subsection{A roadmap of the proof}
\label{appendix:roadmap}

Our proof is based on the Descent Lemma:

\begin{eqnarray*} 
  \Ex f(x_{k+1,0})- \Ex f(x_{k,0}) &\leq&  \Ex \langle \nabla f(x_{k,0}), x_{k+1,0}-x_{k,0}  \rangle+  \frac{L}{2}\Ex\|x_{k+1,0}-x_{k,0}\|_2^2 
\end{eqnarray*}

The expectation $\Ex(\cdot)$ is taken on the whole trajectory. Summing both sides from the initialization $k=t_{0}$ to $k=T$, we have the following re-arranged inequality: (usually we set $t_0 =1$.)

{\small
\begin{eqnarray} \label{eq_descentlemma}
   \Ex \sum_{k=t_{0}}^T \langle \nabla f(x_{k,0}), x_{k,0}-x_{k+1,0}  \rangle&\leq&   \frac{L}{2} \sum_{k=t_0}^T \Ex \|x_{k+1,0}-x_{k,0}\|_2^2  + \Ex f(x_{t_{0},0})- \Ex f(x_{T+1,0}).
\end{eqnarray}
}

To prove the convergence, we need an upper bound for $ \frac{L}{2}\Ex \|x_{k+1,0}-x_{k,0}\|_2^2$, as well as a lower bound for $ \Ex \langle \nabla f(x_{k,0}), x_{k,0}-x_{k+1,0}  \rangle$ (and such a lower bound should be in the order of  $\frac{1}{\sqrt{k}} \Ex\|\nabla f(x_{k,0})\|$).  

The upper bound for  $\frac{L}{2}\Ex \|x_{k+1,0}-x_{k,0}\|_2^2$ is relatively easy to get: according to the update rule of Adam, this term is in the order of $\mathcal{O}(\eta_k^2 \mlki/\vlki)=\mathcal{O}(\eta_k^2\|\nabla f(x_{k,0})\|/\|\nabla f(x_{k,0})\|)= \mathcal{O}(\frac{1}{k})$. Further, recall $\sum_{k=t_{0}}^{T} \frac{1}{k} \leq \log \frac{T+1}{t_{0}}$, so  $\frac{L}{2} \Ex \|x_{k+1,0}-x_{k,0}\|_2^2$ contributes to the $\log$ term in Theorem \ref{thm1}.  The proof is shown in Lemma \ref{lemma_delta}.

However, the lower bound for $ \Ex \langle \nabla f(x_{k,0}), x_{k,0}-x_{k+1,0}  \rangle=  \Ex \left[\sum_{l=1}^d \sum_{i=0}^{n-1} \partial_l f(\xkz)\frac{\mlki}{\sqrt{\vlki}} \right]$  requires sophisticated derivation, we explain as follows. Before taking the expectation, we first work on every possible realization of $\sum_{l=1}^d \sum_{i=0}^{n-1} \partial_l f(\xkz)\frac{\mlki}{\sqrt{\vlki}}$.
We perform the following decomposition for every $l \in [1,d]$.

{\small
\begin{eqnarray} 
\sum_{i=0}^{n-1} \partial_l f(\xkz)\frac{\mlki}{\sqrt{\vlki}} 
  & = &\sum_{i=0}^{n-1} \frac{\partial_l f(\xkz)}{\sqrt{\vlki}} \left(\partial_l f_i(\xkz)  + \mlki- \partial_l f_i(\xkz)\right)  \nonumber \\
  & = &  \underbrace{ \left[ \sum_{i=0}^{n-1} \frac{\partial_l f(\xkz)}{\sqrt{\vlki}}\partial_l f_i(\xkz) \right]}_{(a)}+ \underbrace{   \left[\partial_l f(\xkz)  \sum_{i=0}^{n-1}   \frac{1}{{\sqrt{\vlki}} } (\mlki-\partial_l f_i(\xkz))   \right]   }_{(b)}. \nonumber  \\
\label{proofsketch_descent}
\end{eqnarray}
}


 


First of all, we introduce the following Lemma \ref{lemma_error_beta2} to further decompose $(a)$ and $(b)$. The intuition of this decomposition is as follows: by increasing $\beta_2$, we can control the movement of the moving average factor $v_{l,k,i}$ (similar idea as  \citep{shi2020rmsprop,zou2019sufficient,chen2021towards}). 

\begin{lemma} \label{lemma_error_beta2}
  Under Assumption \ref{assum1}, for those $l$ satisfying $\max_i |\partial_l f_i(x_{k,0})| \geq Q_k:=\triangle_1 \frac{n\sqrt{n}}{\sqrt{k}} \frac{32\sqrt{2}}{(1-\btwo)^n \btwo^n}$, 
  we have 
  the following lower bound for $(a)$ in \eqref{proofsketch_descent}: 

  \begin{eqnarray}
    \sum_{i=0}^{n-1} \frac{\partial_l f(\xkz)}{\sqrt{\vlki}}\partial_l f_i(\xkz) &\geq& 
    \frac{\partial_{l} f\left(x_{k, 0}\right)^{2}}{\sqrt{v_{l, k, 0}}}-\delta_{1} \left|\frac{\partial_l f(\xkz)}{\sqrt{\vlkz}} \right|\sum_{i}\left|\partial_{l} f_i\left(x_{k, 0} \right)\right|,  \label{eq_a}
  \end{eqnarray}

  where $\delta_{1}=\frac{(1-\beta_2)4n \rho_2  }{\beta_2^n}+\left(\frac{1}{\sqrt{\beta_{2}^{n}}}-1\right)$.  $\rho_{2}, \rho_3$ are constants satisfying $\rho_{2} \geq \frac{\left|\max_i \partial_l f_i (x_{k,0})\right|^{2}}{\frac{1}{n} \sum_{i=1}^{n}\left|\partial_l f_i (x_{k,0})\right|^{2}},$ $\rho_{3} \geq \frac{\left|\sum_{i=1}^{n} \partial_{l} f_{i}\left(x_{k, 0}\right)\right|}{\sqrt{\frac{1}{n} \sum_{i=1}^{n}\left|\partial_{l} f_{i}\left(x_{k, 0}\right)\right|^{2}}}$. Note that $\delta_1 \rightarrow 0$ when $\beta_2 \rightarrow 1$.

  Under the same condition, we also have a lower bound for $(b)$ in \eqref{proofsketch_descent}: 
  
  \begin{eqnarray} \label{eq_b}
    &&\sum_{i=0}^{n-1}   \frac{\partial_{l} f\left(x_{k, 0}\right)}{{\sqrt{\vlki}} } (\mlki-  \partial_{l} f_i(\xkz) )  \nonumber\\
      &\geq& \frac{\partial_l f(\xkz)}{\sqrt{\vlkz}} \sum_{i=0}^{n-1}  \left(\mlki-  \partial_{l} f_i(\xkz) \right)  -\delta_1\left|\frac{\partial_l f(\xkz)}{\sqrt{\vlkz}} \right|\sum_{i=0}^{n-1}|\mlki- \partial_{l} f_i(\xkz)|. 
  \end{eqnarray}

  
\end{lemma}

   The proof can be seen in Appendix \ref{appendix_lemma_error_beta2} and it is motivated from \citep{shi2020rmsprop}. 
Lemma \ref{lemma_error_beta2} decomposes both $(a)$ and $(b)$: in the denominator, we approximate $\vlki$ (which is changing with $i$) by $\vlkz$ (which is fixed). Accordingly, the approximation error can be 
 controlled by increasing $\beta_2$ (since $\delta_1 \rightarrow 0$ when $\beta_2 \rightarrow 1$). However, similiarly as  \citep{shi2020rmsprop}, Lemma \ref{lemma_error_beta2} can only be applied to those $l$ with 
 $\max_i |\partial_l f_i(x_{k,0})| \geq Q_k$, so we need to discuss  two cases as below.

 \paragraph{Case 1: unbounded gradient.} Given $x_{k,0}$, consider those  $l$ with   $\max_i |\partial_l f_i(x_{k,0})| \geq Q_k$, we have the following decomposition:

 {\small
 \begin{eqnarray} 
   \sum_{i=0}^{n-1} \partial_l f(\xkz)\frac{\mlki}{\sqrt{\vlki}} 
  & \overset{\eqref{proofsketch_descent}}{\geq} &  \underbrace{ \left[ \sum_{i=0}^{n-1} \frac{\partial_l f(\xkz)}{\sqrt{\vlki}}\partial_l f_i(\xkz) \right]}_{(a)}+ \underbrace{  \left[\partial_l f(\xkz)  \sum_{i=0}^{n-1}   \frac{1}{{\sqrt{\vlki}} } (\mlki-\partial_l f_i(\xkz))   \right]   }_{(b)}. \nonumber  
  \end{eqnarray}
    \begin{eqnarray}
  \quad \quad \quad \quad \quad \quad \quad \quad&\overset{\text{Lemma \ref{lemma_error_beta2}}}{\geq}&
    \frac{\partial_{l} f\left(x_{k, 0}\right)^{2}}{\sqrt{v_{l, k, 0}}}  -   \delta_{1} \left|\frac{\partial_l f(\xkz)}{\sqrt{\vlkz}} \right|  \sum_{i=0}^{n-1}\left|\partial_{l} f_i\left(x_{k, 0} \right)\right| \nonumber \\
    && +  \frac{\partial_l f(\xkz)}{\sqrt{\vlkz}}\sum_{i=0}^{n-1}  \left(\mlki-  \partial_{l} f_i(\xkz) \right)   -  \delta_1\left|\frac{\partial_l f(\xkz)}{\sqrt{\vlkz}} \right| \sum_{i=0}^{n-1}|\mlki- \partial_{l} f_i(\xkz)|  \nonumber 
    \end{eqnarray}
    \begin{eqnarray}
     \quad \quad \quad \quad \quad \quad \quad \quad&\overset{\text{Lemma \ref{lemma_f1}}}{\geq }&
    \underbrace{ \left[ \frac{\partial_{l} f\left(x_{k, 0}\right)^{2}}{\sqrt{v_{l, k, 0}}} \right] }_{(a_1)}  -  \underbrace{\left[ \delta_{1} \sqrt{\frac{2\rho_3^2}{\beta_2^n} } \sum_{i=0}^{n-1}\left|\partial_{l} f_i\left(x_{k, 0} \right)\right|\right] }_{(a_2)}\nonumber \\
    &&+\underbrace{ \left[ \frac{\partial_l f(\xkz)}{\sqrt{\vlkz}} \sum_{i=0}^{n-1}  \left(\mlki-  \partial_{l} f_i(\xkz) \right) \right] }_{(b_1)} - \underbrace{ \left[\delta_1\sqrt{\frac{2\rho_3^2}{\beta_2^n} }  \sum_{i=0}^{n-1}|\mlki- \partial_{l} f_i(\xkz)| \right]}_{(b_2)}, 
    \nonumber \\
    \label{eq_a1a2b1b2}
\end{eqnarray}
  
 }

 \paragraph{Case 2: bounded gradient.}  Given $x_{k,0}$, consider those  $l$ with  $\max_i |\partial_l f_i(x_{k,0})| \leq Q_k$, the analysis degenerates to the ``bounded gradient'' scenario, we have the following lower bound:

 \begin{eqnarray} 
\sum_{i=0}^{n-1} \partial_l f(\xkz)\frac{\mlki}{\sqrt{\vlki}} 
& \geq &  - \sum_{i=0}^{n-1} |\partial_l f(\xkz) | \left|\frac{\mlki}{\sqrt{\vlki}}\right|  \nonumber \\
& \overset{\eqref{eq:m_over_v}}{\geq} & - |\partial_l f(\xkz) | \frac{1-\beta_1}{\sqrt{1-\beta_2}}\frac{1}{1-\frac{\beta_1}{\sqrt{\beta_2}}} n   \nonumber \\
&\overset{(*) }{\geq }& -n Q_k  \frac{1-\beta_1}{\sqrt{1-\beta_2}}\frac{1}{1-\frac{\beta_1}{\sqrt{\beta_2}}} n \nonumber \\
&= &  -\triangle_1 \frac{n^2\sqrt{n}}{\sqrt{k}} \frac{32\sqrt{2}}{(1-\btwo)^n \btwo^n}  \frac{1-\beta_1}{\sqrt{1-\beta_2}}\frac{1}{1-\frac{\beta_1}{\sqrt{\beta_2}}} n  \nonumber \\
&:=& -F_1 \frac{1}{\sqrt{k}}, \nonumber
\end{eqnarray}

where $F_1:= \triangle_1 n^2\sqrt{n} \frac{32\sqrt{2}}{(1-\btwo)^n \btwo^n}  \frac{1-\beta_1}{\sqrt{1-\beta_2}}\frac{1}{1-\frac{\beta_1}{\sqrt{\beta_2}}} n$, $(*)$ is due to the fact that $|\partial_l f(\xkz) | \leq n \max_i |\partial_l f_i(x_{k,0})| \leq n Q_k$.

Combining {\bf Case 1 \& 2} together, we have the following result (note that $(a_1), (a_2), (b_1), (b_2)$  vary with $l$. A more precise notation will be $(a_1)_l$, etc. We drop the subscript for brevity). 

{\small
\begin{eqnarray}
   \sum_{l=1}^d\sum_{i=0}^{n-1} \partial_l f(\xkz)\frac{\mlki}{\sqrt{\vlki}} &\geq & \underbrace{\sum_{l \text{ large}} \left\{ (a_1) -(a_2) +(b_1) - (b_2) \right\} }_{\text{unbounded gradient }}- \underbrace{\sum_{l \text{ small }}  F_1 \frac{1}{\sqrt{k}}}_{\text{bounded gradient }} \nonumber\\ 
     &\geq & \underbrace{\sum_{l \text{ large}} \left\{ (a_1) -(a_2) +(b_1) - (b_2) \right\} }_{\text{unbounded gradient }}-  \underbrace{\sum_{l=1}^d \left\{ F_1 \frac{1}{\sqrt{k}}   \right\}}_{\text{bounded gradient }}  \nonumber \\
  &\geq & \left\{\sum_{l \text{ large}} (a_1) \right\}  + \left\{ \sum_{l \text{ large} } (b_1)\right\}- \left\{  \sum_{l \text{ large }} \{(a_2) +  (b_2) \}\right\} - d F_1 \frac{1}{\sqrt{k}}, \nonumber\\     \label{eq_zhangdecomposition1}
\end{eqnarray}
}
 
where ``$l \text{ large}$'' stands for the gradient component in {\bf Case 1} and ``$l \text{ small}$'' indicates those in {\bf Case 2}. We assume ``$l \text{ large}$'' is not an empty set, otherwise the analysis degenerates to the case with bounded gradient assumption. 

Now we take expectation on \eqref{eq_zhangdecomposition1}. The expectation is taken on all the possible trajectories up to the $i$-th iteration in $k$-th epoch.


{\small
\begin{eqnarray}
   \Ex \left[ \sum_{l=1}^d\sum_{i=0}^{n-1} \partial_l f(\xkz)\frac{\mlki}{\sqrt{\vlki}} \right] 
   &\overset{\eqref{eq_zhangdecomposition1}}{\geq} & \Ex \left\{\sum_{l \text{ large}} (a_1)   +  \sum_{l \text{ large}} (b_1) -  \sum_{l \text{ large }} \{(a_2) +  (b_2) \}\right\} - d F_1 \frac{1}{\sqrt{k}}   , \nonumber \\ \label{eq_zhangdecomposition}
 \end{eqnarray}
}

In the following context, we will discuss how to bound all terms  in \eqref{eq_zhangdecomposition}, respectively.  First and foremost, we derive a lower bound for $\Ex \left[\sum_{l \textbf{ large}} (b_1) \right]$. Since  $\Ex \left[\sum_{l \textbf{ large}} (b_1) \right]$ contains all the historical gradient information, it involves great effort to handle it. We will show that the lower bound of $\Ex \left[\sum_{l \textbf{ large}} (b_1) \right]$ will  vanish when $\beta_2$ is large and $k$ goes to infinity.


Recall $ \Ex \left[\sum_{l \text{ large}}(b_1)\right]   = \Ex \left[\sum_{l \text{ large}} \frac{\partial_l f(\xkz)}{\sqrt{\vlkz}} \sum_{i=0}^{n-1}  \left(\mlki-  \partial_{l} f_i(\xkz) \right) \right]$. When $\beta_1$ is large, $\mlki$ contains heavy historical signals.  It seems unclear how large $\Ex \left[\sum_{l \text{ large}}(b_1)\right] $ would be  when $\beta_1$ goes to 1.
Existing literatures \citep{Manzil2018adaptive,de2018convergence,shi2020rmsprop} take a naive approach: 
	they set  $\beta_1 \approx 0$ so that  $m_{k,i} \approx  \nabla f_{\tau_{k,i}} (x_{k,i})$. Then we get $\delta(\beta_1)  \approx 0$. 
	However, this naive method cannot be applied here since we are interested in practical cases where $\beta_1$ is large in $[0,1)$. We emphasize the  following technical difficulties in bounding $\Ex \left[\sum_{l \text{ large}} (b_1) \right]$ for any $\beta_1 \in [0,1)$:

\begin{itemize}
    \item{\bf Issue (i)} In order to bound  $ \Ex \left[\sum_{l \text{ large}}(b_1)\right]  $, we need to first know how to bound its simpler version: $ \Ex \left[ \sum_{i=0}^{n-1}  \left(\mlki-  \partial_{l} f_i(\xkz) \right) \right]$. This term measures the difference between the current gradient and weighted historical gradients. It seems unclear that how large this term could be when $\beta_1$ goes to 1.  
    \item{\bf Issue (ii)} Even if we can bound $ \Ex \left[ \sum_{i=0}^{n-1}  \left(\mlki-  \partial_{l} f_i(\xkz) \right) \right]$, it is still unclear how to bound  $ \Ex \left[\frac{\partial_l f(\xkz)}{\sqrt{\vlkz}} \sum_{i=0}^{n-1}  \left(\mlki-  \partial_{l} f_i(\xkz) \right) \right]$, which is further multiplied by  a random variable $\frac{\partial_l f(\xkz)}{\sqrt{\vlkz}} $.
    \item{\bf Issue (iii)} Even if we can bound   $ \Ex \left[\frac{\partial_l f(\xkz)}{\sqrt{\vlkz}} \sum_{i=0}^{n-1}  \left(\mlki-  \partial_{l} f_i(\xkz) \right) \right]$, it is still different from $ \Ex \left[\sum_{l \text{ large}}(b_1)\right]  $ which contains additional operation ``$\sum_{l \text{ large}}$" inside the expectation. Note that the set ``$l \text{ large}$" is a random variable which changes along different trajectories, so there is still non-negligible gap between bounding $ \Ex \left[\sum_{l \text{ large}}(b_1)\right]  $ and $ \Ex \left[\frac{\partial_l f(\xkz)}{\sqrt{\vlkz}} \sum_{i=0}^{n-1}  \left(\mlki-  \partial_{l} f_i(\xkz) \right) \right]$.
\end{itemize}

To our best knowledge, there is no general approach to tackle the above issues. 
In the following content, we will overcome difficulties (i), (ii) and (iii) in Step 1, 2 and 3 respectively.  In Step 1, we will discuss how to bound $ \Ex \left[ \sum_{i=0}^{n-1}  \left(\mlki-  \partial_{l} f_i(\xkz) \right) \right]$, which is a simplified version of $\Ex \left[(b_1) \right]$. Bounding this term will shed light on bounding the whole term $ \Ex \left[\sum_{l \text{ large}}(b_1)\right]  $. Then in Step 2 and 3, we will prove several technical lemmas to handle the effect of $\frac{\partial_l f(\xkz)}{\sqrt{\vlkz}}$ and ``$\sum_{l \text{ large}}$", by which we can tackle {\bf (ii)} and {\bf (iii)}. Combining all together we can bound $ \Ex \left[\sum_{l \text{ large}}(b_1)\right]  $.

\paragraph{Bounding $ \Ex \left[\sum_{l \text{ large}}(b_1)\right]  $: Step 1.} We now introduce the key idea of bounding  $ \Ex \left[ \sum_{i=0}^{n-1}  \left(\mlki-  \partial_{l} f_i(\xkz) \right) \right]$. 
In Appendix \ref{appendix:sketch}, we distill our idea into a toy example called "the color-ball model of the 1st kind" and thus we prove Lemma \ref{lemma_toy_2}. This lemma is crucial for bounding $ \Ex \left[ \sum_{i=0}^{n-1}  \left(\mlki-  \partial_{l} f_i(\xkz) \right) \right]$. We refer the readers to Appendix \ref{appendix:sketch} for more explanation.

Lemma \ref{lemma_toy_2} can provide insights in bounding  $\Ex\left[ \sum_{i=0}^{n-1}  \left(\mlki-  \partial_{l} f_i(\xkz) \right) \right]$.  
However, there is still certain gap between $\Ex\left[ \sum_{i=0}^{n-1}  \left(\mlki-  \partial_{l} f_i(\xkz) \right) \right]$ and the quantities in Lemma \ref{lemma_toy_2}. 
We elaborate as follows:
\begin{itemize}
  \item To mimic the color-ball example, we need to expand $\sum_{i=0}^{n-1}\partial_lf_i(\xkz)$ into an {\it infinite} sum sequence: $\sum_{i=0}^{n-1}\partial_lf_i(\xkz) (1-\beta_1) (1+\beta_1+\cdots \beta_1^\infty) $. However, $\sum_{i=0}^{n-1} \mlki $  is a {\it finite} sum  sequence up to the order of $\beta^{kn}$. In contrast, both sequences in the color-ball example are ``finite sum''. As such, there is an error term caused by ``finite sum v.s. infinite sum''. 
  \item When taking the expectation, the variable $x$ in each possible trajectory is different. In contrast, in the color-ball example, $\{a_i\}_{i=0}^2$ are fixed in all shuffling order (so it is much easier to calculate the expectation by summing them up).
  
  To mimic the color-ball example, we repeatedly take conditional expectation at the beginning of each $k$-th epoch. In this way, $x_{k,0}$ will be fix. Despite $x_{k,i}$ is still changing across the trajectory, we can transform $x_{k,i}$  into $x_{k,0}$ by using Lipschitz property.
  \item In each trajectory of $\sum_{i=0}^{n-1} \mlki $, the variable $x$ in the summand of $\mlki$ varies with $k$ and $i$; while the  variable in $\partial_{l} f_i(\xkz)$  is fixed to be $\xkz$. In contrast, in the color-ball example, $\{a_i\}_{i=0}^2$ are the same across the epoch.

  To mimic the final step in the color-ball example (Figure \ref{fig:toy_2_cancel}), at each step of conditional expectation, we need to simultaneously move the variables in $\mlki$ and $\partial_{l} f_i(\xkz)$ (using Lipschitz property) so that they can match and cancel out with each other. This operation will introduce new error terms and it is our duty to put them under control.
\end{itemize}

We omit the proof for bounding $\Ex\left[ \sum_{i=0}^{n-1}  \left(\mlki-  \partial_{l} f_i(\xkz) \right) \right]$ since this is not our actual goal. Instead, we will directly use the above ideas to bound $\Ex\left[\left(b_1 \right) \right]$. To do so, we  need to further tackle the issue {\bf (ii)}  and {\bf (iii)} mentioned before. We explain as follows.

\paragraph{Bounding $ \Ex \left[\sum_{l \text{ large}}(b_1)\right]  $: Step 2.}  


We now resolve issue  {\bf (ii)}, i.e., handle the effect of $\frac{\partial_l f(\xkz)}{\sqrt{\vlkz}}$. The key idea is as follows: when we calculate $\Ex\left[ (b_1) \right]$, we sequentially take conditional expectation $\Ext (\cdot)$, $\Ex_{k-1}(\cdot)$, etc.. When taking  $\Ext(\cdot)$, we will fix all the historical information up to $k$-th epoch, so  $\frac{\partial_l f(\xkz)}{\sqrt{\vlkz}}$ can be regarded as a constant. In this sense,    $\Ext\left[ \frac{\partial_l f(\xkz)}{\sqrt{\vlkz}} \sum_{i=0}^{n-1}  \left(\mlki-  \partial_{l} f_i(\xkz) \right) \right]$ can be calculated following the same idea as color-ball toy example.

However, when taking $\Ex_{k-1}(\cdot)$, $\frac{\partial_l f(\xkz)}{\sqrt{\vlkz}}$ will become a random variable which changes with different trajectories. 
In this case, the color-ball method {\it cannot} be applied. To fix this issue, we introduce the following lemma to change $\frac{\partial_l f(\xkz)}{\sqrt{\vlkz}}$ into $\frac{\partial_l f(x_{k-1,0})}{\sqrt{v_{k-1,0}}}$, which can again be regarded as a fixed constant when taking $\Ex_{k-1}(\cdot)$.

\begin{lemma} \label{lemma_k-k-1}
  Suppose Assumption \ref{assum1} holds and $\beta_1 < \sqrt{\beta_2}$.
  For any integer $j \in [0, k]$, if $\max_i |\partial_l f_i(x_{k,0})| \geq Q_k$,  $\max_i |\partial_l f_i(x_{k-1,0})| \geq Q_{k-1}$, $\cdots$, $\max_i |\partial_l f_i(x_{k-j,0})| \geq Q_{k-j}$ (where $Q_k:=\triangle_1 \frac{n\sqrt{n}}{\sqrt{k}} \frac{32\sqrt{2}}{(1-\btwo)^n \btwo^n}$),
  then we have  the following result:

\begin{eqnarray*}
  \left| \frac{\partial_l f(x_{k,0})}{\sqrt{v_{l,k,0}}} -   \frac{\partial_l f(x_{k-j,0})}{\sqrt{v_{l,k-j,0}}} \right| &\leq&  \frac{1}{1-\frac{1}{\sqrt{\beta_2^n}}} \frac{n^2\triangle_{n(k-j)}}{\sqrt{v_{l,k-j,0}}} +j\sqrt{\frac{2\rho_3^2}{\beta_2^n}}  \frac{1}{\left(1-\frac{(1-\beta_2)4n \rho_2  }{\beta_2^n} \right)}  \delta_1 ,
\end{eqnarray*}
    where $\delta_1$ is defined in \eqref{eq_delta1}.

\end{lemma}

The proof of Lemma \ref{lemma_k-k-1} can be seen in Appendix \ref{appendix_lemma_k-k-1}. To proceed, we combine Lemma \ref{lemma_k-k-1} and the ``color-ball method of the 2nd kind", and thus we can prove Lemma \ref{lemma_toy_new}. This lemma is crucial for our current goal: bounding  $\Ex (b_1)$. We refer the readers to Appendix \ref{appendix:sketch} for more information. 

We emphasize that here are still the following gap between Lemma \ref{lemma_toy_new} and our goal $\Ex (b_1)$. 

\begin{itemize}
    \item We have the similar gap as discussed at the end of the {\bf Step 1}.
    \item  The condition in Lemma \ref{lemma_k-k-1} has requirement on the gradient norm, while this requirement is temporarily ignored in the  color-ball method of the 2nd kind.
    \item  The result  in Lemma \ref{lemma_k-k-1} has additional error terms other than $\mathcal{O}(1/\sqrt{k})$. This is slightly different from the setting in  Lemma \ref{lemma_toy_new}.
\end{itemize}

It requires some technical lemmas to handle these gaps. We fill in these gaps in Lemma \ref{lemma_b1}. The technical details can be seen Appendix \ref{appendix_lemma_b1}. 

\paragraph{Bounding $ \Ex \left[\sum_{l \text{ large}}(b_1)\right]  $: Step 3.} 
Now we shift gear to tackle (iii): handling the random variable `` $l \text{ large}$". We rewrite `` $l \text{ large}$" into the indicator function as follows:

{\small$$ \Ex \left[\sum_{l \text{ large}}(b_1)\right]    = \Ex\left[ \sum_{l \text{ large}} \frac{\partial_l f(\xkz)}{\sqrt{\vlkz}} \sum_{i=0}^{n-1}  \left(\mlki-  \partial_{l} f_i(\xkz) \right) \right] = \Ex\left[ \sum_{l=1}^d \mathbb{I}_k \frac{\partial_l f(\xkz)}{\sqrt{\vlkz}} \sum_{i=0}^{n-1}  \left(\mlki-  \partial_{l} f_i(\xkz) \right) \right], $$}

where $\mathbb{I}_k  :=\mathbb{I}\left( \max_i |\partial_l f_i(x_{k,0})| \geq Q_k:=\triangle_1 \frac{n\sqrt{n}}{\sqrt{k}} \frac{32\sqrt{2}}{(1-\btwo)^n \btwo^n} \right) $ is the indicator function ($\mathbb{I}(A) =1$ when event $A$ holds and $\mathbb{I}(A) =0$ otherwise.)
Similarly as before, when taking $\Ext (\cdot)$, `` $\mathbb{I}_k$" can be regarded as a constant index. Therefore,    $\Ext\left[ \sum_{l =1}^d \mathbb{I}_k  \frac{\partial_l f(\xkz)}{\sqrt{\vlkz}} \sum_{i=0}^{n-1}  \left(\mlki-  \partial_{l} f_i(\xkz) \right) \right]$ can be calculated following the same idea as color-ball model of the 1st kind.
However, when taking $\Ex_{k-1}(\cdot)$, $\mathbb{I}_k$ will become a random variable which changes with different trajectories. 
In this case, the color-ball method {\it cannot} be applied. Similarly as in {\bf Step 2}, we introduce the following lemma to change  $\mathbb{I}_k$ into  $\mathbb{I}_{k-1}$ (defined later), which can again be regarded as a fixed  when taking $\Ex_{k-1}(\cdot)$.

\begin{lemma} \label{lemma_indicator}
    Suppose Assumption \ref{assum1}  holds and $\beta_1 < \sqrt{\beta_2}$. For $0\leq j \leq k$, we  define $\mathbb{I}_{k-j} : = \mathbb{I}\left( \max_i |\partial_l f_i(x_{k-j,0})| \geq \sum_{p=k-j}^{k}Q_p  \right)$, where $Q_k:=\triangle_1 \frac{n\sqrt{n}}{\sqrt{k}} \frac{32\sqrt{2}}{(1-\btwo)^n \btwo^n}$,
 then we have  the following results.
 
  {\small
  \begin{eqnarray}
    \mathbb{I}_k &=&\mathbb{I}\left( \max_i |\partial_l f_i(x_{k,0})| \geq Q_k \text{ and } \max_i |\partial_l f_i(x_{k-j,0})|  \geq \sum_{p=k-j}^{k}Q_p   \right) \nonumber \\
    &+& \mathbb{I}\left( \max_i |\partial_l f_i(x_{k,0})| \geq Q_k \text{ and } \max_i |\partial_l f_i(x_{k-j,0})|  \leq \sum_{p=k-j}^{k}Q_p   \right), \label{eq_indicator_1}
  \end{eqnarray} }

  {\small
 \begin{equation}
 \mathbb{I}\left( \max_i |\partial_l f_i(x_{k,0})| \geq Q_k \text{ and } \max_i |\partial_l f_i(x_{k-j,0})|  \geq \sum_{p=k-j}^{k}Q_p   \right)  = \mathbb{I}\left( \max_i |\partial_l f_i(x_{k-j,0})|  \geq \sum_{p=k-j}^{k}Q_p  \right) = \mathbb{I}_{k-j}.  
     \label{eq_indicator_2}
 \end{equation} }

\end{lemma}

\begin{proof}
Equation \eqref{eq_indicator_1} is straightforward, we only prove \eqref{eq_indicator_2} here.  Under Assumption \ref{assum1}, we have $\left| \partial_l f_i(x_{k,0}) - \partial_l f_i(x_{k-j,0})\right| \leq n \triangle_{n(k-1)}+ \cdots + n \triangle_{n(k-j)} \overset{}{\leq } Q_{k-1}+ \cdots +  Q_{k-j}$. To show the second inequality, it requires comparing the value between $Q_{k}$ and $n \triangle_{n(k)} $, which are both problem-dependent constants. Here, the inequality holds when  $n \triangle_{n(k)} \leq Q_k $. If otherwise, we can always define  $\tilde{Q}_k := \max \{ Q_k, n \triangle_{n(k)} \}$ and the inequality still holds by changing all the $Q_k$ into $\tilde{Q}_k$. We temporarily omit this step for now. 

Since $\left| \partial_l f_i(x_{k,0}) - \partial_l f_i(x_{k-j,0})\right|  \overset{}{\leq } Q_{k-1}+ \cdots +  Q_{k-j}$, the event $\left\{ \max_i |\partial_l f_i(x_{k-j,0})| \geq \sum_{p=k-j}^{k}Q_p  \right\}$ implies the event $\left\{ \max_i |\partial_l f_i(x_{k,0})| \geq Q_k \right\}$, so the proof is completed.
\end{proof}

Now we are ready to bound $ \Ex \left[\sum_{l \text{ large}}(b_1)\right]  $. Combining {\bf Step 1, 2 and 3} together, we prove the following Lemma \ref{lemma_b1}.

\begin{lemma} \label{lemma_b1}
  Under Assumption \ref{assum1}, consider $\beta_1 < \sqrt{\beta_2}$, when k is large such that: $k \geq 4$; $\beta_1^{(k-1)n} \leq \frac{\beta_1^n}{\sqrt{k-1}},$  we have the following result:

\begin{eqnarray*}
  \Ex \left[\sum_{l \text{ large}}(b_1)\right]  \geq -\mathcal{O}\left(\frac{1}{\sqrt{k}}\right) - \mathcal{O}\left( \delta_1 \Ex\left[\sum_{l=1}^d\sum_{i=0}^{n-1} \left|\partial_l f_i(x_{k,0})   \right| \right] \right),
\end{eqnarray*}
where $\delta_1 =\frac{(1-\beta_2)4n \rho_2  }{\beta_2^n}+\left(\frac{1}{\sqrt{\beta_{2}^{n}}}-1\right)$, which goes to 0 when $\beta_2$ goes to 1. 










\end{lemma}

Proof can be seen in Appendix \ref{appendix_lemma_b1}. 
Now we bound the error terms: $\Ex\left[\sum_{l \text{ large }} \{  (a_2)+(b_2) \}\right]$.  
Since they are multiplied by $\delta_1$, these two terms vanish when $\beta_2 \rightarrow 1$.

\paragraph{Bounding  $\Ex\left[\sum_{l \text{ large }} \{(a_2) +  (b_2) \}\right]$.} We bound these two terms in the following Lemma \ref{lemma_da2_db2}.

\begin{lemma}\label{lemma_da2_db2}
Given all the history up to $x_{k,0}$, we denote  $\alpha$ as the index of the coordinate with the greatest gradient:
  $
  \alpha=\arg \max _{l=1,2, \cdots, d}\left|\partial_{l} f\left(x_{k, 0}\right)\right|.
  $
  Under Assumption \ref{assum1} and \ref{assum2}, consider $\beta_1 < \sqrt{\beta_2}$, when k is large such that: $k \geq 4$; $\beta_1^{(k-1)n} \leq \frac{\beta_1^n}{\sqrt{k-1}},$  we have the following results:

{\small
  \begin{eqnarray*}
    \Ex\left[\sum_{l \text{ large }} (a_2) \right] := \Ex \left[ \sum_{l \text{ large }}\delta_{1} \sqrt{\frac{2\rho_3^2}{\beta_2^n} }\sum_{i=0}^{n-1}\left|\partial_{l} f_i\left(x_{k, 0} \right)\right|\right] \overset{\eqref{eq_fi_f_2}}{\leq}   \mathcal{O}\left(\delta_1 \Ex \left|\partial_{\alpha} f\left(x_{k, 0}\right)\right|\right) + \mathcal{O}\left(\delta_1 \sqrt{D_0}\right),
  \end{eqnarray*}
}

{\small
\begin{eqnarray*}
  \Ex\left[\sum_{l \text{ large }} (b_2) \right]:= \Ex \left[ \sum_{l \text{ large }} \delta_1\sqrt{\frac{2\rho_3^2}{\beta_2^n} }\sum_{i=0}^{n-1}|\mlki- \partial_{l} f_i(\xkz)| \right] \leq    \mathcal{O}\left(\delta_1  \Ex \left|\partial_{\alpha} f\left(x_{k, 0}\right)\right|\right) + \mathcal{O}\left(\delta_1 \sqrt{D_0}\right)+  \mathcal{O} \left( \frac{1}{\sqrt{k}}\right),
\end{eqnarray*}
}

where $\delta_1$ is defined in \eqref{eq_delta1}, $D_0$ is defined in Assumption \ref{assum2}.



\end{lemma}

Detailed proof of Lemma \ref{lemma_da2_db2} is shown in Appendix \ref{appendix_lemma_da2_db2}.

\paragraph{Bounding $ \Ex \left\{\sum_{l \text{ large}} (a_1)   +  \sum_{l \text{ large}} (b_1) -  \sum_{l \text{ large }} \{(a_2) +  (b_2) \}\right\}$ .}  
With the help of Lemma \ref{lemma_b1} and  \ref{lemma_da2_db2}, we can bound 
$ \Ex \left\{\sum_{l \text{ large}} (a_1)   +  \sum_{l \text{ large}} (b_1) -  \sum_{l \text{ large }} \{(a_2) +  (b_2) \}\right\}$. 
The intuition is as follows:
\begin{itemize}
  \item $(a_1)$ is in the order of $|\partial_l f(x_{k,0})|$, multiplied by some positive constant;
  \item  $\sum_{l \text{ large}}\{ (a_2)+(b_2) \}$ vanishes when $\beta_2 \rightarrow 1$. 
  \item $\Ex \left\{ \sum_{l \text{ large}} (b_1) \right\}$ vanishes when $\beta_2 \rightarrow 1$ and $k$ goes to infinity.
\end{itemize}
  Therefore, $ \Ex \left\{\sum_{l \text{ large}} (a_1)   +  \sum_{l \text{ large}} (b_1) -  \sum_{l \text{ large }} \{(a_2) +  (b_2) \}\right\}$ is  still in the order of $|\partial_l f(x_{k,0})|$ (multiplied by a positive constant when $\beta_2$ is large). More formal results are shown in Lemma \ref{lemma_a1}.

\begin{lemma} \label{lemma_a1}
  Under Assumption \ref{assum1} and \ref{assum2}, when the hyperparameters satisfy: i) $\beta_1 < \sqrt{\beta_2} <1$,   ii) $\beta_2 $ is large enough such that $A(\beta_2)$ is small enough to satisfy \eqref{condition}, where $A(\beta_2)$ is a non-negative constant  that approaches 0 when $\beta_2$ approaches $1$. More Specifically, $A(\beta_2)$ needs to satisfy ($\rho_1$, $\rho_2$ and $\rho_3$ are defined in Appendix \ref{sec_notations}).

{\tiny
  \begin{eqnarray}\label{condition}
    A(\beta_2) &:=& \left\{ \frac{(1-\beta_2)4n \rho_2  }{\beta_2^n}+\left(\frac{1}{\sqrt{\beta_{2}^{n}}}-1\right) \right\}  \left(\sqrt{\frac{2\rho_3^2}{\beta_2^n} }  4n +     \left(\sqrt{\frac{2\rho_3^2}{\beta_2^n}}  \frac{3n}{\left(1- \frac{(1-\beta_2)4n \rho_2  }{\beta_2^n}\right) }  \right)\frac{1}{(1-\beta_1^n)}  \right)\sqrt{D_{1}} \rho_{1} d \nonumber  \\
    &\leq& \frac{1}{\sqrt{10D_1d}},
  \end{eqnarray} }

  Then, we have the following result when $k$ is large enough such that $\beta_{1}^{(k-1) n} \leq \frac{\beta_{1}^{n}}{\sqrt{k-1}}$ and $k \geq 4$:
  
{\small
  \begin{eqnarray*}
   &&\Ex \left\{\sum_{l \text{ large}} (a_1)   +  \sum_{l \text{ large}} (b_1) -  \sum_{l \text{ large }} \{(a_2) +  (b_2) \}\right\}\\
   &:=&  \mathbb{E}\left[ \sum_{l \text{ large}} \frac{\partial_{l} f\left(x_{k, 0}\right)^{2}}{\sqrt{v_{l, k, 0}}}\right] 
   + \Ex \left[\sum_{l \text{ large}} \frac{\partial_l f(\xkz)}{\sqrt{\vlkz}} \sum_{i=0}^{n-1}  \left(\mlki-  \partial_{l} f_i(\xkz) \right) \right] \\
   &&- d \delta_{1} \sqrt{\frac{2\rho_3^2}{\beta_2^n} }\Ex \left[ \sum_{i=0}^{n-1}\left|\partial_{l} f_i\left(x_{k, 0} \right)\right|\right] - d \delta_1\sqrt{\frac{2\rho_3^2}{\beta_2^n} }\Ex \left[ \sum_{i=0}^{n-1}|\mlki- \partial_{l} f_i(\xkz)| \right]  \\
    &\geq&  \frac{1}{d\sqrt{10D_1 d}}\Ex \min  \left\{ \sqrt{\frac{ 2D_1 d }{D_{0} } }  \|\nabla f(x_{k,0})\|_2^2, \|\nabla f(x_{k,0})\|_1\right\}   \nonumber \\
    &&-\mathcal{O}(\frac{1}{\sqrt{k}})-\mathcal{O}(\sqrt{D_0}).  \nonumber \\
  \end{eqnarray*}}
\end{lemma}

Proof of Lemma \ref{lemma_a1} can be seen in Appendix \ref{appendix_lemma_a1}.

\begin{remark}
Condition \eqref{condition} specifies the smallest threshold of $\beta_2$ to ensure the convergence. This condition can be translated into the threshold funtion $\gamma_1(n)$ mentioned in Theorem \ref{thm1}.
As a rough estimate, Lemma \ref{lemma_a1} requires $\beta_2 \geq \gamma_1(n) = 1- \mathcal{O}\left((1-\beta_1^n)/(n^{2}\rho) \right)$, where $\rho=\rho_1 \rho_2 \rho_3$. As discussed in Appendix \ref{sec_notations}, we have $0 \leq \rho_{3} \leq \sqrt{n} \rho_{1} \leq n$. When $\rho_1, \rho_2, \rho_3$ achieve their upper bound at the same time, we get the worst case bound $\beta_2 \geq \gamma_1(n) =1- \mathcal{O}\left((1-\beta_1^n)/n^{4.5} \right)$. However, $\rho$ is highly dependent on the problem instance $f(x)$ and training process. Our experiments in Appendix \ref{appendix:more_exp}   shows that $\rho$ is often much smaller than its worst case bound, making the threshold of $\beta_2$ lower than it appears to be.  Note that for adaptive gradient methods, we are not the first to provide the threshold on $\beta_2$ for convergence guarantee: a similar threshold on $\beta_2$ was firstly provided for RMSProp by \citep{shi2020rmsprop}. We prove that the threshold also exists for Adam, but with  extra dependence on $\beta_1$. 


\label{remark_rho}
\end{remark}

\begin{remark} \label{remark_f4_sketch}
We emphasize that the constant term $\mathcal{O}(\sqrt{D_0})$ will vanish to 0 as $\beta_2$ goes to 1. This can be seen in the proof in Appendix  \ref{appendix_lemma_a1} (the definition of $F_4$) and Remark \ref{remark_f4}. 
\end{remark}

Based on Lemma \ref{lemma_a1}, we can further rewrite \eqref{eq_zhangdecomposition} as follows.

{\small
\begin{eqnarray}
   \Ex \left[ \sum_{l=1}^d\sum_{i=0}^{n-1} \partial_l f(\xkz)\frac{\mlki}{\sqrt{\vlki}} \right] 
   &\overset{\eqref{eq_zhangdecomposition}}{\geq} & \Ex \left\{\sum_{l \text{ large}} (a_1)   +  \sum_{l \text{ large}} (b_1) -  \sum_{l \text{ large }} \{(a_2) +  (b_2) \}\right\} - d F_1 \frac{1}{\sqrt{k}}   , \nonumber \\ 
   &\overset{\text{Lemma \ref{lemma_a1}}}{\geq}&  \frac{1}{d\sqrt{10D_1 d}}\Ex \min  \left\{ \sqrt{\frac{ 2D_1 d }{D_{0} } }  \|\nabla f(x_{k,0})\|_2^2, \|\nabla f(x_{k,0})\|_1\right\}   \nonumber \\
    &&-\mathcal{O}(\frac{1}{\sqrt{k}})-\mathcal{O}(\sqrt{D_0}) - d\frac{F_1}{\sqrt{k}} \label{eq_zhangdecomposition_3}
 \end{eqnarray} }

The proof of Theorem \ref{thm1} is concluded by plugging \eqref{eq_zhangdecomposition_3}  into Descent Lemma  \eqref{eq_descentlemma} and then taking telescope some from $k=t_0$ to $k=T$. These steps are quite standard in non-convex optimization. We finish these calculation in Lemma \ref{lemma_all_together}

\begin{lemma}\label{lemma_all_together}
When  inequality \eqref{eq_zhangdecomposition_3} holds, we have the following results based on Descent Lemma  \eqref{eq_descentlemma}:

\begin{eqnarray*}
  &&\min_{k \in [1,T]} \Ex \left[\min  \left\{ \sqrt{\frac{ 2D_1 d }{D_{0} } }   \|\nabla f(x_{k,0})\|_2^2, \|\nabla f(x_{k,0})\|_1\right\} \right] \nonumber \\
  && =  \mathcal{O}\left(\frac{\log T }{\sqrt{T}} \right)  + \mathcal{O}(\sqrt{D_0}) .\nonumber 
\end{eqnarray*}

\end{lemma}

Proof of Lemma \ref{lemma_all_together} can be seen in Appendix \ref{appendix:lemma_all}. Now, the whole  proof of Theorem \ref{thm1} is completed.

\subsection{Proof of Lemma \ref{lemma_error_beta2}} \label{appendix_lemma_error_beta2}






 To prove Lemma \ref{lemma_error_beta2}, we only prove \eqref{eq_b}. The proof for \eqref{eq_a} can be analogized from the proof of \eqref{eq_b}.  
We discuss the following two cases:

\paragraph{Case 1:} When  $\partial_{l} f\left(x_{k, 0}\right) (\mlki-  \partial_{l} f_i(\xkz) ) \leq 0$,  we have

\begin{eqnarray}
  \frac{\partial_{l} f\left(x_{k, 0}\right) (\mlki-  \partial_{l} f_i(\xkz) ) }{\sqrt{v_{l, k, i}}} 
  \overset{(a)}{\geq}&  \frac{\partial_{l} f\left(x_{k, 0}\right)(\mlki-  \partial_{l} f_i(\xkz) )}{\sqrt{v_{l, k, 0}}}  \frac{1}{\sqrt{\beta_2^i}}, \nonumber
\end{eqnarray}

where $(a)$ is because of  $\sqrt{v_{l, k, i}} \geq  \sqrt{\beta_2^i} \sqrt{v_{l, k, 0}}  $.

\paragraph{Case 2:}  When  $\partial_{l} f\left(x_{k, 0}\right) (\mlki-  \partial_{l} f_i(\xkz) )  \geq 0$,  we have 


{\small
\begin{eqnarray}
 \frac{\partial_{l} f\left(x_{k, 0}\right) (\mlki-  \partial_{l} f_i(\xkz) ) }{\sqrt{v_{l, k, i}}}
  &\overset{}{=} & \frac{\partial_{l} f\left(x_{k, 0}\right) (\mlki-  \partial_{l} f_i(\xkz) ) }{\sqrt{v_{l, k, 0}}} \frac{1}{\sqrt{ \left(1+ (v_{l,k,i}-v_{l,k,0})/v_{l,k,0} \right)}} \nonumber \\
  &\overset{(*)}{\geq}& \frac{\partial_{l} f\left(x_{k, 0}\right) (\mlki-  \partial_{l} f_i(\xkz) ) }{\sqrt{v_{l, k, 0}}}    \left(1- \frac{|\vlki - \vlkz|}{2\vlkz} \right) \label{eq_1-x/2}
\end{eqnarray}
}

where $(*)$ is due to $\frac{1}{\sqrt{1+x}} \geq 1-\frac{x}{2}$. We further have 

{\small 
\begin{eqnarray*}
|\vlki - \vlkz| &=& (1-\beta_2)\Large(\partial_l f_{\tau_{k,i}}(x_{\tau_{k,i}})^2 + \beta_2 \partial_l f_{\tau_{k,i-1}}(x_{\tau_{k,i-1}})^2+ \cdots + \beta_2^{i-1} \partial_l f_{\tau_{k,1}}(x_{\tau_{k,1}})^2    \\
&&+ (\beta_2^i -1) \partial_l f_{\tau_{k,0}}(x_{\tau_{k,0}})^2 + (\beta_2^{i+1}-1)\partial_l f_{\tau_{k-1,n-1}}(x_{\tau_{k-1,n-1}})^2 + \cdots \Large)\\
&\overset{\text{since } \beta_2^i -1 < 0}{\leq} & (1-\beta_2)\sum_{j=0}^{i-1} \beta_2^j  \partial_l f_{\tau_{k,i-j}}(x_{\tau_{k,i-j}})^2  \\
&\leq& (1-\beta_2)\sum_{j=0}^{i-1} \beta_2^j \left(  \partial_l f_{\tau_{k,i-j}}(x_{\tau_{k,0}})^2 + 2(i-j) \triangle_{nk} |\partial_l f_{\tau_{k,i-j}}(x_{\tau_{k,0}})| + (i-j)^2 \triangle_{nk}^2       \right) \\
&\leq& (1-\beta_2)\sum_{j=0}^{i-1} \beta_2^j \left(  \max_i|\partial_l f_{i}(x_{\tau_{k,0}})|^2 + 2(i-j) \triangle_{nk}  \max_i | \partial_l f_{i}(x_{\tau_{k,0}})| + (i-j)^2 \triangle_{nk}^2       \right) \\
&\overset{\text{(**)}}{\leq} &  (1-\beta_2)\sum_{j=0}^{i-1} \beta_2^j \left(  \max_i|\partial_l f_{i}(x_{\tau_{k,0}})|^2 + 2 \max_i | \partial_l f_{i}(x_{\tau_{k,0}})|^2 + \max_i | \partial_l f_{i}(x_{\tau_{k,0}})|^2     \right)  \\
&\leq& (1-\beta_2) 4n \max_i | \partial_l f_{i}(x_{\tau_{k,0}})|^2 
\end{eqnarray*}
}

where $(**)$ is due to the condition in the Lemma  $\max_i|\partial_l f_{i}(x_{\tau_{k,0}})|\geq Q_k \geq n \triangle_{nk}$. Plug the above result into \eqref{eq_1-x/2} we have

{\small
\begin{eqnarray}
 \frac{\partial_{l} f\left(x_{k, 0}\right) (\mlki-  \partial_{l} f_i(\xkz) ) }{\sqrt{v_{l, k, i}}}
  &\overset{}{\geq} &   \frac{\partial_{l} f\left(x_{k, 0}\right) (\mlki-  \partial_{l} f_i(\xkz) ) }{\sqrt{v_{l, k, 0}}} \left( 1- \frac{(1-\beta_2)4n \max_i | \partial_l f_{i}(x_{\tau_{k,0}})|^2 }{2 \vlkz}\right) \nonumber\\
    &\overset{}{\geq} &   \frac{\partial_{l} f\left(x_{k, 0}\right) (\mlki-  \partial_{l} f_i(\xkz) ) }{\sqrt{v_{l, k, 0}}} \left( 1- \frac{(1-\beta_2)4n \rho_2\sum_{i=0}^{n-1}| \partial_l f_{i}(x_{\tau_{k,0}})|^2/n  }{2 \vlkz}\right)  \nonumber\\
  &\overset{(***)}{\geq} & \frac{\partial_{l} f\left(x_{k, 0}\right) (\mlki-  \partial_{l} f_i(\xkz) ) }{\sqrt{v_{l, k, 0}}} \left( 1- \frac{(1-\beta_2)4n \rho_2  }{\beta_2^n} \right) \label{lemma_f3}
\end{eqnarray}
}
In the last step $(***)$, we use the following Lemma \ref{lemma_f1}, which is based on \citep{shi2020rmsprop}. 

\begin{lemma} 
\label{lemma_f1}
  Under Assumption \ref{assum1},   if the $l$-th component of $ \nabla f_i(x_{k,0})$ satisfies 
  $\max_i |\partial_l f_i(x_{k,0})| \geq Q_k:=\triangle_1 \frac{n\sqrt{n}}{\sqrt{k}} \frac{32\sqrt{2}}{(1-\btwo)^n \btwo^n}$,
  we have

  $$\frac{v_{l, k, 0}}{\frac{1}{n} \sum_{i} \partialtauki(x_{k,0})^{2}} \geq \frac{\beta_{2}^{n}}{2},$$

  $$\frac{v_{l, k, 0}}{\left(\partial_{l} f\left(x_{k, 0}\right)\right)^{2}} \geq \frac{\beta_{2}^{n}}{2 \rho_{3}^{2}}.$$
\end{lemma}


\begin{proof}
  The proof idea of Lemma \ref{lemma_f1} is similar as Lemma F.1 in \citep{shi2020rmsprop}, but with the following differences: 
  \begin{itemize}
    \item [1.] The condition of Lemma F.1  in \citep{shi2020rmsprop} both require the full batch gradient $ |\partial_l f(x_{k,0})| \geq n Q_k$, which is different from our condition. Here, we choose the condition of Lemma \ref{lemma_f1} because it meets the need of our decomposition strategy in \eqref{eq_zhangdecomposition} .
    \item[2.]  The constant $Q_k$ is different. The difference is due to the different update formula of RMSProp and general Adam. 
  \end{itemize}

  Anyhow, Lemma  \ref{lemma_f1} can be easily proved following most of the steps in Lemma F.1 in \citep{shi2020rmsprop}. We omit the proof for brevity.
\end{proof}

Combining the above {\bf Case 1} and {\bf Case 2}, we have

{\small 
\begin{eqnarray}
\sum \frac{\partial_{l} f\left(x_{k, 0}\right) (\mlki-  \partial_{l} f_i(\xkz) ) }{\sqrt{v_{l, k, i}}} &\geq&  \sum_{i-} \frac{\partial_{l} f\left(x_{k, 0}\right)(\mlki-  \partial_{l} f_i(\xkz) )}{\sqrt{v_{l, k, 0}}}  \frac{1}{\sqrt{\beta_2^i}} \nonumber \\ 
&&+\sum_{i+} \frac{\partial_{l} f\left(x_{k, 0}\right) (\mlki-  \partial_{l} f_i(\xkz) ) }{\sqrt{v_{l, k, 0}}}  \left( 1- \frac{(1-\beta_2)4n \rho_2  }{\beta_2^n} \right) \nonumber \\
&=& \sum_{i=0}^{n-1} \frac{\partial_{l} f\left(x_{k, 0}\right)(\mlki-  \partial_{l} f_i(\xkz) )}{\sqrt{v_{l, k, 0}}} \nonumber \\ 
 &&  +  \sum_{i-} \frac{\partial_{l} f\left(x_{k, 0}\right)(\mlki-  \partial_{l} f_i(\xkz) )}{\sqrt{v_{l, k, 0}}}  \left(\frac{1}{\sqrt{\beta_2^i}}-1 \right) \nonumber\\
&&+\sum_{i+} \frac{\partial_{l} f\left(x_{k, 0}\right) (\mlki-  \partial_{l} f_i(\xkz) ) }{\sqrt{v_{l, k, 0}}}  \left( - \frac{(1-\beta_2)4n \rho_2  }{\beta_2^n} \right)  \nonumber \\
&\geq & \sum_{i=0}^{n-1} \frac{\partial_{l} f\left(x_{k, 0}\right)(\mlki-  \partial_{l} f_i(\xkz) )}{\sqrt{v_{l, k, 0}}}  \nonumber \\
&&-\delta_{1}\left|\frac{\partial_{l} f\left(x_{k, 0}\right)}{\sqrt{v_{l, k, 0}}}\right| \sum_{i=0}^{n-1}\left|m_{l, k, i}-\partial_{l} f_{i}\left(x_{k, 0}\right)\right| , \nonumber
\end{eqnarray}
}

where $\delta_{1}=\frac{(1-\beta_2)4n \rho_2  }{\beta_2^n}+\left(\frac{1}{\sqrt{\beta_{2}^{n}}}-1\right)$. Proof is completed.

\subsection{Proof of Lemma \ref{lemma_k-k-1}}
\label{appendix_lemma_k-k-1}

 We first prove the result for $j=1$. We discuss the following two cases. 

\paragraph{Case 1: when $\frac{\partial_l f(x_{k,0})}{\sqrt{v_{l,k,0}}} \geq \frac{\partial_l f(x_{k-1,0})}{\sqrt{v_{l,k-1,0}}} $:     } when $\partial_l f(x_{k,0})\leq 0$, we have 

$$\frac{\partial_l f(x_{k,0})}{\sqrt{v_{l,k,0}}}  \overset{\text{\eqref{lemma_f3}}}{\leq} \frac{\partial_l f(x_{k,0})}{\sqrt{v_{l,k-1,0}}} \left(1-\frac{(1-\beta_2)4n \rho_2  }{\beta_2^n}\right).   $$

When $\partial_l f(x_{k,0}) > 0$, we have

$$\frac{\partial_l f(x_{k,0})}{\sqrt{v_{l,k,0}}}  \overset{}{\leq} \frac{\partial_l f(x_{k,0})}{\sqrt{v_{l,k-1,0}}} \frac{1}{\sqrt{\beta_2^n}}  $$. 

In conclusion, we have: 

{\small
\begin{eqnarray*}
  \frac{\partial_l f(x_{k,0})}{\sqrt{v_{l,k,0}}} &\leq &   \max \left\{  \frac{\partial_l f(x_{k,0})}{\sqrt{v_{l,k-1,0}}} \left(1-\frac{(1-\beta_2)4n \rho_2  }{\beta_2^n}\right), \frac{\partial_l f(x_{k,0})}{\sqrt{v_{l,k-1,0}}} \frac{1}{\sqrt{\beta_2^n}} \right\} \\
  &\leq & \frac{\partial_l f(x_{k,0})}{\sqrt{v_{l,k-1,0}}} + \max \left\{  \frac{\partial_l f(x_{k,0})}{\sqrt{v_{l,k-1,0}}} \left(-\frac{(1-\beta_2)4n \rho_2  }{\beta_2^n}\right), \frac{\partial_l f(x_{k,0})}{\sqrt{v_{l,k-1,0}}} \left(\frac{1}{\sqrt{\beta_2^n}}-1\right) \right\}   \\
  &\leq & \frac{\partial_l f(x_{k,0})}{\sqrt{v_{l,k-1,0}}} + \frac{\left|\partial_l f(x_{k,0})\right|}{\sqrt{v_{l,k-1,0}}} \left(\frac{1}{\sqrt{\beta_2^n}}-1  +\frac{(1-\beta_2)4n \rho_2  }{\beta_2^n}  \right)\\
    &\overset{\text{\eqref{lemma_f3}}}{\leq}&\frac{\partial_l f(x_{k,0})}{\sqrt{v_{l,k-1,0}}} + \frac{\left|\partial_l f(x_{k,0})\right|}{\sqrt{v_{l,k,0}}} \left(\frac{1}{\sqrt{\beta_2^n}}-1  +\frac{(1-\beta_2)4n \rho_2  }{\beta_2^n}  \right) \frac{1}{\left(1-\frac{(1-\beta_2)4n \rho_2  }{\beta_2^n} \right) } 
\end{eqnarray*}
\begin{eqnarray*}
    &\overset{\text{Lemma \ref{lemma_f1}}}{\leq}&\frac{\partial_l f(x_{k,0})}{\sqrt{v_{l,k-1,0}}} + \sqrt{\frac{2\rho_3^2}{\beta_2^n}} \left(\frac{1}{\sqrt{\beta_2^n}}-1  +\frac{(1-\beta_2)4n \rho_2  }{\beta_2^n}  \right) \frac{1}{\left(1-\frac{(1-\beta_2)4n \rho_2  }{\beta_2^n} \right) } \\
    &= & \frac{\partial_l f(x_{k,0})}{\sqrt{v_{l,k-1,0}}} + \sqrt{\frac{2\rho_3^2}{\beta_2^n}}  \frac{\delta_1}{\left(1-\frac{(1-\beta_2)4n \rho_2  }{\beta_2^n} \right) } ,
\end{eqnarray*}

}

where $\delta_1 = \left(\frac{1}{\sqrt{\beta_2^n}}-1  +\frac{(1-\beta_2)4n \rho_2  }{\beta_2^n} \right) $ is a constant that goes to 0 when $\beta_2$ goes to 1. Therefore, we have

\begin{eqnarray*}
  \frac{\partial_l f(x_{k,0})}{\sqrt{v_{l,k,0}}} -   \frac{\partial_l f(x_{k-1,0})}{\sqrt{v_{l,k-1,0}}} &\leq & \frac{\partial_l f(x_{k,0})- \partial_l f(x_{k-1,0}) }{\sqrt{v_{l,k-1,0}}} + \sqrt{\frac{2\rho_3^2}{\beta_2^n}}  \frac{1}{\left(1-\frac{(1-\beta_2)4n \rho_2  }{\beta_2^n} \right) }  \delta_1 \\
  &\overset{\text{Lemma \ref{lemma_delta}}}{ \leq } & \frac{n^2\triangle_{n(k-1)}}{\sqrt{v_{l,k-1,0}}} + \sqrt{\frac{2\rho_3^2}{\beta_2^n}}  \frac{1}{\left(1-\frac{(1-\beta_2)4n \rho_2  }{\beta_2^n} \right)  }  \delta_1. 
\end{eqnarray*}

Note that 
\paragraph{Case 2: when $\frac{\partial_l f(x_{k,0})}{\sqrt{v_{l,k,0}}} \leq \frac{\partial_l f(x_{k-1,0})}{\sqrt{v_{l,k-1,0}}} $:     } when $\partial_l f(x_{k,0})\geq 0$, we have 

$$\frac{\partial_l f(x_{k,0})}{\sqrt{v_{l,k,0}}}  \overset{\text{\eqref{lemma_f3}}}{\geq} \frac{\partial_l f(x_{k,0})}{\sqrt{v_{l,k-1,0}}} \left(1-\frac{(1-\beta_2)4n \rho_2  }{\beta_2^n} \right).   $$

When $\partial_l f(x_{k,0}) < 0$, we have

$$\frac{\partial_l f(x_{k,0})}{\sqrt{v_{l,k,0}}}  \overset{}{\geq} \frac{\partial_l f(x_{k,0})}{\sqrt{v_{l,k-1,0}}} \frac{1}{\sqrt{\beta_2^n}}  $$. 

Following the same strategy as in Case 1, we can show that

\begin{eqnarray*}
  \frac{\partial_l f(x_{k,0})}{\sqrt{v_{l,k,0}}} &\geq & 
  \frac{\partial_l f(x_{k,0})}{\sqrt{v_{l,k-1,0}}} - \sqrt{\frac{2\rho_3^2}{\beta_2^n}}  \frac{1}{\left(1-\frac{(1-\beta_2)4n \rho_2  }{\beta_2^n} \right)}  \delta_1,
\end{eqnarray*}

which further implies 

\begin{eqnarray*}
  \frac{\partial_l f(x_{k,0})}{\sqrt{v_{l,k,0}}} -   \frac{\partial_l f(x_{k-1,0})}{\sqrt{v_{l,k-1,0}}} &\geq & \frac{\partial_l f(x_{k,0})- \partial_l f(x_{k-1,0}) }{\sqrt{v_{l,k-1,0}}} - \sqrt{\frac{2\rho_3^2}{\beta_2^n}}  \frac{1}{\left(1-\frac{(1-\beta_2)4n \rho_2  }{\beta_2^n} \right)}  \delta_1 \\
  &\overset{\text{Lemma \ref{lemma_delta}}}{ \geq } & - \frac{n^2\triangle_{n(k-1)}}{\sqrt{v_{l,k-1,0}}} - \sqrt{\frac{2\rho_3^2}{\beta_2^n}}  \frac{1}{\left(1-\frac{(1-\beta_2)4n \rho_2  }{\beta_2^n} \right)}  \delta_1. 
\end{eqnarray*}

Case 1 and Case 2 together, we have 

$$ \left| \frac{\partial_l f(x_{k,0})}{\sqrt{v_{l,k,0}}} -   \frac{\partial_l f(x_{k-1,0})}{\sqrt{v_{l,k-1,0}}} \right| \leq \frac{n^2\triangle_{n(k-1)}}{\sqrt{v_{l,k-1,0}}} + \sqrt{\frac{2\rho_3^2}{\beta_2^n}}  \frac{1}{\left(1-\frac{(1-\beta_2)4n \rho_2  }{\beta_2^n} \right)}  \delta_1 . $$

Now we consider the case when $j>1$. Based on the above inequality, we have

{\small
\begin{eqnarray*}
  \left| \frac{\partial_l f(x_{k,0})}{\sqrt{v_{l,k,0}}} -   \frac{\partial_l f(x_{k-j,0})}{\sqrt{v_{l,k-j,0}}} \right| &\leq& \left( \frac{n^2\triangle_{n(k-1)}}{\sqrt{v_{l,k-1,0}}} + \frac{n^2\triangle_{n(k-2)}}{\sqrt{v_{l,k-2,0}}} + \cdots + \frac{n^2\triangle_{n(k-j)}}{\sqrt{v_{l,k-j,0}}}\right) + j\sqrt{\frac{2\rho_3^2}{\beta_2^n}}  \frac{1}{\left(1-\frac{(1-\beta_2)4n \rho_2  }{\beta_2^n} \right)}  \delta_1  \\
  &\leq &  \left( \frac{1}{\sqrt{\beta_2^{jn}}} + \frac{1}{\sqrt{\beta_2^{(j-1)n}}} + \cdots + 1\right) \frac{n^2\triangle_{n(k-j)}}{\sqrt{v_{k-j,0}}} + j\sqrt{\frac{2\rho_3^2}{\beta_2^n}}  \frac{1}{\left(1-\frac{(1-\beta_2)4n \rho_2  }{\beta_2^n} \right)}   \delta_1  \\
  &\leq & \frac{1}{1-\frac{1}{\sqrt{\beta_2^n}}} \frac{n^2\triangle_{n(k-j)}}{\sqrt{v_{l,k-j,0}}} +j\sqrt{\frac{2\rho_3^2}{\beta_2^n}}  \frac{1}{\left(1-\frac{(1-\beta_2)4n \rho_2  }{\beta_2^n} \right)}  \delta_1 .
\end{eqnarray*}
}

The proof is completed.




\subsection{Proof of Lemma \ref{lemma_b1}} \label{appendix_lemma_b1}

To prove  Lemma \ref{lemma_b1}, we need to further decompose $\Ex[\sum_{l \text{ large}}(b_1)]$. 
First and foremost, we write $\sum_{i=0}^{n-1}\mlki$ in an explicit form.

{\small
\begin{eqnarray}
  \mlki &=&(1-\beta_1) \{\partial_l f_{\tau_{k,i}}(x_{k,i})+\cdots + \beta_1^{i}\partial_l f_{\tau_{k,0}}(x_{k,0}) \nonumber \\
  && + \beta_1^{i+1}  f_{\tau_{k-1,n-1}}(x_{k-1,n-1})+ \cdots + \beta_1^{i+n}f_{\tau_{k-1,0}}(x_{k-1,0}) \nonumber\\
  && +  \nonumber\\
  && \vdots\nonumber\\
  && +  \nonumber \\
  && + \beta_1^{(k-2)n+i+1} \partial_l f_{\tau_{1,n-1}}(x_{1,n-1})  +\cdots+ \beta_1^{(k-1)n+i} \partial_l f_{\tau_{1,0}}(x_{1,0})  \} \nonumber \\
 && + \BLUE{\beta_1^{(k-1)n+i+1} \partial_l f(x_{1,0}) }\label{proofsketch_m}
\end{eqnarray} }

Since $\partial_l f(x_{1,0})=\partial_l f(x_{1,0}) (1-\beta_1) (1+\beta_1 + \cdots +\beta_1^ \infty )$ and $\partial_l f(x_{1,0})=\sum_{i=0}^{n-1}  \partial_l f_{i}(x_{1,0}),$ we have

\BLUE{
{\small
\begin{eqnarray}
  \beta_1^{(k-1)n+i+1} \partial_l f(x_{1,0}) &=&  (1-\beta_1)\{ \beta_1^{(k-1)n+i+1} \partial_l f_{0}(x_{1,0})+ \cdots + \beta_1^{(k-1)n+i+1} \partial_l f_{n-1}(x_{1,0})  \nonumber \\
  && +  \beta_1^{(k-1)n+i+2} \partial_l f_{0}(x_{1,0})+ \cdots + \beta_1^{(k-1)n+i+2} \partial_l f_{n-1}(x_{1,0})  \nonumber \\
  && +  \nonumber\\
  && \vdots\nonumber\\
  && +  \nonumber \\
  && +   \beta_1^{\infty} \partial_l f_{0}(x_{1,0})+ \cdots + \beta_1^{\infty} \partial_l f_{n-1}(x_{1,0}) \} \label{proofsketch_m2}
\end{eqnarray} 
}
}

Plugging \eqref{proofsketch_m2} into \eqref{proofsketch_m}, we have 

{\small
\begin{eqnarray}
  \mlki &=&(1-\beta_1) \{\partial_l f_{\tau_{k,i}}(x_{k,i})+\cdots + \beta_1^{i}\partial_l f_{\tau_{k,0}}(x_{k,0}) \nonumber \\
  && + \beta_1^{i+1}  f_{\tau_{k-1,n-1}}(x_{k-1,n-1})+ \cdots + \beta_1^{i+n}f_{\tau_{k-1,0}}(x_{k-1,0}) \nonumber\\
  && + \cdots  \nonumber\\
  && + \beta_1^{(k-2)n+i+1} \partial_l f_{\tau_{1,n-1}}(x_{1,n-1})  +\cdots+ \beta_1^{(k-1)n+i} \partial_l f_{\tau_{1,0}}(x_{1,0})  \nonumber \\
 &&\BLUE{ +\beta_1^{(k-1)n+i+1} \partial_l f_{0}(x_{1,0})+ \cdots + \beta_1^{(k-1)n+i+1} \partial_l f_{n-1}(x_{1,0})  } \nonumber \\
 && \BLUE{+  \beta_1^{(k-1)n+i+2} \partial_l f_{0}(x_{1,0})+ \cdots + \beta_1^{(k-1)n+i+2} \partial_l f_{n-1}(x_{1,0})  }\nonumber \\
 &&+ \cdots  \nonumber  \\
 && \BLUE{+   \beta_1^{\infty} \partial_l f_{0}(x_{1,0})+ \cdots + \beta_1^{\infty} \partial_l f_{n-1}(x_{1,0}) \}  }
 \label{proofsketch_m3}
\end{eqnarray}
}

Using \eqref{proofsketch_m3}, we have  (For each `epoch', we suggest readers to read from bottom to the top.)

{\small
\begin{eqnarray}
  M_{l,k}:=\sum_{i=0}^{n-1}\mlki &=& m_{l,k,n-1}+ \cdots+ m_{l,k,0}  \nonumber\\
  &=& (1-\beta_1) \{   \nonumber\\
  &&\partial_l f_{\tau_{k,n-1}}(x_{k,n-1})+ \cdots + \beta_1^{n-1}\partial_l f_{\tau_{k,0}}(x_{k,0})  \nonumber\\
  &&+  \partial_l f_{\tau_{k,n-2}}(x_{k,n-2})+ \cdots+ \beta_1^{n-2}\partial_l f_{\tau_{k,0}}(x_{k,0})   \nonumber\\
  &&+ \cdots  \nonumber\\
  &&\underbrace{+ \partial_l f_{\tau_{k,0}}(x_{k,0}) \quad  \quad  \quad  \quad  \quad  \quad  \quad  \quad  \quad  \quad  \quad  \quad  \quad  \quad }_{k\text{-th epoch}} \nonumber 
  \end{eqnarray}
  
  \begin{eqnarray}
  \quad \quad \quad \quad \quad \quad \quad \quad \quad \quad \quad \quad \quad \quad \quad  
  && + \beta_1^n \partial_l f_{\tau_{k-1,n-1}}(x_{k-1,n-1})+\cdots + \beta_1^{n+n-1}\partial_l f_{\tau_{k-1,0}}(x_{k-1,0})  \nonumber\\
  && + \beta_1^{n-1} \partial_l f_{\tau_{k-1,n-1}}(x_{k-1,n-1})+\cdots + \beta_1^{n+n-2}\partial_l f_{\tau_{k-1,0}}(x_{k-1,0})  \nonumber\\
  && + \cdots  \nonumber\\
  && \underbrace{+ \beta_1 \partial_l f_{\tau_{k-1,n-1}}(x_{k-1,n-1})+\cdots + \beta_1^{n}\partial_l f_{\tau_{k-1,0}}(x_{k-1,0})}_{k-1\text{-th epoch}}  \nonumber\\
  &&+  \nonumber\\
  && \vdots\nonumber\\
  && +  \nonumber 
  \end{eqnarray}
  \begin{eqnarray}
   \quad \quad \quad \quad \quad \quad \quad \quad \quad \quad \quad \quad \quad \quad \quad 
  && + \beta_1^{(k-1)n} \partial_l f_{\tau_{1,n-1}}(x_{1,n-1})+\cdots + \beta_1^{(k-1)n+n-1}\partial_l f_{\tau_{1,0}}(x_{1,0})  \nonumber\\
  && + \beta_1^{(k-2)n+n-1} \partial_l f_{\tau_{1,n-1}}(x_{1,n-1})+\cdots + \beta_1^{(k-1)n+n-2}\partial_l f_{\tau_{1,0}}(x_{1,0})  \nonumber\\
  && + \cdots  \nonumber\\
  && \underbrace{+ \beta_1^{(k-2)n+1} \partial_l f_{\tau_{1,n-1}}(x_{1,n-1})+\cdots + \beta_1^{(k-1)n}\partial_l f_{\tau_{1,0}}(x_{1,0}) }_{1\text{-th epoch}} \nonumber  
 \end{eqnarray}
  \begin{eqnarray}
  \quad \quad \quad \quad \quad \quad \quad \quad \quad \quad \quad \quad \quad \quad\quad && + \BLUE{  \beta_1^{(k-1)n+n}\partial_l f_{0}(x_{1,0}) +\cdots  +\beta_1^{(k-1)n+n}\partial_l f_{n-1}(x_{1,0})   }  \nonumber \\
  &&+ \BLUE{  \beta_1^{(k-1)n+n-1}\partial_l f_{0}(x_{1,0}) +\cdots  +\beta_1^{(k-1)n+n-1}\partial_l f_{n-1}(x_{1,0})   }  \nonumber \\
  && + \cdots \nonumber \\
  &&\underbrace{+ \BLUE{  \beta_1^{(k-1)n+1}\partial_l f_{0}(x_{1,0}) +\cdots  +\beta_1^{(k-1)n+1}\partial_l f_{n-1}(x_{1,0})   + } }_{0\text{-th epoch}}  \nonumber \\
  && \vdots  \} \label{proofsketch_sum_m}
\end{eqnarray}
}



Note that in \eqref{proofsketch_sum_m}, the power of $\beta_1$ grows slower in the blue part (when $k =0$), such a transition will cause trouble in bounding the  $\sum_{i=0}^{n-1}(\mlki- \partial_l f_i(\xkz) ). $  Therefore, we need to define an auxillary sequence $M^{\prime}_{l,k}$ which does not involve such a phase transition. We define  $M^{\prime}_{l,k}$ as follows. (For each `epoch', we suggest readers to read from bottom to the top.) 


{\small
\begin{eqnarray}
  M_{l,k}^{\prime} &:=&  (1-\beta_1) \{   \nonumber\\
  &&\partial_l f_{\tau_{k,n-1}}(x_{k,n-1})+ \cdots + \beta_1^{n-1}\partial_l f_{\tau_{k,0}}(x_{k,0})  \nonumber\\
  &&+  \partial_l f_{\tau_{k,n-2}}(x_{k,n-2})+ \cdots+ \beta_1^{n-2}\partial_l f_{\tau_{k,0}}(x_{k,0})   \nonumber\\
  &&+ \cdots  \nonumber\\
  &&\underbrace{+ \partial_l f_{\tau_{k,0}}(x_{k,0}) \quad  \quad  \quad  \quad  \quad  \quad  \quad  \quad  \quad  \quad  \quad  \quad  \quad  \quad }_{k\text{-th epoch}} \nonumber 
  \end{eqnarray}
  \begin{eqnarray}
  \quad  \quad  \quad  \quad  \quad  \quad  \quad  \quad \quad\quad\quad\quad&& + \beta_1^n \partial_l f_{\tau_{k-1,n-1}}(x_{k-1,n-1})+\cdots + \beta_1^{n+n-1}\partial_l f_{\tau_{k-1,0}}(x_{k-1,0})  \nonumber\\
  && + \beta_1^{n-1} \partial_l f_{\tau_{k-1,n-1}}(x_{k-1,n-1})+\cdots + \beta_1^{n+n-2}\partial_l f_{\tau_{k-1,0}}(x_{k-1,0})  \nonumber\\
  && + \cdots  \nonumber\\
  && \underbrace{+ \beta_1 \partial_l f_{\tau_{k-1,n-1}}(x_{k-1,n-1})+\cdots + \beta_1^{n}\partial_l f_{\tau_{k-1,0}}(x_{k-1,0})  }_{k-1\text{-th epoch}} \nonumber\\
  && +  \nonumber\\
  && \vdots\nonumber\\
  && +  \nonumber \\
  && + \beta_1^{(k-1)n} \partial_l f_{\tau_{1,n-1}}(x_{1,n-1})+\cdots + \beta_1^{(k-1)n+n-1}\partial_l f_{\tau_{1,0}}(x_{1,0})  \nonumber\\
  && + \beta_1^{(k-2)n+n-1} \partial_l f_{\tau_{1,n-1}}(x_{1,n-1})+\cdots + \beta_1^{(k-1)n+n-2}\partial_l f_{\tau_{1,0}}(x_{1,0})  \nonumber\\
  && + \cdots  \nonumber\\
  && \underbrace{+ \beta_1^{(k-2)n+1} \partial_l f_{\tau_{1,n-1}}(x_{1,n-1})+\cdots + \beta_1^{(k-1)n}\partial_l f_{\tau_{1,0}}(x_{1,0}) }_{1\text{-th epoch}}  \nonumber \\ 
  && + \BLUE{  \beta_1^{(k-1)n+n}\partial_l f_{0}(x_{1,0}) +\cdots  +\beta_1^{(k-1)n+n+n-1}\partial_l f_{n-1}(x_{1,0})   }  \nonumber \\
  &&+ \BLUE{  \beta_1^{(k-1)n+n-1}\partial_l f_{0}(x_{1,0}) +\cdots  +\beta_1^{(k-1)n+2(n-1)}\partial_l f_{n-1}(x_{1,0})   }  \nonumber \\
  && + \cdots \nonumber \\
  &&\underbrace{+ \BLUE{  \beta_1^{(k-1)n+1}\partial_l f_{0}(x_{1,0}) +\cdots  +\beta_1^{(k-1)n+n}\partial_l f_{n-1}(x_{1,0})   + } }_{0\text{-th epoch}}  \nonumber \\
  && \vdots  \} \label{proofsketch_sum_mp}
\end{eqnarray}
}

Now we bound $\sum_{i=0}^{n-1}(\mlki- \partial_l f_i(\xkz) )$ with the help of $ M^{\prime}_{l,k}$.  Denoting  $F_{l,k}:=  \sum_{i=0}^{n-1} \partial_l f_i(\xkz)$, we have:


\begin{eqnarray}
  \sum_{i=0}^{n-1}(\mlki- \partial_l f_i(\xkz) ) &=& \sum_{i=0}^{n-1} \mlki-  \sum_{i=0}^{n-1} \partial_l f_i(\xkz)  \nonumber \\
  &:=& M_{l,k}-  F_{l,k}  \nonumber \\
  &\geq &-| M_{l,k}- M^{\prime}_{l,k} |   +  M^{\prime}_{l,k}- F_{l,k}
  \label{proofsketch2}
\end{eqnarray}

To proceed, we  use the relation $1=(1-\beta_1) (1+\beta_1+\beta_1^2+\cdots)$ to rewrite $F_{l,k}$.

\begin{eqnarray}
 F_{l,k}:= \sum_{i=0}^{n-1} \partial_l f_i(\xkz)&=&  (1-\beta_1) \{   \nonumber\\
  &&\partial_l f_{n-1}(x_{k,0})+ \cdots + \beta_1^{n-1}\partial_l f_{n-1}(x_{k,0})  \nonumber\\
  &&+  \partial_l f_{n-2}(x_{k,0})+ \cdots+ \beta_1^{n-1}\partial_l f_{n-2}(x_{k,0})   \nonumber\\
  &&+ \cdots  \nonumber\\
  &&\underbrace{+\partial_l f_{0}(x_{k,0})+ \cdots+ \beta_1^{n-1}\partial_l f_{0}(x_{k,0})  }_{k\text{-th epoch}} \nonumber  \\
  && +  \nonumber\\
  && \vdots\nonumber\\
  && +  \nonumber \end{eqnarray}
  \begin{eqnarray}
  \quad  \quad  \quad  \quad  \quad  \quad  \quad  \quad \quad\quad\quad\quad \quad \quad \quad
  && + \beta_1^{(k-1)n} \partial_l f_{n-1}(x_{k,0})+\cdots + \beta_1^{(k-1)n+n-1}\partial_l f_{n-1}(x_{k,0})  \nonumber\\
  && + \beta_1^{(k-1)n} \partial_l f_{n-2}(x_{k,0})+\cdots + \beta_1^{(k-1)n+n-1}\partial_l f_{n-2}(x_{k,0})  \nonumber\\
  && + \cdots  \nonumber\\
  && \underbrace{+ \beta_1^{(k-1)n} \partial_l f_{0}(x_{k,0})+\cdots + \beta_1^{(k-1)n+n-1}\partial_l f_{0}(x_{k,0}) }_{1\text{-th epoch}}  \nonumber\end{eqnarray}
  \begin{eqnarray}
  \quad  \quad  \quad  \quad  \quad  \quad  \quad  \quad \quad\quad\quad\quad 
  && + \BLUE{  \beta_1^{kn}\partial_l f_{n-1}(x_{k,0}) +\cdots  +\beta_1^{kn+n-1}\partial_l f_{n-1}(x_{k,0})   }  \nonumber \\
  &&+ \BLUE{  \beta_1^{kn}\partial_l f_{n-2}(x_{k,0}) +\cdots  +\beta_1^{kn+n-1}\partial_l f_{n-2}(x_{k,0})   }  \nonumber \\
  && + \cdots \nonumber \\
  &&\underbrace{+ \BLUE{  \beta_1^{kn}\partial_l f_{0}(x_{k,0}) +\cdots  +\beta_1^{kn+n-1}\partial_l f_{0}(x_{k,0})   } }_{0\text{-th epoch}}  \nonumber \\
  && \vdots  \} \label{proofsketch_sum_fp}
\end{eqnarray}

For the sake of better presentation, we rewrite $F_{l,k}$ and $M_{l,k}^{\prime}$ as follows:

$$
F_{l,k}:= \left(F_{l,k}\right)_k + \cdots + \left(F_{l,k}\right)_1+ \left(F_{l,k}\right)_{0} + \left(F_{l,k}\right)_{-1} +\cdots,   
$$

$$
M_{l,k}^{\prime}:= \left(M_{l,k}^{\prime}\right)_k + \cdots + \left(M_{l,k}^{\prime}\right)_1+ \left(M_{l,k}^{\prime}\right)_{0} + \left(M_{l,k}^{\prime}\right)_{-1} +\cdots ,  
$$

where $\left(F_{l,k}\right)_j$ contains the summand of  $F_{l,k}$ in the $j$-th epoch, $j=k,k-1, \cdots -\infty$. Similarly for $M_{l,k}^{\prime}$. Now, we separate $M_{l,k}^{\prime}- F_{l,k}$ into two parts.

\begin{eqnarray*}
 M_{l,k}^{\prime} - F_{l,k}  &=& \left[\left(M_{l,k}^{\prime}\right)_k - \left(F_{l,k}\right)_k +  \cdots \left(M_{l,k}^{\prime}\right)_1- \left(F_{l,k}\right)_1 \right]  \\
  &&+ \left[\left(M_{l,k}^{\prime}\right)_{0} - \left(F_{l,k}\right)_0 +  \cdots \left(M_{l,k}^{\prime}\right)_{-\infty}- \left(F_{l,k}\right)_{-\infty} \right]
\end{eqnarray*}

Now, we rewrite $\Ex[\sum_{l \text{ large}}(b_1)]$: 

\begin{eqnarray}
 && \Ex \left[\sum_{l \text{ large}} \frac{\partial_l f(x_{k,0})}{\sqrt{v_{k,0}}}\sum_{i=0}^{n-1}(\mlki- \partial_l f_i(\xkz) )\right] \nonumber\\ 
 &=&  \Ex \left[ \sum_{l \text{ large}} \frac{\partial_l f(x_{k,0})}{\sqrt{v_{k,0}}} \left(M_{l,k}-  F_{l,k} \right) \right] \nonumber \end{eqnarray}
 \begin{eqnarray}
   \quad  \quad  \quad \quad\quad  \quad  \quad \quad \quad &\geq & -\Ex \left[ \sum_{l \text{  large}} \frac{|\partial_l f(x_{k,0}) |}{\sqrt{v_{k,0}}}| M_{l,k}- M^{\prime}_{l,k} |    \right] + \Ex\left[ \sum_{l \text{ large}} \frac{\partial_l f(x_{k,0}) }{\sqrt{v_{k,0}}}(M^{\prime}_{l,k}- F_{l,k})  \right] \nonumber\\
  &\geq & \underbrace{ - \Ex \left[\sum_{l \text{ large}} \frac{|\partial_l f(x_{k,0}) |}{\sqrt{v_{k,0}}}  | M_{l,k}- M^{\prime}_{l,k} |   \right]   }_{(1)}   \nonumber\end{eqnarray}
 \begin{eqnarray}
   \quad  \quad  \quad \quad\quad  \quad  \quad \quad \quad
  &&+  \underbrace{\Ex \left[ \sum_{l \text{ large}} \frac{\partial_l f(x_{k,0}) }{\sqrt{v_{k,0}}} \left(\left(M_{l,k}^{\prime}\right)_k - \left(F_{l,k}\right)_k +  \cdots+ \left(M_{l,k}^{\prime}\right)_1- \left(F_{l,k}\right)_1   \right) \right] }_{(2)}\nonumber\\
  && - \underbrace{ \Ex \left[ \sum_{l \text{ large}}\frac{ \left|\partial_l f(x_{k,0}) \right|}{\sqrt{v_{k,0}}} \left|\left(M_{l,k}^{\prime}\right)_{0} - \left(F_{l,k}\right)_0 +  \cdots \left(M_{l,k}^{\prime}\right)_{-\infty}- \left(F_{l,k}\right)_{-\infty}  \right| \right]}_{(3)} \nonumber\\
  \label{eq_decompose_b1} 
\end{eqnarray}

We bound $(1)$, $(2)$ and $(3)$ respectively. 
 We bound $(1)$ and $(3)$ in the following Lemma \ref{lemma_M-Mp} and Lemma \ref{lemma_Mp-F_2}. Since the difference of $M_{l,k}-M_{l,k}^\prime$ only occurs in the high order terms of the infinite sum sequence, 
 $(1)$  is expected to vanish as $k$ grows. Similarly for $(3)$.

\begin{lemma} \label{lemma_M-Mp}
  When $k$ is large enough such that: $\beta_1^{(k-1)n} \leq \frac{\beta_1^n}{\sqrt{k-1}}$ and $k \geq 4$, then we have 

  \begin{eqnarray}
   \Ex  \left[ \sum_{l \text{ large}}\frac{ \left|\partial_l f(x_{k,0}) \right|}{\sqrt{v_{k,0}}} \left| M_{l,k}-M_{l,k}^\prime  \right| \right]\leq G_4\frac{1}{\sqrt{k}},
  \end{eqnarray}
  where $G_4 = d \AAA \beta_1^n \sqrt{2}n (n-1)\sum_{i=0}^{n-1} \Ex \left(  |\partial_\alpha f_i (x_{1,0})| \right).$
\end{lemma}

\begin{lemma} \label{lemma_Mp-F_2}
  When $k$ is large enough such that: $\beta_1^{(k-1)n} \leq \frac{\beta_1^n}{\sqrt{k-1}}$ and $k \geq 4$, then we have 

  \begin{eqnarray}
   \Ex \left[ \sum_{l \text{ large}}\frac{ \left|\partial_l f(x_{k,0}) \right|}{\sqrt{v_{k,0}}} \left|\left(M_{l,k}^{\prime}\right)_{0} - \left(F_{l,k}\right)_0 +  \cdots \left(M_{l,k}^{\prime}\right)_{-\infty}- \left(F_{l,k}\right)_{-\infty}  \right|\right]  \leq G_5\frac{1}{\sqrt{k}}, 
  \end{eqnarray}
  where $G_5 = d \AAA \left(\frac{\beta_1^{2n} 2n^3\triangle_1 \sqrt{2}}{\sqrt{n}} \frac{1-\beta_1}{(1-\beta_1^n)^2} + n  \left(1-\beta_1^{n-1} \right) \sum_{i=0}^{n-1}  \Ex|\partial_\alpha f_i (x_{1,0})|  \frac{(1-\beta_1)\beta_1^{n}\sqrt{2}}{1-\beta_1^n} \right) $. 

\end{lemma}

We relegate the proof of Lemma \ref{lemma_M-Mp} and \ref{lemma_Mp-F_2} to Appendix \ref{appendix:lemma_M-Mp} and \ref{appendix:lemma_Mp-F_2}.

Now we  bound $(2)$. This part involves the difficulties (i), (ii) and (iii) mentioned in Appendix \ref{appendix:roadmap}. 
We bound $(2)$ in Lemma \ref{lemma_Mp-F_1}.

\begin{lemma} \label{lemma_Mp-F_1}
  When $k$ is large enough such that: $\beta_1^{(k-1)n} \leq \frac{\beta_1^n}{\sqrt{k-1}}$ and $k \geq 2$, then we have

\begin{eqnarray*}
  &&\Ex \left[ \sum_{l \text{ large}} \frac{\partial_l f(x_{k,0}) }{\sqrt{v_{k,0}}} \left( \left(M_{l,k}^{\prime}\right)_k - \left(F_{l,k}\right)_k +  \cdots+ \left(M_{l,k}^{\prime}\right)_1- \left(F_{l,k}\right)_1   \right)  \right] \\
  &\geq  &  -\frac{G_6}{\sqrt{k}} - G_7 \delta_1 \Ex\left(\sum_{l=1}^d\sum_{i=0}^{n-1} \left|\partial_l f_i(x_{k,0})   \right| \right).
\end{eqnarray*}

  The constant terms $G_6$ and $G_7$ are specified in Appendix \ref{appendix:lemma_Mp-F_1}.

\end{lemma}

Proof is shown in Appendix \ref{appendix:lemma_Mp-F_1}. 

Combining Lemma \ref{lemma_M-Mp}, \ref{lemma_Mp-F_2}  and \ref{lemma_Mp-F_1} together, we conclude the proof. 

\begin{eqnarray*}
  \Ex[\sum_{l \text{ large}}(b_1)] \geq - \frac{1}{\sqrt{k}} \left(G_4+ G_5 +G_6 \right) - G_7 \delta_1 \Ex\left(\sum_{l=1}^d\sum_{i=0}^{n-1} \left|\partial_l f_i(x_{k,0})   \right| \right).
\end{eqnarray*}






\subsection{Proof of Lemma \ref{lemma_M-Mp}}
\label{appendix:lemma_M-Mp}

  By definition of $M_{l,k}$ and $M_{l,k}^\prime$ (see \eqref{proofsketch_sum_m} and \eqref{proofsketch_sum_mp}), they only differ when $k\leq 0$.  More specifically, We have
  (For each `epoch', we suggest readers to read from bottom to the top.)
  \begin{eqnarray}
    \frac{M_{l,k}-M_{l,k}^\prime}{1-\beta_1} & =&   \left(\beta_1^{(k-1)n+n} - \beta_1^{(k-1)n+n+1} \right) \partial_l f_1(x_{1,0})  +\cdots +\left(\beta_1^{(k-1)n+n} - \beta_1^{(k-1)n+n+n-1}\right) \partial_l f_{n-1}(x_{1,0}) \nonumber \\
     && + \cdots  \nonumber \\
     && + \left(\beta_1^{(k-1)n+2} - \beta_1^{(k-1)n+3}  \right) \partial_l f_1(x_{1,0})  +\cdots +  \left(\beta_1^{(k-1)n+2} - \beta_1^{(k-1)n+n+1}\right) \partial_l f_{n-1}(x_{1,0})  \nonumber \\
     && \underbrace{+ \left(\beta_1^{(k-1)n+1} - \beta_1^{(k-1)n+2}\right) \partial_l f_1(x_{1,0})   +\cdots +  \left(\beta_1^{(k-1)n+1} - \beta_1^{(k-1)n+n}\right) \partial_l f_{n-1}(x_{1,0})   }_{\text{0-th epoch} }\nonumber \\
     && + \left(\beta_1^{(k-1)n+n+1} - \beta_1^{kn+n}  \right) \partial_l f_0(x_{1,0})  +\cdots +  \left(\beta_1^{(k-1)n+n+1} - \beta_1^{kn+n+n-1}\right) \partial_l f_{n-1}(x_{1,0})  \nonumber \\
     && + \cdots  \nonumber \\
     && + \left(\beta_1^{(k-1)n+3} - \beta_1^{kn+2}  \right) \partial_l f_0(x_{1,0})  +\cdots +  \left(\beta_1^{(k-1)n+3} - \beta_1^{kn+n+1}\right) \partial_l f_{n-1}(x_{1,0})  \nonumber \\
     && \underbrace{+ \left(\beta_1^{(k-1)n+2} - \beta_1^{kn+1}\right) \partial_l f_0(x_{1,0})   +\cdots +  \left(\beta_1^{(k-1)n+2} - \beta_1^{kn+n}\right) \partial_l f_{n-1}(x_{1,0})   }_{\text{-1-th epoch} }\nonumber \\
     && + \left(\beta_1^{(k-1)n+n+2} - \beta_1^{(k+1)n+n}\right) \partial_l f_0(x_{1,0})   +\cdots +  \left(\beta_1^{(k-1)n+n+2} - \beta_1^{(k+1)n+n+n-1}\right) \partial_l f_{n-1}(x_{1,0})  \nonumber \\
     && + \cdots  \nonumber \\
    && + \left(\beta_1^{(k-1)n+4} - \beta_1^{(k+1)n+2}\right) \partial_l f_0(x_{1,0})   +\cdots +  \left(\beta_1^{(k-1)n+4} - \beta_1^{(k+1)n+n+1}\right) \partial_l f_{n-1}(x_{1,0})  \nonumber \\
     && \underbrace{+ \left(\beta_1^{(k-1)n+3} - \beta_1^{(k+1)n+1}\right) \partial_l f_0(x_{1,0})   +\cdots +  \left(\beta_1^{(k-1)n+3} - \beta_1^{(k+1)n+n}\right) \partial_l f_{n-1}(x_{1,0})   }_{\text{-2-th epoch} }\nonumber \\
     &&+ \cdots \label{proofsketch_m-mp}
  \end{eqnarray}

  We start with the 1st column in the `0-th epoch' (from the bottom to the top).

  \begin{eqnarray}
    \text{The 1st column in the `0-th epoch'} &= &  \beta_1^{(k-1)n+1}(1-\beta_1) | \partial_l f_1 (x_{1,0})  |  \nonumber \\
    && +\beta_1^{(k-1)n+2}(1-\beta_1) | \partial_l f_1 (x_{1,0})  | \nonumber \\
    &&+ \cdots  \nonumber \\
    && + \beta_1^{(k-1)n+n}(1-\beta_1) | \partial_l f_1 (x_{1,0})  | \nonumber \\
    &\leq & \beta_1^{(k-1)n+1} n (1-\beta_1) | \partial_l f_1 (x_{1,0})  |
  \end{eqnarray}

  Similarly, we can bound every column in the in the `0-th epoch', e.g. last ($n-1$-th)  column can be  bounded as follows.
  \begin{eqnarray}
    \text{The last column in the `0-th epoch'} &= &  \beta_1^{(k-1)n+1}(1-\beta_1^{n-1}) | \partial_l f_{n-1} (x_{1,0})  |  \nonumber \\
    && +\beta_1^{(k-1)n+2}(1-\beta_1^{n-1}) | \partial_l f_{n-1} (x_{1,0})  | \nonumber \\
    &&+ \cdots  \nonumber \\
    && + \beta_1^{(k-1)n+n}(1-\beta_1^{n-1}) | \partial_l f_{n-1} (x_{1,0})  | \nonumber \\
    &\leq & \beta_1^{(k-1)n+1} n (n-1)(1-\beta_1) | \partial_l f_{n-1}(x_{1,0})  |
  \end{eqnarray}

  Summing up all the columns, we have

  \begin{eqnarray}
    \text{The `0-th epoch'} &\leq &  \beta_1^{(k-1)n+1} n (n-1)(1-\beta_1) \sum_{i=1}^{n-1}| \partial_l f_{i}(x_{1,0})  | \nonumber \\
    &\leq &  \beta_1^{(k-1)n+1} n (n-1)(1-\beta_1)  \sum_{i=0}^{n-1} |\partial_l 
     f_{i}(x_{1,0})  |
  \end{eqnarray}

  Using the same technique, it can be shown that 

  \begin{eqnarray}
    \text{The `-1-th epoch'} &\leq & 
     \beta_1^{(k-1)n+2} 2n (n-1)(1-\beta_1)  \sum_{i=0}^{n-1} |\partial_l 
    f_{i}(x_{1,0})  |,
  \end{eqnarray}

  \begin{eqnarray}
    \text{The `-2-th epoch'} &\leq & 
     \beta_1^{(k-1)n+3} 3n (n-1)(1-\beta_1)  \sum_{i=0}^{n-1} |\partial_l 
    f_{i}(x_{1,0})  |
  \end{eqnarray}

  $$\cdots $$

  Plugging all these results in \eqref{proofsketch_m-mp},  we have 

{\small
  \begin{eqnarray}
    \frac{ \left|M_{l,k}-M_{l,k}^\prime \right|}{1-\beta_1} &\leq&  \beta_1^{(k-1)n+1} n (n-1)(1-\beta_1)  \sum_{i=0}^{n-1} |\partial_l 
     f_{i}(x_{1,0})  | \left( 1+2\beta_1+3\beta_1^2 +\cdots \right) \nonumber  \\
    &=& \beta_1^{(k-1)n+1} n (n-1)\sum_{i=0}^{n-1} |\partial_l 
     f_{i}(x_{1,0})  |  \frac{1}{1-\beta_1}.\nonumber \\
    &\leq& \beta_1^{(k-1)n} n (n-1)\sum_{i=0}^{n-1} |\partial_l 
     f_{i}(x_{1,0})  |  \frac{1}{1-\beta_1}.\nonumber 
  \end{eqnarray} }

That is to say, when $k$ is large enough such that $\beta_1^{(k-1)n} \leq \frac{\beta_1^n}{\sqrt{k-1}}$ we have 

{\small
\begin{eqnarray}
  \left|M_{l,k}-M_{l,k}^\prime \right| &\leq&\frac{\beta_1^n}{\sqrt{k-1}} n (n-1)\sum_{i=0}^{n-1}|\partial_l 
   f_{i}(x_{1,0})  |\\
   &\overset{\text{When $k \geq 2$}}{\leq }&\frac{\beta_1^n\sqrt{2}}{\sqrt{k}} n (n-1)\sum_{i=0}^{n-1}|\partial_l 
   f_{i}(x_{1,0})  |.
\end{eqnarray} }

Combining with Lemma \ref{lemma_f1}, the proof is completed.

\subsection{Proof of Lemma \ref{lemma_Mp-F_2}}
\label{appendix:lemma_Mp-F_2}
We start with $\left(M_{l,k}^{\prime}\right)_0-\left(F_{l,k}\right)_0$.

{\small
\begin{eqnarray}
  \frac{\left(M_{l,k}^{\prime}\right)_0-\left(F_{l,k}\right)_0}{1-\beta_1} &=& \BLUE{  \beta_1^{(k-1)n+n}\partial_l f_{0}(x_{1,0}) +\cdots  +\beta_1^{(k-1)n+n+n-1}\partial_l f_{n-1}(x_{1,0})   }  \nonumber \\
  &&+ \BLUE{  \beta_1^{(k-1)n+n-1}\partial_l f_{0}(x_{1,0}) +\cdots  +\beta_1^{(k-1)n+2(n-1)}\partial_l f_{n-1}(x_{1,0})   }  \nonumber \\
  && + \cdots \nonumber \\
  &&\underbrace{+ \BLUE{  \beta_1^{(k-1)n+1}\partial_l f_{0}(x_{1,0}) +\cdots  +\beta_1^{(k-1)n+n}\partial_l f_{n-1}(x_{1,0})    } }_{ \left(M_{l,k}^{\prime}\right)_0}  \nonumber \end{eqnarray}
 \begin{eqnarray}
   \quad  \quad  \quad \quad
  &&-  \Bigm\{  \BLUE{  \beta_1^{kn}\partial_l f_{n-1}(x_{k,0}) +\cdots  +\beta_1^{kn+n-1}\partial_l f_{n-1}(x_{k,0})   }  \nonumber \\
  &&+ \BLUE{  \beta_1^{kn}\partial_l f_{n-2}(x_{k,0}) +\cdots  +\beta_1^{kn+n-1}\partial_l f_{n-2}(x_{k,0})   }  \nonumber \\
  && + \cdots \nonumber \\
  &&\underbrace{+ \BLUE{  \beta_1^{kn}\partial_l f_{0}(x_{k,0}) +\cdots  +\beta_1^{kn+n-1}\partial_l f_{0}(x_{k,0})   } }_{\left(F_{l,k}\right)_0} \Bigm\} \nonumber \\
  &:=& \sum_{i=0}^{n-1}\delta_i, \nonumber
\end{eqnarray}
}

where 

{\small
  \begin{eqnarray}
  \delta_i &=&  \sum_{j=1}^n  \left(\beta_1^{(k-1)n+j} \partial_l f_{i}(x_{1,0})  - \beta_1^{kn+j-1}  \partial_l f_{i}(x_{k,0}) \right) \nonumber\\
  &=&  \underbrace{\sum_{j=1}^n \left(\beta_1^{(k-1)n+j}- \beta_1^{kn+j-1} \right) \partial_l f_{i}(x_{1,0})}_{\delta_i^a} + \underbrace{ \sum_{j=1}^n \beta_1^{kn+j-1} \left( \partial_l f_{i}(x_{1,0}) - \partial_l f_{i}(x_{k,0})\right) }_{\delta_i^b}.   \nonumber 
  \end{eqnarray} }

  We further have
  
  {\small
\begin{eqnarray}
  \delta_i^a &=&  \sum_{j=1}^n \beta_1^{(k-1)n+j} \left(1-\beta_1^{n-1} \right) \partial_l f_{i}(x_{1,0}) \nonumber \\
  &\leq & n \beta_1^{(k-1)n+1} \left(1-\beta_1^{n-1} \right) | \partial_l f_{i}(x_{1,0}) |  \nonumber \\
\end{eqnarray}}

{\small
\begin{eqnarray}
  \delta_i^b &\overset{}{\leq }&   \sum_{j=1}^n \beta_1^{kn+j-1} \frac{\triangle_1 (k-1)2n}{\sqrt{n(k-1)-1}} \nonumber \\
  &\leq & \beta_1^{kn} \frac{\triangle_1 (k-1)2n^2}{\sqrt{n(k-1)-1}} . 
\end{eqnarray} }

Therefore, 

{\small
\begin{eqnarray}
  \frac{\left|\left(M_{l,k}^{\prime}\right)_0-\left(F_{l,k}\right)_0 \right|}{1-\beta_1} &= &  \left|\sum_{i=0}^{n-1}\delta_i \right| \nonumber \\
  &\leq & \sum_{i=0}^{n-1} \left(  n \beta_1^{(k-1)n+1} \left(1-\beta_1^{n-1} \right) |\partial_l f_{i}(x_{1,0}) |+  \beta_1^{kn} \frac{\triangle_1 (k-1)2n^2}{\sqrt{n(k-1)-1}}\right) \nonumber \\
  &=&     \beta_1^{kn} \frac{\triangle_1 (k-1)2n^3}{\sqrt{n(k-1)-1}} +   n \beta_1^{(k-1)n+1} \left(1-\beta_1^{n-1} \right) \sum_{i=0}^{n-1} |\partial_l f_{i}(x_{1,0}) |
\end{eqnarray} }

Using the same calculation, we can get the following result: 

{\small
\begin{eqnarray}
  \frac{\left|\left(M_{l,k}^{\prime}\right)_{-1}-\left(F_{l,k}\right)_{-1} \right|}{1-\beta_1}
  &\leq &     \beta_1^{(k+1)n} \frac{\triangle_1 k 2n^3}{\sqrt{n(k-1)-1}} +   n \beta_1^{kn+1} \left(1-\beta_1^{n-1} \right) \sum_{i=0}^{n-1} |\partial_l f_{i}(x_{1,0}) |. \nonumber
\end{eqnarray}
}

Repeat this calculation, we have

{\small
\begin{eqnarray}
 &&\frac{1}{1-\beta_1}\left[\left(M_{l,k}^{\prime}\right)_{0} - \left(F_{l,k}\right)_0 +  \cdots \left(M_{l,k}^{\prime}\right)_{-\infty}- \left(F_{l,k}\right)_{-\infty}  \right]\nonumber \\
 &\leq &  
 \frac{\beta_1^{2n} 2n^3\triangle_1}{\sqrt{n(k-1)-1}} \left( \beta_1^{(k-2)n}(k-1) +\beta_1^{(k-1)n} k +\cdots \right)    \nonumber \\
 &&+ n  \left(1-\beta_1^{n-1} \right) \sum_{i=0}^{n-1} |\partial_l f_{i}(x_{1,0}) | \left(\beta_1^{(k-1)n+1} + \beta_1^{kn+1}+ \cdots \right) \nonumber\\
 &\leq &  \frac{\beta_1^{2n} 2n^3\triangle_1}{\sqrt{n(k-1)-1}} \left( 1 +\beta_1^{n} 2 +\cdots \right)  \nonumber \\
 &&+ n  \left(1-\beta_1^{n-1} \right) \sum_{i=0}^{n-1} |\partial_l f_{i}(x_{1,0}) | \left(\beta_1^{(k-1)n+1} + \beta_1^{kn+1}+ \cdots \right) \nonumber\\
 &\overset{\text{Lemma \ref{lemma_beta}}}{\leq }&\frac{\beta_1^{2n} 2n^3\triangle_1}{\sqrt{n(k-1)-1}} \frac{1}{(1-\beta_1^n)^2} \nonumber \\
 &&+ n  \left(1-\beta_1^{n-1} \right) \sum_{i=0}^{n-1} |\partial_l f_{i}(x_{1,0}) | \frac{\beta_1^{(k-1)n}}{1-\beta_1^n}\nonumber
 \end{eqnarray} }

 $\text{When } \beta_1^{(k-1)n} \leq \frac{\beta_1^n }{\sqrt{k-1}}  $, we further have:

 {\small
 \begin{eqnarray}
 &&\frac{1}{1-\beta_1}\left[\left(M_{l,k}^{\prime}\right)_{0} - \left(F_{l,k}\right)_0 +  \cdots \left(M_{l,k}^{\prime}\right)_{-\infty}- \left(F_{l,k}\right)_{-\infty}  \right]\nonumber \\
 &\overset{}{\leq }& \frac{\beta_1^{2n} 2n^3\triangle_1}{\sqrt{n(k-1)-1}} \frac{1}{(1-\beta_1^n)^2} \nonumber \\
 &&+ n  \left(1-\beta_1^{n-1} \right) \sum_{i=0}^{n-1} |\partial_l f_{i}(x_{1,0}) | \frac{\beta_1^{n}}{1-\beta_1^n} \frac{1}{\sqrt{k-1}}\nonumber\\
 &\overset{\frac{1}{\sqrt{k-1}} \leq \sqrt{\frac{2}{k}}  } {\leq}&  \frac{\beta_1^{2n} 2n^3\triangle_1}{\sqrt{n(k-1)-1}} \frac{1}{(1-\beta_1^n)^2} \nonumber \\
 &&+ n  \left(1-\beta_1^{n-1} \right) \sum_{i=0}^{n-1} |\partial_l f_{i}(x_{1,0}) | \frac{\beta_1^{n}}{1-\beta_1^n} \sqrt{\frac{2}{k}}\nonumber\\
 &\overset{(*)} {\leq} & \frac{\beta_1^{2n} 2n^3\triangle_1 \sqrt{2}}{\sqrt{nk}} \frac{1}{(1-\beta_1^n)^2} \nonumber \\
 &&+ n  \left(1-\beta_1^{n-1} \right) \sum_{i=0}^{n-1} |\partial_l f_{i}(x_{1,0}) | \frac{\beta_1^{n}}{1-\beta_1^n} \sqrt{\frac{2}{k}},\nonumber
\end{eqnarray}
}

where $(*): \frac{1}{\sqrt{n(k-1)-1}} \leq \sqrt{\frac{2}{nk}} \text{ when } k \geq 4 $. Therefore, we have

$$ \left|\left(M_{l,k}^{\prime}\right)_{0} - \left(F_{l,k}\right)_0 +  \cdots \left(M_{l,k}^{\prime}\right)_{-\infty}- \left(F_{l,k}\right)_{-\infty}  \right|\leq \frac{1}{\sqrt{k}} \tilde{G_5}$$,

$\tilde{G_5} = \frac{\beta_1^{2n} 2n^3\triangle_1 \sqrt{2}}{\sqrt{n}} \frac{1-\beta_1}{(1-\beta_1^n)^2} + n  \left(1-\beta_1^{n-1} \right) \sum_{i=0}^{n-1} |\partial_l f_{i}(x_{1,0}) | \frac{(1-\beta_1)\beta_1^{n}\sqrt{2}}{1-\beta_1^n} $. 

Combining with Lemma \ref{lemma_f1}, the proof is completed with constant $G_5$ defined in Lemma \ref{lemma_Mp-F_2}.

\subsection{Proof of Lemma \ref{lemma_Mp-F_1}}
\label{appendix:lemma_Mp-F_1}

In this section, we derive a lower  bound for $\Ex \left[ \sum_{l \text{ large}} \frac{\partial_l f(x_{k,0}) }{\sqrt{v_{k,0}}} \left( \left(M_{l,k}^{\prime}\right)_k - \left(F_{l,k}\right)_k +  \cdots+ \left(M_{l,k}^{\prime}\right)_1- \left(F_{l,k}\right)_1  \right)   \right]$. 
We will use the ideas mentioned in Appendix \ref{appendix:roadmap} (i.e., Step 1,2 and 3). We first rewrite ``$l \text{ large}$" into indicator function as follows:

{\tiny
$$\Ex \left[ \sum_{l \text{ large}} \frac{\partial_l f(x_{k,0}) }{\sqrt{v_{k,0}}} \left( \left(M_{l,k}^{\prime}\right)_k - \left(F_{l,k}\right)_k +  \cdots+ \left(M_{l,k}^{\prime}\right)_1- \left(F_{l,k}\right)_1  \right)   \right] = \Ex \left[ \sum_{l=1}^d  \mathbb{I}_k\frac{\partial_l f(x_{k,0}) }{\sqrt{v_{k,0}}} \left( \left(M_{l,k}^{\prime}\right)_k - \left(F_{l,k}\right)_k +  \cdots+ \left(M_{l,k}^{\prime}\right)_1- \left(F_{l,k}\right)_1  \right)   \right],$$}

where $\mathbb{I}_{k}:=\mathbb{I}\left(\max _{i}\left|\partial_{l} f_{i}\left(x_{k, 0}\right)\right| \geq Q_{k}:=\triangle_{1} \frac{n \sqrt{n}}{\sqrt{k}} \frac{32 \sqrt{2}}{\left(1-\beta_{2}\right)^{n} \beta_{2}^{n}}\right)$. We also define $\mathbb{I}_{k-j}:=\mathbb{I}\left(\max _{i}\left|\partial_{l} f_{i}\left(x_{k-j, 0}\right)\right| \geq \sum_{p=k-j}^{k} Q_{p}\right)$, it will be used later.

We take the conditional expectation over the history before $x_{k,0}$.  We first focus on $ \Ext \left[ \sum_{l=1}^d \mathbb{I}_k \frac{\partial_l f(x_{k,0}) }{\sqrt{v_{k,0}}} \left(\left(M_{l,k}^{\prime}\right)_k - \left(F_{l,k}\right)_k  \right)\right]$ and we delegate the history part for later analysis.

For each possible $\left(M_{l,k}^{\prime}\right)_k - \left(F_{l,k}\right)_k$, we first convert all $x_{k,i}$ into $x_{k,0}$ using Lemma \ref{lemma_delta}.  Since $\left(M_{l,k}^{\prime}\right)_k $ contains $(n-1)+\cdots+1= \frac{n(n-1)}{2}$ terms of $x_{k,i}$ (with $i\neq 0$), we have

\begin{eqnarray*}
  \left(M_{l,k}^{\prime}\right)_k  & \overset{\text{ Lemma \ref{lemma_delta}} }{\geq} & -   \frac{(1-\beta_1)n^2 (n-1)\triangle_{nk}}{2} +(1-\beta_1) \{   \nonumber\\
  &&\partial_l f_{\tau_{k,n-1}}(x_{k,0})+ \cdots + \beta_1^{n-1}\partial_l f_{\tau_{k,0}}(x_{k,0})  \nonumber\\
  &&+  \partial_l f_{\tau_{k,n-2}}(x_{k,0})+ \cdots+ \beta_1^{n-2}\partial_l f_{\tau_{k,0}}(x_{k,0})   \nonumber\\
  &&+ \cdots  \nonumber\\
  &&\underbrace{+ \partial_l f_{\tau_{k,0}}(x_{k,0}) \quad  \quad  \quad  \quad  \quad  \quad  \quad  \quad  \quad  \quad  \quad  \quad  \quad  \quad }_{k\text{-th epoch}} \nonumber  \}\\
  &:=& -  \frac{(1-\beta_1)n^2 (n-1)\triangle_{nk}}{2} + \mathscr{M}_{l,k}.
\end{eqnarray*}

Using the same strategy for analyzing the color-ball toy example, we have

{\small
\begin{eqnarray*}
  \frac{\Ext \left[ \mathscr{M}_{l,k} -\left(F_{l,k}\right)_k  \right] }{1-\beta_1}& = & \frac{1}{n!} \left(-(n-1)!\beta_1-2 (n-1)!\beta_1^2 -\cdots -(n-1)(n-1)! \beta_1^{n-1} \right) \sum_{i=0}^{n-1} \partial_l f_i(x_{k,0}) \nonumber \\
  & =& \left(-\frac{1}{n}\beta_1-\frac{2}{n}\beta_1^2 -\cdots -\frac{n-1}{n} \beta_1^{n-1}\right) \sum_{i=0}^{n-1} \partial_l f_i(x_{k,0})
\end{eqnarray*}
}

{\tiny
\begin{eqnarray*}
 \Ext \left[ \sum_{l=1}^d \mathbb{I}_k \frac{\partial_l f(x_{k,0}) }{\sqrt{v_{k,0}}} \left(\left(M_{l,k}^{\prime}\right)_k - \left(F_{l,k}\right)_k  \right)\right]& \geq & -d\AAA \frac{(1-\beta_1)n^2 (n-1)\triangle_{nk}}{2} + \Ext \left[\sum_{l=1}^d \mathbb{I}_k \frac{\partial_l f(x_{k,0}) }{\sqrt{v_{k,0}}} \left(  \mathscr{M}_{l,k} -\left(F_{l,k}\right)_k \right) \right] \nonumber \\
  &= & -d\AAA\frac{(1-\beta_1)n^2 (n-1)\triangle_{nk}}{2} \nonumber\\
  &&+ \sum_{l=1}^d \mathbb{I}_k (1-\beta_1)\left(-\frac{1}{n}\beta_1-\frac{2}{n}\beta_1^2 -\cdots -\frac{n-1}{n} \beta_1^{n-1}\right)\sum_{i=0}^{n-1} \partial_l f_i(x_{k,0}) \nonumber\\
  &:=& -d\AAA \frac{(1-\beta_1)n^2 (n-1)\triangle_{nk}}{2} + \sum_{l=1}^d \mathbb{I}_k \frac{\partial_l f(x_{k,0}) }{\sqrt{v_{k,0}}} (1-\beta_1)J_1\sum_{i=0}^{n-1} \partial_l f_i(x_{k,0})
\end{eqnarray*} 
}

We denote $J_1:= \left(-\frac{1}{n}\beta_1-\frac{2}{n}\beta_1^2 -\cdots -\frac{n-1}{n} \beta_1^{n-1}\right)$, it will be used repeatedly in the following derivation. 

Now we move one step further exert $ \Ex_{k-1}(\cdot)$. In particular, we need to calculate
{\small 
\begin{eqnarray}
  &&\Ex_{k-1} \left\{\Ext  \left[ \BLUE{\sum_{l=1}^d \mathbb{I}_k \frac{\partial_l f(x_{k,0}) }{\sqrt{v_{k,0}}} \left( \left(M_{l,k}^{\prime}\right)_k-   \left(F_{l,k}\right)_k  \right)} + \RED{\sum_{l=1}^d \mathbb{I}_k \frac{\partial_l f(x_{k,0}) }{\sqrt{v_{k,0}}} \left(
  \left(M_{l,k}^{\prime}\right)_{k-1}     -  \left(F_{l,k}\right)_{k-1}  \right) } \right] \right\} \nonumber \\
  &\geq& 
  \Ex_{k-1} \left\{ \underbrace{\BLUE{ \sum_{l=1}^d \mathbb{I}_k \frac{\partial_l f(x_{k,0}) }{\sqrt{v_{k,0}}} 
    (1-\beta_1)J_1\sum_{i=0}^{n-1} \partial_l f_i(x_{k,0}) }}_{(a)}+  \underbrace{ \RED{\sum_{l=1}^d \mathbb{I}_k \frac{\partial_l f(x_{k,0}) }{\sqrt{v_{k,0}}} \left(
  \left(M_{l,k}^{\prime}\right)_{k-1}     -  \left(F_{l,k}\right)_{k-1}  \right) }}_{(b)}  \right\} \nonumber\\
  &&-  d\AAA\frac{(1-\beta_1)n^2 (n-1)\triangle_{nk}}{2}. 
     \label{eq_cancel_k-1}
 \end{eqnarray}
}

The blue term is the residue from the $k$-th epoch. The red term is the main component in the $(k-1)$-th epoch. Before calculating $\Ex_{k-1}(\cdot)$, we need to  convert all  the variable $x$ into  $x_{k-1,0}$ using Lipschitz property. First of all, we work on $(a)$.
 
 {\small
 \begin{eqnarray*}
   (a)&=&\sum_{l=1}^d \mathbb{I}_k \frac{\partial_l f(x_{k,0}) }{\sqrt{v_{k,0}}} 
    (1-\beta_1)J_1\sum_{i=0}^{n-1} \partial_l f_i(x_{k,0})
    \\
    &\overset{\text{Lemma \ref{lemma_indicator}}}{=}&     \sum_{l=1}^d \mathbb{I}_{k,k-1} \frac{\partial_l f(x_{k,0}) }{\sqrt{v_{k,0}}} 
    (1-\beta_1)J_1\sum_{i=0}^{n-1} \partial_l f_i(x_{k,0}) \\
    &&+ \sum_{l=1}^d \tilde{\mathbb{I}}_{k,k-1} \frac{\partial_l f(x_{k,0}) }{\sqrt{v_{k,0}}} 
    (1-\beta_1)J_1\sum_{i=0}^{n-1} \partial_l f_i(x_{k,0}) \\
    &\overset{\text{Lemma \ref{lemma_f1}}}{\geq}& \sum_{l=1}^d \mathbb{I}_{k,k-1} \frac{\partial_l f(x_{k,0}) }{\sqrt{v_{k,0}}} 
    (1-\beta_1)J_1\sum_{i=0}^{n-1} \partial_l f_i(x_{k,0}) -\sum_{l=1}^d \tilde{\mathbb{I}}_{k,k-1} \AAA
    (1-\beta_1)|J_1||\sum_{i=0}^{n-1} \partial_l f_i(x_{k,0})| \\
    &\geq&  \!\!\!\! \sum_{l=1}^d \mathbb{I}_{k,k-1} \frac{\partial_l f(x_{k,0}) }{\sqrt{v_{k,0}}} 
    (1-\beta_1)J_1\sum_{i=0}^{n-1} \partial_l f_i(x_{k,0}) -\sum_{l=1}^d \tilde{\mathbb{I}}_{k,k-1} \AAA
    (1-\beta_1)|J_1|\left(\sum_{i=0}^{n-1} |\partial_l f_i(x_{k-1,0})|  + n^2 \triangle_{n(k-1)} \right)   \\
    &\geq&  \sum_{l=1}^d \mathbb{I}_{k,k-1} \frac{\partial_l f(x_{k,0}) }{\sqrt{v_{k,0}}} 
    (1-\beta_1)J_1\sum_{i=0}^{n-1} \partial_l f_i(x_{k,0})
    \\
    &&
   -\sum_{l=1}^d \tilde{\mathbb{I}}_{k,k-1} \AAA
    (1-\beta_1)|J_1|\sum_{i=0}^{n-1} |\partial_l f_i(x_{k-1,0})|    - d\AAA
    (1-\beta_1)|J_1|
     n^2 \triangle_{n(k-1)}
    \end{eqnarray*}
        }
     {\small
    \begin{eqnarray*}
    &\overset{\text{Lemma \ref{lemma_f1} and \ref{lemma_delta}}}{\geq} &\sum_{l=1}^d \mathbb{I}_{k,k-1} \frac{\partial_l f(x_{k,0}) }{\sqrt{v_{k,0}}} 
    (1-\beta_1)J_1\sum_{i=0}^{n-1} \partial_l f_i(x_{k-1,0}) \\
    &&- (1-\beta_1)d \AAA |J_1| n^2 \triangle_{n(k-1)} 
    - (1-\beta_1)d \AAA |J_1| n (Q_k+Q_{k-1}) - (1-\beta_1)d \AAA |J_1| n^2 \triangle_{n(k-1)}  \\
    & \geq &\sum_{l=1}^d \mathbb{I}_{k,k-1} \frac{\partial_l f(x_{k-1,0}) }{\sqrt{v_{k-1,0}}} 
    (1-\beta_1)J_1\sum_{i=0}^{n-1} \partial_l f_i(x_{k-1,0}) \\
    && - \sum_{l=1}^d \mathbb{I}_{k,k-1}  \left|\frac{\partial_l f(x_{k-1,0}) }{\sqrt{v_{k-1,0}}} - \frac{\partial_l f(x_{k,0}) }{\sqrt{v_{k,0}}}  \right|
    (1-\beta_1)|J_1 | \left|\sum_{i=0}^{n-1} \partial_l f_i(x_{k-1,0})\right| \\
    &&-  2 (1-\beta_1)d \AAA |J_1| n^2 \triangle_{n(k-1)} 
    - (1-\beta_1)d \AAA |J_1| n (Q_k+Q_{k-1}) \\
              \end{eqnarray*}
        }
     {\small
    \begin{eqnarray*}
    &\overset{\text{Lemma \ref{lemma_k-k-1}}}{\geq}& \sum_{l=1}^d \mathbb{I}_{k,k-1} \frac{\partial_l f(x_{k-1,0}) }{\sqrt{v_{k-1,0}}} 
    (1-\beta_1)J_1\sum_{i=0}^{n-1} \partial_l f_i(x_{k-1,0}) \\
    && - \sum_{l=1}^d \mathbb{I}_{k,k-1}  
    (1-\beta_1)|J_1 | \left|\sum_{i=0}^{n-1} \partial_l f_i(x_{k-1,0})\right| \left( \frac{1}{1-\frac{1}{\sqrt{\beta_2^n}}} \frac{n^2\triangle_{n(k-1)}}{\sqrt{v_{k-1,0}}} +\sqrt{\frac{2\rho_3^2}{\beta_2^n}}  \frac{1}{\left(1- \frac{(1-\beta_2)4n \rho_2  }{\beta_2^n}\right) }  \delta_1\right)\\
    &&-  2(1-\beta_1)d \AAA |J_1| n^2 \triangle_{n(k-1)} 
    - (1-\beta_1)d \AAA |J_1| n(Q_k+Q_{k-1}) \\
\end{eqnarray*} }

 where  $\mathbb{I}_{k,k-j} := \mathbb{I}\left( \max_i |\partial_l f_i(x_{k,0})| \geq Q_k \text{ and } \max_i |\partial_l f_i(x_{k-j,0})|  \geq \sum_{p=k-j}^{k}Q_p   \right)$ and  $\tilde{\mathbb{I}}_{k,k-j} := \mathbb{I}\left( \max_i |\partial_l f_i(x_{k,0})| \geq Q_k \text{ and } \max_i |\partial_l f_i(x_{k-j,0})|  \leq \sum_{p=k-j}^{k}Q_p   \right).$
By Lemma \ref{lemma_indicator}, we know $\mathbb{I}_{k,k-1} = \mathbb{I}_{k-1}$, so we have:

{\small
 \begin{eqnarray*} 
 (a)   & \overset{\text{Lemma \ref{lemma_f1}}}{\geq} & \sum_{l=1}^d \mathbb{I}_{k-1} \frac{\partial_l f(x_{k-1,0}) }{\sqrt{v_{k-1,0}}} 
    (1-\beta_1)J_1\sum_{i=0}^{n-1} \partial_l f_i(x_{k-1,0}) \\
    && - d
    (1-\beta_1)|J_1 |  \frac{1}{1-\frac{1}{\sqrt{\beta_2^n}}} n^2\triangle_{n(k-1)}\sqrt{ \frac{2}{n \beta_2^n}} \\
    &&-   \sum_{l=1}^d
    (1-\beta_1)|J_1 |  \left|\sum_{i=0}^{n-1} \partial_l f_i(x_{k-1,0}) \right|
    \left(\sqrt{\frac{2\rho_3^2}{\beta_2^n}}  \frac{1}{\left(1- \frac{(1-\beta_2)4n \rho_2  }{\beta_2^n}\right) }  \delta_1\right)\\
    &&-  2(1-\beta_1)d \AAA |J_1| n^2 \triangle_{n(k-1)} 
    - (1-\beta_1)d \AAA |J_1| n (Q_k+Q_{k-1})
    \end{eqnarray*} }
    
    {\small 
    \begin{eqnarray*}
     & \overset{\text{Lemma \ref{lemma_delta}}}{\geq} & \BLUE{\sum_{l=1}^d \mathbb{I}_{k-1} \frac{\partial_l f(x_{k-1,0}) }{\sqrt{v_{k-1,0}}} 
    (1-\beta_1)J_1\sum_{i=0}^{n-1} \partial_l f_i(x_{k-1,0})} \\
    && - d
    (1-\beta_1)|J_1 |  \frac{1}{1-\frac{1}{\sqrt{\beta_2^n}}} n^2\triangle_{n(k-1)}\sqrt{ \frac{2}{n \beta_2^n}} \\
    &&-  \sum_{l=1}^d
    (1-\beta_1)|J_1 |  \left(\left|\sum_{i=0}^{n-1} \partial_l f_i(x_{k,0}) \right| + n^2\triangle_{n(k-1)}\right)
    \left(\sqrt{\frac{2\rho_3^2}{\beta_2^n}}  \frac{1}{\left(1- \frac{(1-\beta_2)4n \rho_2  }{\beta_2^n}\right) }   \delta_1\right)\\
    &&- 2(1-\beta_1)d \AAA |J_1| n^2 \triangle_{n(k-1)} 
    - (1-\beta_1)d \AAA |J_1| n (Q_k+Q_{k-1}). \\
 \end{eqnarray*} }
 
 The blue term will be handled using color-ball method when taking conditional expectation $\Ex_{k-1}(\cdot)$.
Now we derive a lower bound for (b). Similarly as before, we rewrite $\left(M_{l,k}^{\prime}\right)_{k-1}$  and $\left(F_{l,k}\right)_{k-1}$ as follows.

  $$
    \left(M_{l,k}^{\prime}\right)_{k-1}   \overset{ \text{Lemma \ref{lemma_delta}}}{\geq}   \mathscr{M}_{l,k-1} - (1-\beta_1)\beta_1 n^3 \triangle_{n(k-1)}     ;
  $$

  $$
  \left(F_{l,k}\right)_{k-1}  \overset{ \text{Lemma \ref{lemma_delta}}}{\geq}   \left(F_{l,k-1}\right)_{k-1}   -  (1-\beta_1)\beta_1^n n^3 \triangle_{n(k-1)};  
  $$

  where

\begin{eqnarray*}
  \mathscr{M}_{l,k-1} &:=&  (1-\beta_1) \{  \\
  && + \beta_1^n \partial_l f_{\tau_{k-1,n-1}}(x_{k-1,0})+\cdots + \beta_1^{n+n-1}\partial_l f_{\tau_{k-1,0}}(x_{k-1,0})  \nonumber\\
  && + \beta_1^{n-1} \partial_l f_{\tau_{k-1,n-1}}(x_{k-1,0})+\cdots + \beta_1^{n+n-2}\partial_l f_{\tau_{k-1,0}}(x_{k-1,0})  \nonumber\\
  && + \cdots  \nonumber\\
  && \underbrace{+ \beta_1 \partial_l f_{\tau_{k-1,n-1}}(x_{k-1,0})+\cdots + \beta_1^{n}\partial_l f_{\tau_{k-1,0}}(x_{k-1,0})  }_{k-1\text{-th epoch}} 
  \} 
\end{eqnarray*}

\begin{eqnarray*}
  \left(F_{l,k-1}\right)_{k-1}:= \sum_{i=0}^{n-1} \partial_l f_i(x_{k-1,0})&=&  (1-\beta_1)\beta_1^n \{   \nonumber\\
   &&\partial_l f_{n-1}(x_{k-1,0})+ \cdots + \beta_1^{n-1}\partial_l f_{n-1}(x_{k-1,0})  \nonumber\\
   &&+  \partial_l f_{n-2}(x_{k-1,0})+ \cdots+ \beta_1^{n-1}\partial_l f_{n-2}(x_{k-1,0})   \nonumber\\
   &&+ \cdots  \nonumber\\
   &&\underbrace{+\partial_l f_{0}(x_{k-1,0})+ \cdots+ \beta_1^{n-1}\partial_l f_{0}(x_{k-1,0})  }_{(k-1)\text{-th epoch}} 
    \}
 \end{eqnarray*}
 
 We now calculate $(b)$. Using the same idea as in $(a)$, we have
 
 {\small
 \begin{eqnarray*}
   (b) &= & \sum_{l=1}^d \mathbb{I}_k \frac{\partial_l f(x_{k,0}) }{\sqrt{v_{k,0}}} \left(
  \left(M_{l,k}^{\prime}\right)_{k-1}    -  \left(F_{l,k}\right)_{k-1}  \right) \\
  &\overset{\text{Lemma \ref{lemma_f1} and \ref{lemma_delta} }}{\geq} & \sum_{l=1}^d \mathbb{I}_k \frac{\partial_l f(x_{k,0}) }{\sqrt{v_{k,0}}} \left(
 \mathscr{M}_{l,k-1}   -  \left(F_{l,k-1}\right)_{k-1}  \right)- d\AAA (1-\beta_1)\left(\beta_1 +\beta_1^n \right) n^3 \triangle_{n(k-1)} \\
  &\overset{\text{Lemma \ref{lemma_indicator}  }}{=} & \sum_{l=1}^d \mathbb{I}_{k,k-1} \frac{\partial_l f(x_{k,0}) }{\sqrt{v_{k,0}}} \left(
 \mathscr{M}_{l,k-1}   -  \left(F_{l,k-1}\right)_{k-1}  \right) + \sum_{l=1}^d \tilde{\mathbb{I}}_{k,k-1} \frac{\partial_l f(x_{k,0}) }{\sqrt{v_{k,0}}} \left(
 \mathscr{M}_{l,k-1}   -  \left(F_{l,k-1}\right)_{k-1}  \right)\\
  &&- d\AAA (1-\beta_1)\left(\beta_1 +\beta_1^n \right) n^3 \triangle_{n(k-1)} \\
    &\overset{\text{Lemma \ref{lemma_f1}  }}{\geq} & \sum_{l=1}^d \mathbb{I}_{k,k-1} \frac{\partial_l f(x_{k,0}) }{\sqrt{v_{k,0}}} \left(
 \mathscr{M}_{l,k-1}   -  \left(F_{l,k-1}\right)_{k-1}  \right) -  \sum_{l=1}^d \tilde{ \mathbb{I}}_{k,k-1}  \AAA \left|
 \mathscr{M}_{l,k-1}   -  \left(F_{l,k-1}\right)_{k-1}  \right|\\
  &&- d\AAA (1-\beta_1)\left(\beta_1 +\beta_1^n \right) n^3 \triangle_{n(k-1)} \\
      &\overset{\text{Def of $\tilde{ \mathbb{I}}_{k,k-1}$ }}{\geq} & \sum_{l=1}^d \mathbb{I}_{k,k-1} \frac{\partial_l f(x_{k,0}) }{\sqrt{v_{k,0}}} \left(
 \mathscr{M}_{l,k-1}   -  \left(F_{l,k-1}\right)_{k-1}  \right) -  d \AAA(1-\beta_1)(\beta_1+\beta_1^n)n^2  \left( Q_k+Q_{k-1}\right)\\
  &&- d\AAA (1-\beta_1)\left(\beta_1 +\beta_1^n \right) n^3 \triangle_{n(k-1)} \\
  &\overset{\text{Lemma \ref{lemma_k-k-1} }}{\geq} & \sum_{l=1}^d \mathbb{I}_{k,k-1} \frac{\partial_l f(x_{k-1,0}) }{\sqrt{v_{k-1,0}}} \left(
 \mathscr{M}_{l,k-1} -  \left(F_{l,k-1}\right)_{k-1}  \right) \\
  &&-  \sum_{l=1}^d\left(\frac{1}{1-\frac{1}{\sqrt{\beta_{2}^{n}}}} \frac{n^2 \triangle_{n(k-1)}}{\sqrt{v_{k-1,0}}}+\sqrt{\frac{2 \rho_{3}^{2}}{\beta_{2}^{n}}} \frac{1}{\left(1- \frac{(1-\beta_2)4n \rho_2  }{\beta_2^n}\right) } 
 \delta_{1}\right) \left|
 \mathscr{M}_{l,k-1} -  \left(F_{l,k-1}\right)_{k-1}  \right|  \\ 
  && -  d \AAA(1-\beta_1)(\beta_1+\beta_1^n)n^2  \left( Q_k+Q_{k-1}\right)- d\AAA (1-\beta_1)\left(\beta_1 +\beta_1^n \right) n^3 \triangle_{n(k-1)} \\
 \end{eqnarray*} }

To proceed, we derive an upper bound for $\left|
 \mathscr{M}_{l,k-1} -  \left(F_{l,k-1}\right)_{k-1}  \right|$.  

\begin{eqnarray*}
  \left|
  \mathscr{M}_{l,k-1}   -  \left(F_{l,k-1}\right)_{k-1}  \right| &\leq &   \left|
  \mathscr{M}_{l,k-1} \right|  +\left|  \left(F_{l,k-1}\right)_{k-1}  \right| \\
  &\leq & (1-\beta_1) \beta_1 n \sum_{i=0}^{n-1} \left|\partial_l f_i(x_{k-1,0})   \right| + (1-\beta_1) \beta_1^n n \sum_{i=0}^{n-1} \left|\partial_l f_i(x_{k-1,0})   \right| 
\end{eqnarray*}

Therefore, we have:

{\small
\begin{eqnarray*}
  (b)   &\geq & \sum_{l=1}^d \mathbb{I}_{k,k-1} \frac{\partial_l f(x_{k-1,0}) }{\sqrt{v_{k-1,0}}} \left(
  \mathscr{M}_{l,k-1}  -  \left(F_{l,k-1}\right)_{k-1}  \right) \\
  &&-  \sum_{l=1}^d\left(\frac{1}{1-\frac{1}{\sqrt{\beta_{2}^{n}}}} \frac{n^2 \triangle_{n(k-1)}}{\sqrt{v_{k-1,0}}}+\sqrt{\frac{2 \rho_{3}^{2}}{\beta_{2}^{n}}} \frac{1}{\left(1- \frac{(1-\beta_2)4n \rho_2  }{\beta_2^n}\right) }  \delta_{1}\right)(1-\beta_1) (\beta_1+\beta_1^n) n \sum_{i=0}^{n-1} \left|\partial_l f_i(x_{k-1,0})   \right|   \\ 
  && -  d \AAA(1-\beta_1)(\beta_1+\beta_1^n)n^2  \left( Q_k+Q_{k-1}\right)- d\AAA (1-\beta_1)\left(\beta_1 +\beta_1^n \right) n^3 \triangle_{n(k-1)}  \end{eqnarray*} }
    
    {\small 
    \begin{eqnarray*}
   &\overset{\text{Lemma \ref{lemma_f1} and \ref{lemma_indicator}}}{\geq} & \RED{ \sum_{l=1}^d \mathbb{I}_{k-1} \frac{\partial_l f(x_{k-1,0}) }{\sqrt{v_{k-1,0}}} \left(
  \mathscr{M}_{l,k-1}   -  \left(F_{l,k-1}\right)_{k-1}  \right)} \\
   &&- d\left(\frac{1}{1-\frac{1}{\sqrt{\beta_{2}^{n}}}} 
   \sqrt{ \frac{2}{n \beta_2^n}}
   \right)(1-\beta_1) (\beta_1+\beta_1^n) n^3   \triangle_{n(k-1)} \\ 
  &&-  \sum_{l=1}^d\sqrt{\frac{2 \rho_{3}^{2}}{\beta_{2}^{n}}} \frac{1}{\left(1- \frac{(1-\beta_2)4n \rho_2  }{\beta_2^n}\right) }  \delta_{1}(1-\beta_1) (\beta_1+\beta_1^n) n \sum_{i=0}^{n-1} \left(\left|\partial_l f_i(x_{k-1,0})   \right|  \right)  \\ 
  && -  d \AAA(1-\beta_1)(\beta_1+\beta_1^n)n^2  \left( Q_k+Q_{k-1}\right)- d\AAA (1-\beta_1)\left(\beta_1 +\beta_1^n \right) n^3 \triangle_{n(k-1)} \\
\end{eqnarray*}
}

The red term will be handled using color-ball method when taking conditional expectation $Ex_{k-1}(\cdot)$. Now we have derived lower bounds for both $(a)$ and $(b)$. Combining together, we have:

{\small
\begin{eqnarray*}
  &&\Ex_{k-1} \left\{\Ext  \left[ \BLUE{\sum_{l=1}^d \mathbb{I}_k \frac{\partial_l f(x_{k,0}) }{\sqrt{v_{k,0}}} \left( \left(M_{l,k}^{\prime}\right)_k-   \left(F_{l,k}\right)_k  \right)} + \RED{\sum_{l=1}^d \mathbb{I}_k \frac{\partial_l f(x_{k,0}) }{\sqrt{v_{k,0}}} \left(
  \left(M_{l,k}^{\prime}\right)_{k-1}     -  \left(F_{l,k}\right)_{k-1}  \right) } \right] \right\} \nonumber \\
  &\overset{\eqref{eq_cancel_k-1}}{\geq}& 
  \Ex_{k-1} \left\{ \underbrace{\BLUE{ \sum_{l=1}^d \mathbb{I}_k \frac{\partial_l f(x_{k,0}) }{\sqrt{v_{k,0}}} 
    (1-\beta_1)J_1\sum_{i=0}^{n-1} \partial_l f_i(x_{k,0}) }}_{(a)}+  \underbrace{ \RED{\sum_{l=1}^d \mathbb{I}_k \frac{\partial_l f(x_{k,0}) }{\sqrt{v_{k,0}}} \left(
  \left(M_{l,k}^{\prime}\right)_{k-1}     -  \left(F_{l,k}\right)_{k-1}  \right) }}_{(b)}  \right\} \nonumber\\
  &&-  d\AAA\frac{(1-\beta_1)n^2 (n-1)\triangle_{nk}}{2} \\
  &\overset{}{\geq}& 
  \Ex_{k-1} \left\{\BLUE{ \sum_{l=1}^d \mathbb{I}_{k-1} \frac{\partial_l f(x_{k-1,0}) }{\sqrt{v_{k-1,0}}} 
    (1-\beta_1)J_1\sum_{i=0}^{n-1} \partial_l f_i(x_{k-1,0}) } +  \RED{ \sum_{l=1}^d \mathbb{I}_{k-1} \frac{\partial_l f(x_{k-1,0}) }{\sqrt{v_{k-1,0}}} \left(
  \mathscr{M}_{l,k-1}   -  \left(F_{l,k-1}\right)_{k-1}  \right)}  \right\}  \nonumber\\
  &&\BLUE{- d
    (1-\beta_1)|J_1 |  \frac{1}{1-\frac{1}{\sqrt{\beta_2^n}}} n^2\triangle_{n(k-1)}\sqrt{ \frac{2}{n \beta_2^n}}  }\\
    && \BLUE{-  \sum_{l=1}^d
    (1-\beta_1)|J_1 |  \left(\left|\sum_{i=0}^{n-1} \partial_l f_i(x_{k,0}) \right| + n^2\triangle_{n(k-1)}\right)
    \left(\sqrt{\frac{2\rho_3^2}{\beta_2^n}}  \frac{1}{\left(1- \frac{(1-\beta_2)4n \rho_2  }{\beta_2^n}\right) }   \delta_1\right) }\\
    &&\BLUE{- 2(1-\beta_1)d \AAA |J_1| n^2 \triangle_{n(k-1)} 
    -  (1-\beta_1)d \AAA |J_1| n (Q_k+Q_{k-1})} \\
    && \RED{- d\left(\frac{1}{1-\frac{1}{\sqrt{\beta_{2}^{n}}}} 
   \sqrt{ \frac{2}{n \beta_2^n}}
   \right)(1-\beta_1) (\beta_1+\beta_1^n) n^2   \triangle_{n(k-1)} }\\ 
  &&\RED{-  \sum_{l=1}^d\sqrt{\frac{2 \rho_{3}^{2}}{\beta_{2}^{n}}} \frac{1}{\left(1- \frac{(1-\beta_2)4n \rho_2  }{\beta_2^n}\right) }  \delta_{1}(1-\beta_1) (\beta_1+\beta_1^n) n \left(\sum_{i=0}^{n-1} \left|\partial_l f_i(x_{k,0})   \right|  +n^3 \triangle_{n(k-1)}\right) } \\ 
  && \RED{-  d \AAA(1-\beta_1)(\beta_1+\beta_1^n)n^2  \left( Q_k+Q_{k-1}\right)- d\AAA (1-\beta_1)\left(\beta_1 +\beta_1^n \right) n^3 \triangle_{n(k-1)} }\\
    &&-  d\AAA\frac{(1-\beta_1)n^2 (n-1)\triangle_{nk}}{2} 
 \end{eqnarray*}
 }

For the first term in the above inequality, we can calculate it using the idea in the color-ball toy example:

{\small
\begin{eqnarray*}
   &&\Ex_{k-1} \left\{\BLUE{ \sum_{l=1}^d \mathbb{I}_{k-1} \frac{\partial_l f(x_{k-1,0}) }{\sqrt{v_{k-1,0}}} 
    (1-\beta_1)J_1\sum_{i=0}^{n-1} \partial_l f_i(x_{k-1,0}) } +  \RED{ \sum_{l=1}^d \mathbb{I}_{k-1} \frac{\partial_l f(x_{k-1,0}) }{\sqrt{v_{k-1,0}}} \left(
  \mathscr{M}_{l,k-1}   -  \left(F_{l,k-1}\right)_{k-1}  \right)}  \right\}  \\
  &=&\sum_{l=1}^d \mathbb{I}_{k-1} \frac{\partial_l f(x_{k-1,0}) }{\sqrt{v_{k-1,0}}} 
    (1-\beta_1) \beta_1^n J_1\sum_{i=0}^{n-1} \partial_l f_i(x_{k-1,0}). 
\end{eqnarray*}
}

To proceed, we further take $\Ex_{k-2} (\cdot)$ and bound  

{\small
  $$\Ex_{k-2} \left\{ \underbrace{\BLUE{ \sum_{l=1}^d \mathbb{I}_{k-1} \frac{\partial_l f(x_{k-1,0}) }{\sqrt{v_{k-1,0}}} 
    (1-\beta_1)\beta_1^n J_1\sum_{i=0}^{n-1} \partial_l f_i(x_{k-1,0}) }}_{(a)}+  \underbrace{ \RED{\sum_{l=1}^d \mathbb{I}_k \frac{\partial_l f(x_{k,0}) }{\sqrt{v_{k,0}}} \left(
  \left(M_{l,k}^{\prime}\right)_{k-2}     -  \left(F_{l,k}\right)_{k-2}  \right) }}_{(b)}  \right\}. $$}
  
  Repeat this process until $k=1$, we have:
  
\begin{eqnarray*}
     && \Ex \left[ \sum_{l \text{ large}} \frac{\partial_l f(x_{k,0}) }{\sqrt{v_{k,0}}} \left( \left(M_{l,k}^{\prime}\right)_k - \left(F_{l,k}\right)_k +  \cdots+ \left(M_{l,k}^{\prime}\right)_1- \left(F_{l,k}\right)_1  \right)   \right]  \\
      &\geq& (1-\beta_1)\beta_1^{kn} J_1 \Ex \left[\sum_{l=1}^d \mathbb{I}_1 \frac{\partial_l f(x_{1,0}) }{\sqrt{v_{1,0}}} 
   \sum_{i=0}^{n-1} \partial_l f_i(x_{1,0})\right] \\
   &&+\BLUE{ \text{Error}_1 } + \RED{\text{Error}_2 } +\text{Error}_3 \\
    &\overset{(*)}{\geq}& -   \frac{\beta_1^n}{\sqrt{k-1}}(1-\beta_1)\left|J_1 \Ex \left[\sum_{l=1}^d \mathbb{I}_1 \frac{\partial_l f(x_{1,0}) }{\sqrt{v_{1,0}}} 
   \sum_{i=0}^{n-1} \partial_l f_i(x_{1,0})\right]  \right|\\
   &&+\BLUE{ \text{Error}_1 } + \RED{\text{Error}_2 } +\text{Error}_3
\end{eqnarray*}

 where $(*)$ holds for large $k$ such that $\beta^{(k-1)n} \leq \frac{\beta_1^n}{\sqrt{k-1}}$.
We specify $\BLUE{ \text{Error}_1 },  \RED{\text{Error}_2 }, \text{Error}_3$ as follows.

{\small
\begin{eqnarray*}
  \BLUE{ \text{Error}_1 } &= &  \BLUE{- d
    (1-\beta_1)|J_1 |  \frac{1}{1-\frac{1}{\sqrt{\beta_2^n}}} n^2\sqrt{ \frac{2}{n \beta_2^n}}  \left(\sum_{j=0}^{\infty}  \beta_1^{jn}\triangle_{n(k-j)}  \right) }\\
    && \BLUE{-  \sum_{l=1}^d
    (1-\beta_1)|J_1 |  \left(\left|\sum_{i=0}^{n-1} \partial_l f_i(x_{k,0}) \right| \right)
    \left(\sqrt{\frac{2\rho_3^2}{\beta_2^n}}  \frac{1}{\left(1- \frac{(1-\beta_2)4n \rho_2  }{\beta_2^n}\right) }   \delta_1\right) \left(\sum_{j=0}^{\infty}  \beta_1^{jn} (j+1)  \right)}   \\
    && \BLUE{-  d
    (1-\beta_1) \triangle_{n(k-1)}|J_1 |  n^2
    \left(\sqrt{\frac{2\rho_3^2}{\beta_2^n}}  \frac{1}{\left(1- \frac{(1-\beta_2)4n \rho_2  }{\beta_2^n}\right) }   \delta_1\right)\left(\sum_{j=0}^{\infty}  \beta_1^{jn} (j+1)^2 \right) }\\
    &&\BLUE{- 2(1-\beta_1)d \AAA |J_1| n^2 \left(\sum_{j=0}^{\infty}  \beta_1^{jn} \triangle_{n(k-1-j)}  \right)} \\
    && \BLUE{- (1-\beta_1)d \AAA |J_1| n \left[ Q_k +Q_{k-1} + \beta_1^n(Q_k+Q_{k-1} + Q_{k-2})+     \cdots   \right]} 
\end{eqnarray*}
}

Since $\delta_2 = \lim_{k \rightarrow \infty} \sum_{j=1}^{k-1} (\beta_1^n)^j \sqrt{\frac{k}{k-j}}$ is a finite constant, we have $ \sum_{j=0}^{\infty}  \beta_1^{jn}\triangle_{n(k-j)}  = \frac{\triangle_{1} \delta_2}{\sqrt{nk}}$. In addition, we have

\begin{eqnarray*}
 && \left[Q_k + Q_{k-1}+ \beta_1^n(Q_k+Q_{k-1}+Q_{k-2})+  \cdots \right] \\
  & \overset{(*)}{ \leq}&  Q_k (1+ \beta_1^n + \cdots) \\
  && + Q_{k-1} + \beta_1^n (Q_{k-1} + Q_{k-2}) + \beta_1^{2n} (Q_{k-1} + Q_{k-2} + Q_{k-3})+\cdots \\
  &\overset{(*)}{\leq}& Q_k (1+ \beta_1^n + \cdots)
  +2Q_{k-1}+ 2\left(\beta_1^n 2Q_{k-1}+ \beta_1^{2n} 3Q_{k-1} +\cdots\right)  \\
  &\overset{\text{Lemma \ref{lemma_beta}}}{ \leq}&
  Q_k \frac{1}{(1-\beta_1^n)} + 
  2Q_{k-1}  \frac{1}{(1-\beta_1^n)^2} \\
  &\leq & 5 Q_k \frac{1}{(1-\beta_1^n)^2}
\end{eqnarray*}

where $(*)$ uses the following fact:  consider integers $k>J>0$, we have $\sum_{j=1}^{J} \frac{1}{\sqrt{k-j}} \leq 2 \frac{J}{\sqrt{k-1}} $.  This inequality can be simply proved by taking the integral over $\frac{1}{\sqrt{k}}$.  The final inequality is due to $Q_{k-1}< 2Q_{k}$.
Now, we have


\begin{eqnarray*}
   \BLUE{ \text{Error}_1 } &\geq &   \BLUE{- d
    (1-\beta_1)|J_1 |  \frac{1}{1-\frac{1}{\sqrt{\beta_2^n}}} n^2\sqrt{ \frac{2}{n \beta_2^n}}  \frac{\triangle_{1} \delta_2}{\sqrt{nk}}}\\
    && \BLUE{-  \sum_{l=1}^d
    (1-\beta_1)|J_1 |  \left(\left|\sum_{i=0}^{n-1} \partial_l f_i(x_{k,0}) \right| \right)
    \left(\sqrt{\frac{2\rho_3^2}{\beta_2^n}}  \frac{1}{\left(1- \frac{(1-\beta_2)4n \rho_2  }{\beta_2^n}\right) }   \delta_1\right) \frac{1}{(1-\beta_1^n)^2}}   \\
    && \BLUE{- 2 d
    (1-\beta_1) |J_1 |  n^2
    \left(\sqrt{\frac{2\rho_3^2}{\beta_2^n}}  \frac{1}{\left(1- \frac{(1-\beta_2)4n \rho_2  }{\beta_2^n}\right) }   \delta_1\right)\frac{1+\beta_1^n}{(1-\beta_1^n)^3} \frac{\triangle_{1}}{\sqrt{n(k-1)}}} \\
    &&\BLUE{- (1-\beta_1)d \AAA |J_1| n^2 \frac{\triangle_{1} \delta_2}{\sqrt{n(k-1)}}}  \BLUE{- (1-\beta_1)d \AAA |J_1| n  5 Q_k  \frac{1}{(1-\beta_1^n)^2}} 
\end{eqnarray*}

\begin{eqnarray*}
 \RED{\text{Error}_2}& = & \RED{- d\left(\frac{1}{1-\frac{1}{\sqrt{\beta_{2}^{n}}}} 
   \sqrt{ \frac{2}{n \beta_2^n}}
   \right)(1-\beta_1) (\beta_1+\beta_1^n) n^3   \left(\sum_{j=0}^{\infty}  \beta_1^{jn} \triangle_{n(k-1-j)}  \right)   }\\ 
  &&\RED{-  \sum_{l=1}^d\sqrt{\frac{2 \rho_{3}^{2}}{\beta_{2}^{n}}} \frac{1}{\left(1- \frac{(1-\beta_2)4n \rho_2  }{\beta_2^n}\right) }  \delta_{1}(1-\beta_1) (\beta_1+\beta_1^n) n \left(\sum_{i=0}^{n-1} \left|\partial_l f_i(x_{k,0})   \right| \right) \left(\sum_{j=0}^{\infty}(j+1)\beta_1^{nj} \right) } \\ 
    &&\RED{-  d\sqrt{\frac{2 \rho_{3}^{2}}{\beta_{2}^{n}}} \frac{1}{\left(1- \frac{(1-\beta_2)4n \rho_2  }{\beta_2^n}\right) }  \delta_{1}(1-\beta_1) (\beta_1+\beta_1^n) n^3\triangle_{n(k-1)}  \left( \sum_{j=0}^{\infty} (j+1)^2\beta_1^{nj}   \right) } \\ 
  && \RED{-  d \AAA(1-\beta_1)(\beta_1+\beta_1^n)n^2  \left(  (Q_k+Q_{k-1}) + \beta_1^n (Q_k+Q_{k-1}+ Q_{k-2})  +\cdots \right) }\\
  &&\RED{- d\AAA (1-\beta_1)\left(\beta_1 +\beta_1^n \right) n^3 \left(\sum_{j=0}^{\infty}  \beta_1^{jn} \triangle_{n(k-1-j)}  \right) }
\end{eqnarray*}

Based on the calculation in Lemma \ref{lemma_beta}, we have

\begin{eqnarray*}
 \RED{\text{Error}_2}& \geq & \RED{- d\left(\frac{1}{1-\frac{1}{\sqrt{\beta_{2}^{n}}}} 
   \sqrt{ \frac{2}{n \beta_2^n}}
   \right)(1-\beta_1) (\beta_1+\beta_1^n) n^3   \frac{\triangle_1}{\sqrt{n}} \frac{\delta_2}{\sqrt{k}}  }\\ 
  &&\RED{-  \sum_{l=1}^d\sqrt{\frac{2 \rho_{3}^{2}}{\beta_{2}^{n}}} \frac{1}{\left(1- \frac{(1-\beta_2)4n \rho_2  }{\beta_2^n}\right) }  \delta_{1}(1-\beta_1) (\beta_1+\beta_1^n) n \left(\sum_{i=0}^{n-1} \left|\partial_l f_i(x_{k,0})   \right| \right) \frac{1}{(1-\beta_1^n)^2}} \\ 
    &&\RED{-  d\sqrt{\frac{2 \rho_{3}^{2}}{\beta_{2}^{n}}} \frac{1}{\left(1- \frac{(1-\beta_2)4n \rho_2  }{\beta_2^n}\right) }  \delta_{1}(1-\beta_1) (\beta_1+\beta_1^n) n^3 \frac{1+\beta_1^n}{(1-\beta_1^n)^3} \frac{\triangle_{1} \delta_2}{\sqrt{n(k-1)}} }  \\ 
  && \RED{-  d \AAA(1-\beta_1)(\beta_1+\beta_1^n)n^2  2Q_k \frac{1}{\beta_1^n(1-\beta_1^n)^2} }\\
  &&\RED{- d\AAA (1-\beta_1)\left(\beta_1 +\beta_1^n \right) n^3 \frac{\triangle_1}{\sqrt{n}} \frac{\delta_2}{\sqrt{k}}. }
\end{eqnarray*}

\begin{equation*}
 \text{Error}_3 =  -  d\AAA(1-\beta_1)n^3  \left(\sum_{j=0}^{\infty}  \beta_1^{jn}\triangle_{n(k-j)}  \right) =-\frac{d\AAA(1-\beta_1)n^3}{\sqrt{k}} \frac{\triangle_{1}}{\sqrt{n}} \delta_2,
\end{equation*}

Since  $\frac{1-\beta_1}{1-\beta_1^n} \leq 1$  and $|J_1| \leq n$, we have

\begin{eqnarray*}
  &&\BLUE{  
    |J_1 |  
    \left(\sqrt{\frac{2\rho_3^2}{\beta_2^n}}  \frac{1}{\left(1- \frac{(1-\beta_2)4n \rho_2  }{\beta_2^n}\right) }   \right) \frac{1-\beta_1}{(1-\beta_1^n)^2}} \RED{+ \sqrt{\frac{2 \rho_{3}^{2}}{\beta_{2}^{n}}} \frac{1}{\left(1- \frac{(1-\beta_2)4n \rho_2  }{\beta_2^n}\right) }  \frac{ 2 n (1-\beta_1)}{(1-\beta_1^n)^2}} \\
    &\leq&     \left(\sqrt{\frac{2\rho_3^2}{\beta_2^n}}  \frac{3n}{\left(1-\frac{1-\beta_{2}}{2}\left(-1+\frac{4 \rho_{2} n}{\beta_{2}^{n}}\right)\right) }  \right)\frac{1}{(1-\beta_1^n)}.
\end{eqnarray*}

Using the fact that
 $\frac{1}{\sqrt{k-1}} \leq \frac{\sqrt{2}}{\sqrt{k}}$ (for $k \geq 2$), we have:

{\small
\begin{eqnarray*}
      &&\Ex \left[ \sum_{l \text{ large}} \frac{\partial_l f(x_{k,0}) }{\sqrt{v_{k,0}}} \left( \left(M_{l,k}^{\prime}\right)_k - \left(F_{l,k}\right)_k +  \cdots+ \left(M_{l,k}^{\prime}\right)_1- \left(F_{l,k}\right)_1  \right)   \right]  \\
      &\geq& 
    -   \frac{\beta_1^n}{\sqrt{k-1}}(1-\beta_1)\left|J_1 \Ex \left[\sum_{l=1}^d \mathbb{I}_1 \frac{\partial_l f(x_{1,0}) }{\sqrt{v_{1,0}}} 
   \sum_{i=0}^{n-1} \partial_l f_i(x_{1,0})\right]  \right|\\
   &&+\BLUE{ \text{Error}_1 } + \RED{\text{Error}_2 } +\text{Error}_3 \\
    &\geq & -\frac{G_6}{\sqrt{k}} - G_7 \delta_1 \Ex\left(\sum_{l=1}^d\sum_{i=0}^{n-1} \left|\partial_l f_i(x_{k,0})   \right| \right),
\end{eqnarray*}
}
where 

\begin{eqnarray*}
  G_6 &:=&  \beta_1^n(1-\beta_1) \sqrt{2}\left|J_1 \right| \AAA \Ex \left[\sum_{l=1}^d  \left|\mathbb{I}_1  
   \sum_{i=0}^{n-1} \partial_l f_i(x_{1,0}) \right| \right]  +\frac{\triangle_{1}}{\sqrt{n}} \delta_2 d\AAA(1-\beta_1)n^3 \\
    &&\BLUE{+ d
    (1-\beta_1)|J_1 |  \frac{1}{1-\frac{1}{\sqrt{\beta_2^n}}} n^2\sqrt{ \frac{2}{n \beta_2^n}}  \frac{\triangle_{1} \delta_2}{\sqrt{n}}}\\
    && \BLUE{+  d
    (1-\beta_1) |J_1 |  n^2
    \left(\sqrt{\frac{2\rho_3^2}{\beta_2^n}}  \frac{1}{\left(1- \frac{(1-\beta_2)4n \rho_2  }{\beta_2^n}\right) }   \delta_1\right)\frac{1+\beta_1^n}{(1-\beta_1^n)^3} \frac{\triangle_{1}\sqrt{2}}{\sqrt{n}}} \\
    &&\BLUE{+ 2(1-\beta_1)d \AAA |J_1| n^2 \frac{\triangle_{1} \delta_2 \sqrt{2}}{\sqrt{n}}}  \BLUE{- 5(1-\beta_1)d \AAA |J_1| n  \frac{1}{(1-\beta_1^n)^2}\Delta_{1} n \sqrt{n} \frac{32 \sqrt{2}}{\left(1-\beta_{2}\right)^{n} \beta_{2}^{n}}  }  \\
    &&\RED{+ d\left(\frac{1}{1-\frac{1}{\sqrt{\beta_{2}^{n}}}} 
   \sqrt{ \frac{2}{n \beta_2^n}}
   \right)(1-\beta_1) (\beta_1+\beta_1^n) n^3   \frac{\triangle_1}{\sqrt{n}} \delta_2  }\\ 
    &&\RED{+  d\sqrt{\frac{2 \rho_{3}^{2}}{\beta_{2}^{n}}} \frac{1}{\left(1- \frac{(1-\beta_2)4n \rho_2  }{\beta_2^n}\right) }  \delta_{1}(1-\beta_1) (\beta_1+\beta_1^n) n^3 \frac{1+\beta_1^n}{(1-\beta_1^n)^3} \frac{\triangle_{1} \delta_2 \sqrt{2}}{\sqrt{n}} }  \\ 
  && \RED{+  2d \AAA(1-\beta_1)(\beta_1+\beta_1^n)n^2   \frac{1}{\beta_1^n(1-\beta_1^n)^2} \Delta_{1} n \sqrt{n} \frac{32 \sqrt{2}}{\left(1-\beta_{2}\right)^{n} \beta_{2}^{n}}  }\\
  &&\RED{+ d\AAA (1-\beta_1)\left(\beta_1 +\beta_1^n \right) n^3 \frac{\triangle_1}{\sqrt{n}} \delta_2, }
\end{eqnarray*}

\begin{eqnarray*}
  G_7 &:=& 
    \left(\sqrt{\frac{2\rho_3^2}{\beta_2^n}}  \frac{3n}{\left(1-\frac{1-\beta_{2}}{2}\left(-1+\frac{4 \rho_{2} n}{\beta_{2}^{n}}\right)\right) }  \right)\frac{1}{(1-\beta_1^n)} 
\end{eqnarray*}

 where $J_1 = \left(-\frac{1}{n} \beta_{1}-\frac{2}{n} \beta_{1}^{2}-\cdots-\frac{n-1}{n} \beta_{1}^{n-1}\right)$,  $\delta_2 = \lim_{k \rightarrow \infty} \sum_{j=1}^{k-1} (\beta_1^n)^j \sqrt{\frac{k}{k-j}}$ (if needed, we can further bound $J_1$ by $n$ for simplicity).
 This conclude the  proof.

\subsection{Proof of Lemma \ref{lemma_da2_db2}}
\label{appendix_lemma_da2_db2}

We now derive upper bounds for $ \sum_{l \text{ large }} (a_2)  $ and $\sum_{l \text{ large }} (b_2) $. The upper bound for $ \sum_{l \text{ large }} (a_2) $ is very straightforward using inequality \eqref{eq_fi_f_2}.

  \begin{eqnarray*}
    \sum_{l \text{ large }} (a_2)  =   \delta_{1} \sqrt{\frac{2\rho_3^2}{\beta_2^n} }\sum_{l \text{ large }}\sum_{i=0}^{n-1}\left|\partial_{l} f_i\left(x_{k, 0} \right)\right| \overset{\eqref{eq_fi_f_2}}{\leq}  \delta_{1} \sqrt{\frac{2\rho_3^2}{\beta_2^n} } \sqrt{D_{1}} \rho_{1} d\left(\left|\partial_{\alpha} f\left(x_{k, 0}\right)\right|+\sqrt{\frac{D_{0}}{D_{1} d}}\right).
  \end{eqnarray*}

Now we shift gear to $\sum_{l \text{ large }} (b_2):=   \delta_1\sqrt{\frac{2\rho_3^2}{\beta_2^n} }\sum_{l \text{ large }} \sum_{i=0}^{n-1}|\mlki- \partial_{l} f_i(\xkz)|  $. To proceed, we need an upper bound for $ \sum_{l \text{ large }} \sum_{i=0}^{n-1}|\mlki- \partial_{l} f_i(\xkz)|$.  For each $i$, we perform the following decomposition.

  \begin{eqnarray}
    \sum_{l \text{ large }}|\mlki- \partial_{l} f_i(\xkz)|  &\leq & \underbrace{ \sum_{l=1}^d  |(1-\beta_1) \left[\beta_1^i \partial_l f_{\tau_{k,0}}(x_{k,0})+ \cdots+ \partialtauki (x_{k,i})\right] -\partial_{l} f_i(\xkz) |}_{(d_1)}   \nonumber \\
    &&+ \underbrace{ \sum_{l=1}^d  \left| \beta_1^{i+1} m_{l,k-1,n-1}\right| }_{(d_2)}. \label{eq_mlki-f}
  \end{eqnarray}
    
  To start, we bound $(d_1)$. 
  



  \begin{eqnarray}
    (d_1) &\leq & \sum_{l=1}^d \left\{ |\partial_l f_{\tau_{k,0}}(x_{k,0})|+ |\partial_l f_{\tau_{k,1}}(x_{k,1})|+\cdots+ |\partialtauki (x_{k,i})| +|\partial_{l} f_i(x_{k,0}) | \right\}  \nonumber \\
    &\leq & \sum_{l=1}^d \left\{ |\partial_l f_{\tau_{k,0}}(x_{k,0})|+ |\partial_l f_{\tau_{k,1}}(x_{k,1})|+\cdots+ |\partial_l f_{\tau_{k,n-1}}(x_{k,n-1})| +|\partial_{l} f_i(x_{k,0}) | \right\}  \nonumber \\
    &\leq & \sum_{l=1}^d \sum_{i=0}^{n-1} \bigm\{  |\partial_{l} f_i(x_{k,0}) |  \bigm\} + \sum_{l=1}^d |\partial_{l} f_i(x_{k,0}) | 
    + d\triangle_{nk} + \cdots + n d\triangle_{nk}   \nonumber \\
    &\overset{\eqref{eq_fi_f_2}}{ \leq } &  
    2\sqrt{D_{1}} \rho_{1} d\left(\left|\partial_{\alpha} f\left(x_{k, 0}\right)\right|+\sqrt{\frac{D_{0}}{D_{1} d}}\right) +\frac{n(n+1)d}{2}  \triangle_{nk}  \nonumber \\
    & =&    2\sqrt{D_{1}} \rho_{1} d\left(\left|\partial_{\alpha} f\left(x_{k, 0}\right)\right|+\sqrt{\frac{D_{0}}{D_{1} d}}\right)   +\frac{(n+1)\sqrt{n}d}{2\sqrt{k}}  \triangle_{1} .
    \label{eq_d1}
  \end{eqnarray}

  Now, we bound $(d_2)$. Recall $\sum_{i=0}^{n-1} \partial_l f_i(x_{1,0}) =\partial_l f(x_{1,0}) $, we have 

  \begin{eqnarray}
    |m_{l,k-1,n-1}| &\leq& (1-\beta_1) \left[   |\partial_l f_{\tau_{k-1,n-1}} (x_{k-1,n-1})|  +\beta_1  |\partial_l f_{\tau_{k-1,n-2}} (x_{k-1,n-2})|  +\cdots \right] \nonumber\\
    & &  +\beta_1^{(k-1)n} \sum_{i=0}^{n-1} |\partial_l f_i(x_{1,0}) | \label{eq_mlk-1n-1}
  \end{eqnarray}
  
  Note that for any $i\in [0,n-1]$, $j\in [0,n-1]$, $t\in [1,k-1]$, we have the following result.

  \begin{eqnarray}
    |\partial_l f_i(x_{k-t,j})| &\leq &  |\partial_l f_i(x_{k,0})|  +  |\partial_l f_i(x_{k,0}) -  \partial_l f_i(x_{k-t,j}) |  \nonumber \\
    & \overset{\text{Lemma }\ref{lemma_delta}}{\leq}&    |\partial_l f_i(x_{k,0})|  +(n-j)\triangle_{(k-t)n} + n \triangle_{(k-t+1)n} + \cdots + n \triangle_{(k-1)n} \nonumber \\
    &\leq & |\partial_l f_i(x_{k,0})| + \triangle_{(k-1)n}+ \triangle_{(k-1)n-1} +\cdots +\triangle_{(k-1)n- [(n-j)+(t-1)n-1]} \nonumber \\
    &\leq & |\partial_l f_i(x_{k,0})| + \triangle_{(k-1)n-1}+ \triangle_{(k-1)n-2} +\cdots +\triangle_{(k-1)n- [(n-j)+(t-1)n]} \nonumber \\
    &\overset{(*)}{\leq}& |\partial_l f_i(x_{k,0})| + \frac{2[(n-j)+(t-1)n] \triangle_1}{ \sqrt{n(k-1)-1}}, \label{eq_fk-tj}
  \end{eqnarray}
  
  where $(*)$ uses the following fact:  consider integers $k>J>0$, we have $\sum_{j=1}^{J} \frac{1}{\sqrt{k-j}} \leq 2 \frac{J}{\sqrt{k-1}} $.



Plugging  \eqref{eq_fk-tj} into \eqref{eq_mlk-1n-1} and  re-arranging the index, we have
{\small
\begin{eqnarray}
  \sum_{l=1}^d   |m_{l,k-1,n-1}| &\overset{\eqref{eq_mlk-1n-1}}{ \leq}&  (1-\beta_1)\sum_{l=1}^d  \left[   |\partial_l f_{\tau_{k-1,n-1}} (x_{k-1,n-1})|  +\beta_1  |\partial_l f_{\tau_{k-1,n-2}} (x_{k-1,n-2})|  +\cdots \right] \nonumber\\
    & &  +\beta_1^{(k-1)n} \sum_{l=1}^d \sum_{i=0}^{n-1} |\partial_l f_i(x_{1,0}) |  \nonumber\\
    &\overset{\eqref{eq_fi_f_2}, \eqref{eq_fk-tj}}{ \leq} & (1-\beta_1) \sum_{q=1}^{(k-1)n} \beta_1^{q-1} \left[ \sqrt{D_{1}} \rho_{1} d\left(\left|\partial_{\alpha} f\left(x_{k, 0}\right)\right|+\sqrt{\frac{D_{0}}{D_{1} d}}\right) +  \frac{2qd \triangle_1}{ \sqrt{n(k-1)-1}} \right]   \nonumber \\
    && +\beta_1^{(k-1)n}  \sum_{l=1}^d\sum_{i=0}^{n-1} |\partial_l f_i(x_{1,0}) |  \nonumber \\
    &\leq & (1-\beta_1) \sum_{q=1}^{\infty} \beta_1^{q-1} \left[ \sqrt{D_{1}} \rho_{1} d\left(\left|\partial_{\alpha} f\left(x_{k, 0}\right)\right|+\sqrt{\frac{D_{0}}{D_{1} d}}\right) +  \frac{2qd \triangle_1}{ \sqrt{n(k-1)-1}} \right]   \nonumber \\
    && +\beta_1^{(k-1)n} \sum_{l=1}^d \sum_{i=0}^{n-1} |\partial_l f_i(x_{1,0}) |  \nonumber \\
    &\overset{\text{Lemma }\ref{lemma_beta}}{ \leq}& \sqrt{D_{1}} \rho_{1} d\left(\left|\partial_{\alpha} f\left(x_{k, 0}\right)\right|+\sqrt{\frac{D_{0}}{D_{1} d}}\right)  +  \frac{1}{1-\beta_1}\frac{2d \triangle_1}{ \sqrt{n(k-1)-1}} \nonumber \\
    && +\beta_1^{(k-1)n} \sum_{l=1}^d \sum_{i=0}^{n-1} |\partial_l f_i(x_{1,0}) |   \nonumber \\
  & \overset{ \text{When $k \geq 4$}
      }{\leq} &  \sqrt{D_{1}} \rho_{1} d\left(\left|\partial_{\alpha} f\left(x_{k, 0}\right)\right|+\sqrt{\frac{D_{0}}{D_{1} d}}\right)  +  \frac{1}{1-\beta_1}\frac{2\sqrt{2}d\triangle_1}{ \sqrt{nk}} \nonumber \\
      && +\beta_1^{(k-1)n} \sum_{l=1}^d \sum_{i=0}^{n-1} |\partial_l f_i(x_{1,0}) | \nonumber \\ 
      & \overset{ (i)
      }{\leq} &  \sqrt{D_{1}} \rho_{1} d\left(\left|\partial_{\alpha} f\left(x_{k, 0}\right)\right|+\sqrt{\frac{D_{0}}{D_{1} d}}\right)  +  \frac{1}{1-\beta_1}\frac{2\sqrt{2}d\triangle_1}{ \sqrt{nk}} \nonumber \\
      && +\frac{\sqrt{2}\beta_1^n}{\sqrt{k}} \sum_{l=1}^d \sum_{i=0}^{n-1} |\partial_l f_i(x_{1,0}) |  \label{eq_mlk-1n-1_done}
\end{eqnarray}
}

where the second last inequality holds because: when $k \geq 4$, $\frac
{1}{\sqrt{n(k-1)-1}} \leq \frac{\sqrt{2}}{\sqrt{nk}}$. $(i)$ holds when $k$ is large enough such that $\beta_1^{(k-1)n} \leq \frac{\beta_1^n}{\sqrt{k-1}}.$ Further,  $ \frac{\beta_1^n}{\sqrt{k-1}}\leq  \frac{\sqrt{2}\beta_1^n}{\sqrt{k}} $ when $k \geq 2$.

Now, we have derived upper bounds for $(d_1)$ and $(d_2)$. Plugging \eqref{eq_mlk-1n-1_done} and \eqref{eq_d1} into \eqref{eq_mlki-f}, we conclude the proof.

{\small
\begin{eqnarray}
  \sum_{l=1}^d  \sum_{i=0}^{n-1}|\mlki- \partialtauki(\xkz)|  &\overset{\eqref{eq_d1}, \eqref{eq_mlk-1n-1_done} }{\leq }& 3n \sqrt{D_{1}} \rho_{1} d\left(\left|\partial_{\alpha} f\left(x_{k, 0}\right)\right|+\sqrt{\frac{D_{0}}{D_{1} d}}\right)+\frac{(n+1)nd \sqrt{n}}{2 \sqrt{k}} \Delta_{1} \nonumber  \\
  && +  \frac{nd}{1-\beta_1}\frac{2\sqrt{2}\triangle_1}{ \sqrt{nk}}+\frac{\sqrt{2}n\beta_1^n}{\sqrt{k}} \sum_{l=1}^d \sum_{i=0}^{n-1} |\partial_l f_i(x_{1,0}) |  \nonumber \\
  &=& \frac{1}{\sqrt{k}} \left[  \frac{d(n+1) n^{\frac{3}{2}}}{2 } \Delta_{1}  +  \frac{d 2\sqrt{2} \sqrt{n}\triangle_1 }{1-\beta_1}+\sqrt{2}n \beta_1^n \sum_{i=0}^{n-1} \|\nabla f_i(x_{1,0}) \|_1  \right] \nonumber \\
  && +3n \sqrt{D_{1}} \rho_{1} d\left(\left|\partial_{\alpha} f\left(x_{k, 0}\right)\right|+\sqrt{\frac{D_{0}}{D_{1} d}}\right). \label{eq_boundingd}
\end{eqnarray}
}

\begin{eqnarray*}
  \sum_{l \text{ large }} (b_2) &= &\delta_1\sqrt{\frac{2\rho_3^2}{\beta_2^n} }\sum_{l \text{ large }} \sum_{i=0}^{n-1}|\mlki- \partial_{l} f_i(\xkz)|  \\
       &\leq&   \delta_1\sqrt{\frac{2\rho_3^2}{\beta_2^n} } \Bigm\{
       \frac{1}{\sqrt{k}} \left[  \frac{d(n+1) n^{\frac{3}{2}}}{2 } \Delta_{1}  +  \frac{d2\sqrt{2} \sqrt{n}\triangle_1 }{1-\beta_1}+\sqrt{2}n \beta_1^n \sum_{i=0}^{n-1}  \|\nabla f_i(x_{1,0}) \|_1   \right] \nonumber \\
    && + 3n \sqrt{D_{1}} \rho_{1} d\left(\left|\partial_{\alpha} f\left(x_{k, 0}\right)\right|+\sqrt{\frac{D_{0}}{D_{1} d}}\right) \Bigm\}.
\end{eqnarray*}

The proof is concluded by adding expectation on both sides of the inequality.

\subsection{Proof of Lemma \ref{lemma_a1}}
\label{appendix_lemma_a1}

\begin{eqnarray*}
  &&   \Ex \left\{\sum_{l \text{ large}} (a_1)   +  \sum_{l \text{ large}} (b_1) -  \sum_{l \text{ large }} \{(a_2) +  (b_2) \}\right\}\\
  &=& \mathbb{E}\left[ \sum_{l \text{ large}} \frac{\partial_{l} f\left(x_{k, 0}\right)^{2}}{\sqrt{v_{l, k, 0}}}\right] 
   + \Ex \left[\sum_{l \text{ large}} \frac{\partial_l f(\xkz)}{\sqrt{\vlkz}} \sum_{i=0}^{n-1}  \left(\mlki-  \partial_{l} f_i(\xkz) \right) \right] \\
   &&- d \delta_{1} \sqrt{\frac{2\rho_3^2}{\beta_2^n} }\Ex \left[ \sum_{i=0}^{n-1}\left|\partial_{l} f_i\left(x_{k, 0} \right)\right|\right] - d \delta_1\sqrt{\frac{2\rho_3^2}{\beta_2^n} }\Ex \left[ \sum_{i=0}^{n-1}|\mlki- \partial_{l} f_i(\xkz)| \right]  \\
  &\overset{\text{Lemma \ref{lemma_b1} and  \ref{lemma_da2_db2} and \ref{lemma_fi_f}}}{\geq}& \Ex\left[\sum_{l \text{ large}}\frac{\partial_{l} f\left(x_{k, 0}\right)^{2}}{\sqrt{v_{l, k, 0}}}  \right] - \frac{1}{\sqrt{k}} \left(G_4+ G_5 +G_6 \right) - G_7 \delta_1 \sqrt{D_{1}} \rho_{1} d\left(\Ex\left|\partial_{\alpha} f\left(x_{k, 0}\right)\right|+\sqrt{\frac{D_{0}}{D_{1} d}}\right) \\
  &&-  \delta_{1} \sqrt{\frac{2\rho_3^2}{\beta_2^n} } \sqrt{D_{1}} \rho_{1} d\left(\Ex\left|\partial_{\alpha} f\left(x_{k, 0}\right)\right|+\sqrt{\frac{D_{0}}{D_{1} d}}\right) \nonumber \\
  &&-   \delta_1\sqrt{\frac{2\rho_3^2}{\beta_2^n} } \left[ \frac{d(n+1) n^{\frac{3}{2}}}{2 } \Delta_{1}  +  \frac{d2\sqrt{2} \sqrt{n}\triangle_1 }{1-\beta_1}+\sqrt{2}n \beta_1^n \sum_{i=0}^{n-1}  \Ex\|\nabla f_i(x_{1,0}) \|_1 \right]\frac{1}{\sqrt{k}} \nonumber  \\
  && - \delta_1\sqrt{\frac{2\rho_3^2}{\beta_2^n} }  3n \sqrt{D_{1}} \rho_{1} d\left( \Ex\left|\partial_{\alpha} f\left(x_{k, 0}\right)\right|+\sqrt{\frac{D_{0}}{D_{1} d}}\right),  \nonumber 
  \end{eqnarray*}

  where constant terms $G_4, G_5, G_6, G_7$ can be seen at the end of  Appendix \ref{appendix:lemma_Mp-F_1}.

Since
$\alpha=\arg \max _{l=1,2, \cdots, d}\left|\partial_{l} f\left(x_{k, 0}\right)\right|$ and $  \sum_{l \text{ large}}\frac{\partial_{l} f\left(x_{k, 0}\right)^{2}}{\sqrt{v_{l, k, 0}}}  \geq \frac{\partial_{\alpha} f\left(x_{k, 0}\right)^{2}}{\sqrt{v_{\alpha, k, 0}}}$, we have

  \begin{eqnarray}
    &&   \Ex \left\{\sum_{l \text{ large}} (a_1)   +  \sum_{l \text{ large}} (b_1) -  \sum_{l \text{ large }} \{(a_2) +  (b_2) \}\right\} \nonumber\\
  &\overset{}{\geq}& \Ex \left\{\frac{\partial_{\alpha} f\left(x_{k, 0}\right)^{2}}{\sqrt{v_{\alpha, k, 0}}} \right\} \nonumber \\
  &&- \left[ G_4+G_5+G_6+ \delta_1\sqrt{\frac{2\rho_3^2}{\beta_2^n} } \left( \frac{d(n+1) n^{\frac{3}{2}}}{2 } \Delta_{1}  +  \frac{d2\sqrt{2} \sqrt{n}\triangle_1 }{1-\beta_1}+\sqrt{2}n \beta_1^n \sum_{i=0}^{n-1}  \Ex\|\nabla f_i(x_{1,0}) \|_1 \right) \right]\frac{1}{\sqrt{k}} \nonumber  \\
  && - \delta_1 \left(\sqrt{\frac{2\rho_3^2}{\beta_2^n} }  4n + G_7  \right)\sqrt{D_{1}} \rho_{1} d\left( \Ex\left|\partial_{\alpha} f\left(x_{k, 0}\right)\right|+\sqrt{\frac{D_{0}}{D_{1} d}}\right)  \nonumber \\
  &=& \Ex \left\{\frac{\partial_{\alpha} f\left(x_{k, 0}\right)^{2}}{\sqrt{v_{\alpha, k, 0}}}\right\}
  - F_2 \frac{1}{\sqrt{k}} -F_3 \Ex\left|\partial_{\alpha} f\left(x_{k, 0}\right)\right| - F_4, \label{eq_a1-da2-db2} 
\end{eqnarray}

where $F_2 := 
\delta_1\sqrt{\frac{2\rho_3^2}{\beta_2^n} } \left[ \frac{d(n+1) n^{\frac{3}{2}}}{2 } \Delta_{1}  +  \frac{d2\sqrt{2} \sqrt{n}\triangle_1 }{1-\beta_1}+\sqrt{2}n \beta_1^n \sum_{i=0}^{n-1}  \|\nabla f_i(x_{1,0}) \|_1 \right]  + G_4+ G_5+ G_6$;

$F_3:= \delta_1 \left(\sqrt{\frac{2\rho_3^2}{\beta_2^n} }  4n + G_7  \right)\sqrt{D_{1}} \rho_{1} d$;

$F_4:=  \delta_1 \left(\sqrt{\frac{2\rho_3^2}{\beta_2^n} }  4n + G_7  \right)\sqrt{D_{1}} \rho_{1} d  \sqrt{\frac{D_{0}}{D_{1} d}}. $





Now, we discuss two cases.

\paragraph{Case (a): when $\left|\partial_{\alpha} f\left(x_{k, 0}\right)\right| \geq 4 \sqrt{2} \frac{\Delta_{1}}{\left(1-\beta_{2}\right) \sqrt{D_{1} n k d}}$.} 

In this case, we have the following result.



{\small
\begin{eqnarray*}
  \frac{\partial_\alpha f(\xkz)^2}{\sqrt{v_{\alpha,k,0}}}  
   & \geq &   \frac{\partial_\alpha f(\xkz)^2}{\sqrt{(1-\beta_2)\left( |\partial_\alpha f_{\tau_{k,0}}(x_{k,0})|^2 + \beta_2 |\partial_\alpha f_{\tau_{k-1,n-1}}(x_{k-1,n-1})|^2 + \cdots   \right)    }} \\ 
  &\geq &  \frac{\partial_\alpha f(\xkz)^2}{\sqrt{(1-\beta_2)\left( |\partial_\alpha f_{\tau_{k,0}}(x_{k,0})|^2 + \beta_2 |\partial_\alpha f_{\tau_{k-1,n-1}}(x_{k,0}) +  \triangle_{n(k-1)}|^2  \cdots      \right) }}
  \\
  &\overset{\text{Lemma \ref{lemma_fi_f}}}{\geq}& \frac{\partial_\alpha f(\xkz)^2}{\sqrt{(1-\beta_2) \left(  \sum_{j=0}^{\infty}\left(\sqrt{\left|\partial_{\alpha} f\left(x_{k, 0}\right)\right|^{2}+\frac{D_{0}}{D_{1} d}} \sqrt{D_{1} d}+\sum_{t=1}^{j} \Delta_{n(k-1)-t}\right)^{2} \beta_{2}^{j}    \right) }}
  \\
  &\geq& 
  \frac{\partial_\alpha f(\xkz)^2}{\sqrt{ (1-\beta_2) \left(  \sum_{j=0}^{\infty} \beta_{2}^{j} \left( \left(\left|\partial_{\alpha} f\left(x_{k, 0}\right)\right|^{2}+\frac{D_{0}}{D_{1} d}\right) D_{1} d + 4 \sqrt{2}j \triangle_{nk} \sqrt{\left|\partial_{\alpha} f\left(x_{k, 0}\right)\right|^{2}+\frac{D_{0}}{D_{1} d}} \sqrt{D_{1}d} + 8j^2\triangle_{nk}^2  \right)    \right) }}
   \\
   &\overset{\text{Lemma \ref{lemma_beta}}}{\geq} & \frac{\partial_\alpha f(\xkz)^2}{\sqrt{D_{1} d\left(\left(\left|\partial_{\alpha} f\left(x_{k, 0}\right)\right|^{2}+\frac{D_{0}}{D_{1} d}\right)+  \sqrt{\left|\partial_{\alpha} f\left(x_{k, 0}\right)\right|^{2}+\frac{D_{0}}{D_{1} d}} \frac{4 \sqrt{2} \triangle_{nk}}{(1-\beta_2)\sqrt{D_{1} d}} + \frac{16 \triangle_{nk}^2}{D_1 d (1-\beta_2)^2  }\right) } } \\
  &\overset{\textbf{Case (a)}}{\geq}&  \frac{\partial_\alpha f(\xkz)^2}{\sqrt{\frac{5}{2} D_{1} d\left(\left|\partial_{\alpha} f\left(x_{k, 0}\right)\right|^{2}+\frac{D_{0}}{D_{1} d}\right)}},
\end{eqnarray*}
}

Now we consider the following two sub-cases.

\begin{itemize}
  \item When $\partial_\alpha f(\xkz)^2 \leq \frac{D_{0}}{D_{1} d} $: 
  \begin{eqnarray*}
    &&
    \frac{\partial_{\alpha} f\left(x_{k, 0}\right)^{2}}{\sqrt{v_{\alpha, k, 0}}}
  - F_2 \frac{1}{\sqrt{k}} -F_3 \left|\partial_{\alpha} f\left(x_{k, 0}\right)\right| - F_4 \\ 
    &&\geq     
    \frac{\partial_{\alpha} f\left(x_{k, 0}\right)^{2}}{\sqrt{5D_0}}
  - F_2 \frac{1}{\sqrt{k}} -F_3 \sqrt{\frac{D_{0}}{D_{1} d}} - F_4 
  \end{eqnarray*}
  
  
  \item When $\partial_\alpha f(\xkz)^2 \geq \frac{D_{0}}{D_{1} d} $: 
  \begin{eqnarray*}
    &&    \frac{\partial_{\alpha} f\left(x_{k, 0}\right)^{2}}{\sqrt{v_{\alpha, k, 0}}}
    - F_2 \frac{1}{\sqrt{k}} -F_3 \left|\partial_{\alpha} f\left(x_{k, 0}\right)\right| - F_4  \\
    && \geq \frac{|\partial_\alpha f(\xkz)|}{\sqrt{5D_1d}}  - F_2 \frac{1}{\sqrt{k}} -F_3 \left|\partial_{\alpha} f\left(x_{k, 0}\right)\right| - F_4  \\
    && =  \left|\partial_{\alpha} f\left(x_{k, 0}\right)\right| \left(\frac{1}{\sqrt{5D_1d}}  -F_3\right) - F_2 \frac{1}{\sqrt{k}}- F_4  \\
    &&\geq   \left|\partial_{\alpha} f\left(x_{k, 0}\right)\right| \left(\frac{1}{\sqrt{5D_1d}}  -F_3\right) - F_2 \frac{1}{\sqrt{k}}-F_3 \sqrt{\frac{D_{0}}{D_{1} d}} -F_4  
  \end{eqnarray*}
  
  
\end{itemize}

Combining together, we have the following results for {\bf Case (a)}.

{\small
\begin{equation*}
   \left\{\frac{\partial_{\alpha} f\left(x_{k, 0}\right)^{2}}{\sqrt{v_{\alpha, k, 0}}}\right\}
  - F_2 \frac{1}{\sqrt{k}} -F_3  \left|\partial_{\alpha} f\left(x_{k, 0}\right)\right| - F_4 \geq  \min\left\{   \frac{\partial_\alpha f(\xkz)^2}{\sqrt{5D_0}},   \left|\partial_{\alpha} f\left(x_{k, 0}\right)\right| \left(\frac{1}{\sqrt{5D_1d}}  -F_3\right)  \right\}   - F_2 \frac{1}{\sqrt{k}}-F_3 \sqrt{\frac{D_{0}}{D_{1} d}} -F_4.
\end{equation*}
}



\paragraph{Case (b):}
When 
$\left|\partial_{\alpha} f\left(x_{k, 0}\right)\right| < 4 \sqrt{2} \frac{\Delta_{1}}{\left(1-\beta_{2}\right) \sqrt{D_{1} n k d}}$,   we have 

\begin{eqnarray}
   \left\{\frac{\partial_{\alpha} f\left(x_{k, 0}\right)^{2}}{\sqrt{v_{\alpha, k, 0}}}\right\}
  - F_2 \frac{1}{\sqrt{k}} -F_3 \left|\partial_{\alpha} f\left(x_{k, 0}\right)\right| - F_4  
  &\geq &  - F_2 \frac{1}{\sqrt{k}} -F_3 \left|\partial_{\alpha} f\left(x_{k, 0}\right)\right| - F_4 \nonumber \\
  &\geq & -F_2 \frac{1}{\sqrt{k}} -F_3 4 \sqrt{2} \frac{\Delta_{1}}{\left(1-\beta_{2}\right) \sqrt{D_{1} n k d}} - F_4 \nonumber\\
   &=& -\frac{1}{\sqrt{k}}G_1- F_4, \label{eq_a1-da2-db2_caseb}
\end{eqnarray}

where $G_1:=   F_2 +  F_3 4 \sqrt{2} \frac{\Delta_{1}}{\left(1-\beta_{2}\right) \sqrt{D_{1} n  d}} $.
Now, the following two claims are both true.

  \begin{itemize}
    \item Claim 1:
    \begin{eqnarray}
      \eqref{eq_a1-da2-db2_caseb} &\geq & \frac{\partial_{\alpha} f\left(x_{k, 0}\right)^{2}}{\sqrt{5 D_{0}}} - \frac{1}{k}\frac{\left(4 \sqrt{2} \frac{\Delta_{1}}{\left(1-\beta_{2}\right) \sqrt{D_{1} n  d}}\right)^2}{\sqrt{5 D_{0}}} -\frac{1}{\sqrt{k}}G_1- F_4. \nonumber 
    \end{eqnarray}
    \item Claim 2: 
    \begin{eqnarray}
      \eqref{eq_a1-da2-db2_caseb} &\geq & \left|\partial_{\alpha} f\left(x_{k, 0}\right)\right|\left(\frac{1}{\sqrt{5 D_{1} d}}-F_3\right) -\frac{1}{\sqrt{k}}4 \sqrt{2} \frac{\Delta_{1}}{\left(1-\beta_{2}\right) \sqrt{D_{1} n  d}}\left(\frac{1}{\sqrt{5 D_{1} d}}-F_3\right) \nonumber \\
      && -\frac{1}{\sqrt{k}}G_1- F_4.  \nonumber
    \end{eqnarray}
  \end{itemize}



   Combining Claim 1 and Claim 2, we have 

    {\small
    \begin{eqnarray}
    \eqref{eq_a1-da2-db2_caseb}  &\geq &  \min \left\{ \frac{\partial_{\alpha} f\left(x_{k, 0}\right)^{2}}{\sqrt{5 D_{0}}}, \left|\partial_{\alpha} f\left(x_{k, 0}\right)\right|\left(\frac{1}{\sqrt{5 D_{1} d}}-F_3\right)\right\} -F_4  \nonumber \\
      && - \frac{1}{\sqrt{k}}\left[\max\left\{ \frac{\left(4 \sqrt{2} \frac{\Delta_{1}}{\left(1-\beta_{2}\right) \sqrt{D_{1} n  d}}\right)^2}{\sqrt{5 D_{0}}}, \frac{4 \sqrt{2}\Delta_{1}}{\left(1-\beta_{2}\right) \sqrt{D_{1} n  d}}\left(\frac{1}{\sqrt{5 D_{1} d}}-F_3\right) \right\}  +G_1 \right] \nonumber \\
      &:= & \min \left\{\frac{\partial_{\alpha} f\left(x_{k, 0}\right)^{2}}{\sqrt{5 D_{0}}},\left|\partial_{\alpha} f\left(x_{k, 0}\right)\right|\left(\frac{1}{\sqrt{5 D_{1} d}}-F_3\right)\right\}  -\frac{1}{\sqrt{k}} G_2 -F_4,
    \end{eqnarray}}

   where $G_2:= \max\left\{ \frac{\left(4 \sqrt{2} \frac{\Delta_{1}}{\left(1-\beta_{2}\right) \sqrt{D_{1} n  d}}\right)^2}{\sqrt{5 D_{0}}}, \frac{4 \sqrt{2}\Delta_{1}}{\left(1-\beta_{2}\right) \sqrt{D_{1} n  d}}\left(\frac{1}{\sqrt{5 D_{1} d}}-F_3\right) \right\}  +G_1 .$

  Combining {\bf Case (a)} and {\bf Case (b)} together, we have 

\begin{eqnarray*}
  \Ex \left\{\sum_{l \text{ large}} (a_1)   +  \sum_{l \text{ large}} (b_1) -  \sum_{l \text{ large }} \{(a_2) +  (b_2) \}\right\} &\geq &  \Ex \min \left\{ \frac{\partial_{\alpha} f\left(x_{k, 0}\right)^{2}}{\sqrt{5 D_{0}}}, \left|\partial_{\alpha} f\left(x_{k, 0}\right)\right|\left(\frac{1}{\sqrt{5 D_{1} d}}-F_3\right)\right\}  \nonumber\\
      &&-\frac{1}{\sqrt{k}} G_2 -F_4   \nonumber \\
  &\geq&  \Ex\min \left\{\frac{ \|\nabla f(x_{k,0})\|_2^2 }{d \sqrt{5 D_{0}} }, \|\nabla f(x_{k,0})\|_1\left(\frac{1}{d\sqrt{5 D_{1} d}}-\frac{F_3}{d}\right)\right\}   \nonumber \\
  &&-\frac{1}{\sqrt{k}} G_2 -F_4  ,
\end{eqnarray*}

where the last inequality is because of $\|\nabla f(x_{k,0})\|_2^2 \leq d \partial_{\alpha} f\left(x_{k, 0}\right)^{2}$,  $\|\nabla f(x_{k,0})\|_1 \leq d |\partial_{\alpha} f\left(x_{k, 0}\right)|$.

Recall $F_3 \rightarrow 0$ when $\beta_2 \rightarrow 1$, so there exists an interval $(1-\epsilon,1 ]$, such that $ \frac{1}{\sqrt{5D_1d}}- F_3 \geq \frac{1}{\sqrt{10D_1 d}}$, or equivalently  
$F_3 \leq \frac{1}{\sqrt{10D_1 d}}$ (note that $F_3$ is the same as ``$A(\beta_2)$'' stated in the condition of Lemma \ref{lemma_a1}). With such a choice of $\beta_2$,  we have the following results by changing all the $F_3$ into $\frac{1}{\sqrt{10D_1 d}}$:

\begin{eqnarray*}
  \Ex \left\{\sum_{l \text{ large}} (a_1)   +  \sum_{l \text{ large}} (b_1) -  \sum_{l \text{ large }} \{(a_2) +  (b_2) \}\right\} 
  &\geq&  \Ex \min \left\{\frac{ \|\nabla f(x_{k,0})\|_2^2 }{d \sqrt{5 D_{0}} }, \|\nabla f(x_{k,0})\|_1\frac{1}{d\sqrt{10D_1 d}}\right\}   \nonumber \\
  &&-\frac{1}{\sqrt{k}} G_2 -F_4.  \nonumber \\
  &\overset{(*)}{\geq}&  \frac{1}{d\sqrt{10D_1 d}} \Ex \min  \left\{ \sqrt{\frac{ 2D_1 d }{D_{0} } }   \|\nabla f(x_{k,0})\|_2^2, \|\nabla f(x_{k,0})\|_1\right\}   \nonumber \\
  &&-\frac{1}{\sqrt{k}} G_2 -F_5.  \nonumber \\
\end{eqnarray*}

In inequality $(*)$, we change $F_4$ into 
$F_5:=  \sqrt{\frac{D_{0}}{D_{1} d}} \frac{1}{\sqrt{10D_1 d}}
$. This is because: first, $F_4 = F_3 \sqrt{\frac{D_0}{D_1d}}$; second, $F_3 \leq \frac{1}{\sqrt{10D_1d}}$.

The proof of Lemma \ref{lemma_a1} is completed.

We restate all the constants as follows ($G_4, G_5, G_6$ are specified in Appendix \ref{appendix:lemma_Mp-F_1}):


$G_2:= \max\left\{ \frac{\left(4 \sqrt{2} \frac{\Delta_{1}}{\left(1-\beta_{2}\right) \sqrt{D_{1} n  d}}\right)^2}{\sqrt{5 D_{0}}}, \frac{4 \sqrt{2}\Delta_{1}}{\left(1-\beta_{2}\right) \sqrt{D_{1} n  d} \sqrt{10 D_{1} d}}\right\}  +G_1$;


$G_1:=   F_2 +   4 \sqrt{2} \frac{\Delta_{1}}{\left(1-\beta_{2}\right) \sqrt{D_{1} n  d} \sqrt{10D_1 d}} $;

 $F_2 := 
\delta_1\sqrt{\frac{2\rho_3^2}{\beta_2^n} } \left[ \frac{d(n+1) n^{\frac{3}{2}}}{2 } \Delta_{1}  +  \frac{d2\sqrt{2} \sqrt{n}\triangle_1 }{1-\beta_1}+\sqrt{2}n \beta_1^n \sum_{i=0}^{n-1}  \|\nabla f_i(x_{1,0}) \|_1 \right]  + G_4+ G_5+ G_6$;




$F_5:=  \sqrt{\frac{D_{0}}{D_{1} d}} \frac{1}{\sqrt{10D_1 d}}.
$

\begin{remark} \label{remark_f4}
We comment that we can always replace $F_5$ in the final result by $F_4$, i.e., we can choose not to apply the last inequality $(*)$ in the proof. The benefit of using $F_4$ is that $F_4$ monotonously decrease to 0 when increasing  $\beta_2$ to 1. This monotone property is not shown in the notation of $F_5$.
Nevertheless, we choose to use $F_5$ since it is a much cleaner constant.
\end{remark}


 \subsection{Proof of Lemma \ref{lemma_all_together}}
 \label{appendix:lemma_all}

Based on Descent Lemma, we have

\begin{eqnarray}
  \sum_{k=t_{0}}^{T} \frac{\eta_0}{\sqrt{nk}} \text{(r.h.s. of \eqref{eq_zhangdecomposition_3})} 
  &\leq& 
  \sum_{k=t_{0}}^{T} \Ex \left\langle\nabla f\left(x_{k, 0}\right), x_{k, 0}-x_{k+1,0}\right\rangle \nonumber \\
  &\leq & \Ex f\left(x_{t_{0}, 0}\right)- \Ex f\left(x_{T+1,0}\right) +\sum_{k=t_{0}}^{T}\frac{L}{2} \Ex \left\|x_{k+1,0}-x_{k, 0}\right\|_{2}^{2}\nonumber\\
  &\overset{\eqref{eq:m_over_v}}{\leq} &  \Ex f\left(x_{t_{0}, 0}\right)-f^*+\sum_{k=t_{0}}^{T} \frac{L  d }{2 }  \left(\frac{ n \eta_{0}}{\sqrt{nk}}\frac{\left(1-\beta_{1}\right)}{\sqrt{1-\beta_{2}}} \frac{1}{1-\frac{\beta_{1}}{\sqrt{\beta_{2}}}}  \right) ^2 \nonumber  \\
  &= &  \Ex f\left(x_{t_{0}, 0}\right)-f^*+\sum_{k=t_{0}}^{T} \frac{L  d n }{2 k}  \left(\eta_{0}\frac{\left(1-\beta_{1}\right)}{\sqrt{1-\beta_{2}}} \frac{1}{1-\frac{\beta_{1}}{\sqrt{\beta_{2}}}}  \right) ^2 \nonumber 
\end{eqnarray}

Plugging in  the r.h.s. of \eqref{eq_zhangdecomposition_3}, we have the following relation after rearranging.

\begin{eqnarray}
  &&\sum_{k=t_{0}}^{T} \frac{\eta_{0}}{\sqrt{n k}} \left[ \frac{1}{d\sqrt{10D_1 d}}  \Ex \min  \left\{ \sqrt{\frac{ 2D_1 d }{D_{0} } }  \|\nabla f(x_{k,0})\|_2^2, \|\nabla f(x_{k,0})\|_1\right\}  
  -\Ex \left[\mathcal{O}(\frac{1}{\sqrt{k}})\right]-\mathcal{O}(\sqrt{D_0})\right] \nonumber\\
  &&\leq \Ex f\left(x_{t_{0}, 0}\right)-f^*+\sum_{k=t_{0}}^{T} \frac{1}{k} \left[ \frac{L  d n }{2}   \left(\eta_{0}\frac{\left(1-\beta_{1}\right)}{\sqrt{1-\beta_{2}}} \frac{1}{1-\frac{\beta_{1}}{\sqrt{\beta_{2}}}}  \right) ^2  \right] \nonumber.
\end{eqnarray}

Recall we have $2\left(\sqrt{T}-\sqrt{t_{0}-1}\right) \leq  \sum_{k=t_{0}}^{T}\frac{1}{\sqrt{k}}$, $\sum_{k=t_{0}}^{T} \frac{1}{k} \leq \log{\frac{T+1}{t_0}}$. Further, since  $\| \cdot\|_1 \geq \| \cdot\|_2$,  we get the following relation when $t_0=1$.



\begin{eqnarray}
  &&\min_{k \in [1,T]}  \Ex \left[\min  \left\{ \sqrt{\frac{ 2D_1 d }{D_{0} } }  \|\nabla f(x_{k,0})\|_2^2, \|\nabla f(x_{k,0})\|_2\right\} \right] \nonumber \\
    && \leq \frac{1 }{\sqrt{T}} \frac{d \sqrt{10D_1 dn}}{2 \eta_0}\left[ \Ex f\left(x_{1, 0}\right)-f^* +  \log (T+1) \left(  \frac{L  d n }{2 }  \left(\eta_{0}\frac{\left(1-\beta_{1}\right)}{\sqrt{1-\beta_{2}}} \frac{1}{1-\frac{\beta_{1}}{\sqrt{\beta_{2}}}}  \right) ^2 +H_2  \right)   \right]+F_5,  \nonumber \\
\end{eqnarray}

We specify all the constant as follows. For all the following constants, We keep the same notation as their appearance in their corresponding lemmas.

$F_5:=   \sqrt{\frac{D_{0}}{D_{1} d}} \frac{1}{\sqrt{10D_1 d}};
$

$H_2:= \frac{\eta_0 \left(G_2+dF_1 \right)}{\sqrt{n}};$

$F_1:= \triangle_1 n^2\sqrt{n} \frac{32\sqrt{2}}{(1-\btwo)^n \btwo^n}  \frac{1-\beta_1}{\sqrt{1-\beta_2}}\frac{1}{1-\frac{\beta_1}{\sqrt{\beta_2}}} n$;


$G_2:= \max\left\{ \frac{\left(4 \sqrt{2} \frac{\Delta_{1}}{\left(1-\beta_{2}\right) \sqrt{D_{1} n  d}}\right)^2}{\sqrt{5 D_{0}}}, \frac{4 \sqrt{2}\Delta_{1}}{\left(1-\beta_{2}\right) \sqrt{D_{1} n  d} \sqrt{10 D_{1} d}}\right\}  +G_1$;


$G_1:=   F_2 +   4 \sqrt{2} \frac{\Delta_{1}}{\left(1-\beta_{2}\right) \sqrt{D_{1} n  d} \sqrt{10D_1 d}} $;

 $F_2 := 
\delta_1\sqrt{\frac{2\rho_3^2}{\beta_2^n} } \left[ \frac{d(n+1) n^{\frac{3}{2}}}{2 } \Delta_{1}  +  \frac{d2\sqrt{2} \sqrt{n}\triangle_1 }{1-\beta_1}+\sqrt{2}n \beta_1^n \sum_{i=0}^{n-1}  \|\nabla f_i(x_{1,0}) \|_1 \right]  + G_4+ G_5+ G_6$;










 $G_5 = d \AAA \left(\frac{\beta_1^{2n} 2n^3\triangle_1 \sqrt{2}}{\sqrt{n}} \frac{1-\beta_1}{(1-\beta_1^n)^2} + n  \left(1-\beta_1^{n-1} \right) \sum_{i=0}^{n-1}  \Ex|\partial_\alpha f_i (x_{1,0})|  \frac{(1-\beta_1)\beta_1^{n}\sqrt{2}}{1-\beta_1^n} \right); $


$G_4 = d \AAA \beta_1^n \sqrt{2}n (n-1)\sum_{i=0}^{n-1} \Ex \left(  |\partial_\alpha f_i (x_{1,0})| \right);$






{\small

\begin{eqnarray*}
  G_6 &:=&  \beta_1^n(1-\beta_1) \sqrt{2}\left|J_1 \right| \AAA \Ex \left[\sum_{l=1}^d  \left|\mathbb{I}_1  
   \sum_{i=0}^{n-1} \partial_l f_i(x_{1,0}) \right| \right]  +\frac{\triangle_{1}}{\sqrt{n}} \delta_2 d\AAA(1-\beta_1)n^3 \\
    &&\BLUE{+ d
    (1-\beta_1)|J_1 |  \frac{1}{1-\frac{1}{\sqrt{\beta_2^n}}} n^2\sqrt{ \frac{2}{n \beta_2^n}}  \frac{\triangle_{1} \delta_2}{\sqrt{n}}}\\
    && \BLUE{+  d
    (1-\beta_1) |J_1 |  n^2
    \left(\sqrt{\frac{2\rho_3^2}{\beta_2^n}}  \frac{1}{\left(1- \frac{(1-\beta_2)4n \rho_2  }{\beta_2^n}\right) }   \delta_1\right)\frac{1+\beta_1^n}{(1-\beta_1^n)^3} \frac{\triangle_{1}\sqrt{2}}{\sqrt{n}}} \\
    &&\BLUE{+ 2 (1-\beta_1)d \AAA |J_1| n^2 \frac{\triangle_{1} \delta_2 \sqrt{2}}{\sqrt{n}}}  \BLUE{- 5(1-\beta_1)d \AAA |J_1| n  \frac{1}{(1-\beta_1^n)^2}\Delta_{1} n \sqrt{n} \frac{32 \sqrt{2}}{\left(1-\beta_{2}\right)^{n} \beta_{2}^{n}}  }  \\
    &&\RED{+ d\left(\frac{1}{1-\frac{1}{\sqrt{\beta_{2}^{n}}}} 
   \sqrt{ \frac{2}{n \beta_2^n}}
   \right)(1-\beta_1) (\beta_1+\beta_1^n) n^3   \frac{\triangle_1}{\sqrt{n}} \delta_2  }\\ 
    &&\RED{+  d\sqrt{\frac{2 \rho_{3}^{2}}{\beta_{2}^{n}}} \frac{1}{\left(1- \frac{(1-\beta_2)4n \rho_2  }{\beta_2^n}\right) }  \delta_{1}(1-\beta_1) (\beta_1+\beta_1^n) n^3 \frac{1+\beta_1^n}{(1-\beta_1^n)^3} \frac{\triangle_{1} \delta_2 \sqrt{2}}{\sqrt{n}} }  \\ 
  && \RED{+  2d \AAA(1-\beta_1)(\beta_1+\beta_1^n)n^2   \frac{1}{\beta_1^n(1-\beta_1^n)^2} \Delta_{1} n \sqrt{n} \frac{32 \sqrt{2}}{\left(1-\beta_{2}\right)^{n} \beta_{2}^{n}}  }\\
  &&\RED{+ d\AAA (1-\beta_1)\left(\beta_1 +\beta_1^n \right) n^3 \frac{\triangle_1}{\sqrt{n}} \delta_2, }
\end{eqnarray*}
}

$\Delta_{1}:=\eta_{0}\frac{L \sqrt{d}}{\sqrt{1-\beta_{2}}} \frac{1-\beta_{1}}{1-\frac{\beta_{1}}{\sqrt{\beta_{2}}}}$;

$\delta_{1} :=\frac{(1-\beta_2)4n \rho_2  }{\beta_2^n}+\left(\frac{1}{\sqrt{\beta_{2}^{n}}}-1\right)$;

 where $J_1 = \left(-\frac{1}{n} \beta_{1}-\frac{2}{n} \beta_{1}^{2}-\cdots-\frac{n-1}{n} \beta_{1}^{n-1}\right)$,  $\delta_2 = \lim_{k \rightarrow \infty} \sum_{j=1}^{k-1} (\beta_1^n)^j \sqrt{\frac{k}{k-j}} = \frac{\beta_1^n}{1-\beta_1^n}$. If needed, we can further bound $J_1$ by $n$ for simplicity.

Further, $\rho_1, \rho_2, \rho_3$ are constants satisfying the following conditions for $\forall l =1,\cdots, d$. Usually, these constants can be different for different problems. In the worst case, we have $0 \leq \rho_{3} \leq \sqrt{n} \rho_{1} \leq n$. $\rho_{3}$ is larger when $\partial_l f_i (x_{k,0})$ are more aligned.

$\rho_{1} \geq \frac{\sum_{i=1}^{n}\left|\partial_l f_i (x_{k,0})\right|}{\sqrt{\sum_{i=1}^{n}\left| \partial_l f_i (x_{k,0})\right|^{2}}};$

$\rho_{2} \geq \frac{\left|\max_i \partial_l f_i (x_{k,0})\right|^{2}}{\frac{1}{n} \sum_{i=1}^{n}\left|\partial_l f_i (x_{k,0})\right|^{2}};$

$\rho_{3} \geq \frac{\left|\sum_{i=1}^{n} \partial_l f_i (x_{k,0})\right|}{\sqrt{\frac{1}{n} \sum_{i=1}^{n}\left|\partial_l f_i (x_{k,0})\right|^{2}}}.$

The  proof of Lemma \ref{lemma_all_together} is  now completed.  This also concludes the whole proof of Theorem \ref{thm1}.



\subsection{Dissussion on Bias Correction Terms and Non-Zero $\epsilon$ (Hyperparameter for Numerical Stability)}
 \label{appendix:bias_correction}

In the above analysis, we focus on  Adam without bias correction terms and we consider $\epsilon=0$ ($\epsilon$ is the hyperparameter for numerical stability in Algorithm \ref{algorithm}).  
For completeness, we now briefly discuss how to incorperate the bias correction terms and non-zero $\epsilon$   into our analysis above. Based on the current convergence proof, we only requires several additional simple changes. We list them as follows.

\paragraph{Adam with bias correction terms:} This bias correction terms are introduced  by \citep{kingma2014adam}). It can be implemented by (1) changing the stepsize $\eta_k$ into $\hat{\eta_k} = \frac{\sqrt{1-\beta_1^k}}{1-\beta_1^k}\eta_k = \frac{\sqrt{1-\beta_1^k}}{1-\beta_1^k} \frac{\eta_0}{\sqrt{k}}$; (2) change the initialization of Algorithm \ref{algorithm} into $m_{1,-1} = v_{1,-1} =0$. More details can be seen in \citep{kingma2014adam}). We now explain how to inlcude this two changes into our analysis.

(1) Change of stepsize: We observe that the new stepsize  $\hat{\eta_k}$ is well bounded around the old stepsize $\eta_k$, i.e.,  $\hat{\eta_k} \in [\sqrt{1-\beta_2} \eta_k, \frac{1}{1-\beta_1}\eta_k]$. Therefore, to prove the convergence of Adam with $\hat{\eta_k}$, we add the follwing steps into the current proof. 

\begin{itemize}
  \item Whenever we need an upper bound on  $\hat{\eta_k}$, we use $\hat{\eta_k}\leq \frac{1}{1-\beta_1}\eta_k$. Then we follow the original analysis with an extra constant $\frac{1}{1-\beta_1}$. This step will appear in Lemma \ref{lemma_delta}. It turns out we only need to  change the constant $\triangle_{nk}$ in Lemma \ref{lemma_delta} into $\frac{1}{1-\beta_1}\frac{\eta_{0} }{\sqrt{nk}}\frac{L\sqrt{d}}{\sqrt{1-\beta_{2}}} \frac{1-\beta_{1} }{1-\frac{\beta_{1}}{\sqrt{\beta_{2}}}} $. The rest of the analysis remains the same.  
  \item Whenver we need a lower bound on  $\hat{\eta_k}$, we use $\hat{\eta_k}\geq \sqrt{1-\beta_2}\eta_k$. Then we follow the original analysis with an extra constant $\sqrt{1-\beta_2}$. This step will appear in Lemma \ref{lemma_all_together} and we only need to change the constant terms in the final result. The rest of the analysis remains the same. 
\end{itemize}

(2) Change of initialization: In our current analysis, we use initialization  $m_{1,-1} = \nabla f(x_0)$ and $v_{1,-1} = \max_i\ \nabla f_i(x_0) \circ \nabla f_i(x_0)$.
Now we explain how to prove convergence with initialization  $m_{1,-1} = v_{1,-1} =0$. 

\begin{itemize}
  \item We use   $m_{1,-1} = \nabla f(x_0)$ at   Lemma \ref{lemma_M-Mp}, in which we bound the difference between $M_{l,k}$ and $M_{l,k}^\prime$. The goal of this lemma is to control $M_{l,k} - M_{l,k}^\prime = \mathcal{O}(\beta_1^{(k-1)n})$. 
  
  As explained in the proof of Lemma \ref{lemma_M-Mp}, after expanding  $M_{l,k}$ and $M_{l,k}^\prime$ into serieses, they only differ when ``$k\leq 0$'' (as shown in equation \eqref{proofsketch_m-mp}). When we use $m_{1,-1}=0$, half of the terms in equation \eqref{proofsketch_m-mp} will become 0. However, it does not affect the result of  Lemma \ref{lemma_M-Mp} since all the terms in  equation \eqref{proofsketch_m-mp} are (at least)  weighted by $\beta_1^{(k-1)n}$. So even if half of them to be 0, the rest of the terms is still in the order of  $\beta_1^{(k-1)n}$. Therefore,  Lemma \ref{lemma_M-Mp} still holds with a few changes on the constant terms. 
  

  \item We use $v_{1,-1} = \max_i\ \nabla f_i(x_0) \circ \nabla f_i(x_0)$ at Lemma \ref{lemma_error_beta2},which is mainly based on Lemma F.1 in \citep{shi2020rmsprop}.  According to \citep{shi2020rmsprop}: to include bias correction terms into analysis, we just need to add one more constraint $k > \frac{8 \sqrt{2}}{1-\beta_2^n} +1$ and then we can reach the same conclusion.  
\end{itemize}

\paragraph{Adam with non-zero $\epsilon$:}In our current analysis, we consider $\epsilon = 0$. In practice, $\epsilon$ is often set to be a small postive number such as $10^{-8}$.
Proving convergence with $\epsilon > 0$
is strictly simpler. 
It only requires a few simple changes based on the current proof. We explain as below. 

When $\epsilon \neq 0$, the new 2nd-order momentum becomes $\hat{v}_{k, 0}:=\sqrt{v_{k, 0}}+\epsilon$. This brings the following changes:

\begin{itemize}
  \item Whenever we want an upper bound on $\hat{v}_{k, 0}$, we can choose one of the following upper bound.
  \begin{itemize}
    \item When $\sqrt{v_{k, 0}} \geq \epsilon$, we have $\hat{v}_{k, 0} \leq 2 \sqrt{v_{k, 0}}$. Then, we follow the same steps in the current proof with minor changes on the constant.
    \item When $\sqrt{v_{k, 0}}<\epsilon$, we have $\hat{v}_{k, 0} \leq 2 \epsilon$. This step, again, decouple the statistical dependency between $\nabla f_{\tau k, 0}(x)$ and $\sqrt{v_{k, 0}}$. It changes Adam into SGD and thus simplifies the proof. Many technical lemmas could be skipped in this case.
  \end{itemize}
  \item Whenever we want a lower bound on $\hat{v}_{k, 0}$, we have $\hat{v}_{k, 0} \geq \epsilon$. This step decouples the statistical dependency between $\nabla f_{\tau k, 0}(x)$ and $\sqrt{v_{k, 0}}$. It changes Adam into SGD and thus simplifies the proof. Many technical lemmas could be skipped in this case.
\end{itemize}


\end{document}